\pdfoutput=1
\documentclass[11pt]{article}
\usepackage[OT1]{fontenc}
\usepackage{smile}
\usepackage{subcaption}
\usepackage[margin=1in]{geometry}
\usepackage[most]{tcolorbox}
\usepackage{tikz}
\usetikzlibrary{arrows.meta,positioning,fit,calc}

\newcommand{\bbR}{\mathbb{R}}
\newcommand{\bbE}{\mathbb{E}}

\renewcommand{\d}{\mathrm{d}}
\newcommand{\wh}{\widehat}

\newcommand{\wt}{\widetilde}
\newcommand{\ol}{\overline}

\title{\bf Stable by Construction: Variational Latent Markov Operators for Long-Horizon PDE Prediction}

\author{%
  Junyi Liao$^\dagger$,\ \ Johann Guilleminot$^\ddagger$, Vahid Tarokh$^\dagger$\\ ~\\ $^\dagger$\small\it Department of Electrical and Computer Engineering, Duke University \\ $^\ddagger$\small\it Department of Mechanical Engineering and Materials Science, Duke University \\ \small\sf \{junyi.liao, johann.guilleminot, vahid.tarokh\}@duke.edu
}

\date{}
\begin{document}
\maketitle

\begin{abstract}
Neural PDE solvers provide efficient surrogates for time-dependent physical systems, but autoregressive prediction over long horizons remains challenging because local errors can induce distribution shift and accumulate under recursive deployment. We develop a variational approach to this problem by introducing latent Markov dynamics in which physical states are represented by latent distributions and evolved through probabilistic transitions. The framework is formulated directly on function spaces and specialized to functional Gaussian models, where structured latent perturbations induce a spectral geometry and variational transition alignment regularizes the learned dynamics. We further analyze how these mechanisms affect autoregressive error propagation, providing a theoretical connection between variational training and long-horizon prediction. We instantiate the framework as the Variational Autoencoding Markov Operator (VAMO), which combines spatially resolved latent fields, structured Gaussian perturbations, and a neural-operator transition. Empirically, we demonstrate the effectiveness of VAMO on several fluid-dynamics benchmarks with prediction horizons extending substantially beyond those represented during training, where it consistently reduces error accumulation and improves rollout stability over several deterministic and noise-injection baselines. Overall, these results highlight variational modeling as a complementary approach to robust long-horizon neural PDE dynamics.
\end{abstract}

\section{Introduction}
Partial differential equations (PDEs) provide a fundamental mathematical formulation of physical systems across science and engineering, but repeated numerical simulation can be costly when solutions are required across many initial conditions, coefficients, or forcing terms. This has motivated data-driven PDE solvers, ranging from physics-informed neural networks \citep{raissi2019physics, pang2019fpinns} to operator-learning methods such as DeepONet, Fourier neural operator (FNO), OFormer, and PiT (\citealp{lu2021learning,li2020fourier}; \citealp{li2022transformer,chen2024positional}). Operator-learning methods seek to approximate mappings between function spaces across families of PDE instances, enabling rapid surrogate evaluation after training \citep{kovachki2023neural}. Depending on the underlying architecture, neural operators can further support properties such as transfer across spatial discretizations, nonlocal modeling of spatial interactions, or adaptation to irregular computational domains \citep{li2020fourier,kovachki2023neural,li2023geo}.

Despite these advances, long-horizon prediction of time-dependent PDEs remains challenging, particularly when a learned local evolution operator is deployed recursively beyond the temporal regime represented during training. Autoregressive models are typically trained on ground-truth one-step transitions but recursively consume their own predictions at inference, so small local errors can shift the rollout away from the training distribution and accumulate over time \citep{sanchez2020learning,brandstetter2022message, lippe2023pde}. This difficulty is further exacerbated when training trajectories cover only a limited temporal horizon, while deployment requires evolving the system substantially beyond that horizon \citep{yin2023continuous,michalowska2024neural, diab2025temporal}. Existing remedies take several forms. Recurrent or sequence-based temporal architectures augment neural operators with learned temporal dynamics to capture longer-range dependencies \citep{li2022transformer,michalowska2024neural}. Multi-step or curriculum-based training more directly mitigates the mismatch between teacher-forced training and autoregressive deployment by progressively exposing the model to longer rollouts or model-generated states \citep{takamoto2023learning,li2022transformer,hagnberger2025calm}. A different strategy avoids recursive prediction over a prescribed interval by learning a direct space-time solution map \citep{li2020fourier,wang2021learning}. Although these strategies mitigate long-horizon prediction errors in some settings, their effectiveness can remain closely tied to the training regime and temporal range represented in the data. Recurrent or sequence-based architectures better capture temporal dependencies but can still suffer from error accumulation under the distribution shift induced by recursive deployment; curriculum or multi-step training reduces the train-test mismatch but may remain tied to the rollout horizons encountered during training; and direct space-time prediction avoids recursion over a prescribed interval but does not naturally extrapolate beyond it. Thus, reliable generalization to substantially longer temporal horizons remains challenging.

At the same time, most neural PDE solvers and operator-learning methods formulate the learned solution map or temporal evolution deterministically, while probabilistic approaches remain comparatively less explored. Existing work includes Bayesian operator learning and uncertainty quantification \citep{yang2022scalable,lin2023bdeeponet,garg2023vbdeeponet}, generative operator models \citep{rahman2022generative}, and more recent diffusion-based approaches for PDE modeling and temporal prediction \citep{lippe2023pde,serrano2024aroma}. More closely related to our formulation, Variational Autoencoding Neural Operators (VANO) extend variational autoencoders to functional data \citep{seidman2023variational}, while variational autoencoders have also been developed directly on function spaces with well-posed infinite-dimensional objectives \citep{bunker2025autoencoders}. These variational formulations, however, primarily concern static function representations or input-output operator mappings rather than recursively evolving latent dynamics. A variational formulation is particularly appealing in the dynamic setting: representing each physical state by a distribution in latent space allows the learned transition to be trained over neighborhoods of the observed trajectory, while distributional alignment between predicted and encoded next states provides an additional mechanism for regularizing the latent evolution. These properties suggest that variational modeling may offer a principled way to improve robustness to the perturbations encountered during autoregressive rollout. However, how such probabilistic latent dynamics should be formulated for PDE evolution, and how their variational structure relates to long-horizon deterministic prediction, remain insufficiently understood.

In this work, we study variational modeling of time-dependent PDE dynamics for long-horizon autoregressive prediction. We formulate PDE evolution through probabilistic latent Markov dynamics, where physical states are encoded into latent distributions, propagated through a stochastic transition model, and decoded back to physical space. Stochasticity is used during training to regularize the latent evolution, while inference follows deterministic mean dynamics. We instantiate this framework as the Variational Autoencoding Markov Operator (VAMO), which combines spatially resolved latent fields, structured Gaussian perturbations, and a neural-operator transition. Our main contributions are:

\begin{itemize}[topsep=0pt,itemsep=0pt]
\item\textit{Variational latent dynamics.} We develop a function-space variational formulation for learning temporal PDE evolution through latent Markov dynamics.
\item\textit{Rollout analysis.} We characterize how latent perturbations and variational transition alignment enter the error propagation of deterministic autoregressive rollouts.
\item\textit{Long-horizon evaluation.} We instantiate the framework as VAMO and evaluate it on three fluid-dynamics benchmarks, where it improves long-horizon rollout stability over deterministic latent, direct neural-operator, and noise-injection baselines.
\end{itemize}

Overall, our study positions variational modeling as a complementary approach to improving the robustness of learned PDE dynamics, rather than merely as a mechanism for probabilistic prediction. By connecting variational training with autoregressive evolution, we aim to broaden the role of probabilistic methods in neural PDE solvers toward long-horizon dynamical modeling.

\subsection{Related Work}
\paragraph{Neural PDE solvers and operator learning.} Neural networks have been widely used for PDE approximation through either physics-informed objectives or data-driven surrogate modeling. Physics-informed neural networks (PINNs) enforce governing equations and boundary or initial conditions through the training loss \citep{raissi2019physics, pang2019fpinns}, while operator-learning methods learn mappings between function spaces across families of PDE instances \citep{lu2022comprehensive,hoop2022cost,kovachki2023neural}. Representative approaches include DeepONet \citep{lu2021learning}, model-reduction-based operator learning \citep{bhattacharya2021model}, graph-based and general neural operators \citep{li2020neural,li2020multipole,kovachki2023neural}, and the Fourier neural operator (FNO) \citep{li2020fourier}. Subsequent extensions consider multiple-input operators, deeper architectures,
nonlinear manifold representations, physics-informed training, alternative spectral
representations, and complex geometries
\citep{jin2022mionet,rahman2023uno,seidman2022nomad,
li2024physics,gupta2021multiwavelet,li2023geo}. Another prominent line develops attention-based operator architectures, including the Galerkin Transformer \citep{cao2021choose}, LOCA \citep{kissas2022learning}, OFormer \citep{li2022transformer}, GNOT \citep{hao2023gnot}, ONO \citep{xiao2023improved}, PiT \citep{chen2024positional}, Transolver \citep{wu2024transolver}, and LNO \citep{wang2024latent}. Our work builds on this operator-learning paradigm and focuses on stable temporal evolution.

\paragraph{Learning physical dynamics.} A related line of work focuses explicitly on learning the temporal evolution of physical systems. Autoregressive simulators commonly learn a local update rule and repeatedly apply it to advance the physical state, including graph-based physical simulators, mesh-based models, neural PDE solvers, and convolutional encoder-decoder surrogates \citep{sanchez2020learning,pfaff2020learning,brandstetter2022message,stachenfeld2021learned,geneva2020modeling}. Since repeated prediction can amplify local errors and shift the model away from the state distribution encountered during training, prior works have explored perturbing training states, multi-step or curriculum training, adaptive temporal
stepping, and explicitly stabilized learned dynamics to improve rollout robustness \citep{sanchez2020learning,li2022transformer,hagnberger2025calm, wu2026tante,linot2023stabilized,lippe2023pde}. Another prominent direction evolves the system in a learned latent space, using recurrent models, continuous latent dynamics, Koopman-inspired representations, or learned latent operators \citep{wiewel2019latent,lusch2018deep,morton2018deep, pan2020physics,yin2023continuous,kontolati2024learning}. Along this direction, \cite{geneva2022transformers} further combines Koopman-informed embeddings with Transformers for temporal prediction, while OFormer \citep{li2022transformer} performs recurrent time-marching on spatially resolved latent representations. More recent approaches such as LNO and CALM-PDE similarly model time-dependent PDEs through compressed latent representations \citep{wang2024latent,hagnberger2025calm}. Our work is most closely related to this latent-dynamics perspective, but differs in formulating the latent evolution variationally.

\paragraph{Probabilistic and variational modeling.} Probabilistic PDE surrogates have been studied primarily for uncertainty quantification. Early work includes Bayesian convolutional encoder-decoder models for stochastic PDEs and physics-constrained generative surrogates \citep{zhu2018bayesian,zhu2019physics}. Similar ideas have subsequently been extended to operator learning through Bayesian, probabilistic, and variational neural-operator formulations \citep{yang2022scalable,lin2023bdeeponet,garg2023vbdeeponet,bulte2025probabilistic}. ore closely related to our formulation, several generative approaches model distributions over function-valued data. Generative Adversarial Neural Operators (GANO) learn distributions directly on function spaces using neural operators \citep{rahman2022generative}. Variational Autoencoding Neural Operators (VANO) extend variational autoencoders to functional data \citep{seidman2023variational}, while autoencoders have also been formulated directly on function spaces \citep{bunker2025autoencoders}. Probabilistic learning of physical systems has additionally been developed directly at the field level, including information-field-theoretic approaches for uncertainty-aware modeling \citep{alberts2023physics}. Diffusion-based methods provide another probabilistic generative approach to PDE modeling: diffusion models have been formulated as probabilistic neural operators \citep{haitsiukevich2024diffusion}, while Physics-Informed Diffusion Models incorporate PDE constraints into diffusion training \citep{bastek2025physics}. For temporal prediction, PDE-Refiner uses diffusion-inspired denoising for long autoregressive rollouts \citep{lippe2023pde}, and AROMA employs diffusion-based latent dynamics \citep{serrano2024latent,serrano2024aroma}.

\subsection{Roadmap}
\label{subsec:roadmap}

The remainder of the paper is organized as follows. \S\ref{sec:problem} introduces the time-dependent PDE learning problem, the long-horizon autoregressive prediction setting, and the variational autoencoder preliminaries used throughout the paper. \S\ref{sec:var_latent_dyn} develops variational latent dynamics directly on function spaces, first establishing the abstract variational formulation, then specializing to functional Gaussian models and analyzing the propagation of errors under deterministic mean rollout. \S\ref{sec:vamo} introduces the Variational Autoencoding Markov Operator (VAMO), including its spatially resolved architecture, structured Gaussian latent perturbations, and training objective. \S\ref{sec:experiments} evaluates the resulting framework on long-horizon fluid-dynamics benchmarks. 

\paragraph{Notation.}
We use calligraphic letters such as $\mathcal U$ and $\mathcal Z$ for physical and latent state spaces, and bold uppercase letters $\mathbf U$ and $\mathbf Z$ for the corresponding function-valued states. Lowercase bold letters $\mathbf u$ and $\mathbf z$ denote their finite-dimensional discretizations when needed. For a measurable space $\mathcal X$, $\Delta(\mathcal X)$ denotes the set of probability measures on $\mathcal X$. For two probability measures $P$ and $Q$, we write $P\ll Q$ for absolute continuity, $D_{\mathrm{KL}}(P\Vert Q)$ for the Kullback-Leibler divergence, and $\mathsf N(m,K)$ for a Gaussian measure with mean $m$ and covariance operator $K$.
For a linear operator $K$, let $\|K\|_{\mathrm{op}}$ and $\operatorname{Tr}(K)$ denote its operator norm and trace, respectively. Expectations are written as $\mathbb E$, with subscripts indicating the corresponding distribution when needed.

\section{Problem Setting and Preliminaries}
\label{sec:problem}

We first introduce the time-dependent PDE learning problem considered throughout this work and formalize the autoregressive prediction setting. We then briefly review the variational autoencoder framework that motivates the latent probabilistic formulation developed in \S\ref{sec:var_latent_dyn}.

\subsection{PDE Evolution and Flow Maps}
\label{subsec:pde_flow}
We consider a class of time-dependent partial differential equations (PDEs) defined on a spatial domain $\Omega \subset \mathbb{R}^{d_x}$. Let
\[
    \mathbf{U}(t):\Omega\rightarrow\mathbb{R}^{c}
\]
denote the physical state at time $t$, where $c$ is the number of physical channels. We abstract the governing dynamics, conditional on fixed problem-specific quantities, as an autonomous evolution equation
\begin{equation}
    \partial_t \mathbf{U}(t)=\mathcal{F}(\mathbf{U}(t)),
    \label{eq:pde_evolution}
\end{equation}
where $\mathcal{F}$ is a generally nonlinear differential operator. In concrete PDEs, $\mathcal{F}$ may depend on spatial derivatives of $\mathbf{U}$, physical parameters, forcing fields, boundary conditions, or other problem-specific quantities. When these quantities vary across trajectories, we treat them as additional conditioning inputs and suppress them from the notation for simplicity.

For a fixed time increment $h>0$, the PDE induces a solution operator, or flow map,
\begin{equation}
    \mathcal{S}_h: \mathbf{U}(t)\mapsto \mathbf{U}(t+h).
    \label{eq:flow_map}
\end{equation}
Given discrete times $t_n=nh$, we write $\mathbf{U}_n:=\mathbf{U}(t_n)$, so that the exact discrete-time evolution satisfies
\begin{equation}
    \mathbf{U}_{n+1}=\mathcal{S}_h(\mathbf{U}_n).
    \label{eq:discrete_flow}
\end{equation}
Thus, learning the temporal PDE dynamics can be viewed as learning an approximation to the local
flow map $\mathcal{S}_h$ from observed trajectories.

\paragraph{Autoregressive rollout.}
Let the training data consist of trajectories
\[
    \mathcal{D}=\left\{
        \left(\mathbf{U}_0^{(i)},\mathbf{U}_1^{(i)},\ldots,
        \mathbf{U}_{N_{\rm train}}^{(i)}\right)
    \right\}_{i=1}^{N_{\rm data}},
\]
possibly together with the problem-dependent conditions suppressed above. A local data-driven solver learns the transition $\mathbf{U}_n\mapsto \mathbf{U}_{n+1}$
from adjacent states in these trajectories. Denoting the learned local flow map by $\widehat{\mathcal{S}}_h$, prediction from an initial state $\mathbf{U}_0$ is performed autoregressively as
\begin{equation}
    \widehat{\mathbf{U}}_0=\mathbf{U}_0,
    \qquad
    \widehat{\mathbf{U}}_{n}
    =\widehat{\mathcal{S}}_h(\widehat{\mathbf{U}}_{n-1}),
    \qquad n=1,2,\cdots,N_{\rm test}.
    \label{eq:autoregressive_rollout}
\end{equation}
This creates an important distinction between training and deployment. During supervised training on adjacent state pairs, $\widehat{\mathcal{S}}_h$ is evaluated on ground-truth states $\mathbf{U}_n$, whereas during autoregressive rollout it is repeatedly evaluated on its own predictions $\widehat{\mathbf{U}}_n$. Consequently, local prediction errors can move the rollout away from the state distribution represented in the training trajectories and propagate through subsequent predictions.

Our primary interest is temporal extrapolation, where the available training trajectories cover only a finite horizon while prediction extends substantially beyond it,
\[
    N_{\rm test}>N_{\rm train}.
\]
This setting has also been studied in recent work on long-time PDE forecasting and temporal extrapolation \citep{yin2023continuous,michalowska2024neural,diab2025temporal}. In our experiments, the training and test initial conditions are sampled from the same problem-specific distributions, so the principal extrapolation occurs along the temporal direction. The goal is therefore not only to approximate the local flow map accurately, but also to maintain stable predictions under repeated autoregressive application.

\subsection{Variational Autoencoders}
\label{subsec:vae_review}

We briefly review the variational autoencoder (VAE) framework \citep{kingma2013auto, seidman2023variational} in a form that will be used throughout the paper. Let $\mathcal U$ denote the data space and $\mathcal Z$ the latent space. We denote a data sample by $\mathbf U\in\mathcal U$ and its latent representation by $\mathbf Z\in\mathcal Z$. A variational autoencoder consists of two main probabilistic components. The \emph{encoder} $Q_\phi$ defines a conditional distribution
\[
    \mathcal{U}\ni\mathbf U\,\mapsto\, Q_\phi(\d\mathbf Z\,|\,\mathbf U)\in\Delta(\mathcal Z)
\]
over latent representations $\mathbf Z\in\mathcal Z$ given an observation $\mathbf U\in\mathcal U$. Similarly, the \emph{decoder} $P_\psi$ is a conditional kernel
\[
    \mathcal{Z}\ni\mathbf Z\,\mapsto\, P_\psi(\d\mathbf U\,|\,\mathbf{Z})\in\Delta(\mathcal U)
\]
over observations conditioned on the latent representation. Here $\phi\in\Phi$ and $\psi\in\Psi$ denote the encoder and decoder parameters, respectively.

The generative model is completed by prescribing a reference latent distribution $\mathsf P_{\mathcal Z}\in\Delta(\mathcal Z)$. A latent sample is first drawn according to $\mathbf Z\sim\mathsf P_{\mathcal Z}$, and the corresponding observation is generated from $\mathbf U\sim P_\psi(\cdot\mid\mathbf Z)$. The encoder $Q_\phi(\cdot\mid\mathbf U)$ provides a tractable variational approximation to the generally intractable posterior distribution of $\mathbf Z$ conditioned on $\mathbf U$.

To train the model, one maximizes a variational lower bound on the data likelihood. Let $\mathsf W_{\mathcal U}$ be a reference measure on $\mathcal U$, and suppose that the decoder likelihood is dominated by $\mathsf W_{\mathcal U}$. Under the usual absolute-continuity conditions, the evidence lower bound (ELBO) takes the form
\begin{equation}
\begin{aligned}
    \log
    \frac{\d P_\psi^{\,\mathcal U}}
         {\d\mathsf W_{\mathcal U}}(\mathbf U)
    \geq
    &\;
    \mathbb E_{\mathbf Z\sim Q_\phi(\cdot\mid\mathbf U)}
    \left[
        \log
        \frac{\d P_\psi(\cdot\mid\mathbf Z)}
             {\d\mathsf W_{\mathcal U}}(\mathbf U)
    \right]-
    D_{\mathrm{KL}}
    \left(
        Q_\phi(\cdot\mid\mathbf U)
        \,\Vert\,
        \mathsf P_{\mathcal Z}
    \right),
\end{aligned}
\label{eq:vae_elbo}
\end{equation}
where
\[
P_\psi^{\,\mathcal U}(\d\mathbf U)=\int_{\mathcal Z} P_\psi(\d\mathbf U\mid\mathbf Z)\,\mathsf P_{\mathcal Z}(\d\mathbf Z)
\]
is the marginal distribution induced by the generative model.

The two terms in \eqref{eq:vae_elbo} play complementary roles. The first encourages latent samples produced by the encoder to retain sufficient information for reconstruction through the decoder, while the KL term regularizes the encoded distribution toward the prescribed latent prior. Thus, a variational autoencoder combines probabilistic autoencoding with distributional regularization of the latent representation.

While the standard VAE formulation above regularizes the latent representation against a fixed reference distribution \(\mathsf P_{\mathcal Z}\), corresponding to a static latent prior, variational latent-variable models have also been extended to sequential and dynamical settings through learned latent dynamics, including state-space models \citep{krishnan2015deep,krishnan2017structured,rangapuram2018deep}, latent ordinary differential equations \citep{rubanova2019latent,yildiz2019ode2vae}, and latent stochastic differential equations \citep{ha2018adaptive,tzen2019neural,hasan2021identifying}. Motivated by these perspectives, we replace the static latent prior of a standard VAE with a conditional transition kernel
\[
P_\theta(d\mathbf Z_{n+1}\mid \mathbf Z_n),
\]
so that variational regularization is imposed on the evolution of the latent state. Our focus differs from these finite-dimensional latent-dynamics models in developing the resulting latent Markov formulation directly on function spaces for time-dependent PDE evolution. In \S\ref{sec:var_latent_dyn}, we give a detailed construction of this framework.

\section{Variational Latent Dynamics on Function Spaces}\label{sec:var_latent_dyn}
In this section, we develop the theoretical foundation of variational latent dynamics for time-dependent PDEs. We first formulate the model abstractly on physical and latent function spaces and characterize the variational objectives governing latent representation and transition alignment in \S\ref{sec:abstract_vld}. Then in \S\ref{sec:functional_gaussian_latent_markov}, we specialize to functional Gaussian distributions, where Cameron-Martin theory provides tractable likelihood representations and the covariance structure induces a spectral geometry on latent perturbations and residuals. Finally, in \S\ref{sec:err_var_dyn}, we analyze long-horizon error propagation under deterministic mean rollout, highlighting two complementary mechanisms of variational training: noise injection promotes local stability, while variational alignment provides direct control of latent transition errors and their accumulation over the rollout horizon.

\subsection{Abstract Variational Latent Dynamics}\label{sec:abstract_vld}
We formulate the latent dynamics directly at the level of function spaces \citep{bogachev1998gaussian}, before introducing a particular Gaussian parameterization or spatial discretization. Let $\mathcal U$ denote the space of physical PDE states and $\mathcal Z$ the corresponding latent state space.

\begin{assumption}[State spaces]
\label{assump:state_spaces}
The physical state space $\mathcal U$ and latent state space $\mathcal Z$ are real separable Banach spaces, equipped with their Borel $\sigma$-algebras.
\end{assumption}

This setting is sufficiently general to include many state spaces arising in PDE modeling \citep{stuart2010inverse,dashti2013bayesian}, while providing the standard measurable structure required to define conditional probability distributions on $\mathcal U$ and $\mathcal Z$. In particular, separability ensures that these spaces have the regularity needed for the probabilistic constructions below.

Built on this setting, we study a variational latent formulation of the PDE flow map. Instead of directly learning a deterministic propagator in the physical space, we introduce a latent representation $\mathbf{Z}_n\in \mathcal Z$ for each physical state $\mathbf{U}_n\in \mathcal U$ and model the one-step evolution through the latent pathway
\begin{equation*}
    \mathbf{U}_n\xrightarrow{\text{encode}}
    \mathbf{Z}_n\xrightarrow{\text{evolve}}
    \mathbf{Z}_{n+1}\xrightarrow{\text{decode}}
    \mathbf{U}_{n+1}.
\end{equation*}
Along this pathway, an encoder maps the physical field to a latent representation, a latent transition model propagates this representation in time, and a decoder maps the propagated latent state back to the physical space. This formulation provides a probabilistic description of the learned dynamics while allowing the evolution to be regularized in a structured latent space.

Concretely, for the pathway above, we introduce encoder, transition and decoder kernels, parameterized by $\phi\in\Phi,\theta\in\Theta$ and $\psi\in\Psi$, respectively:
\begin{equation*}
\begin{aligned}
\text{(encoder)}\qquad &\mathcal{U}\ni\mathbf{U}_n\,\mapsto \,Q_\phi(\d\mathbf{Z}_n \,|\, \mathbf{U}_n)\in\Delta\left(\mathcal{Z}\right),\\
\text{(transition)}\qquad &\mathcal{Z}\ni\mathbf{Z}_n\,\mapsto\, P_{\,\theta}(\d\mathbf{Z}_{n+1} \,|\, \mathbf{Z}_n)\in\Delta(\mathcal{Z}),\\
\text{(decoder)}\qquad &\mathcal{Z}\ni\mathbf{Z}_n\,\mapsto \,P_\psi(\d\mathbf{U}_{n+1}\,|\, \mathbf{Z}_{n+1})\in\Delta(\mathcal{U}).
\end{aligned}
\end{equation*}
Together, these components define a probabilistic approximation of the one-step PDE flow map and yield a variational latent Markov model for time-dependent PDEs.

The formulation is independent of the particular physical variables used to represent the underlying PDE state. For example, $\mathbf U_n$ may denote the vorticity field in incompressible Navier-Stokes or the density, velocity, and pressure fields in compressible Euler. The same latent probabilistic framework therefore applies across different PDE systems through the corresponding physical-state space and data distribution.

Before specializing the model to functional Gaussian kernels, we first establish the variational foundation of the latent Markov formulation at the level of general function-space probability measures. We begin by deriving lower bounds on the one-step predictive likelihood, including a conditional functional ELBO that motivates the reconstruction and latent transition consistency terms. We then study the dynamics term at the population level and characterize the latent Markov transition targeted by its minimization. These results do not rely on any particular parameterization and therefore apply to the general formulation above.

\begin{proposition}[Functional lower bounds]
\label{prop:conditional_functional_elbo}
Let $\mathcal U$ and $\mathcal Z$ be Banach spaces. Let
$\mathsf W_{\mathcal U}$ be a fixed reference measure on $\mathcal U$. For a consecutive pair $(\mathbf U_n,\mathbf U_{n+1})\in\mathcal{U}\times\mathcal{U}$, suppose that the encoder, latent transition, and decoder are given by Markov kernels
\[
    Q_\phi(\d\mathbf Z_n\,|\, \mathbf U_n)\in\Delta(\mathcal{Z}),\quad
    P_\theta(\d\mathbf Z_{n+1}\,|\, \mathbf Z_n)\in\Delta(\mathcal{Z}),\quad
    P_\psi(\d\mathbf U_{n+1}\,|\, \mathbf Z_{n+1})\in\Delta(\mathcal{U}).
\]
Assume that the decoder kernel is absolutely continuous with respect to $\mathsf W_{\mathcal U}$:
\[
    P_\psi(\cdot\mid \mathbf Z)\ll \mathsf W_{\mathcal U}.
\]
Then the following lower bounds hold.
\begin{itemize}
\item[(a)] (Predictive decoder likelihood bound).
\begin{equation}
    \log \frac{\d P_{\theta,\psi,\phi}(\cdot\,|\,\mathbf{U}_n)}{\d\mathsf{W}_\mathcal{U}}(\mathbf{U}_{n+1})\geq\bbE_{\mathbf{Z}_{n+1}\sim P_\theta(\cdot|\mathbf{Z}_n),\mathbf{Z}_n\sim Q_\phi(\cdot|\mathbf{U}_n)}\left[\log \frac{\d P_\psi(\cdot\,|\,\mathbf{Z}_{n+1})}{\d\mathsf{W}_\mathcal{U}}(\mathbf{U}_{n+1})\right].\label{eq:predlb}
\end{equation}
\item[(b)] (Conditional functional ELBO). Assume that the relevant likelihood densities exist and that
\[
    Q_\phi(\cdot\,|\, \mathbf U_{n+1})
    \ll
    P_\theta(\cdot\,|\, \mathbf Z_n)
\]
for $Q_\phi(\cdot\,|\, \mathbf U_n)$-almost every $\mathbf Z_n$. Then
\[
\begin{aligned}
    \log \frac{\d P_{\theta,\psi,\phi}(\cdot\,|\,\mathbf{U}_n)}{\d\mathsf{W}_\mathcal{U}}(\mathbf{U}_{n+1})
    &\geq\mathbb{E}_{\mathbf Z_{n+1}\sim Q_\phi(\cdot\,|\, \mathbf U_{n+1})}
    \left[
        \log\frac{\d P_\psi(\cdot\,|\, \mathbf Z_{n+1})}{\d\mathsf{W}_\mathcal{U}}(\mathbf U_{n+1})
    \right]
    \\
    &\quad -\mathbb E_{\mathbf Z_n\sim Q_\phi(\cdot\,|\, \mathbf U_n)}
    \left[D_{\mathrm{KL}}\left(
    Q_\phi(\cdot\,|\, \mathbf U_{n+1})\,\Vert\,P_\theta(\cdot\,|\, \mathbf Z_n)
    \right)\right].
\end{aligned}
\]
\end{itemize}
\end{proposition}

\begin{remark}
Equivalently, the first bound gives the predictive negative log-likelihood objective
\begin{equation}
    \mathcal L_{\mathrm{NLL}}^{\mathrm{pred}}(\mathbf U_n,\mathbf U_{n+1})=\bbE_{\mathbf{Z}_{n+1}\sim P_\theta(\cdot|\mathbf{Z}_n),\mathbf{Z}_n\sim Q_\phi(\cdot|\mathbf{U}_n)}\left[-\log \frac{\d P_\psi(\cdot\,|\,\mathbf{Z}_{n+1})}{\d\mathsf{W}_{\mathcal U}}(\mathbf{U}_{n+1})\right].\label{eq:pred_nll}
\end{equation}
Similarly, the negative conditional ELBO is
\begin{equation}
\begin{aligned}
    \mathcal L_{\mathrm{ELBO}}(\mathbf U_n,\mathbf U_{n+1})
    =
    &-\mathbb E_{\mathbf Z_{n+1}\sim Q_\phi(\cdot\,|\, \mathbf U_{n+1})}
    \left[
        \log \frac{\d P_\psi(\cdot\,|\, \mathbf Z_{n+1})}{\d\mathsf W_{\mathcal U}}(\mathbf U_{n+1})
    \right]
    \\
    &+
    \mathbb E_{\mathbf Z_n\sim Q_\phi(\cdot\,|\, \mathbf U_n)}
    \left[
    D_{\mathrm{KL}}
    \left(
        Q_\phi(\cdot\,|\, \mathbf U_{n+1})
        \,\Vert\,
        P_\theta(\cdot\,|\, \mathbf Z_n)
    \right)
    \right].
\end{aligned}\label{eq:condnegelbo}
\end{equation}
Notably, the ELBO can be understood in the extended-real sense. If $Q_\phi(\cdot\mid \mathbf U_{n+1})$ is not absolutely continuous with respect to $P_\theta(\cdot\mid\mathbf Z_n)$, the corresponding KL divergence is infinite and the lower bound becomes trivial. A finite, nontrivial ELBO therefore requires
\[
    Q_\phi(\cdot\mid \mathbf U_{n+1})
    \ll
    P_\theta(\cdot\mid \mathbf Z_n)
\]
for $Q_\phi(\cdot\mid \mathbf U_n)$-almost every $\mathbf Z_n$.
\end{remark}

The conditional ELBO in \eqref{eq:condnegelbo} separates two complementary learning objectives. The reconstruction term
\[
    -
    \mathbb E_{\mathbf Z_{n+1}\sim Q_\phi(\cdot\,|\,\mathbf U_{n+1})}
    \big[\log P_\psi(\mathbf U_{n+1}\,|\,\mathbf Z_{n+1})\big],
\]
encourages latent samples drawn from $Q_\phi(\cdot\mid\mathbf U_{n+1})$ to retain the information needed to reconstruct the physical state $\mathbf U_{n+1}$. It therefore promotes consistency between the encoder and decoder and ensures that the latent variables remain informative representations of the underlying PDE states. The second term,
\[
    \mathbb E_{\mathbf Z_n\sim Q_\phi(\cdot\,|\,\mathbf U_n)}
    \left[D_{\mathrm{KL}}
    \left(
        Q_\phi(\cdot\,|\,\mathbf U_{n+1})\,\Vert\,P_\theta(\cdot\,|\,\mathbf Z_n)
    \right)\right],
\]
enforces consistency of the latent dynamics. Given a latent state $\mathbf Z_n$ encoded from $\mathbf U_n$, the transition kernel $P_\theta(\cdot\mid\mathbf Z_n)$ predicts a distribution over the next latent state. The KL divergence aligns this prediction with the posterior distribution obtained by encoding the true next state $\mathbf U_{n+1}$. Thus, the variational objective couples two complementary requirements: \textit{physical reconstruction} and \textit{latent transition alignment}. The preceding ELBO gives a sample-level variational objective. We next characterize its latent dynamics term at the population level and identify the transition kernel targeted by its minimization.

\begin{proposition}[Population decomposition of the ELBO dynamics term]
\label{prop:population_elbo_transition_consistency}
Let $\rho$ be a probability measure on consecutive PDE states $(\mathbf U,\mathbf U^+)\in\mathcal U\times\mathcal U$, and let $Q_\phi(\cdot\mid\mathbf U)$ be an encoder kernel. For a transition kernel $P$ from $\mathcal Z$ to $\mathcal Z$, define the population dynamics loss
\[
\mathcal R(P)=\mathbb E_{(\mathbf U,\mathbf U^+)\sim\rho} \mathbb E_{\mathbf Z\sim Q_\phi(\cdot\mid\mathbf U)} \left[D_{\mathrm{KL}}\left(Q_\phi(\cdot\mid\mathbf U^+)
\,\Vert\,P(\cdot\mid\mathbf Z)\right)\right].
\]
Let $\Gamma_\phi$ be the joint law of $(\mathbf{U},\mathbf{U}^+,\mathbf{Z},\mathbf{Z}^+)$ induced by
\[
    (\mathbf U,\mathbf U^+)\sim\rho,\qquad
    \mathbf Z\sim Q_\phi(\cdot\mid\mathbf U),\qquad
    \mathbf Z^+\sim Q_\phi(\cdot\mid\mathbf U^+),
\]
and let $\Pi_\phi(\cdot\mid\mathbf Z)$ be the conditional law of
$\mathbf Z^+$ given $\mathbf Z$ under $\Gamma_\phi$. Then
\begin{equation}
\mathcal R(P)=\bbE_{(\mathbf{Z},\mathbf{U}^+)\sim\Gamma_\phi}\left[D_{\mathrm{KL}}
    \left(Q_\phi(\cdot\,|\,\mathbf U^+)\,\Vert\,\Pi_\phi(\cdot\,|\,\mathbf Z)
    \right)\right]+
\bbE_{\mathbf{Z}\sim\Gamma_\phi}\left[D_{\mathrm{KL}}
    \left(\Pi_\phi(\cdot\,|\,\mathbf Z)\,\Vert\,P(\cdot\,|\,\mathbf Z)
    \right)\right].\label{eq:elbokldecomp}
\end{equation}
Consequently, among all Markov kernels
$P$, the population dynamics loss is minimized by
\[
    P^\star(\cdot\,|\,\mathbf Z)=\Pi_\phi(\cdot\,|\,\mathbf Z)
\]
for $\Gamma_\phi$-almost every $\mathbf Z$. 
\end{proposition}

Proposition~\ref{prop:population_elbo_transition_consistency} separates the population dynamics loss into an encoder-induced term and a transition-model mismatch:
\begin{itemize}
\item The first term in \eqref{eq:elbokldecomp} depends only on the encoder and the distribution of consecutive physical states. It quantifies the dispersion of the encoded next-state posterior $Q_\phi(\cdot\mid\mathbf U^+)$ around the aggregated conditional latent transition $\Pi_\phi(\cdot\mid\mathbf Z)$ and is therefore irreducible when optimizing over the transition kernel $P$.

\item The second term is the only component that depends on $P$. It measures the discrepancy between the learned Markov transition $P(\cdot\mid\mathbf Z)$ and the encoder-induced conditional transition $\Pi_\phi(\cdot\mid\mathbf Z)$. Hence, for a fixed encoder, minimizing the population dynamics objective is equivalent to matching the learned transition to $\Pi_\phi$ almost surely under the latent marginal distribution. In particular, if the model class $\{P_\theta\}_{\theta\in\Theta}$ contains $\Pi_\phi$, any global population minimizer satisfies
\[
    P_{\theta^\star}(\cdot\mid\mathbf Z)
    =
    \Pi_\phi(\cdot\mid\mathbf Z)
\]
up to $\Gamma_\phi$-null sets. Therefore, rather than identifying the original PDE generator $\mathcal F$ directly, the ELBO dynamics term encourages the model to learn the Markov transition induced by the encoder on latent representations of true PDE trajectories.
\end{itemize}

The preceding results characterize the variational latent dynamics without imposing a specific form on the underlying kernels. We now specialize this framework to
functional Gaussian kernels and study the covariance-induced geometry of the latent dynamics through their Cameron-Martin spaces and spectral structure.

\subsection{Functional Gaussian Latent Markov Model}
\label{sec:functional_gaussian_latent_markov}

We now specialize the abstract variational formulation to Gaussian distributions on function spaces \citep{bogachev1998gaussian,kuo2006gaussian,da2014stochastic,bunker2025autoencoders}. This specialization serves two purposes. First, the Cameron-Martin theorem converts the Radon-Nikodym derivatives in the functional variational objective into tractable reconstruction and transition terms. Second, the covariance operators induce a geometry on the latent space that determines how perturbations and transition residuals are weighted across different functional directions. 

We briefly review the necessary background on Gaussian measures and Cameron-Martin spaces in Appendix~\S\ref{app:gaussian_measures_banach}. Recall that a functional Gaussian measure $\mathsf N(m,K)$ is characterized by its mean element $m$ and covariance operator $K$. Accordingly, we consider the functional Gaussian parameterization
\begin{equation*}
\begin{aligned}
    \text{(encoder)}\qquad Q_\phi(\cdot\mid\mathbf U)
    &=\mathsf N\bigl(m_\phi(\mathbf U),K_\phi(\mathbf U)\bigr),
    \\
    \text{(transition)}\ \qquad P_\theta(\cdot\mid\mathbf Z)
    &=\mathsf N\bigl(T_\theta(\mathbf Z),\alpha_\theta(\mathbf Z)^2 K\bigr),
    \\
    \text{(decoder)\,}\qquad P_\psi(\cdot\mid\mathbf Z)
    &=\mathsf N\bigl(D_\psi(\mathbf Z), K_\psi\bigr).
\end{aligned}
\end{equation*}
Here $m_\phi$ denotes the encoder mean, $T_\theta$ the mean latent transition, and $D_\psi$ the decoder mean. The state-dependent covariance $K_\phi(\mathbf U)$ describes uncertainty in the encoded representation, while $K$ specifies the spatial structure of the latent transition noise and $\alpha_\theta(\mathbf Z)>0$ controls its amplitude. The decoder covariance $K_\psi$ is fixed.

These Gaussian distributions are defined on the underlying function spaces rather than on a particular spatial discretization. The finite-dimensional parameterizations used in computation are introduced later in \S\ref{subsec:structured_latent_perturbations}. Here we focus on the function-space structure induced by the corresponding Gaussian measures.

\paragraph{Cameron-Martin representation of the likelihood.}
The likelihood terms in the functional variational bounds of Proposition~\ref{prop:conditional_functional_elbo} are expressed through Radon-Nikodym derivatives with respect to reference measures. For Gaussian kernels, the Cameron-Martin theorem \citep{stroock2010probability} gives an explicit representation of these derivatives in terms of the geometry induced by the covariance operator. We first record this consequence for the decoder likelihood.

\begin{proposition}[Cameron-Martin form of the decoder likelihood]
\label{prop:cm_decoder_likelihood}
Let $\mathcal U$ be a separable Banach space and let $\mathsf W_{\mathcal U}$ be a centered Gaussian measure on $\mathcal U$. Denote its Cameron-Martin space by $\mathcal H_{\mathcal U}$, continuously embedded in $\mathcal U$. Suppose the decoder kernel is given by the
Cameron-Martin shift
\[
    P_\psi(\d\mathbf U\mid \mathbf Z)
    =
    \mathsf W_{\mathcal U}^{D_\psi(\mathbf Z)}(\d\mathbf U),
\]
where $D_\psi(\mathbf Z)\in \mathcal H_{\mathcal U}$.
Then
\[
    P_\psi(\cdot\mid \mathbf Z)\ll \mathsf W_{\mathcal U},
\]
and the Radon-Nikodym derivative is
\begin{equation}
\frac{\d P_\psi(\cdot\mid \mathbf Z)}{\d\mathsf W_{\mathcal U}
}(\mathbf U)
=\exp\left(\langle D_\psi(\mathbf Z),\mathbf U\rangle^\sim
-\frac12\|D_\psi(\mathbf Z)\|_{\mathcal H_{\mathcal U}}^2
\right),\label{eq:cm_decoder_rn}
\end{equation}
where $\langle\cdot,\cdot\rangle^\sim$ denotes the Paley-Wiener map associated with $\mathsf W_{\mathcal U}$. Consequently, the decoder negative log-likelihood with respect to
$\mathsf W_{\mathcal U}$ is
\begin{equation}
-\log\frac{\d P_\psi(\cdot\mid \mathbf Z)}{\d\mathsf W_{\mathcal U}}(\mathbf U)
    =
    \frac12
    \|D_\psi(\mathbf Z)\|_{\mathcal H_{\mathcal U}}^2
    -
    \langle D_\psi(\mathbf Z),\mathbf U\rangle^\sim.
    \label{eq:cm_decoder_nll}
\end{equation}
In particular, if $\mathbf U\in \mathcal H_{\mathcal U}$, then
\[
    \langle D_\psi(\mathbf Z),\mathbf U\rangle^\sim=
    \langle D_\psi(\mathbf Z),\mathbf U\rangle_{\mathcal H_{\mathcal U}},
\]
and therefore
\begin{equation}
-\log\frac{\d P_\psi(\cdot\mid \mathbf Z)}{\d\mathsf W_{\mathcal U}}(\mathbf U)
=
\frac12\|\mathbf U-D_\psi(\mathbf Z)\|_{\mathcal H_{\mathcal U}}^2-\frac12\|\mathbf U\|_{\mathcal H_{\mathcal U}}^2.
\label{eq:cm_decoder_squared_norm}
\end{equation}
Thus, up to a term depending only on the observed field $\mathbf U$, the Gaussian decoder negative log-likelihood is equivalent to the squared Cameron-Martin reconstruction error.
\end{proposition}

\begin{remark}[Deterministic decoder mean]
The probabilistic decoder is parameterized through the deterministic mean map $D_\psi$, while the observation uncertainty is represented by the fixed Gaussian reference measure. Proposition~\ref{prop:cm_decoder_likelihood}
therefore connects the functional Gaussian likelihood to the deterministic decoder used in computation. In particular, after finite-dimensional discretization with isotropic covariance, the Cameron-Martin reconstruction
term reduces to the usual squared Euclidean  reconstruction loss.
\end{remark}

\paragraph{Spectral geometry of the latent transition.}
The same Gaussian structure also determines how latent transition errors are measured. We now assume that $\mathcal Z$ is a separable Hilbert space and consider a positive, self-adjoint, trace-class covariance operator, which admits a spectral decomposition and the corresponding Karhunen-Lo\`eve representation for Gaussian random elements \citep{da2014stochastic,bogachev1998gaussian}. The resulting Cameron-Martin norm reveals how the covariance spectrum assigns different weights to different latent directions.

\begin{proposition}[Spectral geometry of the latent transition]
\label{prop:general_spectral_residual_control}
Let $\mathcal Z$ be a separable Hilbert space and let linear operator $K:\mathcal Z\to\mathcal Z$ be self-adjoint, positive-definite, and trace-class. Let $\{(\lambda_i,\mathbf E_i)\}_{i\geq1}$ be an eigensystem of $K$, where
$\{\mathbf E_i\}_{i\geq1}$ is an orthonormal basis of $\mathcal Z$, eigenvalues $\lambda_i>0$, and $\sum_{i=1}^{\infty}\lambda_i<\infty$. Fix
\[
    \mathsf W_{\mathcal Z}
    =
    \mathsf N(0,K)
\]
and consider the latent transition
\[
    P_\theta(\cdot\mid\mathbf Z)
    =
    \mathsf N
    \bigl(
        T_\theta(\mathbf Z),
        \alpha_\theta(\mathbf Z)^2K
    \bigr),
\]
where $T_\theta(\mathbf Z)\in\mathcal H_K$ and
$\alpha_\theta(\mathbf Z)>0$. Then:

\begin{itemize}[topsep=0pt,itemsep=2pt]
\item[(a)]
The Cameron-Martin space associated with $\mathsf W_{\mathcal Z}$ is
\[
    \mathcal H_K
    =
    \left\{
        \mathbf R=\sum_{i=1}^{\infty}r_i\mathbf E_i:
        \sum_{i=1}^{\infty}\frac{|r_i|^2}{\lambda_i}<\infty
    \right\},
\]
with norm
\begin{equation}
    \|\mathbf R\|_{\mathcal H_K}^2
    =
    \sum_{i=1}^{\infty}
    \frac{|r_i|^2}{\lambda_i}.
    \label{eq:cm_spectral_norm}
\end{equation}
\item[(b)]
The transition admits the Karhunen--Lo\`eve representation
\begin{equation}
    \mathbf Z^+
    =
    T_\theta(\mathbf Z)
    +
    \alpha_\theta(\mathbf Z)
    \sum_{i=1}^{\infty}
        \sqrt{\lambda_i}\,\xi_i\mathbf E_i,
    \qquad
    \xi_i\overset{\mathrm{i.i.d.}}{\sim}\mathsf N(0,1),
    \label{eq:transition_kl_expansion}
\end{equation}
where the series converges in $L^2(\Omega;\mathcal Z)$ and almost surely in $\mathcal Z$. Moreover, for $\mathbf Z^+\in\mathcal H_K$,
\begin{equation}
\begin{aligned}
    -\log\frac{\d P_\theta(\cdot\mid\mathbf Z)}{\d\mathsf N(0,\alpha_\theta(\mathbf Z)^2K)}(\mathbf Z^+)
    &=\frac{1}{2\alpha_\theta(\mathbf Z)^2}\|\mathbf Z^+ -T_\theta(\mathbf Z)\|_{\mathcal H_K}^2 -\frac{1}{2\alpha_\theta(\mathbf Z)^2}\|\mathbf Z^+\|_{\mathcal H_K}^2 .
\end{aligned}
\label{eq:transition_cm_density}
\end{equation}
\end{itemize}
\end{proposition}

\begin{remark}[Spectral-coordinate interpretation]
Equivalently, writing
\[
    T_\theta(\mathbf Z)=\sum_{i=1}^{\infty}t_i(\mathbf Z)\mathbf E_i,\qquad\mathbf Z^+=\sum_{i=1}^{\infty}z_i^+\mathbf E_i,
\]
the representation \eqref{eq:transition_kl_expansion} gives
\[
    z_i^+=t_i(\mathbf Z)+\alpha_\theta(\mathbf Z)\sqrt{\lambda_i}\,\xi_i,\qquad
    \operatorname{Var}(z_i^+\mid\mathbf Z)=\alpha_\theta(\mathbf Z)^2\lambda_i.
\]
Hence the covariance spectrum simultaneously determines the directions and scales of the injected transition noise and the geometry of the corresponding Cameron-Martin residual. Directions with smaller $\lambda_i$ receive less stochastic variation and a larger weight $1/\lambda_i$ in the residual norm, whereas directions with larger $\lambda_i$ are perturbed more strongly and penalized less.
\end{remark}

A particularly relevant specialization arises when the latent space is a periodic function space and the covariance operator is diagonal in the Fourier basis, as for Laplacian-resolvent covariances \citep{stuart2010inverse,dashti2013bayesian}. In this setting, the spectral weighting induced by the Cameron-Martin norm admits a direct Sobolev interpretation.

\begin{corollary}[Sobolev geometry of latent residuals]
\label{cor:sobolev_residual_control}
Under the setting of Proposition~\ref{prop:general_spectral_residual_control},
suppose
\[
    \mathcal Z=L^2(\mathbb T^{d_x};\mathbb R^{d_z})
\]
and
\[
    K=(I-\ell^2\Delta)^{-s},\qquad \ell>0, \qquad s>\frac{d_x}{2},
\]
where the Laplacian acts componentwise. With the Fourier basis
\[
    \phi_m(x)=e^{2\pi i m\cdot x},
    \qquad
    m\in\mathbb Z^{d_x},
\]
the covariance eigenvalues are
\[
    \lambda_m
    =
    \left(
        1+4\pi^2\ell^2|m|^2
    \right)^{-s}.
\]
Consequently,
\begin{equation}
\begin{aligned}
    \left\|
        \mathbf Z^+
        -
        T_\theta(\mathbf Z)
    \right\|_{\mathcal H_K}^2
    &=
    \sum_{m\in\mathbb Z^{d_x}}
    \sum_{a=1}^{d_z}
    \left(
        1+4\pi^2\ell^2|m|^2
    \right)^s
    \left|
        \widehat Z_a^+(m)
        -
        \widehat{T_\theta(\mathbf Z)}_a(m)
    \right|^2 .
\end{aligned}
\label{eq:sobolev_latent_residual}
\end{equation}
Thus the Cameron-Martin space $\mathcal H_K$ is $H^s(\mathbb T^{d_x};\mathbb R^{d_z})$ as a set, with an equivalent length-scale-weighted Sobolev norm. In particular, higher-frequency latent residuals receive increasingly large weight as the covariance spectrum decays.
\end{corollary}

\subsection{Error Propagation in Variational Latent Dynamics}
\label{sec:err_var_dyn}
We now study how the variational training objective affects the stability of long-horizon latent rollout. We first make explicit the one-step objective analyzed below. Let $\mu_n$ denote the joint distribution of consecutive physical states $(\mathbf U_n,\mathbf U_{n+1})$, and consider the Gaussian encoder and transition
\[
    \widetilde{\mathbf Z}_n=m_\phi(\mathbf U_n)+\Xi_n,
    \qquad
    \Xi_n \sim \mathsf N(0,K_\phi(\mathbf U_n)),
\]
and
\[
    \widetilde{\mathbf Z}_n^+=T_\theta(\widetilde{\mathbf Z}_n)+\alpha_\theta(\widetilde{\mathbf Z}_n)\mathbf G_n,
    \qquad
    \mathbf G_n\sim\mathsf N(0,K).
\]
For notational convenience, define the decoder negative log-likelihood
\[
    \ell_\psi(\mathbf U;\mathbf Z)
    := -\log\frac{\d P_\psi(\cdot\mid\mathbf Z)}{\d\mathsf W_{\mathcal U}}(\mathbf U).
\]
For a consecutive pair $(\mathbf U_n,\mathbf U_{n+1})$, define
\begin{align*}
    \mathcal L_{\mathrm{pred}}(\mathbf U_n,\mathbf U_{n+1})
    &:=\mathbb E_{\Xi_n,\mathbf G_n}
    \left[\ell_\psi(\mathbf U_{n+1};\widetilde{\mathbf Z}_n^+)\right],
    \\
    \mathcal L_{\mathrm{KL}}(\mathbf U_n,\mathbf U_{n+1})
    &:=\mathbb E_{\widetilde{\mathbf Z}_n\sim
        Q_\phi(\cdot\mid\mathbf U_n)}
    \left[D_{\mathrm{KL}}\left(
        Q_\phi(\cdot\mid\mathbf U_{n+1})\,\Vert\,P_\theta(\cdot\mid\widetilde{\mathbf Z}_n)
    \right)\right],
    \\
    \mathcal L_{\mathrm{rec}}(\mathbf U_n)
    &:=\mathbb E_{\widetilde{\mathbf Z}_n\sim Q_\phi(\cdot\mid\mathbf U_n)}
    \left[
        \ell_\psi(\mathbf U_n;\widetilde{\mathbf Z}_n)
    \right].
\end{align*}
The population training objective takes the form
\begin{equation}
    \mathcal L^{(n)}(\phi,\theta,\psi)
    =\mathbb E_{(\mathbf U_n,\mathbf U_{n+1})\sim\mu_n}
    \left[
        \mathcal L_{\mathrm{pred}}(\mathbf U_n,\mathbf U_{n+1})
        +\beta\mathcal L_{\mathrm{KL}}(\mathbf U_n,\mathbf U_{n+1})
        +\lambda_{\mathrm{rec}}\mathcal L_{\mathrm{rec}}(\mathbf U_{n+1})
    \right],
    \label{eq:functional_training_objective}
\end{equation}
with the empirical objective obtained by averaging over observed one-step pairs.

By Proposition~\ref{prop:cm_decoder_likelihood}, whenever the corresponding physical state belongs to $\mathcal H_{\mathcal U}$,
\[
    \ell_\psi(\mathbf U;\mathbf Z)
    =\frac12
    \|\mathbf U-D_\psi(\mathbf Z)\|_{\mathcal H_{\mathcal U}}^2
    + c(\mathbf U),
\]
where $c(\mathbf U)$ is independent of the model parameters. Hence the decoder likelihood terms are equivalent, for optimization purposes, to squared Cameron-Martin reconstruction errors, up to additive data-dependent constants.

The objective \eqref{eq:functional_training_objective} can be viewed as a weighted combination of the two variational bounds introduced in \S\ref{sec:abstract_vld}. The predictive term corresponds to the negative decoder-likelihood bound in
Proposition~\ref{prop:conditional_functional_elbo}(a), while the reconstruction and KL terms correspond to the two components of the negative conditional ELBO in
Proposition~\ref{prop:conditional_functional_elbo}(b). The coefficients $\beta$ and $\lambda_{\mathrm{rec}}$ allow for relative weighting of these terms, including the scaling induced by the fixed decoder covariance.

For convenience, our subsequent rollout analysis is stated instead in the ambient physical-state norm $\|\cdot\|_{\mathcal U}$. The two geometries are compatible because the Cameron-Martin space is continuously embedded in $\mathcal U$:
\[
    \mathcal H_{\mathcal U}\hookrightarrow\mathcal U,
    \qquad
    \|\mathbf v\|_{\mathcal U}
    \leq
    C_{\mathcal U}
    \|\mathbf v\|_{\mathcal H_{\mathcal U}},
    \quad
    \mathbf v\in\mathcal H_{\mathcal U},
\]
for some $C_{\mathcal U}>0$. Thus, whenever the relevant residual lies in $\mathcal H_{\mathcal U}$, likelihood control also yields control in the ambient norm used for rollout errors. After finite-dimensional discretization, the isotropic decoder covariance used in our implementation reduces these likelihood terms to the squared Euclidean losses described in \S\ref{subsec:vamo_training}.

\paragraph{Deterministic rollout.} At inference time, we deploy the deterministic mean dynamics. Given an initial physical state $\mathbf U_0$, the rollout is initialized and propagated according to
\begin{equation}
    \widehat{\mathbf Z}_0=m_\phi(\mathbf U_0),
    \quad
    \widehat{\mathbf Z}_n=T_\theta(\widehat{\mathbf Z}_{n-1}),
    \quad
    \widehat{\mathbf U}_n=D_\psi(\widehat{\mathbf Z}_n),\qquad n=1,2,\cdots,T/\Delta t.
    \label{eq:deterministic_mean_rollout}
\end{equation}
Hence the stochasticity in \eqref{eq:functional_training_objective} is used during training, whereas the deployed solver follows the deterministic mean latent dynamics. Our analysis below explains how this stochastic training objective can nevertheless control the error accumulated by the deterministic rollout.

We identify two complementary mechanisms. First, sampling from the encoder and transition distributions replaces pointwise fitting by training over Gaussian neighborhoods in latent space, providing a noise-injection regularization effect. Second, the variational KL term directly controls the mismatch between the learned transition and the encoded distribution of the true next state. We study these two mechanisms separately below.

\subsubsection{Noise Injection and Local Stability}
\label{subsubsec:noise_injection}
For the analysis below, let $(\mathbf U_0^\star,\ldots,\mathbf U_N^\star)\sim \mu$ be a reference trajectory, and define its encoder-mean latent representation by
\[
    \mathbf Z_n^\star:=m_\phi(\mathbf U_n^\star).
\]
Specializing the stochastic pathway above to this trajectory, let
\[
    \widetilde{\mathbf Z}_n
    =\mathbf Z_n^\star+\Xi_n,
    \qquad
    \widetilde{\mathbf Z}_n^+
    =T_\theta(\widetilde{\mathbf Z}_n)
      +\alpha_\theta(\widetilde{\mathbf Z}_n)G_n,
\]
where
\[
    \Xi_n\sim\mathsf N(0,K_\phi(\mathbf U_n^\star)),
    \qquad
    G_n\sim\mathsf N(0,K),
\]
are conditionally independent given the reference trajectory.

We first quantify the local sensitivity probed by stochastic sampling in the predictive pathway. Noise injection during training has long been associated
with regularization of the learned mapping \citep{bishop1995training}, and has also been used in learned physical simulators to improve robustness under autoregressive rollout \citep{sanchez2020learning}. In our setting, the stochastic objective evaluates the transition and decoder over Gaussian neighborhoods of the reference latent trajectory. To characterize this effect, define the \textit{noisy one-step prediction error}
\begin{equation}
    \epsilon_n
    :=
    \left(
        \mathbb E
        \left[
            \left\|
                D_\psi(\widetilde{\mathbf Z}_n^+)
                -
                \mathbf U_{n+1}^\star
            \right\|_{\mathcal U}^2
        \right]
    \right)^{1/2},
    \label{eq:noisy_prediction_error}
\end{equation}
the \textit{encoder-noise sensitivity} of the latent transition
\begin{equation}
    \eta_n
    :=
    \left(
        \mathbb E
        \left[
            \left\|
                T_\theta(\widetilde{\mathbf Z}_n)
                -
                T_\theta(\mathbf Z_n^\star)
            \right\|_{\mathcal Z}^2
        \right]
    \right)^{1/2},
    \label{eq:encoder_noise_sensitivity}
\end{equation}
and the \textit{transition-noise sensitivity} of the decoder
\begin{equation}
    \rho_n
    :=
    \left(
        \mathbb E
        \left[
            \left\|
                D_\psi(\widetilde{\mathbf Z}_n^+)
                -
                D_\psi(T_\theta(\widetilde{\mathbf Z}_n))
            \right\|_{\mathcal U}^2
        \right]
    \right)^{1/2}.
    \label{eq:transition_noise_robustness}
\end{equation}
Here the expectations are taken jointly over the reference trajectory and the corresponding Gaussian perturbations. The quantity $\eta_n$ directly measures the sensitivity of the latent transition to perturbations of its encoded input, whereas $\rho_n$ measures the sensitivity of the decoder to perturbations around the predicted latent state.

The following result makes this local-stability interpretation explicit by bounding the two sensitivity terms in terms of the noise magnitude and the regularity of the learned maps.

\begin{proposition}[Gaussian sensitivity bounds]
\label{prop:gaussian_sensitivity_bounds}
Let the setting and notation of \S\ref{subsubsec:noise_injection} hold.
Assume that
\[
    \operatorname{Tr}\bigl(K_\phi(\mathbf U)\bigr)
    \leq \varsigma,
    \qquad
    \alpha_\theta(\mathbf Z)\leq \bar\alpha
\]
on the relevant regions, and suppose that $T_\theta:\mathcal Z\to\mathcal Z$
and $D_\psi:\mathcal Z\to\mathcal U$ are Lipschitz with constants $\Lambda_T$ and $\Lambda_D$, respectively. Then
\begin{equation}
    \eta_n
    \leq
    \Lambda_T\sqrt{\varsigma},
    \qquad
    \rho_n
    \leq
    \Lambda_D\bar\alpha
    \sqrt{\operatorname{Tr}(K)}.
    \label{eq:lipschitz_sensitivity_bounds}
\end{equation}
Suppose, in addition, that $T_\theta$ is Fr\'echet differentiable on the relevant encoder-mean neighborhoods and satisfies the uniform second-order remainder bound
\[
    \left\|
        T_\theta(\mathbf z+\mathbf h)
        -
        T_\theta(\mathbf z)
        -
        J_{T_\theta}(\mathbf z)\mathbf h
    \right\|_{\mathcal Z}
    \leq
    \frac{M_T}{2}\|\mathbf h\|_{\mathcal Z}^2.
\]
Then
\begin{equation}
\begin{aligned}
    \eta_n
    \leq\left(\bbE\left[\Vert J_{T_\theta}(\mathbf Z^\star_n)\Vert_\mathrm{op}^2\right]\varsigma\right)^{1/2}+\frac{\sqrt{3}M_T}{2}\varsigma.
    \label{eq:tau_jacobian_bound}
\end{aligned}
\end{equation}
Likewise, suppose that $D_\psi$ is Fr\'echet differentiable on the relevant transition-noise neighborhoods and satisfies
\[
    \left\|
        D_\psi(\mathbf z+\mathbf h)
        -
        D_\psi(\mathbf z)
        -
        J_{D_\psi}(\mathbf z)\mathbf h
    \right\|_{\mathcal U}
    \leq
    \frac{M_D}{2}\|\mathbf h\|_{\mathcal Z}^2.
\]
Then
\begin{equation}
\rho_n\leq\ol\alpha\left(\bbE\left[\Vert J_{D_\psi}(T_\theta(\wt{\mathbf Z}_n))\Vert_\mathrm{op}^2\right]\right)^{1/2}\sqrt{\operatorname{Tr} K}+\frac{\sqrt{3}M_D\ol\alpha^2}{2}\operatorname{Tr} K.
    \label{eq:rho_jacobian_banach}
\end{equation}
\end{proposition}

Proposition~\ref{prop:gaussian_sensitivity_bounds} gives two complementary views of the sensitivity induced by Gaussian perturbations. The Lipschitz bounds in \eqref{eq:lipschitz_sensitivity_bounds} provide global control in terms of the overall noise magnitude, whereas the local estimates \eqref{eq:tau_jacobian_bound}-\eqref{eq:rho_jacobian_banach} show that, for small perturbations, the sensitivities are governed primarily by the local Fr\'echet derivatives of the transition and decoder. In particular, $\eta_n$ measures the response of $T_\theta$ to encoder perturbations, while $\rho_n$ measures the response of $D_\psi$ to perturbations around the predicted latent state. The higher-order terms vanish with the corresponding noise scales.

We can interpret how stochastic training acts on these sensitivity terms from an optimization perspective. The predictive objective evaluates
\[
    D_\psi\left(
        T_\theta(\mathbf Z_n^\star+\Xi_n)
        +
        \alpha_\theta(\widetilde{\mathbf Z}_n)\mathbf G_n
    \right)
\]
over Gaussian neighborhoods of the nominal latent trajectory rather than only at their centers. Large local variations of $T_\theta$ in the encoder-noise directions or of $D_\psi$ in the transition-noise directions can therefore increase the noisy prediction error. Through the reparameterized samples, back-propagation updates the model using errors evaluated throughout these neighborhoods. Thus, although the sensitivity terms are not explicitly penalized, noise-injection training implicitly encourages local stability in the regions and directions explored by the Gaussian perturbations.

\subsubsection{Variational Alignment and Rollout Control}\label{subsubsec:variational_kl_alignment}
The previous subsection studies how stochastic sampling regularizes the local sensitivity of the learned transition and decoder. We now turn to the variational KL term. At the distributional level, Proposition~\ref{prop:population_elbo_transition_consistency} shows that this term aligns the learned latent transition with the encoder-induced dynamics of true PDE trajectories. Under the Gaussian parameterization, this probabilistic alignment further yields quantitative control of the corresponding latent means. The following lemma makes this connection precise by bounding the discrepancy between the transition mean and the mean of a target latent distribution in terms of their KL divergence.

\begin{lemma}[KL control of latent means]
\label{lem:kl_latent_mean_control}
Let
\[
P_\theta(\cdot\mid\mathbf{Z})=\mathsf N\left(T_\theta(\mathbf{Z}),\alpha_\theta(\mathbf{Z})^2K\right),
\]
where $K$ is positive, self-adjoint, and trace-class on $\mathcal Z$. Let $Q$ be any probability measure on $\mathcal Z$ with finite second moment and mean
\[
    m_Q:=\int_{\mathcal Z} y\,\d Q(y).
\]
Assume that
\[
D_{\mathrm{KL}}\left(Q\,\Vert\,P_\theta(\cdot\,|\,\mathbf{Z})\right)<\infty .
\]
Then
\[
\left\|m_Q-T_\theta(\mathbf{Z})\right\|_{\mathcal Z}
\le\alpha_\theta(\mathbf{Z})\sqrt{
    2\|K\|_{\mathrm{op}} D_{\mathrm{KL}}\left(Q\,\Vert\,P_\theta(\cdot\,|\,\mathbf{Z})\right)
}.
\]
\end{lemma}
Lemma~\ref{lem:kl_latent_mean_control} converts the dynamic KL divergence into quantitative control of the discrepancy between latent means. In particular, when the transition distribution is Gaussian, a small KL divergence between the encoded next-state distribution and the predicted transition distribution implies that their mean states must also be close. This is the key step that connects the probabilistic alignment enforced by the variational objective to the deterministic mean dynamics used at inference. The proof of Lemma~\ref{lem:kl_latent_mean_control} relies on a Gaussian transportation inequality \citep{riedel2017transportation} and is deferred to Appendix~\S\ref{app:kl-aware-rollout}.

We now introduce the remaining population quantities needed for the rollout analysis. Motivated by Lemma~\ref{lem:kl_latent_mean_control}, for the random reference trajectory, define the dynamic KL error
\[
    \kappa_n:=\left(\mathbb E
        \left[D_{\mathrm{KL}}
            \left(Q_\phi(\cdot\mid\mathbf U_{n+1}^\star)\,\Vert\, P_\theta(\cdot\mid\widetilde{\mathbf Z}_n)
            \right)\right]
    \right)^{1/2}.
\]
Thus, $\kappa_n$ measures, at the population level, how well the learned transition distribution from the perturbed current latent state matches the encoder distribution of the true next physical state. It is the direct population counterpart of the dynamic KL term appearing in the training objective.

In addition, we do not assume exact reconstruction of the physical state from its encoder mean, and therefore define
\[
    \delta_n:=\left(\mathbb E
        \left[\|D_\psi(\mathbf Z_n^\star)-\mathbf U_n^\star\|_{\mathcal U}^2\right]
    \right)^{1/2}.
\]
The quantity $\delta_n$ measures the discrepancy introduced when the encoder-mean latent state is decoded back to the physical space. Here the expectations are taken over the reference trajectory and, where applicable, the encoder perturbation.

Together with the sensitivity quantities studied in \S\ref{subsubsec:noise_injection}, we have the necessary ingredients for analyzing deterministic autoregressive rollout. The following proposition characterizes how the corresponding one-step discrepancies propagate and accumulate over the rollout horizon.

\begin{proposition}[Population KL-aware rollout control with predictive refinement]\label{prop:population_kl_rollout}
Let the setting and notation of \S\ref{subsubsec:noise_injection} hold. Assume that
\[
    0<\alpha_\theta(\mathbf Z)\leq\bar\alpha
\]
on the relevant latent region, and that $T_\theta$ and $D_\psi$ are Lipschitz with constants $\Lambda_T$ and $\Lambda_D$, respectively. Then the following hold.

\begin{enumerate}[label=(\alph*),topsep=0pt,itemsep=3pt]

\item
The population latent one-step defect satisfies
\begin{equation}
    \left(
        \mathbb E
        \left[
            \left\|
                T_\theta(\mathbf Z_n^\star)
                -
                \mathbf Z_{n+1}^\star
            \right\|_{\mathcal Z}^2
        \right]
    \right)^{1/2}
    \leq
    \eta_n+\bar\alpha\sqrt{2\|K\|_{\mathrm{op}}}\,\kappa_n.
    \label{eq:population_latent_one_step}
\end{equation}

\item
For the deterministic mean rollout
\eqref{eq:deterministic_mean_rollout},
\begin{equation}
    \left(
        \mathbb E
        \left[
            \|\widehat{\mathbf U}_n-\mathbf U_n^\star\|_{\mathcal U}^2
        \right]
    \right)^{1/2}
    \leq\Lambda_D
    \sum_{j=0}^{n-1}\Lambda_T^{\,n-1-j}\left(\eta_j+\bar\alpha\sqrt{2\|K\|_{\mathrm{op}}}\,\kappa_j
    \right)
    +
    \delta_n.
    \label{eq:population_kl_physical_rollout}
\end{equation}

\item Incorporating the noisy predictive error gives the refined bound
\begin{equation}
\begin{aligned}
\left(\mathbb E\left[\|\widehat{\mathbf U}_n-\mathbf U_n^\star\|_{\mathcal U}^2\right]\right)^{1/2}
&\leq\Lambda_D
    \sum_{j=0}^{n-2}
    \Lambda_T^{\,n-1-j}\left(\eta_j+\bar\alpha\sqrt{2\|K\|_{\mathrm{op}}}\,\kappa_j\right)+\Lambda_D\eta_{n-1}
    \\
    &\qquad
    +\min\left\{ \Lambda_D\bar\alpha\sqrt{2\|K\|_{\mathrm{op}}}\,\kappa_{n-1}+\delta_n, \epsilon_{n-1}+\rho_{n-1}\right\},
\end{aligned}
\label{eq:population_refined_rollout}
\end{equation}
for every $n\geq 1$, where the sum is understood as zero when $n=1$.
\end{enumerate}
\end{proposition}

\begin{remark}[Interpretation of the rollout bound]
\label{rem:population_rollout_interpretation}
The terms appearing in Proposition~\ref{prop:population_kl_rollout} are controlled by different components of the training objective. The dynamic KL error $\kappa_n$ is the population counterpart of the latent consistency term $\mathcal L_{\mathrm{KL}}$, while the reconstruction error $\delta_n$ is associated with the encoder-decoder reconstruction objective. The noisy prediction error $\epsilon_n$ is controlled through $\mathcal L_{\mathrm{pred}}$, and the sensitivity quantities $\eta_n$ and $\rho_n$ are studied in Proposition~\ref{prop:gaussian_sensitivity_bounds}. As discussed in \S\ref{subsubsec:noise_injection}, these sensitivity terms are not explicitly optimized, but stochastic training evaluates the model over Gaussian neighborhoods and thereby implicitly encourages local stability in the corresponding perturbation directions.

Part~(b) makes explicit how the two terms in the conditional ELBO \eqref{eq:condnegelbo} enter deterministic rollout. The static reconstruction term is reflected by the reconstruction error $\delta_n$, which measures the discrepancy between the true physical state and the decoder applied to its encoder-mean representation. The dynamic KL term is reflected by $\kappa_n$, which controls the mismatch between the predicted transition mean and the encoder mean of the true next state. The additional sensitivity term $\eta_n$ accounts for the fact that the KL objective is evaluated at a perturbed encoder input, whereas deterministic rollout propagates the encoder mean. These one-step latent discrepancies are accumulated through the mean transition with geometric factors $\Lambda_T^{n-1-j}$, while $\delta_n$ enters when the propagated latent state is decoded back to the physical space. Thus, the two principal terms of the conditional ELBO reappear naturally in the rollout error bound, linking variational training to deterministic long-horizon accuracy.

Part~(c) additionally incorporates the predictive objective into the same deterministic rollout analysis. At the final prediction step, the physical error admits two complementary upper bounds: one obtained from the KL-controlled latent mean discrepancy together with reconstruction,
\[
    \Lambda_D\bar\alpha\sqrt{2\|K\|_{\mathrm{op}}}\,
    \kappa_{n-1}+\delta_n,
\]
and the other from the noisy predictive error together with transition-noise sensitivity,
\[
    \epsilon_{n-1}+\rho_{n-1}.
\]
Taking the minimum simply selects the tighter analytical estimate for the same deterministic prediction and therefore sharpens the physical-space rollout bound. In particular, this refinement reflects the complementary roles of the variational and predictive components of the training objective.

Finally, we discuss the error amplification in rollouts and characterize the growth rate of the rollout error bound more explicitly. Suppose that the one-step error and sensitivity terms remain uniformly bounded over the rollout horizon:
\[
\eta_j \leq \bar{\eta}, \qquad
\kappa_j \leq \bar{\kappa}, \qquad
\delta_j \leq \bar{\delta}.
\]
Then Proposition~\ref{prop:population_kl_rollout} (b) implies
\[
\left(
\mathbb{E}
\|\widehat{\mathbf{U}}_n - \mathbf{U}_n^\star\|_{\mathcal{U}}^2
\right)^{1/2}
\leq
\begin{cases}
\displaystyle
\Lambda_D
\left(
\bar{\eta}+\bar{\alpha}\sqrt{2\|K\|_{\mathrm{op}}}\,\bar{\kappa}
\right)
\frac{1-\Lambda_T^n}{1-\Lambda_T}
+\bar{\delta},
& \Lambda_T < 1,\\[1.2em]
\displaystyle
n\,\Lambda_D
\left(
\bar{\eta}
+
\bar{\alpha}\sqrt{2\|K\|_{\mathrm{op}}}\,\bar{\kappa}
\right)
+\bar{\delta},
& \Lambda_T = 1,\\[1.2em]
\displaystyle
\Lambda_D
\left(
\bar{\eta}+\bar{\alpha}\sqrt{2\|K\|_{\mathrm{op}}}\,\bar{\kappa}
\right)
\frac{\Lambda_T^n - 1}{\Lambda_T - 1}
+\bar{\delta},
& \Lambda_T > 1.
\end{cases}
\]
Therefore, under uniformly controlled error and sensitivity terms, the bound remains uniformly bounded for $\Lambda_T < 1$, grows at most linearly for $\Lambda_T = 1$, and permits geometric growth for $\Lambda_T > 1$. The result thus separates two ingredients of long-horizon accuracy: the magnitude of the local one-step discrepancies and their amplification through the learned latent transition.
\end{remark}

While our main analysis applies deterministic mean-map rollout at inference, Appendix~\S\ref{app:stochastic_envelope} gives a complementary high-probability characterization of stochastic rollouts, showing that their latent deviations from the mean-map trajectory remain within a two-sided Gaussian envelope.

\paragraph{Generic autoregressive amplification.} To distinguish the generic effect of autoregressive error accumulation from the mechanisms specific to variational latent dynamics, we consider a generic deterministic predictor $F_\vartheta:\mathcal U\to\mathcal U$, which is deployed recursively as
\[
\widehat{\mathbf U}_0=\mathbf U_0^\star,
\qquad
\widehat{\mathbf U}_{n+1}
=
F_\vartheta(\widehat{\mathbf U}_n).
\]
Define its population one-step prediction error by
\[
\epsilon_n^{\mathrm{dir}}
:=
\left(\mathbb E
\left[\|F_\vartheta(\mathbf U_n^\star)-\mathbf U_{n+1}^\star\|_{\mathcal U}^2\right]\right)^{1/2}.
\]
The following result isolates how these one-step errors accumulate solely through repeated application of the learned predictor.

\begin{proposition}[Generic autoregressive error amplification]
\label{prop:generic_rollout}
Suppose that $F_\vartheta$ is Lipschitz on the relevant region with constant
$\Lambda_F$. Then
\begin{equation}
\left(
\mathbb E
\left[\|\widehat{\mathbf U}_n-\mathbf U_n^\star\|_{\mathcal U}^2
\right]\right)^{1/2}
\leq\sum_{j=0}^{n-1}\Lambda_F^{\,n-1-j}\epsilon_j^{\mathrm{dir}}.
\label{eq:generic_rollout}
\end{equation}
\end{proposition}

Proposition~\ref{prop:generic_rollout} applies to deterministic autoregressive predictors such as the direct neural-operator baseline considered in our subsequent experiments. In particular, \eqref{eq:generic_rollout} shows that geometric error amplification is a generic feature of autoregressive prediction rather than a consequence of the variational formulation. Although the variational latent model evolves autoregressively in latent rather than physical space, Proposition~\ref{prop:population_kl_rollout} exhibits the same amplification mechanism through powers of $\Lambda_T$. Thus, the preceding analysis characterizes how the variational training  controls the one-step quantities that are subsequently amplified during latent rollout.

The sensitivity analysis in \S\ref{subsubsec:noise_injection} suggests a complementary role for stochastic training in the autoregressive amplification mechanism. Both the generic deterministic bound and the variational rollout bound contain geometric amplification factors like $\Lambda_F$ and $\Lambda_T$, governed respectively by the stability of the learned one-step maps. In the variational latent formulation, encoder perturbations expose the transition map to neighborhoods of the encoded trajectory, while transition perturbations probe the decoder around predicted latent states. Proposition~\ref{prop:gaussian_sensitivity_bounds} shows that these sensitivities are locally governed by the Fréchet derivatives of the corresponding maps. Sharp local fluctuations in rollout-relevant neighborhoods can therefore increase the noisy predictive loss and are implicitly discouraged through reparameterized back-propagation. In this way, variational training can regularize the effective local expansion encountered along the rollout trajectory and thereby mitigate geometric error amplification over long horizons.

\section{Variational Autoencoding Markov Operator}
\label{sec:vamo}

The variational latent dynamics developed in
\S\ref{sec:var_latent_dyn} are formulated directly at the function-space level and are therefore independent of a particular spatial discretization or neural architecture. In this section, we introduce the \emph{Variational Autoencoding Markov Operator} (VAMO), a discretized neural-operator realization of this framework for time-dependent PDEs. VAMO realizes the encoder-transition-decoder structure of \S\ref{sec:abstract_vld} using spatially resolved physical and latent fields. Figure~\ref{fig:vamo_architecture} summarizes the variational training procedure and the corresponding mean latent rollout. We first construct its deterministic backbone from a resolution-preserving encoder and decoder together with a residual neural-operator transition in \S\ref{subsec:vamo_architecture}. We then augment this backbone with structured Gaussian latent perturbations in \S\ref{subsec:structured_latent_perturbations} and formulate the resulting training objective in \S\ref{subsec:vamo_training}.

\subsection{Spatially Resolved Architecture}
\label{subsec:vamo_architecture}
To obtain a computational realization of the function-space framework, we first discretize the physical and latent fields on a finite spatial grid and specify the deterministic mean maps underlying VAMO. This subsection focuses on the encoder mean, latent mean transition, and decoder mean that form the deterministic backbone of the model; the Gaussian covariance structure and stochastic sampling used for variational training are introduced separately in \S\ref{subsec:structured_latent_perturbations}.

\begin{figure}[H]
  \centering
  \resizebox{\linewidth}{!}{
\begin{tikzpicture}[
  x=1cm,y=1cm,font=\fontsize{10}{11}\selectfont,>=Latex,
  physical/.style={circle,draw=black!65,fill=black!3,line width=.8pt,
    minimum size=1.03cm,inner sep=2pt},
  decoded/.style={physical,draw=teal!65!black,fill=teal!5},
  latent/.style={rounded corners=5pt,draw=blue!55!black,fill=blue!4,
    line width=.8pt,minimum height=1.03cm,minimum width=2.7cm,
    inner sep=6pt,align=center},
  flow/.style={->,line width=.85pt,draw=black!75},
  enc/.style={flow,draw=blue!65!black},
  dec/.style={flow,draw=teal!70!black},
  loss/.style={draw=orange!70!black,dashed,line width=.85pt},
  noise/.style={->,draw=violet!75!black,line width=.7pt},
  labeltext/.style={font=\fontsize{9}{10.5}\selectfont,align=center,fill=white,inner sep=3pt},
  rowlabel/.style={anchor=east,font=\fontsize{9}{10.5}\selectfont\sffamily,text=black!60,align=right},
  title/.style={anchor=west,font=\bfseries\sffamily},
  note/.style={font=\fontsize{10}{10.5}\selectfont,align=center,text=black!70}
]
\node[title] at (-.3,8.0) {(a) Variational training};
\node[anchor=east,font=\fontsize{10}{10.5}\selectfont\sffamily,text=black!60] at (16.5,8.0)
  {Adjacent reference states};
\node[rowlabel] at (.5,6.0) {Decoded\\physical states};
\node[rowlabel] at (.5,3.0) {Latent\\representations};
\node[rowlabel] at (.5,0) {Reference\\physical states};
\node[physical] (u) at (3.0,0) {$\mathbf u_n$};
\node[physical] (unext) at (13.4,0) {$\mathbf u_{n+1}$};
\draw[flow,draw=black!35] (u) -- node[labeltext,text=black!55] {Reference evolution} (unext);
\node[latent] (q) at (3.0,3.0)
  {$q_n=Q_\phi(\cdot\mid\mathbf u_n)$\\[4pt]$\mathbf z_n\sim q_n$};
\node[latent,minimum width=3.15cm] (p) at (8.2,3.0)
  {$p_n=P_\theta(\cdot\mid\mathbf z_n)$\\[4pt]$\mathbf z_n^+\sim p_n$};
\node[latent,minimum width=3.15cm] (qn) at (13.4,3.0)
  {$q_{n+1}=$\\[4pt]$Q_\phi(\cdot\mid\mathbf u_{n+1})$};
\draw[enc] (u) -- node[labeltext,right=4pt] {Encoder\\$Q_\phi$} (q);
\draw[enc] (unext) -- node[labeltext,right=4pt] {Shared encoder\\$Q_\phi$} (qn);
\draw[flow] (q.east) -- node[labeltext,above=6pt] {Transition\\$P_\theta$} (p.west);
\node[note] at (5.6,2.10) {$T_\theta=I+\mathcal G_\theta^{\mathrm{FNO}}$};
\draw[loss] ($(p.east)+(0,.18)$) --
  node[labeltext,above=5pt,text=orange!70!black] {$\mathcal L_{\mathrm{KL}}$}
  node[labeltext,below=5pt,text=orange!70!black,font=\fontsize{7}{8.5}\selectfont,inner sep=1pt]
    {$D_{\mathrm{KL}}(q_{n+1}\Vert p_n)$}
  ($(qn.west)+(0,.18)$);
\node[decoded] (rec) at (3.0,6.0) {$\widetilde{\mathbf u}_n$};
\node[decoded] (pred) at (8.2,6.0) {$\widehat{\mathbf u}_{n+1}$};
\draw[dec] (q.north) -- node[labeltext,left=3pt] {$D_\psi$} (rec.south);
\draw[dec] (p.north) -- node[labeltext,left=3pt] {$D_\psi$} (pred.south);
\node[note,above=5pt] at (rec.north) {Reconstruction};
\node[note,above=5pt] at (pred.north) {Prediction};
\node[labeltext,text=violet!75!black] (ne) at (5.15,4.85)
  {$S_\phi(\mathbf u_n)K_{\mathrm{enc}}^{1/2}\boldsymbol\xi$\\[-1pt]
   {\fontsize{8}{9.5}\selectfont Pointwise amplitude}};
\draw[noise] (ne.south west) -- (q.north east);
\node[labeltext,text=violet!75!black] (nt) at (10.5,4.65)
  {$\alpha_\theta(\mathbf z_n)K_{\mathrm{tr}}^{1/2}\boldsymbol\xi^{\mathrm{tr}}$\\[-1pt]
   {\fontsize{8}{9.5}\selectfont Scalar amplitude}};
\draw[noise] (nt.south west) -- (p.north east);
\draw[loss] (rec.west) -- (1.15,6.0) --
  node[labeltext,rotate=90,text=orange!70!black] {$\mathcal L_{\mathrm{rec}}$}
  (1.15,0) -- (u.west);
\draw[loss] (pred.east) -- (16.05,6.0) --
  node[labeltext,rotate=90,text=orange!70!black] {$\mathcal L_{\mathrm{pred}}$}
  (16.05,0) -- (unext.east);
\node[font=\fontsize{10}{11}\selectfont] at (8.1,-.95)
  {$\mathcal L=\mathcal L_{\mathrm{pred}}+\beta\mathcal L_{\mathrm{KL}}
      +\lambda_{\mathrm{rec}}\mathcal L_{\mathrm{rec}}$};
\draw[black!15,line width=.7pt] (-.3,-1.55) -- (16.5,-1.55);
\node[title] at (-.3,-2.20) {(b) Mean latent rollout};
\node[anchor=east,font=\fontsize{10}{10}\selectfont\sffamily,text=black!60] at (16.5,-2.20)
  {Encode once; noise disabled};
\node[rowlabel] at (.5,-3.8) {Predicted\\physical states};
\node[rowlabel] at (.5,-6.3) {Initial state /\\latent evolution};
\node[physical] (ic) at (1.8,-6.3) {$\mathbf u_0$};
\node[latent,minimum width=1.45cm] (z0) at (4.5,-6.3) {$\widehat{\mathbf z}_0$};
\node[latent,minimum width=1.45cm] (z1) at (7.3,-6.3) {$\widehat{\mathbf z}_1$};
\node[latent,minimum width=1.45cm] (z2) at (10.1,-6.3) {$\widehat{\mathbf z}_2$};
\node[latent,minimum width=1.45cm] (zN) at (15.1,-6.3) {$\widehat{\mathbf z}_N$};
\node (dots) at (12.6,-6.3) {$\cdots$};
\draw[enc] (ic) -- node[labeltext,above=4pt] {$m_\phi$} (z0);
\draw[flow] (z0) -- node[labeltext,above=4pt] {$T_\theta$} (z1);
\draw[flow] (z1) -- node[labeltext,above=4pt] {$T_\theta$} (z2);
\draw[flow] (z2) -- node[labeltext,above=4pt] {$T_\theta$} (dots);
\draw[flow] (dots) -- node[labeltext,above=4pt] {$T_\theta$} (zN);
\node[decoded] (outN) at (15.1,-3.8) {$\widehat{\mathbf u}_N$};
\node[decoded] (out1) at (7.3,-3.8) {$\widehat{\mathbf u}_1$};
\node[decoded] (out2) at (10.1,-3.8) {$\widehat{\mathbf u}_2$};
\node at (12.6,-3.8) {$\cdots$};
\draw[dec] (zN) -- node[labeltext,right=3pt] {$D_\psi$} (outN);
\draw[dec] (z1) -- node[labeltext,right=3pt] {$D_\psi$} (out1);
\draw[dec] (z2) -- node[labeltext,right=3pt] {$D_\psi$} (out2);
\node[note] at (4.5,-7.25) {$\widehat{\mathbf z}_0=m_\phi(\mathbf u_0)$};
\node[note] at (11.65,-7.25)
  {$\widehat{\mathbf z}_{n+1}=T_\theta(\widehat{\mathbf z}_n),\quad
    \widehat{\mathbf u}_{n+1}=D_\psi(\widehat{\mathbf z}_{n+1})$};
\end{tikzpicture}}
  \caption{Architecture and deployment of VAMO.
  (a) Variational training on an adjacent reference pair $(\mathbf u_n,\mathbf u_{n+1})$. The encoder samples $\mathbf z_n$ from $q_n=Q_\phi(\cdot\mid\mathbf u_n)$,
  and the latent transition samples $\mathbf z_n^+$ from $p_n=P_\theta(\cdot\mid\mathbf z_n)$. A shared decoder produces the reconstruction $\widetilde{\mathbf u}_n=D_\psi(\mathbf z_n)$ and prediction $\widehat{\mathbf u}_{n+1}=D_\psi(\mathbf z_n^+)$.
  Structured Gaussian perturbations act at both latent sampling stages; the KL term aligns the transition distribution with the encoded next-state distribution $q_{n+1}=Q_\phi(\cdot\mid\mathbf u_{n+1})$. Dashed lines indicate loss comparisons, with expectations as defined in \S\ref{subsec:vamo_training}.
  (b) Mean latent rollout encodes the initial condition once, repeatedly applies $T_\theta$, and decodes the evolved latent states through step $N$ without sampling or re-encoding physical predictions. Circles denote physical states; rounded rectangles denote spatially resolved latent representations and their distributions.}
  \label{fig:vamo_architecture}
\end{figure}

Let $\{x_1,x_2,\cdots,x_{N_\mathrm{res}}\}\subset\Omega$ denote a spatial grid on the domain $\Omega$ with $N_{\mathrm{res}}$ discretization points. Evaluating a physical field with $C_u$ channels and its latent counterpart with $C_z$ channels on this grid gives
\[ 
\mathbf u_n \in \mathbb R^{C_u\times N_{\mathrm{res}}}, \quad
\mathbf z_n \in \mathbb R^{C_z\times N_{\mathrm{res}}}. 
\] 
Thus, spatial discretization converts the function-space variables $\mathbf U_n\in\mathcal U$ and $\mathbf Z_n\in\mathcal Z$ into finite-dimensional arrays while preserving their spatial organization. In particular, the latent representation remains a field over the computational grid rather than being collapsed into a single latent vector. Under this discretization, the encoder, transition, and decoder mean maps from \S\ref{sec:abstract_vld} are represented by 
\[ 
m_\phi: \mathbb R^{C_u\times N_{\mathrm{res}}} \rightarrow \mathbb R^{C_z\times N_{\mathrm{res}}}, \quad 
T_\theta: \mathbb R^{C_z\times N_{\mathrm{res}}} \rightarrow \mathbb R^{C_z\times N_{\mathrm{res}}},\quad\text{and}\quad
D_\psi: \mathbb R^{C_z\times N_{\mathrm{res}}} \rightarrow \mathbb R^{C_u\times N_{\mathrm{res}}}. 
\] 
The encoder $m_\phi$ first lifts the physical state to a spatially resolved latent representation, which is subsequently evolved by $T_\theta$ and mapped back to the physical variables through $D_\psi$. We implement the encoder and decoder using resolution-preserving convolutional residual networks \citep{he2016deep}, while the latent dynamics are modeled by a Fourier neural operator \citep{li2020fourier}. Specifically, the mean transition takes the residual form 
\begin{equation} 
T_\theta(\mathbf z) = \mathbf z + \mathcal G_\theta^{\mathrm{FNO}}(\mathbf z), \label{eq:vamo_residual_transition} 
\end{equation} 
where $\mathcal G_\theta^{\mathrm{FNO}}$ is an FNO acting on the latent feature field. The residual parameterization represents one-step evolution through an increment of the current latent state, consistent with the time-discretized PDE flow map, while the Fourier layers provide nonlocal interactions across the spatial domain. The three mean maps therefore define the deterministic backbone 
\[ 
\mathbf u_n \xrightarrow{\,m_\phi\,} \mathbf z_n \xrightarrow{\,T_\theta\,} \mathbf z_{n+1} \xrightarrow{\,D_\psi\,} \mathbf u_{n+1}. 
\] 
VAMO augments this backbone with state-dependent Gaussian perturbations of the encoded and propagated latent fields. We next describe how the corresponding covariance structures are constructed on the discretized grid.

\subsection{Structured Gaussian Latent Perturbations}
\label{subsec:structured_latent_perturbations}

The deterministic backbone in \S\ref{subsec:vamo_architecture} specifies the mean evolution of the latent state. To realize the variational formulation, VAMO augments this evolution with Gaussian perturbations whose covariance preserves the spatial structure of the latent field. Rather than injecting independent noise at individual grid points, we generate correlated perturbations through fixed spectral covariance operators and modulate their amplitudes according to the current state.

\paragraph{Spectral covariance structure.}
We first describe the covariance model shared by the encoder and transition perturbations. On a periodic spatial grid, let $K_{\ell,s}$ denote the discrete analogue of the Laplacian-resolvent covariance
\[
    K_{\ell,s}=(I-\ell^2\Delta)^{-s},
\]
which is diagonal in the Fourier basis. Its square root therefore acts as
\begin{equation}
    \widehat{K_{\ell,s}^{1/2}\boldsymbol\xi}(k)
    =
    \left(1+\ell^2|k|^2\right)^{-s/2}
    \widehat{\boldsymbol\xi}(k),
    \label{eq:discrete_spectral_filter}
\end{equation}
where $\boldsymbol\xi$ is a standard Gaussian field and $k$ denotes the discrete spatial frequency. Thus, $\ell>0$ controls the correlation length of the perturbation, while $s>0$ determines the spectral decay at high frequencies. This construction is the finite-dimensional counterpart of the Laplacian-resolvent Gaussian covariance and Sobolev geometry discussed in \S\ref{sec:functional_gaussian_latent_markov}.

\paragraph{Complementary encoder and transition perturbations.}
For the Gaussian kernels, we use separate covariance operators
\[
    K_{\mathrm{enc}}=K_{\ell_{\mathrm{enc}},s},
    \quad
    K_{\mathrm{tr}}=K_{\ell_{\mathrm{tr}},s}
\]
for the encoder and latent transition, with correlation lengths chosen independently. In practice, we use $\ell_{\mathrm{enc}}<\ell_{\mathrm{tr}}$ and combine these covariance structures with different state-dependent modulations.

\begin{enumerate}[topsep=2pt]
\item\textit{State-dependent encoder distribution.} In addition to the encoder mean $m_\phi(\mathbf u)$, the encoder network produces a pointwise log-variance modulation field
\[
    \ell_\phi(\mathbf u)
    \in
    \mathbb R^{C_z\times N_{\mathrm{res}}}.
\]
We define the corresponding standard-deviation multiplier
\[
    \sigma_\phi(\mathbf u):=\exp\left(\frac12\ell_\phi(\mathbf u)\right),
\]
and let $S_\phi(\mathbf u)$ denote pointwise multiplication by $\sigma_\phi(\mathbf u)$. A latent state is generated via the reparameterization
\begin{equation}
    \mathbf z
    =
    m_\phi(\mathbf u)
    +
    S_\phi(\mathbf u)
    K_{\mathrm{enc}}^{1/2}\boldsymbol\xi,\qquad
    \boldsymbol\xi\sim\mathsf N(0,I),
    \label{eq:discrete_encoder_reparam}
\end{equation}
or equivalently,
\[
    \mathbf z
    \sim
    \mathsf N\bigl(
        m_\phi(\mathbf u),
        K_\phi(\mathbf u)
    \bigr),\quad K_\phi(\mathbf u)=S_\phi(\mathbf u)K_{\mathrm{enc}}S_\phi^*(\mathbf u).
\]
Hence, the fixed covariance operator $K_{\mathrm{enc}}$ specifies the underlying spatial correlation structure, while the learned multiplier $S_\phi(\mathbf u)$ adapts the perturbation magnitude across latent channels, spatial locations, and input states.

\item\textit{Stochastic latent transition.} The transition distribution uses the same structured-noise principle, but modulates the fixed covariance $K_{\mathrm{tr}}$ through a scalar state-dependent amplitude. Given a latent state $\mathbf z$, we sample
\begin{equation}
    \mathbf z^+
    =
    T_\theta(\mathbf z)
    +
    \alpha_\theta(\mathbf z)
    K_{\mathrm{tr}}^{1/2}\boldsymbol\xi^{\mathrm{tr}},
    \qquad
    \boldsymbol\xi^{\mathrm{tr}}
    \sim
    \mathsf N(0,I),
    \label{eq:discrete_transition_reparam}
\end{equation}
which defines
\begin{equation}
    P_\theta(\cdot\mid\mathbf z)
    =
    \mathsf N\!\left(
        T_\theta(\mathbf z),
        \alpha_\theta(\mathbf z)^2K_{\mathrm{tr}}
    \right).
    \label{eq:discrete_transition_kernel}
\end{equation}
The amplitude $\alpha_\theta(\mathbf z)>0$ is predicted by a lightweight network head and adapts the overall perturbation magnitude to the current latent state. The precise numerical parameterization of $\alpha_\theta$ and $\ell_\phi$ is deferred to Appendix~\S\ref{app:implementation_details}.
\end{enumerate}
This construction yields an intentional local-global asymmetry between the two stochastic components. The encoder combines the relatively small correlation length $\ell_{\mathrm{enc}}$ with a pointwise, spatially heterogeneous multiplier, thereby probing robustness to localized and fine-scale variations in the encoded physical state. In contrast, the transition instead combines the larger correlation length $\ell_{\mathrm{tr}}$ with a single state-dependent amplitude applied globally across the latent field, producing smoother and more spatially coherent perturbations around the evolved latent state. Thus, the encoder probes spatially heterogeneous local uncertainty, whereas the transition probes coherent variations of the predicted latent dynamics. Together, these complementary perturbation mechanisms realize the stochastic training effects analyzed in \S\ref{subsubsec:noise_injection}.

\paragraph{Decoder likelihood.}
Finally, we equip the discretized decoder with an isotropic Gaussian likelihood,
\begin{equation}
    P_\psi(\cdot\mid\mathbf z)
    =
    \mathsf N\!\left(
        D_\psi(\mathbf z),
        \sigma_u^2 I
    \right),
    \label{eq:discrete_decoder_likelihood}
\end{equation}
where $\sigma_u>0$ is fixed. This is the finite-dimensional specialization of the Gaussian decoder considered in Proposition~\ref{prop:cm_decoder_likelihood}. In the present setting, the Cameron-Martin space is the entire discretized physical space and its norm is a scaled Euclidean norm:
\[
    \|\mathbf v\|_{\mathcal H_{\sigma_u^2 I}}^2
    =
    \frac{1}{\sigma_u^2}\|\mathbf v\|_2^2.
\]
Consequently, up to an additive constant, the decoder negative
log-likelihood is
\[
    -\log P_\psi(\mathbf u\mid\mathbf z)
    =
    \frac{1}{2\sigma_u^2}
    \|\mathbf u-D_\psi(\mathbf z)\|_2^2
    +\mathrm{const}.
\]
Therefore, the Cameron-Martin reconstruction geometry of the function-space formulation reduces exactly, up to a fixed scaling, to the squared Euclidean norm loss used in the discretized model.

Unlike the encoder and transition distributions, we do not sample from the decoder likelihood. Given a latent state $\mathbf z$, prediction uses $D_\psi(\mathbf z)$ directly, which is both the mean and the \textit{maximum-likelihood} output of the isotropic Gaussian decoder. Hence the stochasticity used for variational training is confined to the latent representation and transition rather than being propagated into the physical prediction.

The fixed likelihood scale $\sigma_u>0$ determines the relative weighting between reconstruction error and latent KL regularization in the ELBO. Indeed, multiplying the negative ELBO by $2\sigma_u^2$ converts the likelihood term to an unweighted squared Euclidean loss while rescaling the KL term. Accordingly, rather than treating $\sigma_u$ as an additional tunable uncertainty parameter, we fix it and absorb this relative scaling into the KL coefficient $\beta$ in the training objective introduced below.

\subsection{Training Objective}
\label{subsec:vamo_training}
Having specified the discretized encoder, transition, and decoder distributions, we now formulate the training objective on observed PDE trajectories as the finite-dimensional counterpart of the variational objective studied in \S\ref{sec:err_var_dyn}. Let $\{\mathbf u_n\}_{n=0}^{N}$ denote a trajectory from the training set. For each adjacent pair $(\mathbf u_n,\mathbf u_{n+1})$, we draw
\[
    \mathbf z_n\sim Q_\phi(\cdot\mid\mathbf u_n),
    \qquad
    \mathbf z_n^+\sim P_\theta(\cdot\mid\mathbf z_n),
\]
using the reparameterizations in \eqref{eq:discrete_encoder_reparam} and \eqref{eq:discrete_transition_reparam}. The resulting one-step objective contains three components:
\begin{align}
    \mathcal L_{\mathrm{pred}}^{(n)}
    &:=
    \mathbb E_{\mathbf z_n\sim Q_\phi(\cdot\mid\mathbf u_n)}
    \mathbb E_{\mathbf z_n^+\sim P_\theta(\cdot\mid\mathbf z_n)}
    \left[
        \left\|
            \mathbf u_{n+1}
            -
            D_\psi(\mathbf z_n^+)
        \right\|_2^2
    \right],
    \label{eq:vamo_pred_loss}
    \\
    \mathcal L_{\mathrm{KL}}^{(n)}
    &:=
    \mathbb E_{\mathbf z_n\sim Q_\phi(\cdot\mid\mathbf u_n)}
    \left[
        D_{\mathrm{KL}}
        \left(
            Q_\phi(\cdot\mid\mathbf u_{n+1})
            \,\Vert\,
            P_\theta(\cdot\mid\mathbf z_n)
        \right)
    \right],
    \label{eq:vamo_kl_loss}
    \\
    \mathcal L_{\mathrm{rec}}^{(n)}
    &:=
    \mathbb E_{\mathbf z_n\sim Q_\phi(\cdot\mid\mathbf u_n)}
    \left[
        \left\|
            \mathbf u_n-D_\psi(\mathbf z_n)
        \right\|_2^2
    \right].
    \label{eq:vamo_rec_loss}
\end{align}
We optimize the trajectory-averaged objective
\begin{equation}
    \mathcal L(\theta,\phi,\psi)
    =
    \frac{1}{N}
    \sum_{n=0}^{N-1}
    \left(
        \mathcal L_{\mathrm{pred}}^{(n)}
        +
        \beta\,\mathcal L_{\mathrm{KL}}^{(n)}
        +
        \lambda_{\mathrm{rec}}\,
        \mathcal L_{\mathrm{rec}}^{(n)}
    \right),
    \label{eq:implemented_vlm_loss}
\end{equation}
where the expectations are estimated by Monte Carlo samples through the reparameterized encoder and transition distributions \citep{kingma2013auto}.

The three terms correspond directly to the mechanisms analyzed in \S\ref{sec:err_var_dyn}, with their computational roles illustrated in Figure~\ref{fig:vamo_architecture}(a). The predictive term trains the full stochastic pathway $\mathbf u_n\to\mathbf z_n\to\mathbf z_n^+\to\widehat{\mathbf u}_{n+1}$, the KL term aligns the predicted latent transition with the distribution obtained by encoding the true next state, and the reconstruction term maintains an informative latent representation of the physical field.

\paragraph{Relation to the function-space objective.} 
Equation~\eqref{eq:implemented_vlm_loss} is a finite-dimensional realization of a weighted combination of the predictive negative log-likelihood and the negative conditional ELBO. To make this relation explicit, consider
\[
    \lambda_{\mathrm{pred}}
    \mathcal L_{\mathrm{pred}}^{\mathrm{NLL}}
    +
    \lambda_{\mathrm{ELBO}}
    \left(
        \mathcal L_{\mathrm{KL}}+\mathcal L_{\mathrm{rec}}^{\mathrm{NLL}}
    \right).
\]
Under the isotropic decoder likelihood \eqref{eq:discrete_decoder_likelihood}, both likelihood terms reduce to squared Euclidean losses with the factor $(2\sigma_u^2)^{-1}$. Normalizing the coefficient of the predictive Euclidean loss to one therefore gives, up to additive constants,
\[
    \mathcal L_{\mathrm{pred}}
    +\underbrace{
    2\sigma_u^2\frac{\lambda_{\mathrm{ELBO}}}{\lambda_{\mathrm{pred}}}}_{\beta}\mathcal L_{\mathrm{KL}}
    +\underbrace{\frac{\lambda_{\mathrm{ELBO}}}{\lambda_{\mathrm{pred}}}}_{\lambda_{\mathrm{rec}}}\mathcal L_{\mathrm{rec}}.
\]
Hence, up to an overall rescaling, the relative weighting of the predictive bound and conditional ELBO in Proposition~\ref{prop:conditional_functional_elbo}, together with the decoder likelihood scale, is represented equivalently by the two coefficients $\beta$ and $\lambda_{\mathrm{rec}}$ in \eqref{eq:implemented_vlm_loss}.

Notably, the conditional ELBO in
Proposition~\ref{prop:conditional_functional_elbo} naturally contains next-state reconstruction, whereas \eqref{eq:vamo_rec_loss} uses the current state. Define the step-$n$ reconstruction term
\[
    \mathcal R_n
    :=
    \mathbb E_{\mathbf z_n\sim Q_\phi(\cdot\mid\mathbf u_n)}
    \left[
        -\log P_\psi(\mathbf u_n\mid\mathbf z_n)
    \right].
\]
Over a complete trajectory, the ELBO reconstruction terms sum to
\[
    \sum_{n=0}^{N-1}\mathcal R_{n+1}
    =
    \sum_{n=1}^{N}\mathcal R_n,
\]
whereas the implemented current-state reconstruction sums to $\sum_{n=0}^{N-1}\mathcal R_n$. The two therefore differ only through the boundary contributions $\mathcal R_0$ and $\mathcal R_N$. We use the current-state form because it reuses the latent sample $\mathbf z_n$ already drawn for the transition step, avoiding an additional reparameterized sample from $Q_\phi(\cdot\mid\mathbf u_{n+1})$ for reconstruction. This provides a modest computational efficiency gain during training while leaving all interior reconstruction contributions unchanged.

\paragraph{Mean latent rollout.}
At inference, VAMO follows the mean latent dynamics illustrated in Figure~\ref{fig:vamo_architecture}(b). Given an initial condition $\mathbf u_0$, we initialize $\widehat{\mathbf z}_0=m_\phi(\mathbf u_0)$ and  recursively compute
\[
    \widehat{\mathbf z}_{n+1}
    =T_\theta(\widehat{\mathbf z}_n),
    \qquad
    \widehat{\mathbf u}_{n+1}
    =D_\psi(\widehat{\mathbf z}_{n+1}).
\]
The initial condition is encoded once, and subsequent evolution proceeds entirely in latent space, with Gaussian perturbations disabled and decoded predictions used only as physical outputs.

\section{Numerical Experiments}
\label{sec:experiments}

In this section, we evaluate VAMO on three fluid-dynamics benchmarks with distinct dynamical characteristics: the one-dimensional compressible Euler equations, two-dimensional compressible fluid dynamics, and forced two-dimensional incompressible Navier--Stokes equations. Together, these problems cover inviscid wave propagation, coupled compressible flow, and long-time vortex dynamics under different levels of physical dissipation. Our experiments are designed to answer three main questions:
\begin{itemize}[itemsep=0pt, topsep=2pt]
\item Can variational latent dynamics improve the accuracy and stability of long-horizon autoregressive PDE prediction?
\item Are the improvements attributable to the structured variational formulation, rather than latent compression or generic noise injection alone?
\item Does VAMO better preserve physically relevant spatial and spectral structures throughout the rollout?
\end{itemize}

For all benchmarks, the model is trained using trajectories restricted to a finite training horizon and evaluated by autoregressive rollout over a \textit{substantially longer test horizon}. During inference, the model is initialized from the exact initial condition and receives no additional ground-truth states. We report errors accumulated over the entire test rollout. This evaluation therefore measures both prediction accuracy within the temporal regime represented during training and stability beyond the training horizon.

\subsection{Benchmarks and Baselines}
\label{sec:benchmarks_baselines}

\subsubsection{Benchmarks Problems}
\label{sec:benchmarks}

We consider three representative fluid-dynamics systems, summarized in Table~\ref{tab:benchmark_summary}. We provide a brief description here and defer the complete governing equations, data-generation procedures, discretization schemes, and dataset configurations to Appendix~\S\ref{app:benchmark_details}.

\paragraph{1D compressible Euler equations.}
The one-dimensional Euler equations describe inviscid compressible flow through conservation of mass, momentum, and total energy. We represent the physical state using the primitive variables
\[
    \mathbf{u}(t,x)=\bigl(\rho(t,x),\, v(t,x),\, p(t,x)\bigr),\quad t\in[0,T],\ \ x\in\mathbb{T}_L.
\]
where $\rho$, $v$, and $p$ denote density, velocity, and pressure, respectively. Since the system contains no explicit viscous dissipation, prediction errors in wave speed and phase can accumulate over time, while nonlinear evolution may generate sharp gradients and discontinuous structures. This benchmark therefore tests long-horizon propagation of compressible waves in a controlled one-dimensional setting.
\paragraph{2D compressible fluid dynamics.}
We consider a two-dimensional viscous compressible-flow benchmark with a physical configuration similar to the CFD setting studied in PDEBench \citep{takamoto2022pdebench}. The state is represented by
\[
    \mathbf{u}(t,\mathbf{x})
    =
    \bigl(
        \rho(t,\mathbf{x}),
        v_x(t,\mathbf{x}),
        v_y(t,\mathbf{x}),
        p(t,\mathbf{x})
    \bigr),\quad t\in[0,T],\ \mathbf{x}\in\mathbb{T}^2.
\]
This system couples density, pressure, and the two velocity components through nonlinear conservation laws and viscous transport. It provides a challenging multi-field benchmark in which local prediction errors can propagate across physical channels and spatial locations.

\paragraph{2D incompressible Navier-Stokes equations.}
We study the forced two-dimensional incompressible Navier-Stokes equations in vorticity form on the unit torus \citep{li2022transformer}. The learned state is the scalar vorticity field
\[
    \omega(t,\mathbf{x}),\quad t\in[0,T],\ \mathbf{x}\in\mathbb{T}^2.
\]
while the corresponding velocity field is recovered nonlocally through the Biot-Savart operator. We consider several viscosity levels and two families of external forcing. Under weak viscosity, small errors in the predicted vorticity can induce nonlocal velocity errors that feed back into future advection. This benchmark is therefore particularly suited to evaluating long-time rollout stability and preservation of spatial and spectral flow structure.

\begin{table}[t]
\small
    \centering
    \caption{
    Summary of the fluid-dynamics benchmarks.
    The three systems span one-dimensional inviscid compressible waves,
    two-dimensional viscous compressible flow, and forced incompressible
    vorticity dynamics.
    }
    \label{tab:benchmark_summary}
    \begin{tabular}{llll}
        \toprule
        Benchmark
        & PDE type
        & Fields
        & Main challenge \\
        \midrule
        1D Euler
        & Compressible, inviscid
        & $(\rho,v,p)$
        & Wave propagation, phase error, shocks \\
        2D CFD
        & Compressible, viscous
        & $(\rho,\mathbf{v},p)$
        & Coupled multi-field dynamics, conservation \\
        2D Navier-Stokes
        & Incompressible, viscous
        & $\omega$
        & Nonlocal transport and long-time stability \\
        \bottomrule
    \end{tabular}
\end{table}

\subsubsection{Compared Methods}
\label{sec:compared_methods}
We compare VAMO with three baselines that isolate the effects of direct physical-space evolution, generic noise injection, and deterministic latent dynamics. All models are trained as one-step predictors and evaluated through autoregressive rollout. \begin{description}[topsep=2pt,itemsep=0pt]
\item\textbf{FNO.} We use the Fourier neural operator \citep{li2020fourier} as a direct physical-space baseline. Unlike the trajectory-to-trajectory setting considered in the original FNO experiments, our model maps the current physical snapshot to the next snapshot and is applied recursively over the test horizon. 
\item\textbf{FNO+Noise.} This baseline uses the same autoregressive FNO architecture, but injects Gaussian perturbations into the lifted feature field during training. It tests whether generic feature-space noise regularization alone can improve long-horizon stability. 
\item\textbf{FNO-AE.} This model combines a deterministic autoencoder with a residual FNO latent transition. The initial physical state is encoded once, the latent field is evolved recursively, and the resulting latent states are decoded into physical predictions. It therefore isolates deterministic latent evolution from the variational training used by VAMO. 
\item\textbf{FNO-VAMO.} This denotes the complete proposed model described in \S\ref{sec:vamo}. FNO-AE and FNO-VAMO use closely matched latent architectures, allowing their comparison to isolate the effect of the variational formulation. 
\end{description}
The roles of the four comparisons are summarized in Table~\ref{tab:comparison_logic}. 
\begin{table}[t] 
\centering 
\caption{ Experimental role of each model comparison. } \label{tab:comparison_logic} 
\begin{tabular}{ll} 
\toprule Comparison & Question addressed \\ 
\midrule 
FNO-VAMO vs.\ FNO & Does the complete framework improve long-horizon prediction? \\ 
FNO-VAMO vs.\ FNO-AE & Is variational training necessary beyond latent modeling? \\ 
FNO-VAMO vs.\ FNO+Noise & Is structured latent stochasticity better than generic noise? \\ 
FNO-AE vs.\ FNO & Does the deterministic latent architecture alone improve stability? \\ 
\bottomrule 
\end{tabular} 
\end{table} 
For a fair comparison, the models use matched FNO backbones whenever applicable, including the number of Fourier layers, retained modes, and hidden width. All methods are trained using the same trajectories, data splits, optimization budget, and model-selection protocol. Complete architectural and optimization details are provided in Appendix~\S\ref{app:implementation_details}.

\paragraph{Evaluation protocol.}
All models are trained as one-step predictors and evaluated by autoregressive rollout from the exact initial condition, without teacher forcing or ground-truth reinitialization. Errors are computed over the full test horizon and averaged across test trajectories. Relative $L^2$ error is the primary metric across all benchmarks. We additionally report relative $L^1$ error for the 1D Euler equations, relative $H^1$ error for the 2D benchmarks, and errors in enstrophy, palinstrophy, and the isotropic energy spectrum for incompressible Navier-Stokes. Complete metric definitions, including the benchmark-specific temporal aggregation and numerical implementation, are provided in Appendix~\S\ref{app:benchmark_details}.

\subsection{Main Results}\label{sec:main_results}
We first compare the long-horizon rollout performance of VAMO and the three baselines across the fluid-dynamics benchmarks. For each problem, all models are trained and evaluated using the same trajectories and autoregressive protocol described in \S\ref{sec:benchmarks_baselines}. The tables report errors accumulated over the full test horizon. We focus here on the overall quantitative comparison and defer the temporal evolution of rollout errors, qualitative predictions, and additional physical diagnostics to the subsequent sections.

\subsubsection{1D Compressible Euler Equations}
\label{sec:euler_results}

We evaluate the models on periodic 1D Euler trajectories with domain lengths $L\in\{5,10,15,20\}$ and corresponding spatial resolutions $1024$, $2048$, $3072$, and $4096$, keeping the grid spacing fixed across settings. Initial conditions are generated from normalized periodic Gaussian random fields with physical correlation length $\ell=1$. Models are trained on trajectories over $[0,1.5]$ and rolled out autoregressively to $T_{\mathrm{test}}=10$ with snapshot interval $\Delta t=0.1$. Complete data-generation and solver details are provided in Appendix~\S\ref{app:euler_data_generation}. We report full-rollout relative $L^1$ and $L^2$ errors for density, velocity, and pressure in Table~\ref{tab:euler1d_results}.

\begin{table}[t]
\centering
\caption{Results on the 1D Euler equations with different domain lengths $L$. We report rollout-aggregated relative $L^1$ and $L^2$ errors for density $\rho$, pressure $p$, velocity $u$, and their channel average. Lower is better.}
\label{tab:euler1d_results}
\resizebox{\textwidth}{!}{
\begin{tabular}{llcccccccc}
\toprule
\multirow{2}{*}{$L$}
& \multirow{2}{*}{Method}
& \multicolumn{2}{c}{Density $\rho$}
& \multicolumn{2}{c}{Pressure $p$}
& \multicolumn{2}{c}{Velocity $u$}
& \multicolumn{2}{c}{Channel Mean} \\
\cmidrule(lr){3-4}
\cmidrule(lr){5-6}
\cmidrule(lr){7-8}
\cmidrule(lr){9-10}
&
& Rel.\,$L^2$ & Rel.\,$L^1$
& Rel.\,$L^2$ & Rel.\,$L^1$
& Rel.\,$L^2$ & Rel.\,$L^1$
& Rel.\,$L^2$ & Rel.\,$L^1$ \\
\midrule
\multirow{4}{*}{$5$}
& FNO       & 0.0194 & 0.00883 & 0.0249 & 0.00951 & 0.1409 & 0.0596 & 0.0617 & 0.0260 \\
& FNO+Noise & 0.0200 & 0.00962 & 0.0259 & 0.0109 & 0.1488 & 0.0727 & 0.0649 & 0.0311 \\
& FNO-AE    & 0.0161 & 0.00737 & 0.0208 & 0.00873 & 0.1195 & 0.0553 & 0.0521 & 0.0238 \\
& FNO-VAMO  & $\mathbf{0.0132}$ & $\mathbf{0.00704}$ & $\mathbf{0.0168}$ & $\mathbf{0.00756}$ & $\mathbf{0.0926}$ & $\mathbf{0.0515}$ & $\mathbf{0.0409}$ & $\mathbf{0.0220}$ \\
\midrule
\multirow{4}{*}{$10$}
& FNO       &0.0206 & 0.00801 & 0.0277 & 0.00942 & 0.1325 & 0.0512 & 0.0602 & 0.0229 \\
& FNO+Noise & 0.0212 & 0.00835 & 0.0286 & 0.00986 & 0.1370 & 0.0540 & 0.0623 & 0.0241 \\
& FNO-AE    & 0.0181 & 0.00709 & 0.0245 & 0.00913 & 0.1263 & 0.0553 & 0.0563 & 0.0238 \\
& FNO-VAMO  & $\mathbf{0.0141}$ & $\mathbf{0.00625}$ & $\mathbf{0.0189}$ & $\mathbf{0.00701}$ & $\mathbf{0.0924}$ & $\mathbf{0.0450}$ & $\mathbf{0.0418}$ & $\mathbf{0.0194}$ \\
\midrule
\multirow{4}{*}{$15$}
& FNO       & 0.0266 & 0.00919 & 0.0369 & 0.0113 & 0.1656 & 0.0646 & 0.0763 & 0.0284 \\
& FNO+Noise  & 0.0263 & 0.00918 & 0.0366 & 0.0113 & 0.1642 & 0.0649 & 0.0757 & 0.0285 \\
& FNO-AE    & 0.0214 & 0.00854 & 0.0290 & 0.0102 & 0.1300 & 0.0571 & 0.0601 & 0.0253 \\
& FNO-VAMO  & $\mathbf{0.0173}$ & $\mathbf{0.00733}$ & $\mathbf{0.0232}$ & $\mathbf{0.00844}$ & $\mathbf{0.1056}$ & $\mathbf{0.0475}$ & $\mathbf{0.0487}$ & $\mathbf{0.0211}$ \\
\midrule
\multirow{4}{*}{$20$}
& FNO       & 0.0380 & 0.0127 & 0.0516 & 0.0159 & 0.2227 & 0.0855 & 0.1099 & 0.0380 \\
& FNO+Noise & 0.0378 & 0.0126 & 0.0511 & 0.0158 & 0.2208 & 0.0851 & 0.1032 & 0.0378 \\
& FNO-AE    & 0.5394 & 0.2259 & 0.5434 & 0.2261 & 2.1236 & 1.0253 & 1.0688 & 0.4924 \\
& FNO-VAMO  & $\mathbf{0.0211}$ & $\mathbf{0.00993}$ & $\mathbf{0.0255}$ & $\mathbf{0.00955}$ & $\mathbf{0.1200}$ & $\mathbf{0.0641}$ & $\mathbf{0.0555}$ & $\mathbf{0.0279}$ \\
\bottomrule
\end{tabular}
}
\end{table}

As shown in Table~\ref{tab:euler1d_results}, FNO-VAMO achieves the lowest channel-averaged relative $L^2$ and $L^1$ errors for all four domain lengths. Relative to the strongest baseline in each setting (FNO-AE for $L=5,10,15$ and FNO+Noise for $L=20$), FNO-VAMO reduces the channel-mean relative $L^2$ error by approximately $21.5\%$, $25.8\%$, $19.0\%$, and $46.2\%$ for $L=5,10,15$, and $20$, respectively. The improvement is consistent across density, pressure, and velocity, with velocity remaining the most challenging physical channel.

The baseline comparison also distinguishes variational latent training from simpler alternatives. Generic noise injection provides little consistent improvement over the direct FNO, while the deterministic latent model performs competitively for $L\leq 15$ but becomes unstable at $L=20$. In contrast, FNO-VAMO maintains a channel-mean relative $L^2$ error below $0.056$ across all four domain lengths, despite the larger physical domain and increasingly long-range wave interactions. These results suggest that the variational formulation improves the robustness of recursively evolved latent dynamics rather than merely providing a better one-step approximation.

\subsubsection{2D Compressible Fluid Dynamics}
\label{sec:cfd2d_results}

We evaluate the models on 2D compressible Navier-Stokes trajectories on the periodic unit torus $\mathbb{T}^2=\bbR^2/\mathbb{Z}^2$ in a low-viscosity, low-Mach-number regime with $\mu=\zeta=10^{-8}$ and $M=0.1$. Initial conditions are sampled from periodic Gaussian random fields. Models are trained on $t\in[0,1]$ and evaluated by autoregressive rollout to $T_{\mathrm{test}}=10$ with snapshot interval $\Delta t=0.1$. The low-viscosity regime is chosen to preserve nontrivial dynamics over the long test horizon; complete data-generation and numerical details are provided in Appendix~\S\ref{app:cfd2d_data_generation}.

\begin{table}[p]
\centering
\caption{Results on 2D compressible fluid dynamics. We report rollout-aggregated relative $L^2$ and $H^1$ errors for density $\rho$, pressure $p$, velocity $\mathbf v=(v_x,v_y)$, and their channel average. Lower is better.}
\label{tab:cfd2d_results}
\resizebox{\textwidth}{!}{
\begin{tabular}{lcccccccc}
\toprule
\multirow{2}{*}{Method}
& \multicolumn{2}{c}{Density $\rho$}
& \multicolumn{2}{c}{Pressure $p$}
& \multicolumn{2}{c}{Velocity $\mathbf v$}
& \multicolumn{2}{c}{Channel Mean} \\
\cmidrule(lr){2-3}
\cmidrule(lr){4-5}
\cmidrule(lr){6-7}
\cmidrule(lr){8-9}
& Rel.\,$L^2$ & Rel.\,$H^1$
& Rel.\,$L^2$ & Rel.\,$H^1$
& Rel.\,$L^2$ & Rel.\,$H^1$
& Rel.\,$L^2$ & Rel.\,$H^1$ \\
\midrule
FNO        & 0.1248 & 2.2440 & 0.1674 & 3.2094 & 0.8345 & 2.7868 & 0.3756 & 2.7467 \\
FNO+Noise  & 0.1422 & 2.3479 & 0.2073 & 3.8650 & 0.9055 & 3.0393 & 0.4183 & 3.0841 \\
FNO-AE     & 0.0972 & 3.3152 & 0.0331 & 1.7665 & 0.6276 & 1.9364 & 0.2526 & 2.3394 \\
FNO-VAMO   & $\mathbf{0.0290}$ & $\mathbf{0.6375}$ & $\mathbf{0.0090}$ & $\mathbf{0.2612}$ & $\mathbf{0.1316}$ & $\mathbf{0.3145}$ & $\mathbf{0.0565}$ & $\mathbf{0.4044}$ \\
\bottomrule
\end{tabular}
}
\end{table}

\begin{figure}[p]
    \centering
    \includegraphics[width=\linewidth]{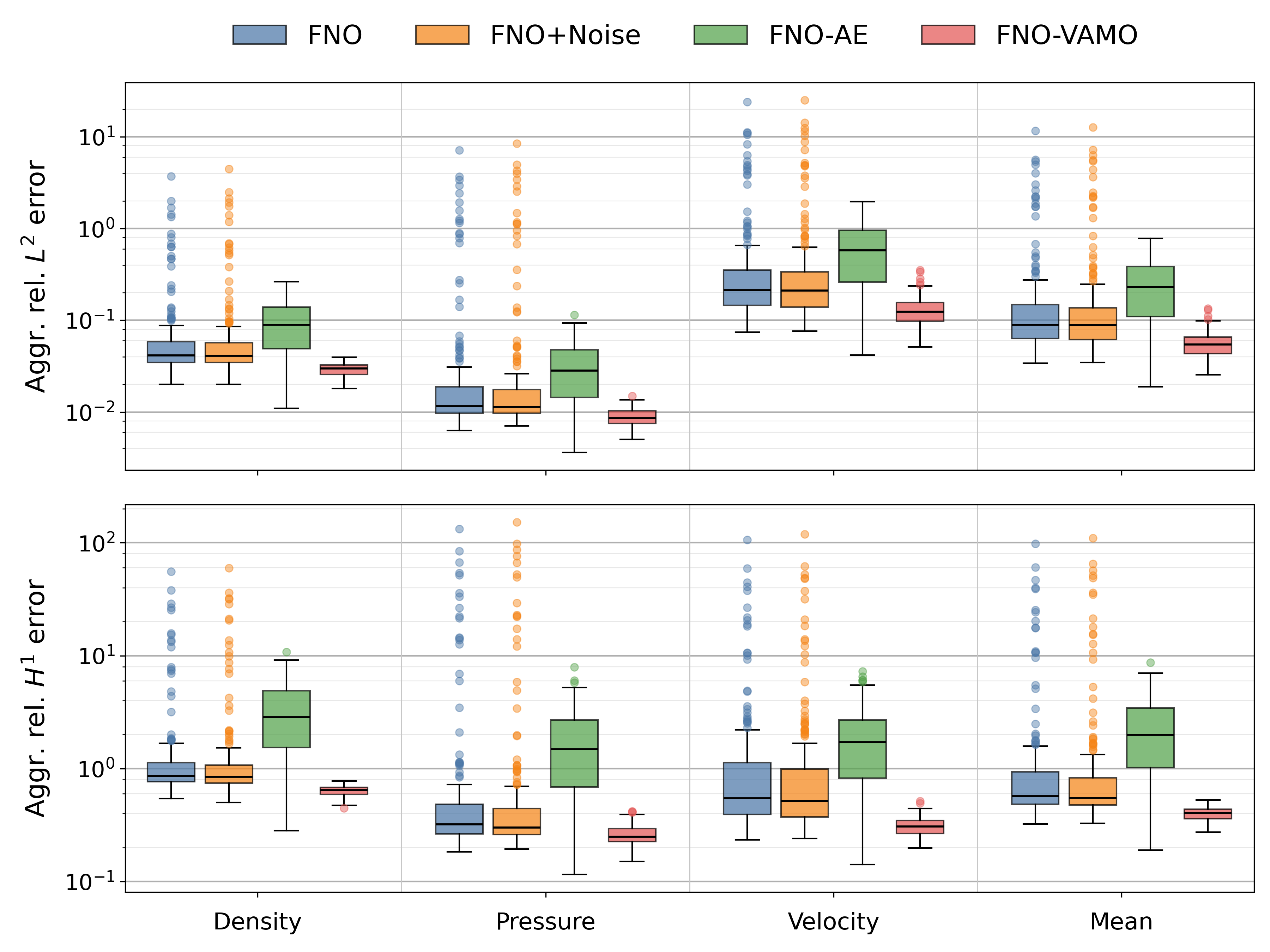}
    \caption{Distribution of per-trajectory full-rollout relative $L^2$ and $H^1$ errors for the two-dimensional compressible-flow benchmark. Results are shown separately for density, pressure, velocity, and their channel mean. The logarithmic scale highlights the heavy upper tails of the baseline methods and the tighter error distribution achieved by FNO-VAMO.}
    \label{fig:cfd2d_error_distribution}
\end{figure}

As shown in Table~\ref{tab:cfd2d_results}, FNO-VAMO substantially outperforms all baselines across every physical component and both error metrics. Relative to the strongest baseline (FNO-AE), it reduces the channel-averaged relative $L^2$ error from $0.2526$ to $0.0565$ ($-77.6\%$) and the relative $H^1$ error from $2.3394$ to $0.4044$ ($-82.7\%$). The gains are consistent across density, pressure, and velocity, with particularly large improvements for the velocity field, which is the most challenging component for the baseline models.

The baseline comparisons also clarify the source of the improvement. Generic noise injection does not improve upon the direct FNO, while the deterministic latent model lowers the $L^2$ error but remains much less accurate in $H^1$, indicating poor preservation of spatial derivatives despite improved field-level accuracy. In contrast, FNO-VAMO reduces both metrics simultaneously, suggesting better preservation of the spatial structure of the coupled compressible dynamics. Figure~\ref{fig:cfd2d_error_distribution} further shows that this advantage is reflected across individual test trajectories: FNO-VAMO exhibits lower typical errors, tighter distributions, and substantially fewer large-error outliers, especially in relative $H^1$. Thus, its improvement reflects more reliable long-horizon prediction across the test distribution rather than gains on only a small subset of trajectories.

\subsubsection{2D Incompressible Navier-Stokes Equations}
\label{sec:ns2d_results}
We consider the forced 2D incompressible Navier-Stokes
equations in vorticity form on the periodic unit torus $\mathbb{T}^2=\bbR^2/\mathbb{Z}^2$, with viscosities $\nu\in\{10^{-3},10^{-4},10^{-5}\}$ following the standard FNO benchmark setting \citep{li2020fourier}. We extend this setting by independently sampling a time-independent forcing field for each trajectory from either a Gaussian random field (\textsc{GRF}) or a sparse mixture of sinusoidal Fourier modes (\textsc{Wave}). For each viscosity, a single model is trained jointly on 1,600 trajectories, with 800 from each forcing family, and evaluated separately on the two forcing types. The training horizons are $[0,10]$, $[0,12]$, and $[0,8]$ for $\nu=10^{-3}$, $10^{-4}$, and $10^{-5}$, respectively, with corresponding test horizons $T=50$, $30$, and $20$. Complete data-generation and numerical details are provided in Appendix~\S\ref{app:ns2d_data_generation}.

\begin{table}[p]
\centering
\caption{Results on the 2D incompressible Navier-Stokes equations. We evaluate six test settings with viscosity $\nu\in\{10^{-3},10^{-4},10^{-5}\}$ and forcing type in \{GRF, \textsc{Wave}\}. The training horizons are $[0,10]$, $[0,12]$, and $[0,8]$ for $\nu=10^{-3}$, $10^{-4}$, and $10^{-5}$, respectively, while the corresponding test horizons are $[0,50]$, $[0,30]$, and $[0,20]$. We report rollout-aggregated errors. Lower is better.}
\label{tab:ns2d_main_results}
\resizebox{\textwidth}{!}{
\begin{tabular}{lllccccc}
\toprule
Data Setting
& Forcing 
& Method 
& Rel.\,$L^2$ 
& Rel.\,$H^1$ 
& $\mathcal{E}$-error 
& $\mathcal{P}$-error 
& Spec. error \\
\midrule
\multirow{8}{*}{$\begin{aligned}
&\nu=10^{-3}\\ &T_\text{train}=10 \\ &T_\text{test}=50 
\end{aligned}$}
& \multirow{4}{*}{GRF}
& FNO         & 0.1453 & 0.8015 & 0.1151 & 3.784 & 0.0239 \\
& 
& FNO+Noise   & 0.0375 & 0.0985 & 0.0196 & 0.1773 & 0.0206 \\
&
& FNO-AE      & 22.31 & 292.0 & 2447 & $4.358\!\times\! 10^5$ & 452.0 \\
&
& FNO-VAMO    & $\mathbf{0.0203}$ & $\mathbf{0.0264}$ & $\mathbf{0.0119}$ & $\mathbf{0.0123}$ & $\mathbf{0.0138}$ \\
\cmidrule(lr){2-8}
& \multirow{4}{*}{\textsc{Wave}}
& FNO         & 0.0698 & 0.1769 & 0.0430 & 0.4039 & 0.0116 \\
&
& FNO$+$Noise   & 0.0686 & 0.1148 & 0.0406 & 0.1778 & 0.0241 \\
&
& FNO-AE      & 5.623 & 55.84 & 565.0 & $5.753\!\times\! 10^4$ & 124.3 \\
&
& FNO-VAMO    & $\mathbf{0.0556}$ & $\mathbf{0.0687}$ & $\mathbf{0.00808}$ & $\mathbf{0.00940}$ & $\mathbf{0.0106}$ \\
\midrule
\multirow{8}{*}{$\begin{aligned}
&\nu=10^{-4}\\ &T_\text{train}=12 \\ &T_\text{test}=30 
\end{aligned}$}
& \multirow{4}{*}{GRF}
& FNO         & 0.4068 & 3.109 & 0.7338 & 45.84 & 0.0456 \\
&
& FNO+Noise   & 0.2987 & 2.003 & 0.5866 & 35.15 & 0.0569 \\
&
& FNO-AE      & 2.691 & 31.93 & 23.19 & 3608 & 0.8164 \\
&
& FNO-VAMO    & $\mathbf{0.0403}$ & $\mathbf{0.1139}$ & $\mathbf{0.0139}$ & $\mathbf{0.0512}$ & $\mathbf{0.0146}$ \\
\cmidrule(lr){2-8}
& \multirow{4}{*}{\textsc{Wave}}
& FNO         & 0.3570 & 1.613 & 0.6580 & 14.51 & 0.0497 \\
&
& FNO+Noise   & 0.3407 & 1.449 & 0.0666 & 15.73 & 0.0499 \\
&
& FNO-AE      & 1.188 & 10.19 & 7.213 & 580.3 & 0.4303 \\
&
& FNO-VAMO    & $\mathbf{0.1198}$ & $\mathbf{0.2212}$ & $\mathbf{0.0386}$ & $\mathbf{0.1008}$ & $\mathbf{0.0390}$ \\
\midrule
\multirow{8}{*}{$\begin{aligned}
&\nu=10^{-5}\\ &T_\text{train}=8 \\ &T_\text{test}=20 
\end{aligned}$}
& \multirow{4}{*}{GRF}
& FNO         & 0.2987 & 2.262 & 0.4862 & 32.90 & 0.0382 \\
&
& FNO+Noise   & 0.4089 & 3.061 & 1.179 & 75.58 & 0.0909 \\
&
& FNO-AE      & 0.5991 & 1.335 & 0.2788 & 1.354 & 0.1551 \\
&
& FNO-VAMO    & $\mathbf{0.0841}$ & $\mathbf{0.3858}$ & $\mathbf{0.0211}$ & $\mathbf{0.2415}$ & $\mathbf{0.00929}$ \\
\cmidrule(lr){2-8}
& \multirow{4}{*}{\textsc{Wave}}
& FNO & 0.2993 & 1.969 & 0.3519 & 19.14 & 0.0314 \\
&
& FNO+Noise   & 0.3450 & 2.254 & 0.5831 & 31.22 & 0.0617 \\
&
& FNO-AE      & 0.6719 & 1.158 & 0.2098 & 0.5878 & 0.2518 \\
&
& FNO-VAMO    & $\mathbf{0.1469}$ & $\mathbf{0.4509}$ & $\mathbf{0.0460}$ & $\mathbf{0.2734}$ & $\mathbf{0.0247}$ \\
\bottomrule
\end{tabular}
}
\end{table}

As shown in Table~\ref{tab:ns2d_main_results}, FNO-VAMO achieves the lowest error across all 6 viscosity-forcing settings and all five evaluation metrics. The gains are particularly strong in the more challenging low-viscosity regimes. Relative to the strongest baseline in each setting, FNO-VAMO reduces the relative $L^2$ error by approximately $65$--$87\%$ for $\nu=10^{-4}$ and $51$--$72\%$ for $\nu=10^{-5}$, with corresponding relative $H^1$ reductions of about $85$--$94\%$ and $61$--$71\%$, respectively.

The improvements extend beyond field-level accuracy. FNO-VAMO also consistently yields lower enstrophy, palinstrophy, and spectral errors, indicating better preservation of both overall vorticity magnitude and small-scale spatial structure. Generic noise injection is not consistently beneficial, while the deterministic latent baseline becomes highly unstable in several settings, especially at $\nu=10^{-3}$ and $\nu=10^{-4}$. These comparisons suggest that neither latent evolution nor unstructured noise alone is sufficient for reliable long-horizon prediction.

Figure~\ref{fig:ns2d_error_distribution} further shows that the advantage of FNO-VAMO is reflected not only in lower average error, but also in the distribution across individual trajectories. In the two more challenging regimes $\nu=10^{-4}$ and $\nu=10^{-5}$, FNO-VAMO generally achieves comparable or lower median errors while exhibiting narrower spreads and substantially fewer large-error outliers. This indicates that the variational formulation improves not only average accuracy but also the reliability of long-horizon rollout across trajectories.

\begin{figure}[ht]
    \centering
    \begin{subfigure}{\linewidth}
        \centering
        \caption{NS2D per-trajectory rollout errors, $\nu=10^{-4}$.}
        \includegraphics[width=\linewidth]
        {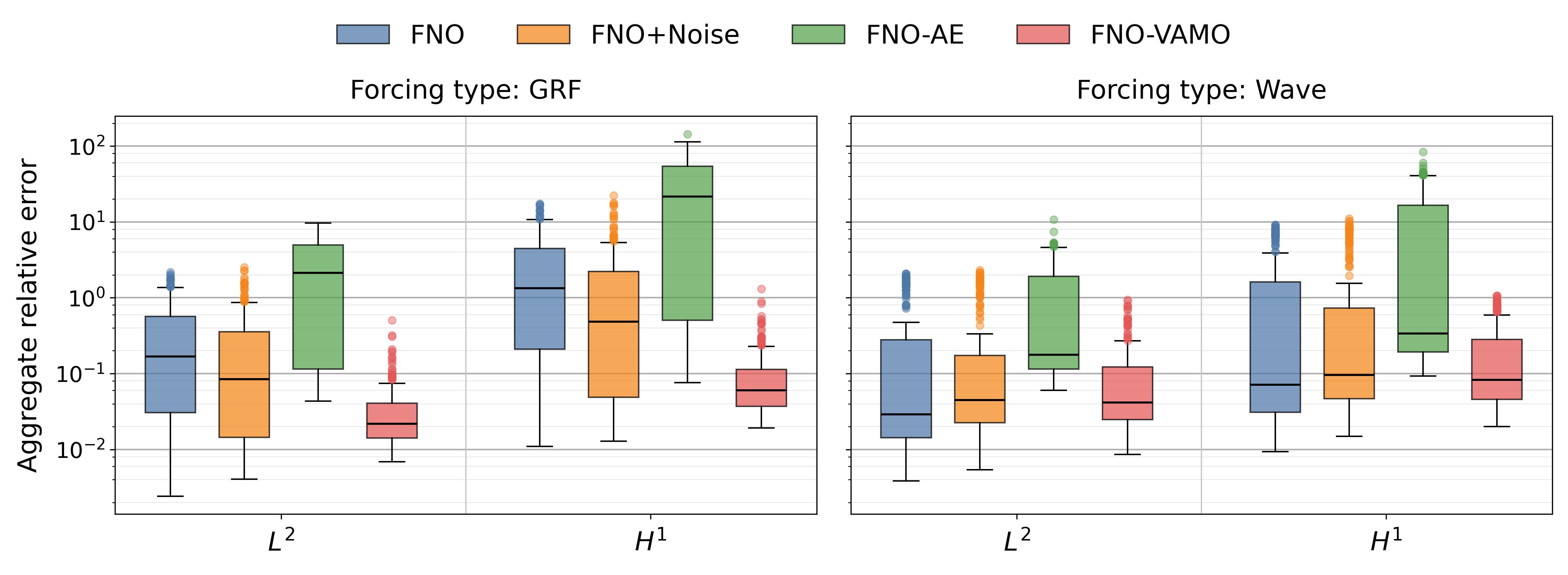}
        \label{fig:ns2d_nu4_error_distribution}
    \end{subfigure}
    \vspace{-1cm}
    
    \begin{subfigure}{\linewidth}
        \centering
        \caption{NS2D per-trajectory rollout errors, $\nu=10^{-5}$.}
        \includegraphics[width=\linewidth]
        {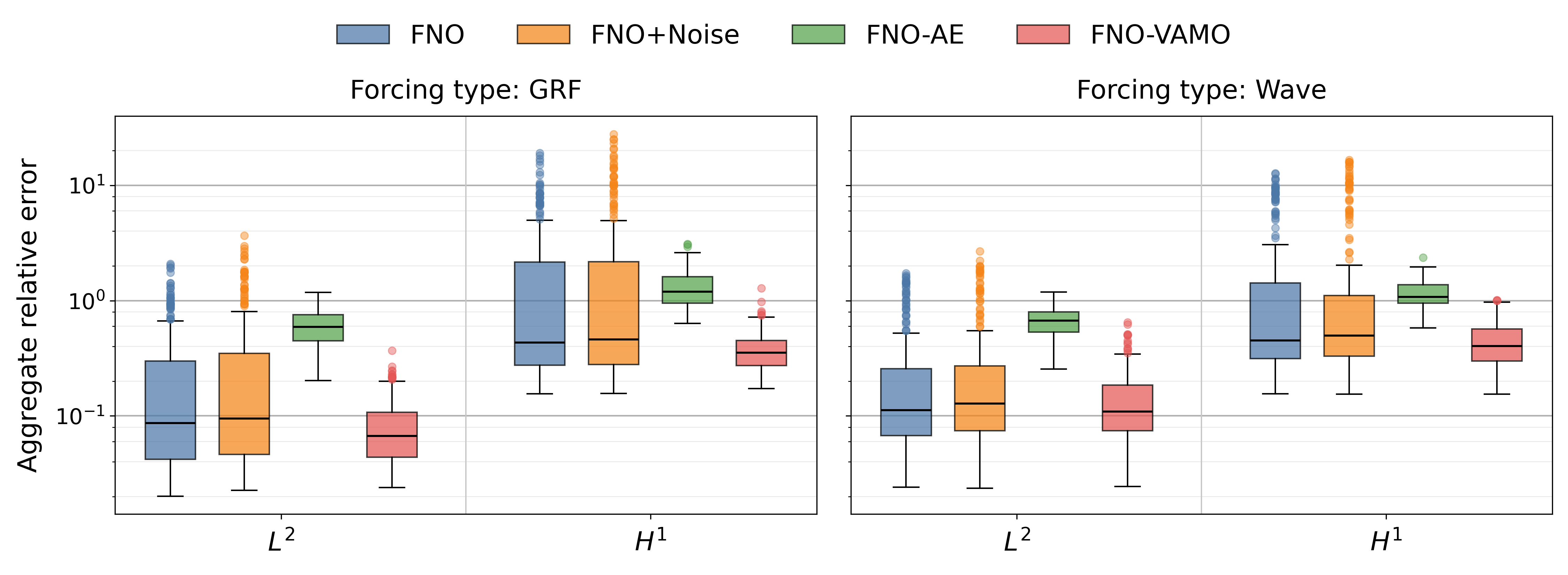}
        \label{fig:ns2d_nu5_error_distribution}
    \end{subfigure}
    \vspace{-1cm}
    \caption{
    Distribution of per-trajectory full-rollout relative $L^2$ and $H^1$ errors for the 2D incompressible Navier-Stokes benchmark. Results are shown separately for GRF and \textsc{Wave} forcing. The logarithmic scale highlights both the typical error and the heavy upper tails associated with unstable rollouts. VAMO generally reduces the median error and suppresses large-error outliers.
    }
    \label{fig:ns2d_error_distribution}
\end{figure}

\subsection{Long-Horizon Error Growth}
\label{sec:error_growth}

The trajectory-aggregated results in \S\ref{sec:main_results} do not show how prediction errors evolve during autoregressive rollout. We therefore report per-snapshot errors for representative 1D Euler settings with $L=5$ and $L=20$ (Figure~\ref{fig:euler_trend}), the 2D compressible-flow benchmark (Figure~\ref{fig:cfd2derror}), and the two more challenging 2D incompressible Navier--Stokes regimes, $\nu=10^{-4}$ and $\nu=10^{-5}$, under both forcing families (Figures~\ref{fig:ns2d_nu4_trend}--\ref{fig:ns2d_nu5_trend}). Results for more benchmarks are provided in Appendix~\S\ref{app:additional_quantitative}. In each figure, the shaded region denotes the temporal horizon represented during training.

\begin{figure}[p]
    \centering
    \begin{subfigure}{\linewidth}
        \centering
        \includegraphics[width=1.0\linewidth]{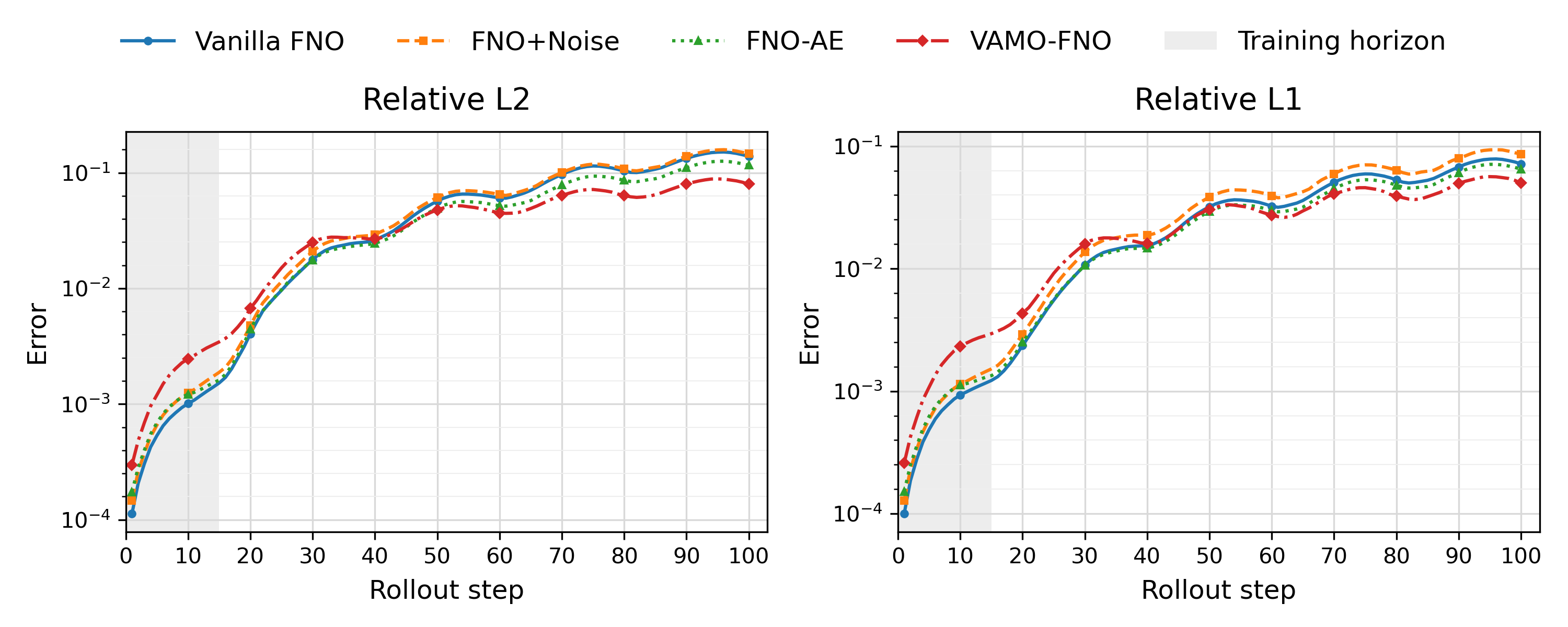}
        \caption{Domain length $L=5$.}
        \label{fig:eulerl5}
    \end{subfigure}

    \vspace{0.5em}
    
    \begin{subfigure}{\linewidth}
        \centering
        \includegraphics[width=1.0\linewidth]{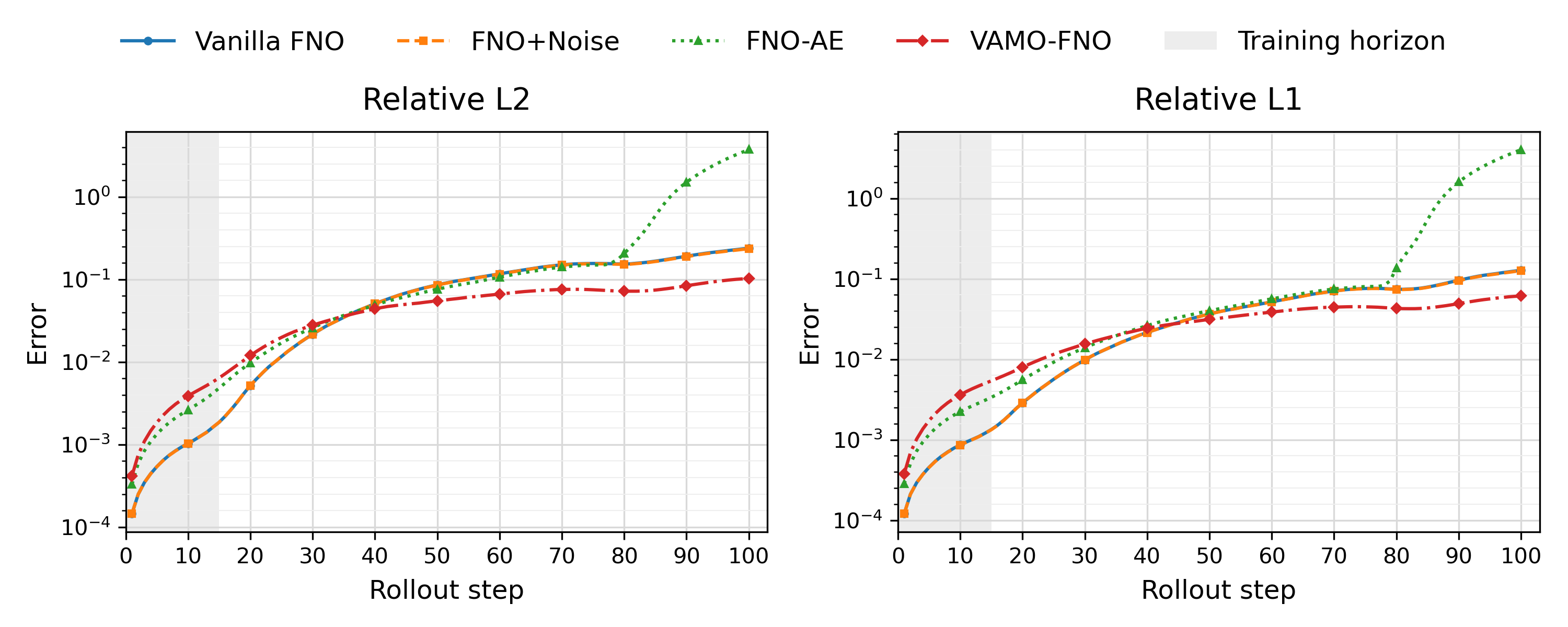}
        \caption{Domain length $L=20$.}
        \label{fig:eulerl20}
    \end{subfigure}
    \caption{Error trend for representative 1D Euler Benchmarks.}
    \label{fig:euler_trend}
\end{figure}

\begin{figure}[p]
    \centering
    \includegraphics[width=1.0\linewidth]{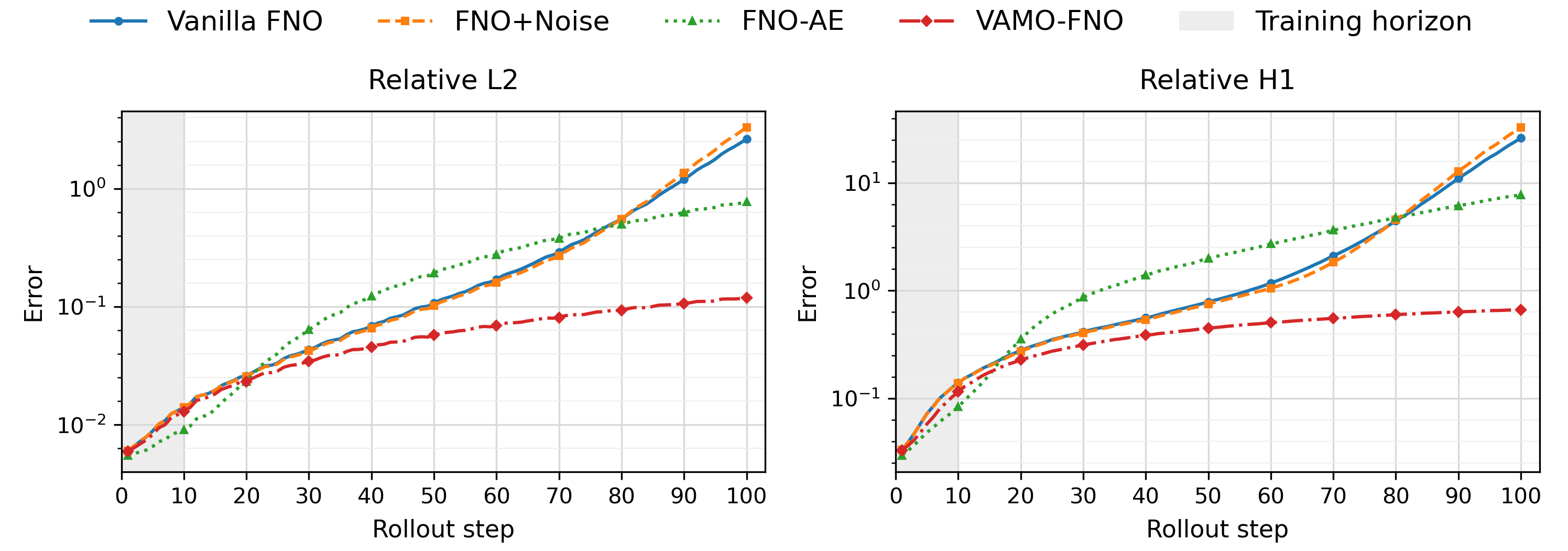}
    \caption{Error trend for 2D CFD Benchmark.}
    \label{fig:cfd2derror}
\end{figure}

\begin{figure}[p]
    \centering
    \begin{subfigure}{\linewidth}
        \centering
        \includegraphics[width=1.0\linewidth]{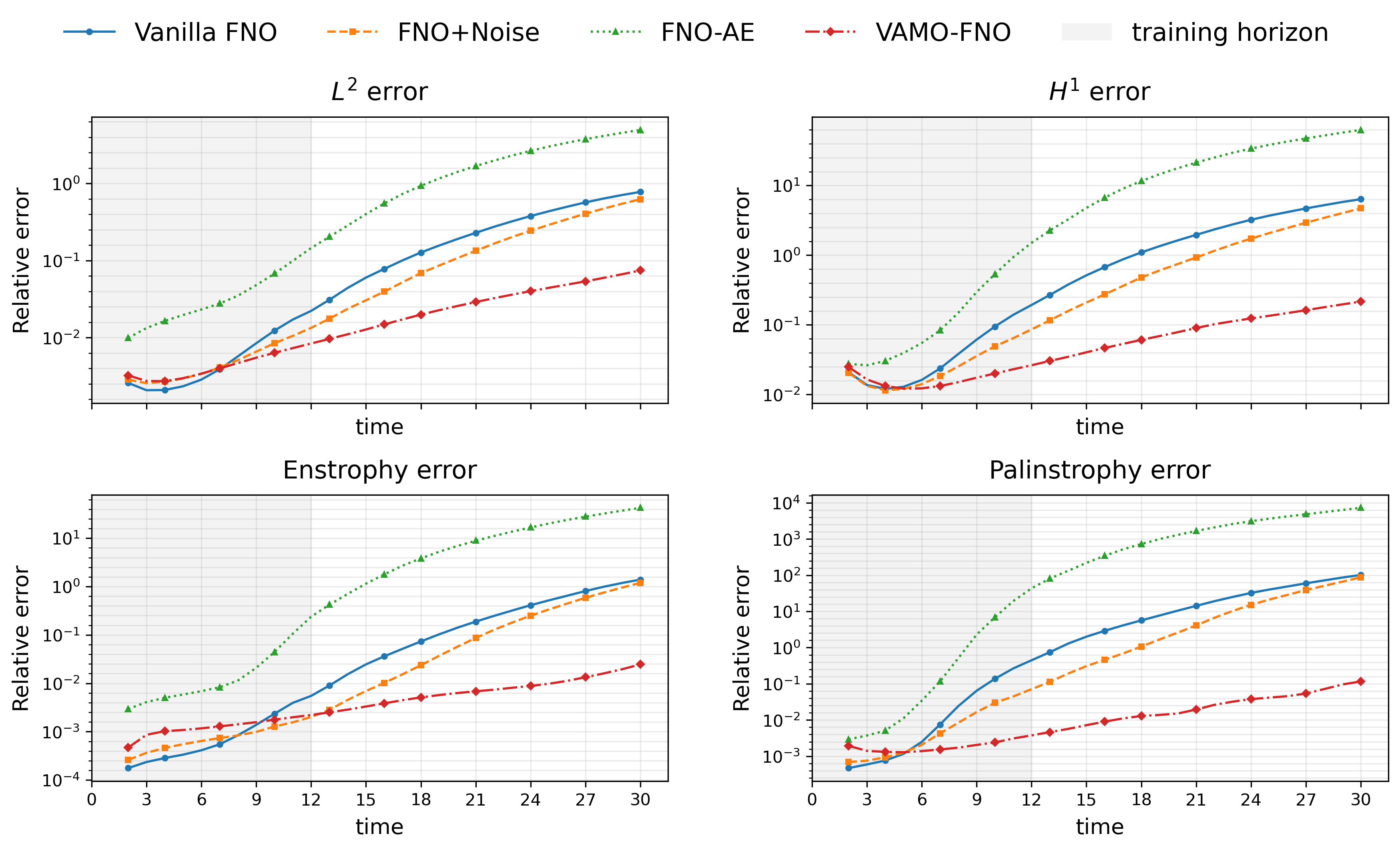}
        \caption{\textsc{GRF} forcing type.}
        \label{fig:nu4grf}
    \end{subfigure}

    \vspace{0.5em}
    
    \begin{subfigure}{\linewidth}
        \centering
        \includegraphics[width=1.0\linewidth]{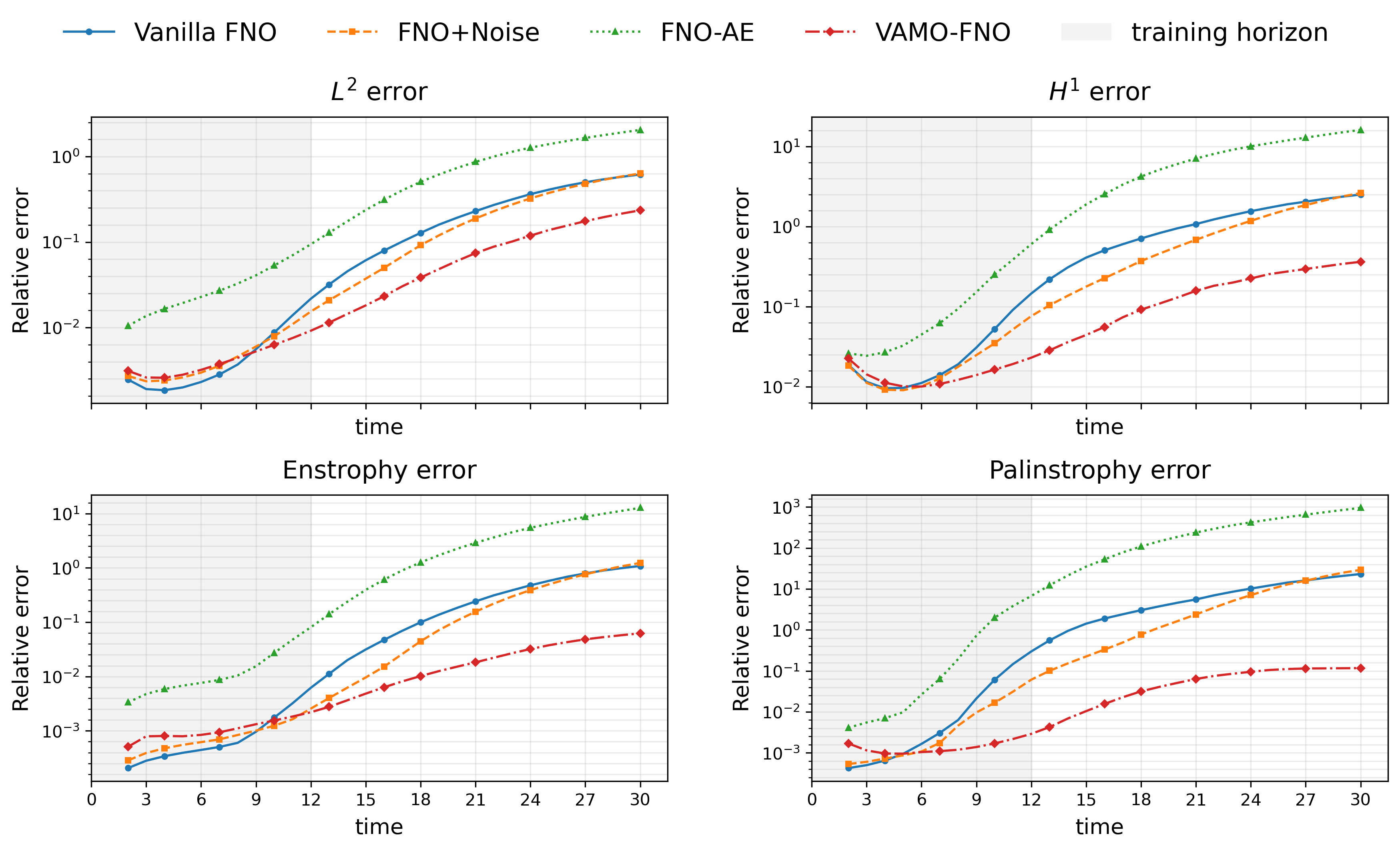}
        \caption{\textsc{Wave} forcing type.}
        \label{fig:nu4wave}
    \end{subfigure}
    \caption{Error trend for 2D incompressible NS Benchmark with $\nu=10^{-4}$.}
    \label{fig:ns2d_nu4_trend}
\end{figure}

\begin{figure}[p]
    \centering
    \begin{subfigure}{\linewidth}
        \centering
        \includegraphics[width=1.0\linewidth]{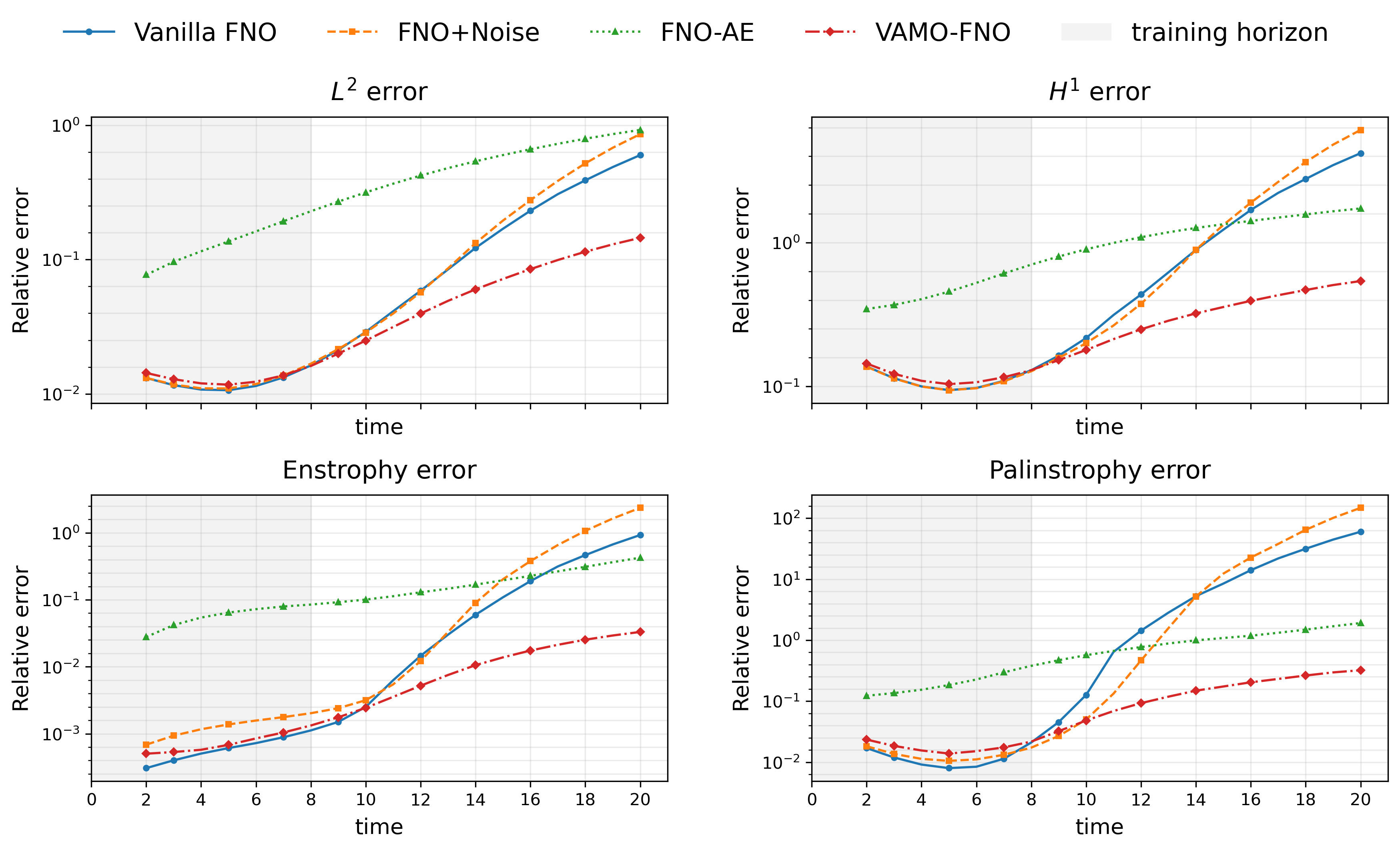}
        \caption{\textsc{GRF} forcing type.}
        \label{fig:nu5grf}
    \end{subfigure}

    \vspace{0.5em}
    
    \begin{subfigure}{\linewidth}
        \centering
        \includegraphics[width=1.0\linewidth]{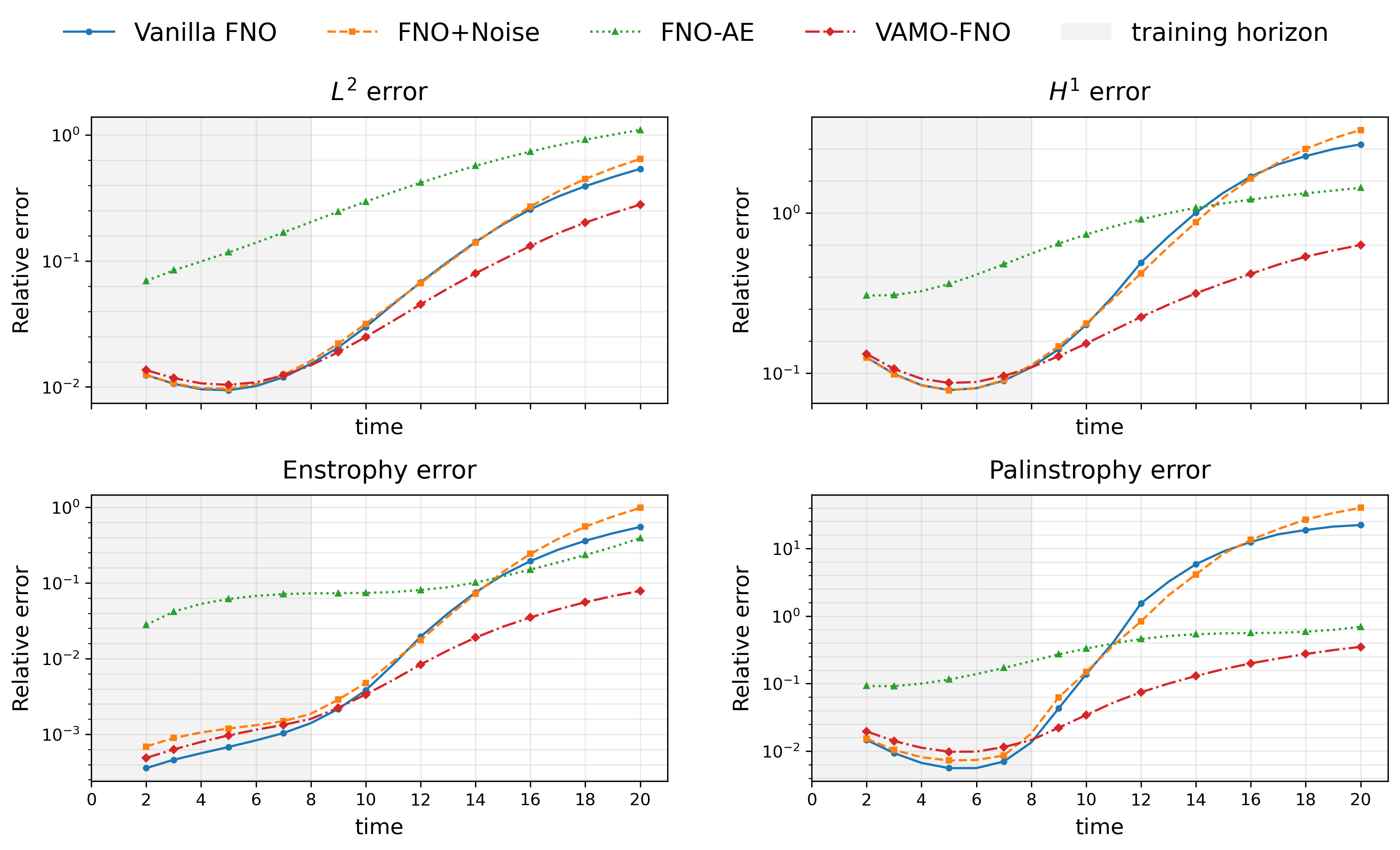}
        \caption{\textsc{Wave} forcing type.}
        \label{fig:nu5wave}
    \end{subfigure}
    \caption{Error trend for 2D incompressible NS Benchmark with $\nu=10^{-5}$.}
    \label{fig:ns2d_nu5_trend}
\end{figure}

Across the three PDE benchmarks, the main advantage of VAMO appears in the later stages of rollout. Its error can be comparable to, or occasionally higher than, that of the direct baselines at early in-distribution snapshots, particularly in easier settings such as the 1D Euler equation with $L=5$. However, the baseline errors generally grow more rapidly as the rollout proceeds, whereas VAMO exhibits substantially slower error accumulation. Consequently, VAMO achieves lower late-time error and better full-rollout accuracy.

This distinction becomes especially clear in the more challenging settings. For Euler with $L=20$, the deterministic latent model remains competitive initially but becomes unstable during the later rollout, while VAMO maintains controlled $L^1$ and $L^2$ errors. On the 2D compressible-flow benchmark, the $L^2$ and $H^1$ errors of the baselines grow rapidly after the training horizon, whereas VAMO remains substantially more stable. A similar pattern is observed in the low-viscosity incompressible Navier--Stokes regimes $\nu=10^{-4}$ and $\nu=10^{-5}$ under both forcing families, where VAMO consistently limits the growth of field-level errors and physically relevant quantities such as enstrophy and palinstrophy. These results indicate that the primary benefit of the variational formulation lies in \textit{improving long-horizon stability} rather than uniformly reducing the error at every individual prediction step.

\begin{figure}[H]
    \centering

    \begin{subfigure}{\linewidth}
        \centering
        \includegraphics[width=0.98\linewidth]
        {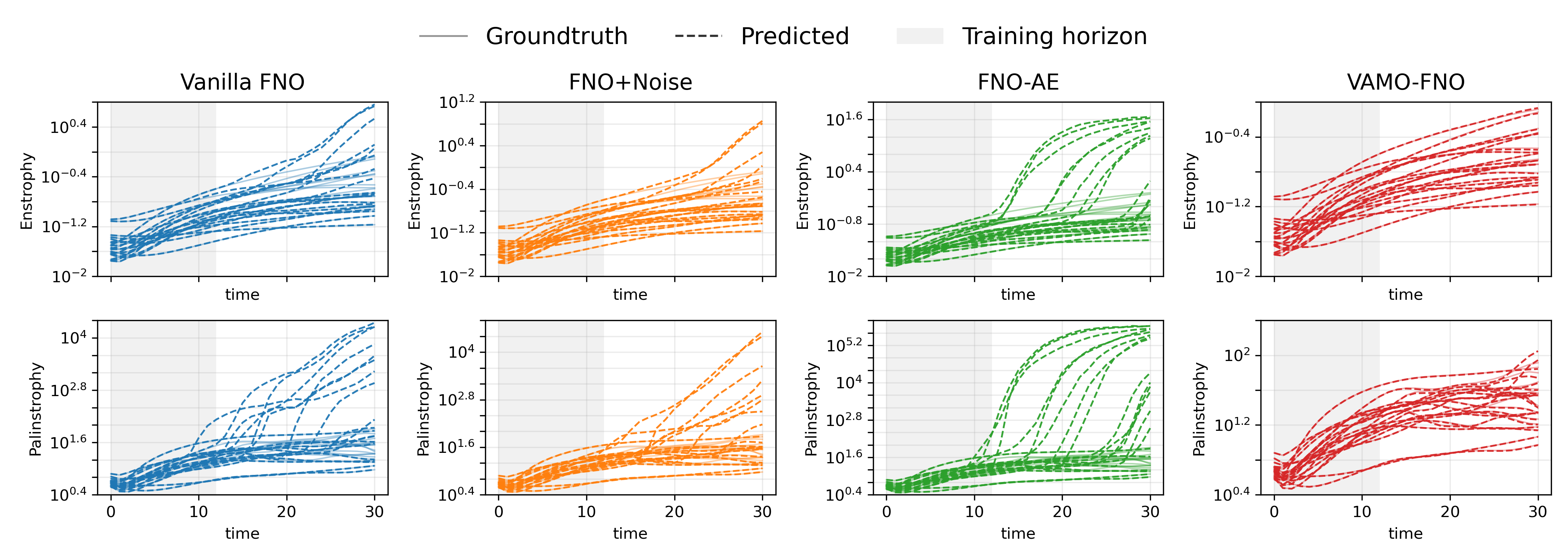}
        \caption{Viscosity $\nu=10^{-4}$, \textsc{GRF} forcing.}
        \label{fig:nsdiag_nu4_grf}
    \end{subfigure}

    \vspace{0.5em}

    \begin{subfigure}{\linewidth}
        \centering
        \includegraphics[width=0.98\linewidth]
        {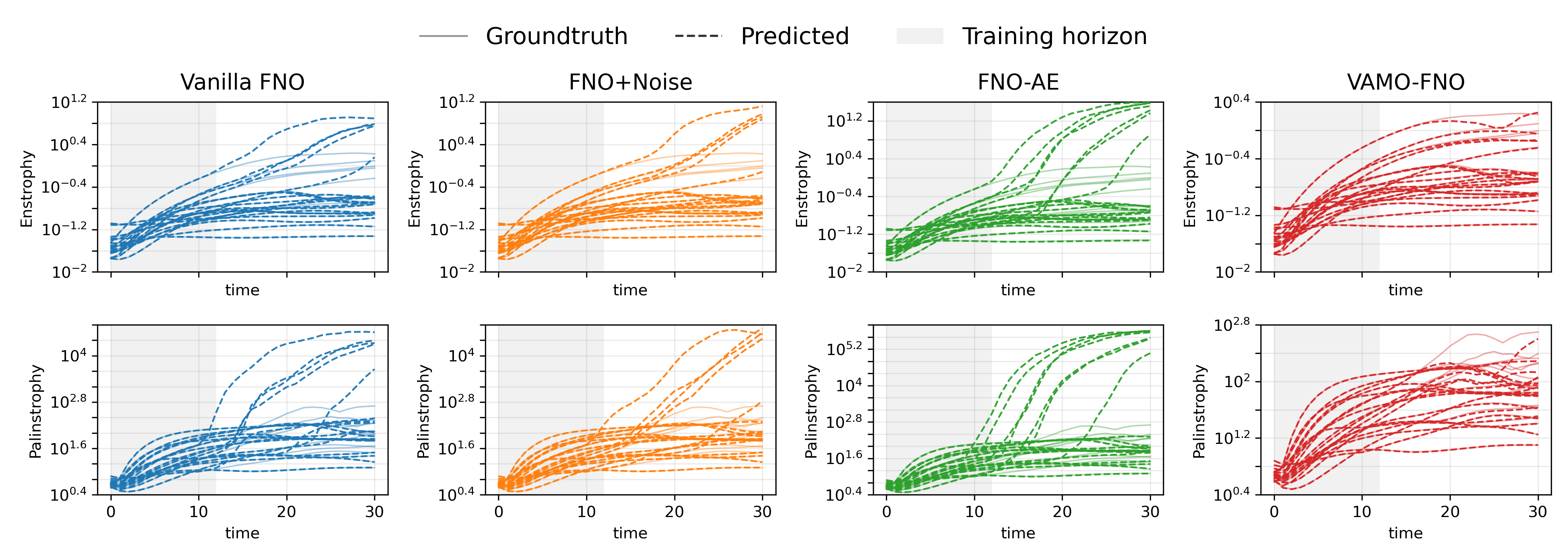}
        \caption{Viscosity $\nu=10^{-4}$, \textsc{Wave} forcing.}
        \label{fig:nsdiag_nu4_wave}
    \end{subfigure}

    \caption{Physical diagnostics for the 2D incompressible Navier-Stokes benchmark. Each panel compares enstrophy and palinstrophy along 20 ground-truth and predicted trajectories. Solid translucent curves denote ground truth, and dashed curves denote predictions.}
    \label{fig:ns2d_physical_diagnostics}
\end{figure}

\begin{figure}[ht]
    \ContinuedFloat
    \centering

    \begin{subfigure}{\linewidth}
        \centering
        \includegraphics[width=\linewidth]
        {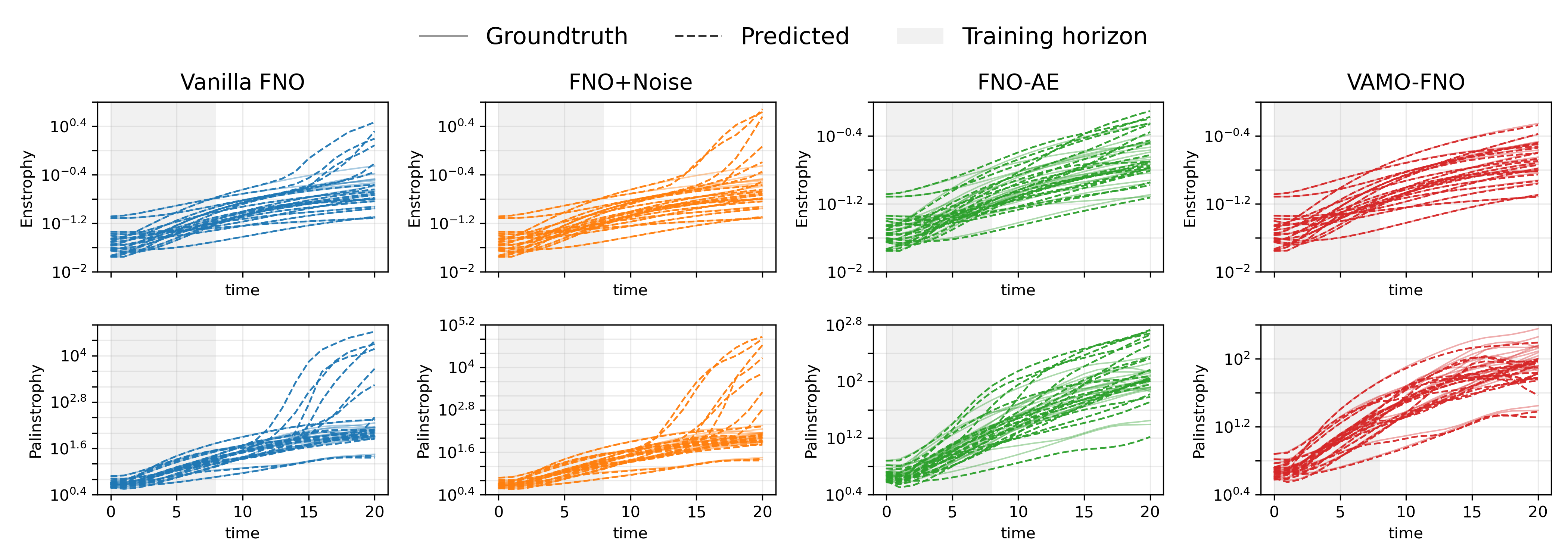}
        \caption{Viscosity $\nu=10^{-5}$, \textsc{GRF} forcing.}
        \label{fig:nsdiag_nu5_grf}
    \end{subfigure}

    \vspace{0.5em}

    \begin{subfigure}{\linewidth}
        \centering
        \includegraphics[width=0.98\linewidth]
        {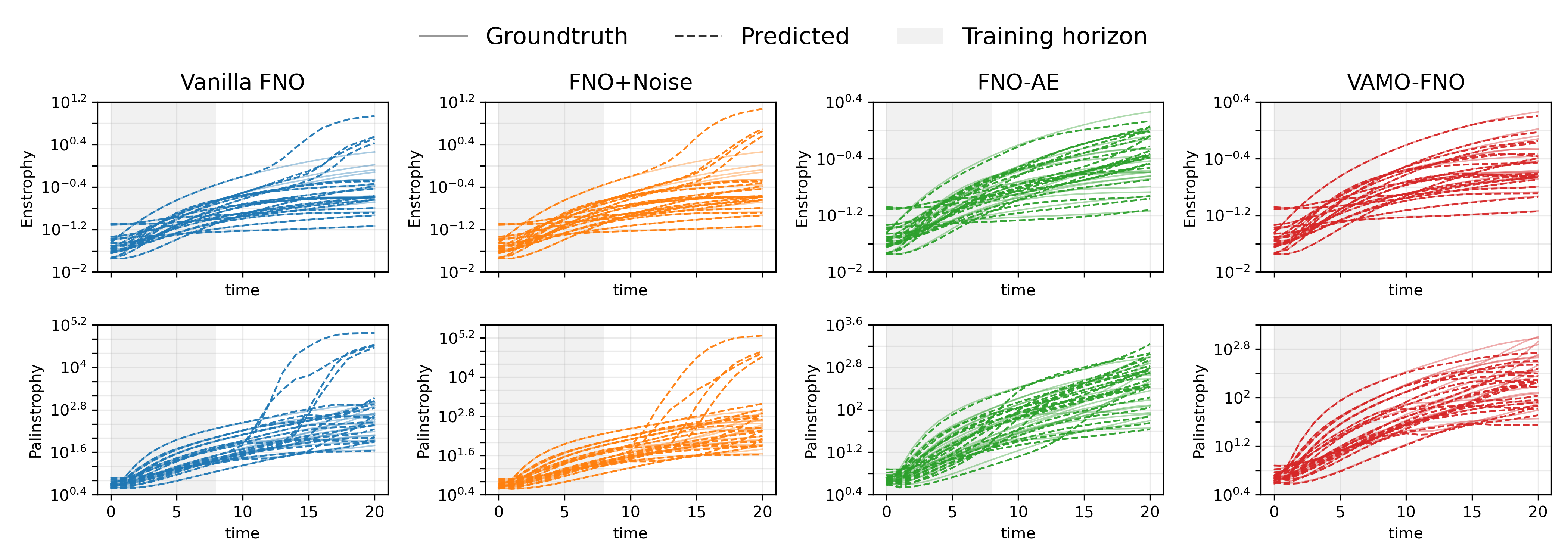}
        \caption{Viscosity $\nu=10^{-5}$, \textsc{Wave} forcing.}
        \label{fig:nsdiag_nu5_wave}
    \end{subfigure}

    \caption[]{Physical diagnostics for the 2D incompressible Navier-Stokes benchmark (continued).}
\end{figure}

\begin{figure}[p]
    \centering

    \begin{subfigure}[t]{\linewidth}
        \centering
        \includegraphics[width=0.98\linewidth]
        {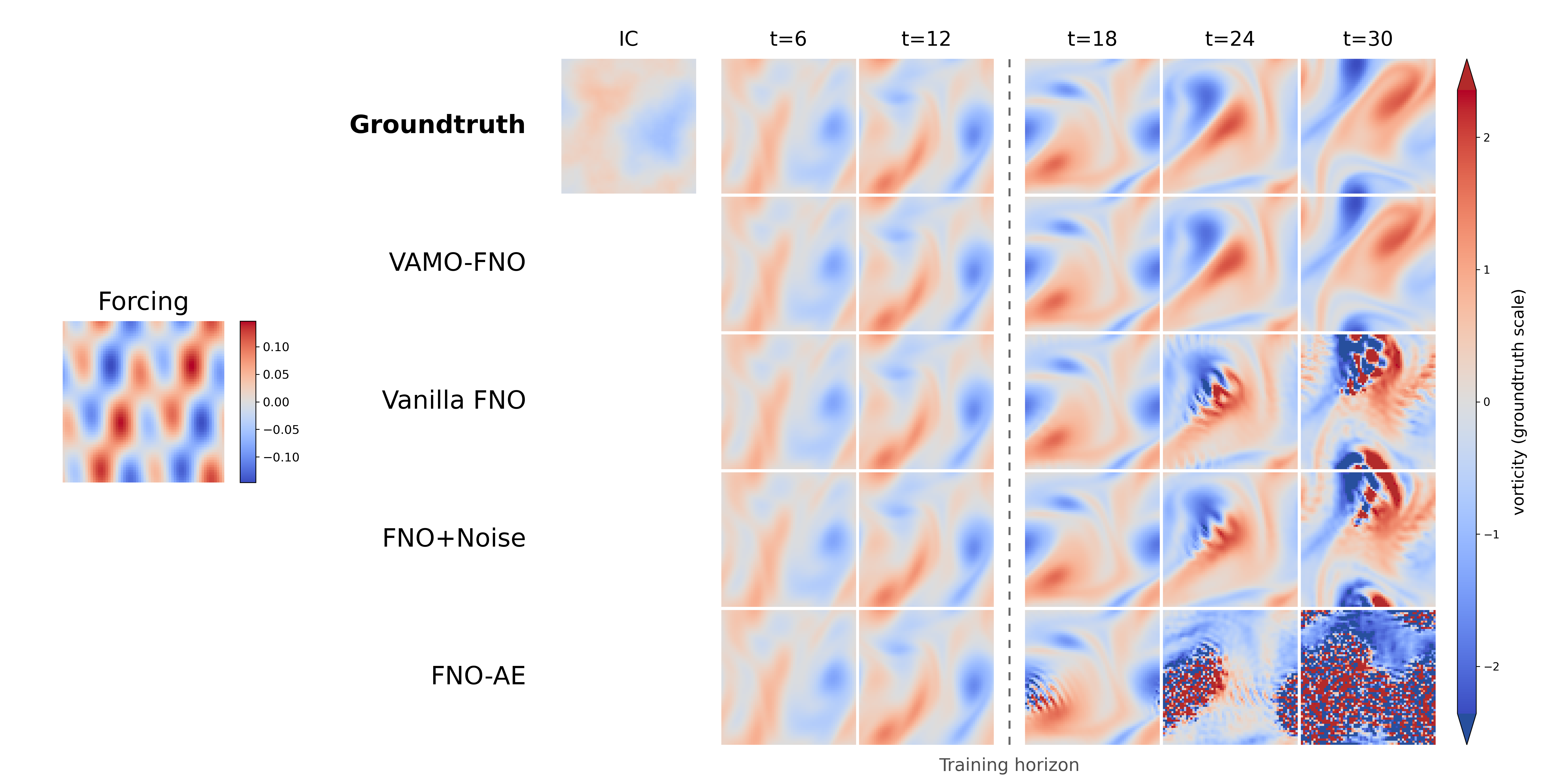}
        \caption{$\nu=10^{-4}$, \textsc{Wave} forcing, Sample 26.}
        \label{fig:ns2d_nu4_grf_trajectory_mb}
    \end{subfigure}

    \vspace{0.8em}
    
    \begin{subfigure}[t]{\linewidth}
        \centering
        \includegraphics[width=\linewidth]
        {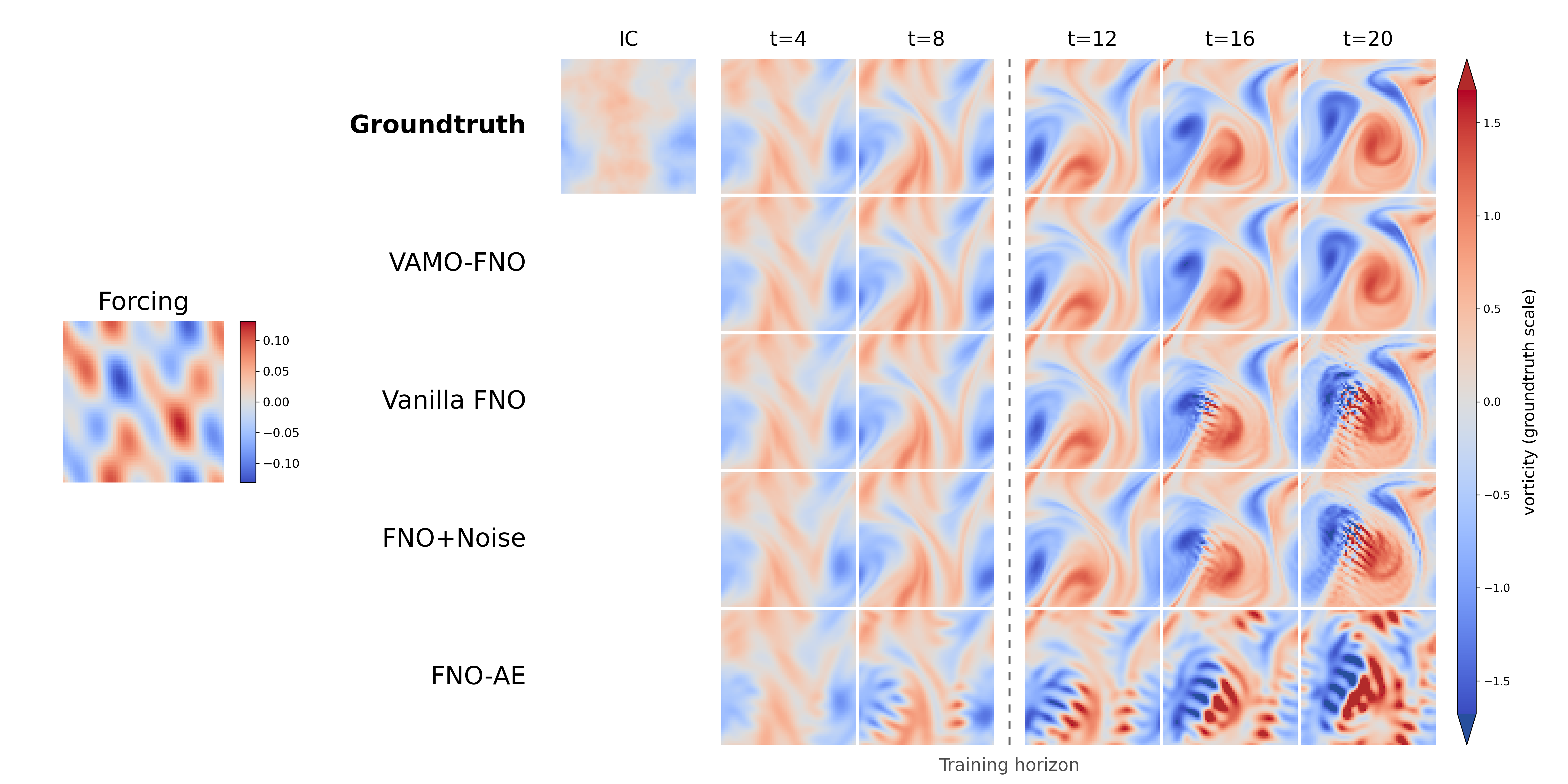}
        \caption{$\nu=10^{-5}$, \textsc{Wave} forcing, Sample 7.}
        \label{fig:ns2d_nu5_wave_trajectory_mb}
    \end{subfigure}
    \caption{
    Representative autoregressive rollouts for the 2D incompressible Navier--Stokes equations with $\nu\in\{10^{-4},10^{-5}\}$. Each panel shows the initial condition, the trajectory-specific forcing, the ground-truth vorticity, and predictions from the four evaluated methods at selected times. All predicted snapshots use the corresponding ground-truth color scale. 
    }
    \label{fig:ns2d_qualitative}
\end{figure}

\subsection{Physical-Statistic Stability}
\label{sec:physical_diagnostics}
We further examine whether the predicted Navier-Stokes trajectories preserve physically relevant flow statistics. Figure~\ref{fig:ns2d_physical_diagnostics} compares the enstrophy and palinstrophy of representative ground-truth and predicted trajectories for the two more challenging regimes, $\nu=10^{-4}$ and $\nu=10^{-5}$, under both forcing families; results for $\nu=10^{-3}$ are deferred to Appendix~\S\ref{app:additional_quantitative}. High enstrophy corresponds to large overall vorticity magnitude, while high palinstrophy reflects strong vorticity gradients and therefore finer-scale spatial structure.

Across the four challenging settings, the baseline failures are strongly trajectory-dependent. The largest deviations and occasional explosions occur primarily on trajectories with high ground-truth enstrophy or palinstrophy. In these cases, small prediction errors can generate excessive vorticity magnitude or spurious small-scale structure, which is reflected by increasingly inflated enstrophy and, more prominently, palinstrophy. FNO-AE is the least stable, with predicted statistics growing by several orders of magnitude in some trajectories. FNO and FNO+Noise are more stable overall, but still tend to overestimate enstrophy and especially palinstrophy on the more demanding trajectories. In contrast, FNO-VAMO tracks the scale and temporal evolution of these statistics substantially more closely and avoids the pronounced late-time inflation observed in the baselines.

The qualitative rollouts in Figure~\ref{fig:ns2d_qualitative} provide a complementary view of the same late-time instability. The baseline predictions remain close to the reference during the early rollout but progressively develop oscillatory and small-scale artifacts beyond the training horizon, with the deterministic latent model deteriorating most severely. FNO-VAMO instead preserves the dominant vorticity structures over the long rollout. Together with the aggregate errors in Table~\ref{tab:ns2d_main_results}, these diagnostics show that the advantage of VAMO extends beyond lower field error to improved control of spurious small-scale growth during autoregressive prediction.

\subsection{Ablation Studies}
To isolate the contributions of the main variational components, we conduct ablations on the two more challenging Navier--Stokes regimes, $\nu=10^{-4}$ and $\nu=10^{-5}$. We consider three variants of our VAMO framework: (i) setting the KL weight to $\beta=0$ while retaining both stochastic components, (ii) removing encoder noise, and (iii) removing transition noise.

In each case, the remaining architecture and training settings are kept unchanged. As in the main experiments, a single model is trained jointly on GRF- and \textsc{Wave}-forced trajectories for each viscosity and evaluated separately on the two forcing families. The aggregate results are reported in Table~\ref{tab:ns2d_variational_ablation}, while Figure~\ref{fig:ns2d_abl_trend} shows how the corresponding $L^2$ and $H^1$ errors evolve throughout the rollout. The temporal trends reinforce the aggregate comparison: removing the KL term leads to rapid error growth, removing encoder noise produces a clear but less severe degradation, and removing transition noise remains close to the full model.

\begin{figure}[p]
    \centering
    \begin{subfigure}{\linewidth}
        \centering
        \includegraphics[width=1.0\linewidth]{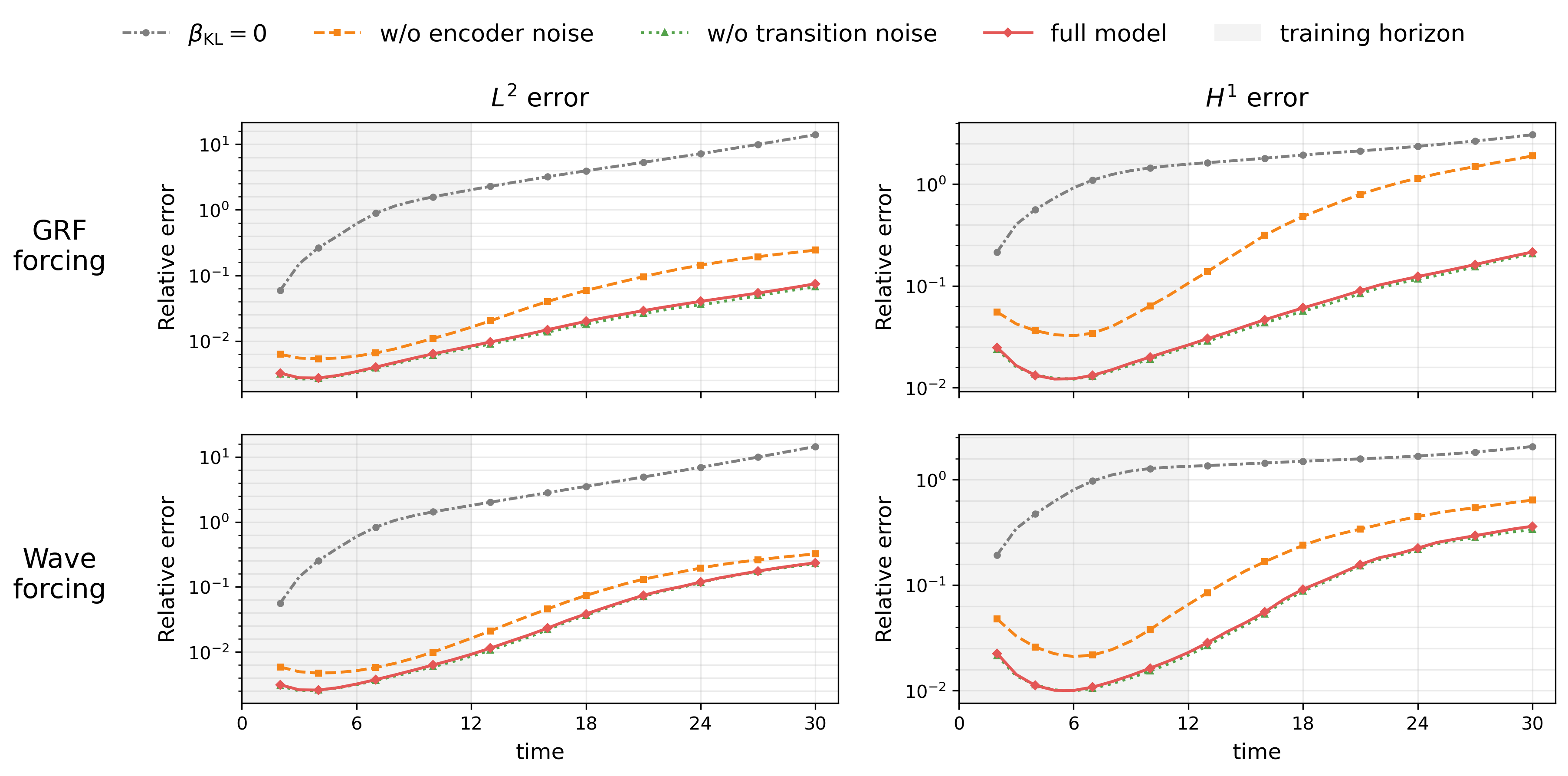}
        \caption{Viscosity $\nu=10^{-4}$.}
        \label{fig:nu4_abl}
    \end{subfigure}

    \vspace{0.5em}
    
    \begin{subfigure}{\linewidth}
        \centering
        \includegraphics[width=1.0\linewidth]{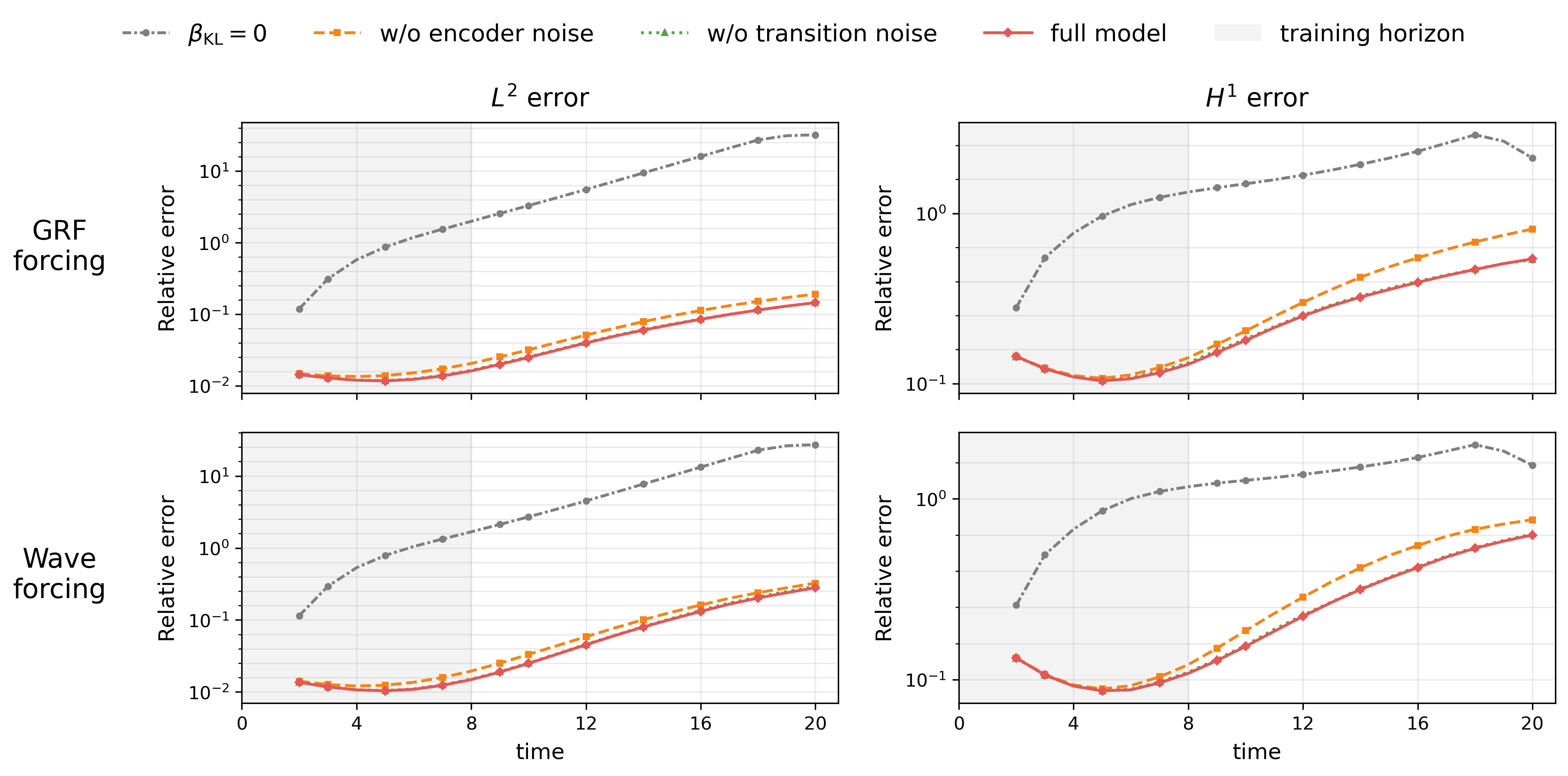}
        \caption{Viscosity $\nu=10^{-5}$.}
        \label{fig:nu5_abl}
    \end{subfigure}
    \caption{
Per-snapshot ablation results on the two-dimensional incompressible Navier-Stokes benchmark for (a) $\nu=10^{-4}$ and (b) $\nu=10^{-5}$. We report relative $L^2$ and $H^1$ errors under both GRF and \textsc{Wave} forcing. The shaded region denotes the training horizon. Removing the KL term causes rapid error growth, removing encoder noise leads to a clear but less severe degradation,
while removing transition noise remains close to the full VAMO model throughout the rollout.
}
    \label{fig:ns2d_abl_trend}
\end{figure}

\begin{table}[t]
\centering
\caption{
Ablation of the latent KL-consistency term on the two-dimensional
incompressible Navier--Stokes benchmark. We compare full VAMO with a variant that sets the KL weight to $\beta=0$ while retaining the same architecture, stochastic encoder and transition, and all other training settings. Results are reported for $\nu\in\{10^{-4},10^{-5}\}$ under GRF and \textsc{Wave} forcing. Lower is better.
}
\label{tab:ns2d_variational_ablation}
\resizebox{\textwidth}{!}{
\begin{tabular}{lllccccc}
\toprule
Viscosity
& Forcing
& Method
& Rel.\,$L^2$
& Rel.\,$H^1$
& $\mathcal{E}$-error
& $\mathcal{P}$-error
& Spec.\ error \\
\midrule

\multirow{8}{*}{$\nu=10^{-4}$}
& \multirow{4}{*}{GRF}
& VAMO ($\beta=0$)
& 7.017 & 2.087 & 92.33 & 3.727 & 4.185 \\
& & VAMO w/o encoder noise
& 0.1373 & 1.018 & 0.0472 & 3.755 & 0.0257 \\
& & VAMO w/o transition noise
& $\mathbf{0.0366}$ & $\mathbf{0.1086}$ & $\mathbf{0.00977}$ & ${0.0573}$ & $\mathbf{0.0102}$ \\
& & Full VAMO
& 0.0403 & 0.1139 & ${0.0139}$ & $\mathbf{0.0512}$ & ${0.0146}$ \\
\cmidrule(lr){2-8}

& \multirow{4}{*}{\textsc{Wave}}
& VAMO ($\beta=0$)
& 6.381 & 1.548 & 82.16 & 1.425 & 7.769 \\
& & VAMO w/o encoder noise
& 0.1827 & 0.4276 & 0.0390 & 0.3667 & 0.0492 \\
& & VAMO w/o transition noise &
$\mathbf{0.1181}$ & $\mathbf{0.2128}$ & $\mathbf{0.0349}$ & $0.1026$ & $\mathbf{0.0303}$ \\
& & Full VAMO
& ${0.1198}$ & ${0.2212}$ & ${0.0386}$ & $\mathbf{0.1008}$ & ${0.0390}$ \\
\midrule

\multirow{8}{*}{$\nu=10^{-5}$}
& \multirow{4}{*}{GRF}
& VAMO ($\beta=0$)
& 17.54 & 2.098 & 614.8 & 3.145 & 28.56 \\
& & VAMO w/o encoder noise
& 0.1105 & 0.5422 & 0.0241 & 0.4616 & 0.0164 \\
& & VAMO w/o transition noise
& 0.0844 & 0.3870 & 0.0224 & $\mathbf{0.2371}$ & 0.0103 \\
& & Full VAMO
& $\mathbf{0.0841}$ & $\mathbf{0.3858}$ & $\mathbf{0.0211}$ & ${0.2415}$ & $\mathbf{0.00929}$ \\
\cmidrule(lr){2-8}

& \multirow{4}{*}{\textsc{Wave}}
& VAMO ($\beta=0$)
& 14.55 & 1.573 & 424.7 & 1.329 & 41.27\\
& & VAMO w/o encoder noise
& 0.1743 & 0.5676 & $\mathbf{0.0405}$ & 0.3532 & 0.0268 \\
& & VAMO w/o transition noise
& 0.1523 & 0.4555 & 0.0488 & 0.2820 & 0.0271 \\
& & Full VAMO
& $\mathbf{0.1469}$ & $\mathbf{0.4509}$ & 0.0460 & $\mathbf{0.2734}$ & $\mathbf{0.0247}$ \\
\bottomrule
\end{tabular}
}
\end{table}

The empirical effects of these ablations can be interpreted through the population rollout analysis in Proposition~\ref{prop:population_kl_rollout}. In particular, the refined bound \eqref{eq:population_refined_rollout} gives
\begin{equation}
\begin{aligned}
\left(
\mathbb E
\left[
\|\widehat{\mathbf U}_n-\mathbf U_n^\star\|_{\mathcal U}^2
\right]
\right)^{1/2}
&\leq
\Lambda_D
\sum_{j=0}^{n-2}
\Lambda_T^{\,n-1-j}
\left(
\eta_j
+
\bar\alpha\sqrt{2\|K\|_{\mathrm{op}}}\,\kappa_j
\right)
+
\Lambda_D\eta_{n-1}
\\
&\qquad
+
\min\left\{
\Lambda_D\bar\alpha\sqrt{2\|K\|_{\mathrm{op}}}\,
\kappa_{n-1}+\delta_n,
\;
\epsilon_{n-1}+\rho_{n-1}
\right\},
\end{aligned}
\label{eq:ablation_rollout_recall}
\end{equation}
where $\kappa_j$ measures the latent KL mismatch, $\eta_j$ measures the sensitivity of the decoded transition to encoder-side perturbations, and $\rho_j$ measures sensitivity to transition-side perturbations. The bound highlights an important asymmetry: $\kappa_j$ and $\eta_j$ enter throughout the recursively accumulated rollout history and are weighted by the geometric factors $\Lambda_T^{\,n-1-j}$, whereas $\rho_{n-1}$ appears only in the terminal predictive refinement. This provides a useful lens for interpreting the markedly different effects of the three ablations:
\begin{itemize}[itemsep=0pt]
\item \textbf{KL consistency.}
Removing the KL term produces by far the largest degradation in Table~\ref{tab:ns2d_variational_ablation}. Across both viscosities and forcing families, the $\beta=0$ model becomes dramatically less accurate in field-level and physical metrics, with relative $L^2$, enstrophy, and spectral errors increasing by orders of magnitude. This behavior is consistent with the role of $\kappa_j$ in \eqref{eq:ablation_rollout_recall}: the KL-controlled latent mismatch enters at every previous rollout step and is accumulated with the geometric factors $\Lambda_T^{\,n-1-j}$. Thus, without explicit alignment between the predicted latent transition and the encoded next-state distribution, latent transition errors can be repeatedly amplified throughout autoregressive evolution.

\item\textbf{Encoder noise.}
Removing encoder noise also consistently worsens performance, although less severely than removing the KL term. The degradation is particularly visible in derivative-sensitive and physical-statistic errors. This observation agrees with the role of the encoder-neighborhood sensitivity $\eta_j$, which, like $\kappa_j$, appears throughout the geometrically weighted rollout sum. Encoder-side perturbations expose the learned transition to a neighborhood of the encoded trajectory during training, thereby encouraging reduced sensitivity to deviations from the nominal latent state. The ablation suggests that this neighborhood regularization is important precisely because such input-side sensitivity can be propagated and amplified over many subsequent rollout steps.

\item\textbf{Transition noise.}
In contrast, removing transition noise affects the performance of VAMO only marginally. The ablated model remains close to full VAMO and is slightly better on several metrics for $\nu=10^{-4}$, while full VAMO is generally competitive or marginally better for $\nu=10^{-5}$. This observation is consistent with the refined rollout bound: unlike $\kappa_j$ and $\eta_j$, the transition-noise sensitivity $\rho_{n-1}$ appears only in the terminal predictive refinement and does not enter the geometrically amplified rollout history. Thus, transition noise appears to act primarily as a secondary local regularizer whose contribution is more dependent on the particular dynamical regime.
\end{itemize}

Taken together, the ablations reveal a clear empirical hierarchy among the variational components. Latent KL consistency is essential for stable evolution, encoder-side neighborhood regularization provides a second substantial source of robustness, and transition noise has a smaller, problem-dependent effect. This hierarchy closely mirrors the structure of the rollout analysis: quantities governing latent consistency and input-side sensitivity enter the recursively amplified part of the error bound, whereas transition-output sensitivity enters only the terminal predictive refinement. While the ablations modify the perturbation distributions and therefore should not be viewed as direct consequences of the bound, their behavior provides empirical support for the mechanism identified by the theory.

\section{Conclusion and Outlook}
\label{sec:conclusion}

In this work, we developed a variational framework for learning time-dependent PDE dynamics with an emphasis on long-horizon autoregressive prediction. The proposed formulation represents physical states through latent distributions and models their evolution with a probabilistic Markov transition, while retaining deterministic mean dynamics at inference. At the theoretical level, we formulated the model directly on function spaces, characterized the geometry induced by functional Gaussian distributions, and connected stochastic latent perturbations and variational transition alignment to the propagation of deterministic rollout errors. We instantiated this framework as the Variational Autoencoding Markov Operator (VAMO), which combines spatially resolved latent representations, structured Gaussian perturbations, and a neural-operator transition. Across the compressible Euler, compressible-flow, and incompressible Navier--Stokes benchmarks, VAMO consistently exhibited slower error accumulation and improved long-horizon stability compared with direct, deterministic-latent, and generic noise-injection baselines.

\paragraph{Limitations.} This work also leaves several limitations. Our theoretical analysis is intended primarily to expose the mechanisms through which the variational objective enters autoregressive error propagation; the resulting bounds rely on regularity assumptions and are not expected to be quantitatively tight for general nonlinear PDE dynamics. For our empirical study, the current realization uses an FNO-based latent transition and spectrally structured Gaussian perturbations on regular grids, and our experiments focus on fluid-dynamics systems with a fixed local prediction interval. The results therefore do not establish that the same architecture or stochastic parameterization is universally preferable across other PDE classes, geometries, discretizations, or temporal training regimes.

\paragraph{Future directions.} Several extensions of this work are natural. First, the present objective is formulated from local transition pairs. A multi-step variational objective could instead be derived by introducing latent distributions over several consecutive transitions, providing a principled basis for combining variational dynamics with multi-step or curriculum-based training. Such a formulation may further reduce the mismatch between training and long autoregressive deployment. Second, VAMO is not tied to the FNO transition used in our experiments: the latent evolution operator could be replaced by geometry-aware or attention-based architectures such as Geo-FNO, OFormer, or PiT, allowing the variational formulation to be studied together with spatial-resolution transfer, irregular geometries, and more flexible discretizations. More broadly, the probabilistic latent dynamics developed here provide a starting point for studying how distributional modeling and deterministic long-horizon prediction can be combined in neural PDE solvers.

More broadly, the variational latent-dynamics framework developed here provides a starting point for studying how probabilistic modeling can support reliable deterministic evolution in neural PDE solvers and operator learning.

\bibliographystyle{ims}
\bibliography{reference}

\newpage
\appendix
\section{Gaussian Measures and Cameron-Martin Spaces}
\subsection{Gaussian Measures on Banach Spaces}
\label{app:gaussian_measures_banach}
We briefly review several basic notions for Gaussian measures on Banach spaces \citep{bogachev1998gaussian,kuo2006gaussian}. Let $\mathcal X$ be a real separable Banach space equipped with its Borel $\sigma$-algebra, and let $\mathcal X^*$ denote its dual space. A Borel probability measure $\mu$ on $\mathcal X$ is called a \textit{Gaussian measure} if, for every continuous linear functional $f\in \mathcal X^*$, the pushforward
\[
    f_\sharp \mu=\mu\circ f^{-1}
\]
is a Gaussian probability measure on $\mathbb R$, possibly degenerate. Equivalently, if a random element $u\sim \mu$, then $f(u)$ is a real-valued Gaussian random variable for every $f\in\mathcal X^*$.

The \textit{mean} of $\mu$ is defined weakly through its action on linear functionals. For each $f\in\mathcal X^*$, one writes
\[
    m_\mu(f)
    :=
    \int_{\mathcal X} f(u)\,\d\mu(u).
\]
This defines a linear functional on $\mathcal X^*$, hence an element of the bidual $\mathcal X^{**}$. If there exists an element $m\in\mathcal X$ such that
\[
    f(m)
    =
    \int_{\mathcal X} f(x)\,\d\mu(x),
    \qquad \forall f\in\mathcal X^*,
\]
then $m$ is called the \textit{mean element} of $\mu$. In this case we write $m = \mathbb E_{X\sim\mu}[X]$ in the weak sense. For Gaussian measures on separable Banach spaces, the mean can typically be identified with an element of the underlying Banach space $\mathcal X$.

The \textit{covariance} of $\mu$ is also naturally defined through dual pairings. Assuming that $\mu$ has mean $m\in\mathcal X$, its \textit{covariance bilinear form} is
\[
    C_\mu(f,g)
    :=
    \int_{\mathcal X}
    \bigl(f(u)-f(m)\bigr)
    \bigl(g(u)-g(m)\bigr)
    \,\d\mu(u),
    \qquad f,g\in\mathcal X^*.
\]
Equivalently, $C_\mu(f,g)=\operatorname{Cov}_{X\sim\mu}(f(X),g(X))$.
Thus $C_\mu:\mathcal X^*\times \mathcal X^*\to\mathbb R$ is a positive semidefinite bilinear form. In particular,
\[
    C_\mu(f,f)
    =
    \operatorname{Var}_\mu(f(u))
    \ge 0.
\]
The covariance bilinear form induces a \textit{covariance operator}. Abstractly, one may define a mapping $K_\mu:\mathcal X^*\to \mathcal X^{**}$
by
\[
    (K_\mu f)(g)
    :=
    C_\mu(f,g),
    \qquad f,g\in\mathcal X^*.
\]
When $K_\mu f$ can be identified with an element of $\mathcal X$, we regard the covariance operator as a mapping $K_\mu:\mathcal X^*\to \mathcal X$
satisfying
\[
    g(K_\mu f)
    =
    C_\mu(f,g),
    \qquad f,g\in\mathcal X^*.
\]
This is the natural Banach-space analogue of the finite-dimensional covariance matrix. Indeed, if $\mathcal X=\mathbb R^d$ and $\mu=\mathsf N(m,\Sigma)$, then each $f\in\mathcal X^*$ has the form $f(u)=a^\top u$ for some $a\in\mathbb R^d$, and
\[
    C_\mu(f,g)
    =
    a^\top \Sigma b,
    \qquad
    g(u)=b^\top u.
\]
In this case the covariance operator is simply multiplication by the covariance matrix $\Sigma$.

\subsection{Abstract Wiener Space and the Cameron-Martin Theorem}
\label{app:cameron_martin}

We next review the Cameron-Martin space associated with a Gaussian measure on a Banach space. A convenient framework is provided by the notion of an abstract Wiener space \citep{stroock2010probability}.

\begin{definition}[Abstract Wiener space]
\label{def:abstract_wiener_space}
Let $\mathcal H$ be a real separable Hilbert space and let $\mathcal X$ be a real separable Banach space. Suppose that $\mathcal H$ is continuously and densely embedded into $\mathcal X$ through an injective linear map
\[
    \iota:\mathcal H\hookrightarrow \mathcal X.
\]
The triple $(\mathcal H,\mathcal X,\mu)$ is called an \emph{abstract Wiener space} if $\mu$ is a centered Gaussian measure on $\mathcal X$ whose covariance structure is induced by the embedding $\iota$ in the following sense: for every $f\in\mathcal X^*$, the pushforward $f_\sharp\mu$ is a centered real Gaussian measure satisfying
\[
    \int_{\mathcal X} f(u)^2\,\d\mu(u)
    =
    \|\iota^*f\|_{\mathcal H}^2,
\]
where $\iota^*:\mathcal X^*\to\mathcal H$ is the adjoint map defined by
\[
    \langle \iota^*f,h\rangle_{\mathcal H}
    =
    f(\iota h),
    \qquad f\in\mathcal X^*,\ h\in\mathcal H.
\]
The Hilbert space $\mathcal H$ is called the \emph{Cameron-Martin space} of $\mu$.
\end{definition}
Equivalently, the covariance operator $K_\mu:\mathcal X^*\to\mathcal X$ satisfies
\[
    K_\mu=\iota\iota^*
\]
in the weak sense that
\[
    f(K_\mu g)=
    \langle \iota^* f,\iota^* g\rangle_{\mathcal H},
    \qquad f,g\in\mathcal X^*.
\]
Thus the Cameron-Martin inner product determines the covariance structure of all continuous linear observations of the Gaussian random element.

We now define the Paley-Wiener map associated with the abstract Wiener space. For any $f,g\in\mathcal X^*$, we have
\[
    \mathbb E_{X\sim\mu}\left[f(X)g(X)\right]
    =
    \langle \iota^* f,\iota^* g\rangle_{\mathcal H}.
\]
Hence the assignment $\iota^* f\mapsto f(\cdot)$
is an isometry from the subspace $\iota^*(\mathcal X^*)\subset\mathcal H$ into $L^2(\mathcal X,\mu)$. Since $\iota^*(\mathcal X^*)$ is dense in $\mathcal H$ for an abstract Wiener space, this isometry extends uniquely to all of $\mathcal H$. This extension is called the Paley-Wiener map. For $h\in\mathcal H$, we denote its image by
\[
    \langle h,\cdot\rangle^\sim
    \in L^2(\mathcal X,\mu).
\]
Thus $\langle h,u\rangle^\sim$ is a centered Gaussian random variable satisfying
\[
    \mathbb E_\mu
    \left[
        \langle h,\cdot\rangle^\sim
        \langle k,\cdot\rangle^\sim
    \right]
    =
    \langle h,k\rangle_{\mathcal H},
    \qquad h,k\in\mathcal H.
\]
In particular, for $f\in\mathcal X^*$,
\[
    \langle \iota^*f,x\rangle^\sim
    =
    f(x),
    \qquad \mu\text{-a.s.}
\]
This identity is the main point of the construction: the Paley-Wiener map extends ordinary linear observations $f(x)$ from directions of the form $\iota^*f$ to arbitrary Cameron-Martin directions $h\in\mathcal H$.

The key property of the Cameron-Martin space is that it characterizes the directions along which the Gaussian measure can be translated without changing its measure class. For $h\in\mathcal X$, define the translated measure
\[
    \mu_h(A)
    :=
    \mu(A-h),
    \qquad A\in\mathcal B(\mathcal X).
\]
Equivalently, if $u\sim\mu$, then $u+h\sim\mu_h$.

\begin{theorem}[Cameron-Martin theorem]
\label{thm:cameron_martin}
Let $(\mathcal H,\mathcal X,\mu)$ be an abstract Wiener space. For $h\in\mathcal X$, let $\mu_h$ be the translated measure defined by $\mu_h=\mu(\cdot-h)$.
Then $\mu_h\ll\mu$ if and only if $h\in\iota(\mathcal H)$. Identifying $\mathcal H$ with its image $\iota(\mathcal H)\subset\mathcal X$, if $h\in\mathcal H$, then
\[
    \frac{\d\mu_h}{\d\mu}(u)
    =
    \exp\left(
        \langle h,u\rangle^\sim
        -
        \frac12\|h\|_{\mathcal H}^2
    \right).
\]
If $h\notin\iota(\mathcal H)$, then $\mu_h\perp\mu$.
\end{theorem}

Here the notation $\langle h,u\rangle^\sim$ should not be confused with the inner product $\langle h,u\rangle_{\mathcal H}$ in the Hilbert space. In infinite dimensions, a typical sample $u\sim\mu$ does not belong to $\mathcal H$, so $\langle h,u\rangle_{\mathcal H}$ is generally not defined. The Paley-Wiener variable $\langle h,u\rangle^\sim$ is the rigorous substitute for this formal pairing.

\paragraph{Construction from a Gaussian measure.}
Conversely, the preceding abstract Wiener-space structure can be constructed from a centered Gaussian measure on a separable Banach space. Let $\mu$ be a centered Gaussian measure on $\mathcal X$ with covariance operator $K_\mu:\mathcal X^*\to\mathcal X$, and let
\[
    \mathcal G_\mu
    :=
    \overline{\mathcal X^*}^{\,L^2(\mathcal X,\mu)}
\]
denote the closure of the continuous linear functionals in $L^2(\mathcal X,\mu)$. For $f\in\mathcal G_\mu$, define $R_\mu f\in\mathcal X$ weakly by
\[
    R_\mu f
    :=
    \int_{\mathcal X}x f(x)\,\d\mu(x),
\]
meaning that
\[
    g(R_\mu f)
    =
    \int_{\mathcal X}g(x)f(x)\,\d\mu(x),
    \qquad
    \forall g\in\mathcal X^*.
\]
For $f\in\mathcal X^*$, one has
\[
    R_\mu f=K_\mu f.
\]
The corresponding Cameron--Martin space is
\[
    \mathcal H_\mu
    :=
    R_\mu(\mathcal G_\mu)
    \subset\mathcal X,
\]
equipped with the inner product induced from $L^2(\mathcal X,\mu)$:
\[
    \langle R_\mu f,R_\mu g\rangle_{\mathcal H_\mu}
    :=
    \int_{\mathcal X}f(x)g(x)\,\d\mu(x).
\]
With the canonical embedding $\iota:\mathcal H_\mu\hookrightarrow\mathcal X$, the triple $(\mathcal H_\mu,\mathcal X,\mu)$ is an abstract Wiener space of the form described above. Moreover, the inverse map
\[
    R_\mu^{-1}:
    \mathcal H_\mu
    \to
    \mathcal G_\mu
    \subset L^2(\mathcal X,\mu)
\]
coincides with the Paley-Wiener map:
\[
    \langle h,\cdot\rangle^\sim=R_\mu^{-1}h.
\]
In particular, for $h=K_\mu f$ with $f\in\mathcal X^*$,
\[
    \langle h,x\rangle^\sim
    =\langle K_\mu f,x\rangle^\sim=f(x),
    \qquad
    \mu\text{-a.s.}
\]
This identity explains why the Paley-Wiener map provides the rigorous counterpart of the formal pairing $\langle h,x\rangle_{\mathcal H_\mu}$. In infinite dimensions, a typical sample $x\sim\mu$ need not belong to $\mathcal H_\mu$, so this Hilbert-space inner product is generally not defined.

Specializing Theorem~\ref{thm:cameron_martin} to the construction above gives the following form of Cameron-Martin theorem, which is used later in our analysis.

\begin{theorem}[Cameron-Martin theorem for Gaussian measures]
\label{thm:cameron_martin2}
Let $\mu$ be a centered Gaussian measure on a real, separable Banach space $\mathcal X$, and let $\mathcal H_\mu$ be its Cameron-Martin space. For $h\in\mathcal{X}$,
the translated measure $\mu_h\ll \mu$ if and only if $h\in\mathcal{H}_\mu$. Moreover, if $h\in\mathcal H_\mu$, then
\[
    \frac{\d\mu_h}{\d\mu}(u)=\exp\left(\langle h,u\rangle^\sim-\frac12\|h\|_{\mathcal H_\mu}^2\right).
\]
If $h\notin\mathcal H_\mu$, then $\mu_h\perp \mu$.
\end{theorem}

For a non-centered Gaussian measure with mean $m\in\mathcal X$ and the same covariance structure, the Cameron-Martin space is unchanged. In that case, for $h\in\mathcal H_\mu$,
\[
    \frac{\d\mathsf N(m+h,K_\mu)}{\d\mathsf N(m,K_\mu)}(u)
    =
    \exp\left(
        \langle h,u-m\rangle^\sim
        -
        \frac12\|h\|_{\mathcal H_\mu}^2
    \right).
\]
Thus the Cameron-Martin norm is the infinite-dimensional analogue of the Mahalanobis norm. In particular, shifts outside $\mathcal H_\mu$ produce singular Gaussian measures, while shifts inside $\mathcal H_\mu$ preserve absolute continuity.

\section{Proofs and Auxiliary Results}
\label{app:proofs}
This appendix collects the proofs and auxiliary results supporting the theoretical developments of \S\ref{sec:var_latent_dyn}. The organization follows the progression of the main text: variational bounds, functional Gaussian geometry, local sensitivity under Gaussian perturbations, KL-aware deterministic rollout, and finally the high-probability stochastic extension.

\subsection{Variational Bounds and Population Alignment}
\label{app:variational_proofs}
We first provide the proofs of the abstract variational results in \S\ref{sec:abstract_vld}. We begin with the one-step variational bounds and then turn to the population decomposition of the latent dynamics term.

\begin{proof}[Proof of Proposition \ref{prop:conditional_functional_elbo}]
(a) Define the latent predictive kernel
\begin{equation*}
    K_{\theta,\phi}(\d\mathbf{Z}_{n+1},\d\mathbf{Z}_n\,|\,\mathbf{U}_n)=P_\theta(\d\mathbf{Z}_{n+1}\,|\,\mathbf{Z}_n)\,Q_\phi(\d\mathbf{Z}_n\,|\,\mathbf{U}_n).
\end{equation*}
Then
\begin{equation*}
    P_{\theta,\psi,\phi}(\cdot\,|\,\mathbf{U}_n)=\int_{\mathcal{Z}\times\mathcal{Z}} P_\psi(\cdot\,|\,\mathbf{Z}_{n+1})\,K_{\theta,\phi}(\d\mathbf{Z}_{n+1},\d\mathbf{Z}_n\,|\,\mathbf{U}_n),
\end{equation*}
and $\mathsf{W}_\mathcal{U}$-almost every $\mathbf{U}_{n+1}$, we have
\begin{equation*}
    \frac{\d P_{\theta,\psi,\phi}(\cdot\,|\,\mathbf{U}_n)}{\d\mathsf{W}_\mathcal{U}}(\mathbf{U}_{n+1})=\int_{\mathcal{Z}\times\mathcal{Z}} \frac{\d P_\psi(\cdot\,|\,\mathbf{Z}_{n+1})}{\d\mathsf{W}_\mathcal{U}}(\mathbf{U}_{n+1})\,K_{\theta,\phi}(\d\mathbf{Z}_{n+1},\d\mathbf{Z}_n\,|\,\mathbf{U}_n).
\end{equation*}
By Jensen's inequality,
\begin{align*}
    \log \frac{\d P_{\theta,\psi,\phi}(\cdot\,|\,\mathbf{U}_n)}{\d\mathsf{W}_\mathcal{U}}(\mathbf{U}_{n+1})&\geq\int_{\mathcal{Z}\times\mathcal{Z}} \log \frac{\d P_\psi(\cdot\,|\,\mathbf{Z}_{n+1})}{\d\mathsf{W}_\mathcal{U}}(\mathbf{U}_{n+1})\,K_{\theta,\phi}(\d\mathbf{Z}_{n+1},\d\mathbf{Z}_n\,|\,\mathbf{U}_n)\\
    &=\bbE_{\mathbf{Z}_{n+1},\mathbf{Z}_n\sim K_{\theta,\phi}(\cdot|\mathbf{U}_n)}\left[\log \frac{\d P_\psi(\cdot\,|\,\mathbf{Z}_{n+1})}{\d\mathsf{W}_\mathcal{U}}(\mathbf{U}_{n+1})\right].
\end{align*}
Expanding $K_{\theta,\phi}$ gives \eqref{eq:predlb}.
\item (b) For fixed $\mathbf Z_n\in\mathcal{Z}$, define the latent-conditioned predictive distribution
\begin{equation*}
    P_{\theta,\psi}(\d\mathbf U_{n+1}\mid \mathbf Z_n)
    =
    \int_{\mathcal Z}
    P_\psi(\d\mathbf U_{n+1}\mid \mathbf Z_{n+1})
    P_\theta(\d\mathbf Z_{n+1}\mid \mathbf Z_n).
\end{equation*}
Then we have the decomposition
\begin{equation}
    \frac{\d P_{\theta,\psi,\phi}(\cdot\,|\,\mathbf{U}_n)}{\d\mathsf{W}_\mathcal{U}}(\mathbf{U}_{n+1})=\int_{\mathcal{Z}}\frac{\d P_{\theta,\psi}(\cdot\,|\,\mathbf{Z}_n)}{\d\mathsf{W}_\mathcal{U}}(\d\mathbf{U}_{n+1})\,Q_\phi(\d\mathbf{Z}_n\,|\,\mathbf{U}_n),\quad\text{for $\mathsf{W}_\mathcal{U}$-a.e.}\ \mathbf{U}_{n+1}.\label{eq:vaedecomp}
\end{equation}
We choose a reference measure $\Lambda$ on $\mathcal{Z}$ such that both $Q_\phi(\cdot\,|\,\mathbf{U}_{n+1})$ and $P_\theta(\cdot\,|\,\mathbf{Z}_n)$ are absolutely continuous with respect to $\Lambda$, for instance, $\Lambda=Q_\phi(\cdot\,|\,\mathbf{U}_{n+1})/2+P_\theta(\cdot\,|\,\mathbf{Z}_n)/2$, and define the likelihood densities
$$p_\theta=\frac{\d P_\theta(\cdot\,|\,\mathbf{Z}_n)}{\d\Lambda},\quad q_\phi=\frac{\d Q_\theta(\cdot\,|\,\mathbf{U}_{n+1})}{\d\Lambda}.$$
Then we factorize $P_{\theta,\psi}$ and apply Jensen's inequality to obtain
\begin{align*}
    \log \frac{\d P_{\theta,\psi}(\cdot\,|\,\mathbf{Z}_n)}{\d\mathsf{W}_\mathcal{U}}&(\mathbf{U}_{n+1})
    =\log\int_{\mathcal Z}
    \frac{\d P_\psi(\cdot\,|\, \mathbf Z_{n+1})}{\d\mathsf{W}_\mathcal{U}}(\mathbf U_{n+1})\,
    p_\theta(\mathbf Z_{n+1}\,|\, \mathbf Z_n)\,\Lambda(\d\mathbf{Z}_{n+1})\\
    &=\log\int_{\mathcal Z}
    \frac{\d P_\psi(\cdot\,|\, \mathbf Z_{n+1})}{\d\mathsf{W}_\mathcal{U}}(\mathbf U_{n+1})\,
    \frac{p_\theta(\mathbf Z_{n+1}\,|\, \mathbf Z_n)}{q_\phi(\mathbf Z_{n+1}\,|\, \mathbf U_{n+1})}\,q_\phi(\mathbf Z_{n+1}\,|\, \mathbf U_{n+1})\Lambda(\d\mathbf{Z}_{n+1})\\
    &\geq\int_{\mathcal Z}
    \left[\log \frac{\d P_\psi(\cdot\,|\, \mathbf Z_{n+1})}{\d\mathsf{W}_\mathcal{U}}(\mathbf U_{n+1})-\log
    \frac{q_\phi(\mathbf Z_{n+1}\,|\, \mathbf U_{n+1})}{p_\theta(\mathbf Z_{n+1}\,|\, \mathbf Z_n)}\right]\! q_\phi(\mathbf Z_{n+1}\,|\,\mathbf U_{n+1})\Lambda(\d\mathbf{Z}_{n+1}),
\end{align*}
with the convention $q_\phi(\mathbf Z_{n+1}\,|\, \mathbf U_{n+1})\log q_\phi(\mathbf Z_{n+1}\,|\, \mathbf U_{n+1})=0$ outside the support of $Q_\phi(\cdot\,|\, \mathbf U_{n+1})$. For $Q_\phi(\cdot\,|\,\mathbf{U}_{n+1})\ll P_\theta(\cdot\,|\,\mathbf{Z}_n)$, the above result is equivalent to
\begin{align*}
    \log \frac{\d P_{\theta,\psi}(\cdot\,|\,\mathbf{Z}_n)}{\d\mathsf{W}_\mathcal{U}}(\mathbf{U}_{n+1})&\geq\int_{\mathcal Z}
    \left[\log \frac{\d P_\psi(\cdot\,|\, \mathbf Z_{n+1})}{\d\mathsf{W}_\mathcal{U}}(\mathbf U_{n+1})-\log
    \frac{\d Q_\phi(\cdot\,|\, \mathbf U_{n+1})}{P_\theta(\cdot\,|\, \mathbf Z_n)}(\mathbf Z_{n+1})\right] Q_\phi(\d\mathbf Z_{n+1}\,|\, \mathbf U_{n+1})\\
    &=\bbE_{\mathbf{Z}_{n+1}\sim Q_\phi(\cdot\,|\,\mathbf{U}_{n+1})}\left[\log \frac{\d P_\psi(\cdot\,|\, \mathbf Z_{n+1})}{\d\mathsf{W}_\mathcal{U}}(\mathbf U_{n+1})\right]-D_\mathrm{KL}\left(Q_\phi(\cdot\,|\,\mathbf{U}_{n+1})\,\|\,P_\theta(\cdot\,|\,\mathbf{Z}_n)\right).
\end{align*}
Again we use Jensen's inequality to \eqref{eq:vaedecomp}:
\begin{equation*}
    \log \frac{\d P_{\theta,\psi,\phi}(\cdot\,|\,\mathbf{U}_n)}{\d\mathsf{W}_\mathcal{U}}(\mathbf{U}_{n+1})\geq\bbE_{\mathbf{Z}_n\sim Q_\phi(\cdot|\mathbf{U}_n)}\left[\log \frac{\d P_{\theta,\psi}(\cdot\,|\,\mathbf{Z}_n)}{\d\mathsf{W}_\mathcal{U}}(\mathbf{U}_{n+1})\right].
\end{equation*}
Combining the two inequalities gives
\begin{align*}
    \log \frac{\d P_{\theta,\psi,\phi}(\cdot\,|\,\mathbf{U}_n)}{\d\mathsf{W}_\mathcal{U}}(\mathbf{U}_{n+1})&\geq\bbE_{\mathbf{Z}_{n+1}\sim Q_\phi(\cdot\,|\,\mathbf{U}_{n+1})}\left[
    \log\frac{\d P_\psi(\cdot\,|\, \mathbf Z_{n+1})}{\d\mathsf{W}_\mathcal{U}}(\mathbf U_{n+1})\right]\\
    &\quad -\bbE_{\mathbf{Z}_n\sim Q_\phi(\cdot|\mathbf{U}_n)}\left[D_\mathrm{KL}\left(Q_\phi(\cdot\,|\,\mathbf{U}_{n+1})\,\|\,P_\theta(\cdot\,|\,\mathbf{Z}_n)\right)\right],
\end{align*}
which proves the claim.
\end{proof}

We next lift the dynamic KL term from the one-step objective to the population level and identify the transition kernel targeted by its minimization.

\begin{proof}[Proof of Proposition \ref{prop:population_elbo_transition_consistency}]
As in the proof of Proposition~\ref{prop:conditional_functional_elbo},
we choose a common dominating measure on $\mathcal Z$ as reference and, with a slight abuse of notation, write
$Q_\phi(\mathbf Z^+\,|\, \mathbf U^+)$,
$P(\mathbf Z^+\,|\, \mathbf Z)$, and
$\Pi_\phi(\mathbf Z^+\,|\, \mathbf Z)$ for the corresponding densities with respect to the dominating measure. Expanding the KL divergence in
$\mathcal R(P)$ gives
\begin{equation*}
\begin{aligned}
\mathcal R(P)&=
\mathbb E_{(\mathbf U,\mathbf U^+)\sim\rho}
\mathbb E_{\mathbf Z\sim Q_\phi(\cdot\,|\,\mathbf U)}
\mathbb E_{\mathbf Z^+\sim Q_\phi(\cdot\,|\,\mathbf U^+)}
\left[\log
\frac{\d Q_\phi(\cdot\,|\,\mathbf U^+)
}{\d P(\cdot\,|\,\mathbf Z)}
(\mathbf Z^+)\right]
\\
&=\mathbb E_{(\mathbf Z,\mathbf Z^+,\mathbf{U}^+)\sim\Gamma_\phi}
\left[
\log \frac{\d Q_\phi(\cdot\,|\,\mathbf U^+)}{\d\Pi_\phi(\cdot\,|\,\mathbf{Z})}(\mathbf Z^+)
\right]+
\mathbb E_{(\mathbf Z,\mathbf Z^+)\sim\Gamma_\phi}
\left[
\log\frac{\d\Pi_\phi(\cdot\,|\,\mathbf{Z})}{\d P(\cdot\,|\,\mathbf Z)}(\mathbf Z^+)
\right].
\end{aligned}
\end{equation*}
The first term depends only on the encoder and the data distribution, and is therefore independent of $P$. For the second term, disintegrate $\Gamma_\phi$ as
\[
\Gamma_\phi(\d\mathbf Z,\d\mathbf Z^+,\d\mathbf{U}^+)=
\Gamma_\phi(\d\mathbf{Z},\d\mathbf{U}^+)\,Q_\phi(\d\mathbf{Z}^+|\mathbf{U}^+)
\]
Then
\begin{equation}
\begin{aligned}
&\mathbb E_{(\mathbf Z,\mathbf Z^+,\mathbf{U}^+)\sim\Gamma_\phi}
\left[
\log \frac{\d Q_\phi(\cdot\,|\,\mathbf U^+)}{\d \Pi_\phi(\cdot\,|\,\mathbf{Z})}(\mathbf Z^+)
\right]\\
&\quad=\int_{\mathcal{U}\times\mathcal{U}}\int_\mathcal{Z} D_\mathrm{KL}\left(Q_\phi(\cdot\,|\,\mathbf{U}^+)\,\|\,\Pi_\phi(\cdot\,|\,\mathbf{Z})\right) Q_\phi(\d\mathbf{Z}\,|\,\mathbf{U})\,\rho(\d\mathbf{U},\d\mathbf{U}^+).
\end{aligned}\label{eq:elbokldecompaux1}
\end{equation}
For the second term, disintegrate $\Gamma_\phi$ as
\[
\Gamma_\phi(\d\mathbf Z,\d\mathbf Z^+)=\Pi_\phi(\d\mathbf Z^+\,|\,\mathbf Z)\,\Gamma_\phi(\d\mathbf Z).
\]
Then
\begin{equation}
\begin{aligned}
\mathbb E_{(\mathbf Z,\mathbf Z^+)\sim\Gamma_\phi}
\left[
\log\frac{\d\Pi_\phi(\cdot\,|\,\mathbf{Z})}{\d P(\cdot\,|\,\mathbf Z)}(\mathbf Z^+)
\right]=\int_{\mathcal Z}
D_{\mathrm{KL}}
\left(\Pi_\phi(\cdot\,|\,\mathbf Z)\,\Vert\,P(\cdot\,|\,\mathbf Z)\right)
\Gamma_\phi(\d\mathbf Z),
\end{aligned}\label{eq:elbokldecompaux2}
\end{equation}
Combining the identities \eqref{eq:elbokldecompaux1} and \eqref{eq:elbokldecompaux2} proves \eqref{eq:elbokldecomp}.

The first term in \eqref{eq:elbokldecomp} is independent of $P$. The second term is nonnegative and equals zero if and only if $P(\cdot\,|\,\mathbf Z)=\Pi_\phi(\cdot\,|\,\mathbf Z)$ for
$\Gamma_\phi$-almost every $\mathbf Z$. Therefore $P^\star=\Pi_\phi$ is the population minimizer over all Markov kernels.
\end{proof}

We now specialize the abstract formulation to functional Gaussian kernels and their associated Cameron-Martin and spectral geometry.

\subsection{Gaussian Likelihood and Spectral Geometry}
\label{app:gaussian_geometry_proofs}
This subsection collects the arguments underlying the functional Gaussian specialization in \S\ref{sec:functional_gaussian_latent_markov}. We first establish the Cameron-Martin representation of the decoder likelihood in the physical state space. We then turn to the latent Hilbert space, where a trace-class covariance admits a spectral representation that determines both the Gaussian perturbations and the associated Cameron-Martin geometry. Finally, we specialize this general construction to the Laplacian-resolvent covariance on the torus, which yields the Sobolev interpretation used in Corollary~\ref{cor:sobolev_residual_control}.

\begin{proof}[Proof of Proposition \ref{prop:cm_decoder_likelihood}]
Since $D_\psi(\mathbf Z)\in \mathcal H_{\mathcal U}$, the
Cameron-Martin theorem [Theorem \ref{thm:cameron_martin2}] implies that the shifted Gaussian measure $\mathsf W_{\mathcal U}^{D_\psi(\mathbf Z)}$ is absolutely continuous with respect to $\mathsf W_{\mathcal U}$ and satisfies
\[
    \frac{\d\mathsf W_{\mathcal U}^{D_\psi(\mathbf Z)}}{\d\mathsf W_{\mathcal U}}(\mathbf U)=\exp\left(\langle D_\psi(\mathbf Z),\mathbf U\rangle^\sim
    -\frac12 \|D_\psi(\mathbf Z)\|_{\mathcal H_{\mathcal U}}^2
    \right).
\]
This proves \eqref{eq:cm_decoder_rn}. Taking negative logarithms gives \eqref{eq:cm_decoder_nll}. If $\mathbf U\in\mathcal H_{\mathcal U}$, then the Paley-Wiener map agrees with the Cameron-Martin inner product:
\[
    \langle D_\psi(\mathbf Z),\mathbf U\rangle^\sim
    =
    \langle D_\psi(\mathbf Z),\mathbf U\rangle_{\mathcal H_{\mathcal U}}.
\]
Hence
\[
\begin{aligned}
-\log\frac{\d P_\psi(\cdot\mid \mathbf Z)}{\d\mathsf W_{\mathcal U}}(\mathbf U)
&=\frac12
\|D_\psi(\mathbf Z)\|_{\mathcal H_{\mathcal U}}^2-
\langle D_\psi(\mathbf Z),\mathbf U\rangle_{\mathcal H_{\mathcal U}}\\
&=\frac12\|\mathbf U-D_\psi(\mathbf Z)\|_{\mathcal H_{\mathcal U}}^2-\frac12\|\mathbf U\|_{\mathcal H_{\mathcal U}}^2,
\end{aligned}
\]
which proves \eqref{eq:cm_decoder_squared_norm}.
\end{proof}

We next turn from the decoder likelihood on the physical state space to the Gaussian geometry of the latent transition. We first present a standard auxiliary result that constructs a Gaussian random element from a positive trace-class covariance and characterizes its Cameron-Martin space in spectral coordinates \citep{da2014stochastic,bogachev1998gaussian}.
\begin{lemma}[Gaussian measure induced by a trace-class covariance]
\label{lem:trace_class_gaussian_cm_kl}
Let $\mathcal Z$ be a separable Hilbert space and let $K:\mathcal Z\to\mathcal Z$ be a positive, self-adjoint, trace-class operator. Let $\{(\lambda_i,\mathbf E_i)\}_{i\ge 1}$ be an eigensystem of $K$, where $\{\mathbf E_i\}_{i\ge 1}$ is an orthonormal basis of $\mathcal Z$,
$\lambda_i>0$, and
\[
    \operatorname{Tr}(K)=\sum_{i=1}^\infty \lambda_i<\infty .
\]
Let $\{\xi_i\}_{i\ge 1}$ be i.i.d. standard Gaussian random variables. Then the series
\[
    \mathbf G
    =\sum_{i=1}^\infty \sqrt{\lambda_i}\xi_i\mathbf E_i
\]
converges in $L^2(\Omega;\mathcal Z)$ and almost surely in $\mathcal Z$. Moreover, $\mathbf G$ is a centered $\mathcal Z$-valued Gaussian random element with covariance operator $K$, i.e.,
\[
    \mathbf G\sim \mathsf N(0,K).
\]
The Cameron-Martin space of $\mathsf N(0,K)$ is
\[
    \mathcal H_K=
    \left\{\mathbf h=\sum_{i=1}^\infty h_i\mathbf E_i:
        \sum_{i=1}^\infty \frac{|h_i|^2}{\lambda_i}<\infty
    \right\},
\]
equipped with the inner product
\[
\langle \mathbf h,\mathbf g\rangle_{\mathcal H_K} =\sum_{i=1}^\infty \frac{h_i g_i}{\lambda_i},
\]
and norm
\[
    \|\mathbf h\|_{\mathcal H_K}^2=\sum_{i=1}^\infty \frac{|h_i|^2}{\lambda_i}.
\]
Equivalently,
\[
    \mathcal H_K=\operatorname{Range}(K^{1/2}),
    \qquad
    \|\mathbf h\|_{\mathcal H_K}
    =\|K^{-1/2}\mathbf h\|_{\mathcal Z},
\]
where $K^{-1/2}$ is understood on $\operatorname{Range}(K^{1/2})$.
\end{lemma}

\begin{proof}
\textit{Step I: Convergence.} For each positive integer $N$, we define
\[
    \mathbf G_N
    =
    \sum_{i=1}^N\sqrt{\lambda_i}\xi_i\mathbf E_i .
\]
For $M>N$, by orthonormality of $\{\mathbf E_i\}_{i\ge 1}$ and independence of
$\{\xi_i\}_{i\ge 1}$,
\[
    \mathbb E\|\mathbf G_M-\mathbf G_N\|_{\mathcal Z}^2
    =
    \mathbb E
    \left\|
        \sum_{i=N+1}^M\sqrt{\lambda_i}\xi_i\mathbf E_i
    \right\|_{\mathcal Z}^2
    =
    \sum_{i=N+1}^M\lambda_i .
\]
Since $\sum_{i=1}^\infty \lambda_i<\infty$, the sequence
$\{\mathbf G_N\}_{N\ge 1}$ is Cauchy in $L^2(\Omega;\mathcal Z)$. Therefore it converges in $L^2(\Omega;\mathcal Z)$ to some $\mathcal Z$-valued random element $\mathbf G$. In particular,
\[
    \mathbb E\|\mathbf G\|_{\mathcal Z}^2
    =
    \lim_{N\to\infty}\mathbb E\|\mathbf G_N\|_{\mathcal Z}^2
    =
    \sum_{i=1}^\infty \lambda_i
    =
    \operatorname{Tr}(K)
    <\infty.
\]
Moreover, $\{\mathbf G_N\}_{N\ge 1}$ is an $L^2$-bounded martingale with respect to the filtration
\[
    \mathcal F_N=\sigma(\xi_1,\ldots,\xi_N).
\]
Indeed, for $M>N$,
\[
    \mathbb E[\mathbf G_M\mid \mathcal F_N]=\mathbf G_N,
\]
and
\[
    \sup_{N\ge 1}\mathbb E\|\mathbf G_N\|_{\mathcal Z}^2
    =
    \sup_{N\ge 1}\sum_{i=1}^N\lambda_i
    \leq
    \sum_{i=1}^\infty \lambda_i
    <\infty.
\]
Therefore, by the Hilbert-space martingale convergence theorem, $\mathbf G_N$ converges almost surely and in $L^2(\Omega;\mathcal Z)$ to a $\mathcal Z$-valued random element. The $L^2$ limit is unique, so the almost sure limit is the same random element $\mathbf G$.

\item\textit{Step II: Verify the Karhunen-Loève expansion.} We next verify that the random element $\mathbf G$ has covariance operator $K$. For any $\mathbf f\in\mathcal Z$,
\[
    \langle \mathbf G,\mathbf f\rangle_{\mathcal Z}
    =
    \sum_{i=1}^\infty \sqrt{\lambda_i}\xi_i
    \langle \mathbf E_i,\mathbf f\rangle_{\mathcal Z},
\]
where the convergence holds in $L^2(\Omega)$. Hence $\langle \mathbf G,\mathbf f\rangle_{\mathcal Z}$ is a centered Gaussian random variable. Therefore $\mathbf G$ is a centered Gaussian random element. Moreover,
for any $\mathbf f,\mathbf g\in\mathcal Z$,
\[
\begin{aligned}
    \mathbb E
    \left[
        \langle \mathbf G,\mathbf f\rangle_{\mathcal Z}
        \langle \mathbf G,\mathbf g\rangle_{\mathcal Z}
    \right]
    &=
    \sum_{i=1}^\infty \lambda_i
    \langle \mathbf E_i,\mathbf f\rangle_{\mathcal Z}
    \langle \mathbf E_i,\mathbf g\rangle_{\mathcal Z}  \\
    &=
    \left\langle
        \sum_{i=1}^\infty \lambda_i
        \langle \mathbf f,\mathbf E_i\rangle_{\mathcal Z}\mathbf E_i,
        \mathbf g
    \right\rangle_{\mathcal Z}=
    \langle K\mathbf f,\mathbf g\rangle_{\mathcal Z}.
\end{aligned}
\]
Thus the covariance operator of $\mathbf G$ is $K$, and $\mathbf G\sim\mathsf N(0,K)$.

\item\textit{Step III: Identify the Cameron-Martin space.} Since $K$ is positive and self-adjoint,
\[
    K^{1/2}\mathbf f
    =
    \sum_{i=1}^\infty \sqrt{\lambda_i}
    \langle \mathbf f,\mathbf E_i\rangle_{\mathcal Z}\mathbf E_i .
\]
Therefore $\mathbf h\in \operatorname{Range}(K^{1/2})$ if and only if there
exists $\mathbf f=\sum_{i=1}^\infty f_i\mathbf E_i\in\mathcal Z$ such that
\[
    \mathbf h=K^{1/2}\mathbf f
    =
    \sum_{i=1}^\infty \sqrt{\lambda_i}f_i\mathbf E_i .
\]
Writing $h_i=\sqrt{\lambda_i}f_i$, this is equivalent to
\[
    \sum_{i=1}^\infty \frac{|h_i|^2}{\lambda_i}
    =
    \sum_{i=1}^\infty |f_i|^2
    <\infty .
\]
Hence
\[
    \mathcal H_K
    =
    \operatorname{Range}(K^{1/2})
    =
    \left\{
        \mathbf h=\sum_{i=1}^\infty h_i\mathbf E_i:
        \sum_{i=1}^\infty \frac{|h_i|^2}{\lambda_i}<\infty
    \right\}.
\]
The Cameron-Martin inner product is the pullback inner product through
$K^{1/2}$:
\[
    \langle \mathbf h,\mathbf g\rangle_{\mathcal H_K}
    =
    \langle K^{-1/2}\mathbf h,K^{-1/2}\mathbf g\rangle_{\mathcal Z}
    =
    \sum_{i=1}^\infty \frac{h_i g_i}{\lambda_i}.
\]
This proves the claim.
\end{proof}

The preceding lemma provides the two ingredients needed for
Proposition~\ref{prop:general_spectral_residual_control}: it identifies the Cameron-Martin geometry induced by $K$ and gives the Karhunen-Lo\`eve representation of a Gaussian perturbation with covariance $K$. Applying these facts to the scaled and shifted transition measure $\mathsf N(T_\theta(\mathbf Z),\alpha_\theta(\mathbf Z)^2K)$ yields the spectral form of the latent transition.

\begin{proof}[Proof of Proposition \ref{prop:general_spectral_residual_control}]
Part~(a) follows directly from Lemma~\ref{lem:trace_class_gaussian_cm_kl}. Indeed, since $K$ is positive, self-adjoint, and trace-class, the centered Gaussian measure $\mathsf W_{\mathcal Z}=\mathsf N(0,K)$ is well-defined on $\mathcal Z$, and its Cameron-Martin space is
\[
    \mathcal H_K
    =\operatorname{Range}(K^{1/2})
    =\left\{
        \mathbf R=\sum_{i=1}^\infty r_i\mathbf E_i:
        \sum_{i=1}^\infty \frac{|r_i|^2}{\lambda_i}<\infty
    \right\},
\]
with norm
\[
    \|\mathbf R\|_{\mathcal H_K}^2=\sum_{i=1}^\infty \frac{|r_i|^2}{\lambda_i}.
\]
We now prove part~(b). Fix $\mathbf Z$ and write
\[
    a=\alpha_\theta(\mathbf Z),
    \qquad
    m=T_\theta(\mathbf Z).
\]
Then
\[
    P_\theta(\cdot\mid \mathbf Z)=\mathsf N(m,a^2K).
\]
The Cameron-Martin space of $\mathsf N(0,a^2K)$ is the same set $\mathcal H_K$, but with scaled norm
\[
    \|\mathbf h\|_{\mathcal H_{a^2K}}^2
    =\frac{1}{a^2}\|\mathbf h\|_{\mathcal H_K}^2.
\]
Since $m\in\mathcal H_K$, the Cameron-Martin theorem implies that $\mathsf N(m,a^2K)$ is absolutely continuous with respect to $\mathsf N(0,a^2K)$, and for $\mathbf Z^+\in\mathcal H_K$,
\[
    \log
    \frac{\d \mathsf N(m,a^2K)}{\d \mathsf N(0,a^2K)}(\mathbf Z^+)
    =\frac{1}{a^2}\langle \mathbf Z^+,m\rangle_{\mathcal H_K}
    -\frac{1}{2a^2}\|m\|_{\mathcal H_K}^2.
\]
Therefore
\[
\begin{aligned}
    -\log\frac{\d\mathsf N(m,a^2K)}{\d \mathsf N(0,a^2K)}(\mathbf Z^+)
    &=\frac{1}{2a^2}\|m\|_{\mathcal H_K}^2-\frac{1}{a^2}
    \langle \mathbf Z^+,m\rangle_{\mathcal H_K}\\
    &=\frac{1}{2a^2}\|\mathbf Z^+-m\|_{\mathcal H_K}^2
    -\frac{1}{2a^2}\|\mathbf Z^+\|_{\mathcal H_K}^2 .
\end{aligned}
\]
Substituting back $a=\alpha_\theta(\mathbf Z)$ and
$m=T_\theta(\mathbf Z)$ yields
\[
    -\log
    \frac{\d P_\theta(\cdot\mid\mathbf Z)}{\d\mathsf N(0,\alpha_\theta(\mathbf Z)^2K)}(\mathbf Z^+)
    =\frac{1}{2\alpha_\theta(\mathbf Z)^2}
    \|\mathbf Z^+-T_\theta(\mathbf Z)\|_{\mathcal H_K}^2-
    \frac{1}{2\alpha_\theta(\mathbf Z)^2}\|\mathbf Z^+\|_{\mathcal H_K}^2.
\]
Finally, by Lemma~\ref{lem:trace_class_gaussian_cm_kl}, the random series
\[
    \sum_{i=1}^\infty \sqrt{\lambda_i}\xi_i\mathbf E_i
\]
defines a $\mathcal Z$-valued Gaussian random element with law $\mathsf N(0,K)$. Hence
\[
    \alpha_\theta(\mathbf Z)
    \sum_{i=1}^\infty \sqrt{\lambda_i}\xi_i\mathbf E_i
    \sim\mathsf N(0,\alpha_\theta(\mathbf Z)^2K).
\]
Adding the mean $T_\theta(\mathbf Z)$ gives
\[
    \mathbf Z^+
    =
    T_\theta(\mathbf Z)
    +
    \alpha_\theta(\mathbf Z)
    \sum_{i=1}^\infty \sqrt{\lambda_i}\xi_i\mathbf E_i
    \sim
    \mathsf N(T_\theta(\mathbf Z),\alpha_\theta(\mathbf Z)^2K),
\]
which is exactly the transition kernel $P_\theta(\cdot\mid\mathbf Z)$. The spectral-coordinate statement in the remark follows by taking inner
products with $\mathbf E_i$.
\end{proof}

Proposition~\ref{prop:general_spectral_residual_control} applies to an arbitrary positive, self-adjoint, trace-class covariance operator. For the spatially structured Gaussian perturbations considered in this work, a particularly useful choice is the Laplacian-resolvent family
\[
    K_{\ell,s}=(I-\ell^2\Delta)^{-s}
\]
on a periodic domain. The next lemma makes the corresponding Fourier structure explicit. In particular, it identifies the covariance spectrum, the condition under which $K_{\ell,s}$ is trace-class on $L^2$, and the resulting Cameron-Martin space. These properties will allow the abstract spectral geometry above to be interpreted as a Sobolev geometry.
\begin{lemma}[Laplacian-resolvent covariance on the torus]
\label{lem:laplacian_resolvent_torus}
Let
\[
    \mathcal Z
    =
    L^2(\mathbb T^{d_x};\mathbb R^{d_z}),
\]
and let
\[
    K_{\ell,s}
    :=
    (I-\ell^2\Delta)^{-s},
    \qquad
    \ell>0,\quad s>0,
\]
where the Laplacian acts componentwise. Let
\[
    \mathcal Z_{\mathbb C}
    =
    L^2(\mathbb T^{d_x};\mathbb C^{d_z})
\]
be the complexification of $\mathcal Z$. For
$m\in\mathbb Z^{d_x}$, define
\[
    \phi_m(x)=e^{2\pi i m\cdot x},
\]
and, for the standard basis
$\{\mathbf e_a\}_{a=1}^{d_z}$ of $\mathbb R^{d_z}$, set
\[
    \mathbf E_{m,a}(x)
    =
    \phi_m(x)\,\mathbf e_a.
\]
Then the following hold.

\begin{enumerate}
\item[(a)]
The family $\{\mathbf E_{m,j}\}_{m\in\mathbb Z^{d_x},\,1\le j\le d_z}$ is an orthonormal eigenbasis of $\mathcal Z_{\mathbb C}$ for $K_{\ell,s}$, with
\begin{equation}
    K_{\ell,s}\mathbf E_{m,j}=\lambda_m\mathbf E_{m,j},
    \qquad
    \lambda_m=\left(1+4\pi^2\ell^2|m|^2\right)^{-s}.
    \label{eq:laplacian_resolvent_eigenvalues}
\end{equation}
Equivalently, $K_{\ell,s}$ is the Fourier multiplier with symbol $\lambda_m$.

\item[(b)]
The operator $K_{\ell,s}$ is self-adjoint, positive definite, and compact. Moreover, it is trace-class if and only if
\[
    s>\frac{d_x}{2}.
\]
In this case,
\begin{equation}
    \operatorname{Tr}(K_{\ell,s})
    =
    d_z
    \sum_{m\in\mathbb Z^{d_x}}
    \left(1+4\pi^2\ell^2|m|^2\right)^{-s}.
    \label{eq:laplacian_resolvent_trace}
\end{equation}

\item[(c)]
If $s>d_x/2$, then $\mathsf N(0,K_{\ell,s})$ defines a Gaussian measure on $\mathcal Z$, and its Cameron-Martin space is
\[
    \mathcal H_{K_{\ell,s}}
    =
    \operatorname{Range}(K_{\ell,s}^{1/2}),
\]
with norm
\begin{equation}
    \|\mathbf h\|_{\mathcal H_{K_{\ell,s}}}^2
    =
    \sum_{m\in\mathbb Z^{d_x}}
    \sum_{j=1}^{d_z}
    \left(
        1+4\pi^2\ell^2|m|^2
    \right)^s
    |\widehat h_j(m)|^2.
    \label{eq:laplacian_resolvent_cm_norm}
\end{equation}
Consequently,
\[
    \mathcal H_{K_{\ell,s}}
    =
    H^s(\mathbb T^{d_x};\mathbb R^{d_z})
\]
as sets, with equivalent norms.
\end{enumerate}
\end{lemma}

\begin{proof}
We first recall the spectral decomposition of the periodic Laplacian. For each 
$m\in\mathbb Z^{d_x}$, 
\[
    \phi_m(x)=e^{2\pi i m\cdot x}
\]
solves the eigenvalue problem
\[
    -\Delta\phi_m=4\pi^2|m|^2\phi_m.
\]
Since $\{\phi_m\}_{m\in\mathbb Z^{d_x}}$ is an orthonormal basis of $L^2(\mathbb T^{d_x};\mathbb C)$, the vector-valued functions
\[
    \mathbf E_{m,j}(x)=\phi_m(x)e_j, \qquad m\in\mathbb Z^{d_x},\ j=1,\dots,d_z,
\]
form an orthonormal basis of $L^2(\mathbb T^{d_x};\mathbb C^{d_z})$, and hence also give the usual complex Fourier representation of $L^2(\mathbb T^{d_x};\mathbb R^{d_z})$ subject to the standard conjugate-symmetry constraint on real-valued fields. Because the Laplacian acts componentwise, we have
\[
    \Delta \mathbf E_{m,j}=(\Delta\phi_m)\mathbf e_j=-4\pi^2|m|^2\mathbf E_{m,j}.
\]
Therefore,
\[
    (I-\ell^2\Delta)\mathbf E_{m,j}
    =
    \left(1+4\pi^2\ell^2|m|^2\right)\mathbf E_{m,j}.
\]
Applying the functional calculus for the positive self-adjoint operator $I-\ell^2\Delta$, we obtain
\[
    K \mathbf E_{m,j}
    =(I-\ell^2\Delta)^{-s}\mathbf E_{m,j}
    =\left(1+4\pi^2\ell^2|m|^2\right)^{-s}\mathbf E_{m,j}.
\]
Thus $\mathbf{E}_{m,j}$ is an eigenfunction of $K$ with eigenvalue
\[
    \lambda_m
    =
    \left(1+4\pi^2\ell^2|m|^2\right)^{-s}.
\]
It follows immediately that for any Fourier expansion
\[
    Z(x)
    =
    \sum_{m\in\mathbb Z^{d_x}}\widehat Z(m)\phi_m(x),
    \qquad \widehat Z(m)\in\mathbb C^{d_z},
\]
the action of $K$ is given by
\[
    KZ
    =
    \sum_{m\in\mathbb Z^{d_x}}
    \lambda_m\widehat Z(m)\phi_m(x).
\]
Hence $K$ is a Fourier multiplier with multiplier 
$\lambda_m$. Since all eigenvalues are positive and converge to zero, $K_{\ell,s}$ is positive, self-adjoint, and compact. It is trace-class if and only if
\[
    \sum_{m\in\mathbb Z^{d_x}}
    (1+4\pi^2\ell^2|m|^2)^{-s}
    <\infty.
\]
For large $|m|$,
\[
    (1+4\pi^2\ell^2|m|^2)^{-s}
    \asymp |m|^{-2s},
\]
and the corresponding lattice sum converges if and only if
$2s>d_x$. Accounting for the $d_z$ latent channels gives
\eqref{eq:laplacian_resolvent_trace}.

Finally, when $s>d_x/2$, the trace-class property ensures that $\mathsf N(0,K_{\ell,s})$ is a well-defined Gaussian measure on $\mathcal Z$. By the spectral characterization of Cameron-Martin spaces [Lemma~\ref{lem:trace_class_gaussian_cm_kl}],
\[
    \mathcal H_{K_{\ell,s}}
    =
    \operatorname{Range}(K_{\ell,s}^{1/2}),
\]
and
\[
    \|\mathbf h\|_{\mathcal H_{K_{\ell,s}}}^2
    =\sum_{m\in\mathbb Z^{d_x}}
    \sum_{j=1}^{d_z}
    \lambda_m^{-1}
    |\widehat h_j(m)|^2=\sum_{m\in\mathbb Z^{d_x}}
    \sum_{j=1}^{d_z}
    \left(1+4\pi^2\ell^2|m|^2\right)^s
    |\widehat h_j(m)|^2,
\]
which proves \eqref{eq:laplacian_resolvent_cm_norm}. Recall the usual Fourier definition of the Sobolev norm on the torus:
\[
    \|\mathbf h\|_{H^s(\mathbb T^{d_x};\mathbb R^{d_z})}^2
    =\sum_{m\in\mathbb Z^{d_x}}\sum_{j=1}^{d_z}
    \left(1+4\pi^2|m|^2\right)^s|\widehat h_j(m)|^2.
\]
Since $\ell>0$ is fixed,
\[
    (1+4\pi^2\ell^2|m|^2)^s\asymp(1+4\pi^2|m|^2)^s.
\]
Hence this norm is equivalent to the standard
$H^s(\mathbb T^{d_x};\mathbb R^{d_z})$ norm.
\end{proof}

Corollary~\ref{cor:sobolev_residual_control} now follows by applying this Cameron-Martin geometry to the transition residual $\mathbf Z^+-T_\theta(\mathbf Z)$ appearing in Proposition~\ref{prop:general_spectral_residual_control}.

\begin{proof}[Proof of Corollary~\ref{cor:sobolev_residual_control}]
By Lemma~\ref{lem:laplacian_resolvent_torus}, the eigenvalues of
$K=(I-\ell^2\Delta)^{-s}$ in the Fourier basis are
\[
    \lambda_m
    =
    \left(
        1+4\pi^2\ell^2|m|^2
    \right)^{-s}.
\]
Proposition~\ref{prop:general_spectral_residual_control} therefore gives
\[
\begin{aligned}
    \left\|
        \mathbf Z^+-T_\theta(\mathbf Z)
    \right\|_{\mathcal H_K}^2
    &=
    \sum_{m\in\mathbb Z^{d_x}}
    \sum_{a=1}^{d_z}
    \lambda_m^{-1}
    \left|
        \widehat Z_a^+(m)
        -
        \widehat{T_\theta(\mathbf Z)}_a(m)
    \right|^2 \\
    &=
    \sum_{m\in\mathbb Z^{d_x}}
    \sum_{a=1}^{d_z}
    \left(
        1+4\pi^2\ell^2|m|^2
    \right)^s
    \left|
        \widehat Z_a^+(m)
        -
        \widehat{T_\theta(\mathbf Z)}_a(m)
    \right|^2 .
\end{aligned}
\]
The equivalence with the standard $H^s$ norm follows from
Lemma~\ref{lem:laplacian_resolvent_torus}.
\end{proof}

\subsection{Gaussian Perturbations and Local Sensitivity}
\label{app:sensitivity_proofs}
We next establish the auxiliary estimates used in the local-sensitivity analysis of \S\ref{subsubsec:noise_injection}. We first collect Gaussian moment bounds and then use them to control the effect of Gaussian perturbations under a second-order local expansion.

\begin{lemma}[Moment estimates]\label{lem:moment_gaussian}
Let $\mathcal Z$ be a separable Hilbert space, and let $K:\mathcal Z\to\mathcal Z$ be positive, self-adjoint, and trace-class operator. Assume the random element $\mathbf G\sim\mathsf{N}(0,K)$. Then
\begin{itemize}
    \item[(a)]$\bbE\Vert\mathbf G\Vert_\mathcal{Z}^4\leq 3(\operatorname{Tr} K)^2.$
    \item[(b)]$\bbE\Vert\mathbf G\Vert_\mathcal{Z}^8\leq 105(\operatorname{Tr} K)^4.$
\end{itemize}
\end{lemma}
\begin{proof}
Let $\{(\lambda_i,\mathbf E_i)\}_{i\ge 1}$ be an eigensystem of $K$. By Karhunen-Lo\`eve expansion,
\[
    \mathbf G
    =
    \sum_{i=1}^\infty\sqrt{\lambda_i}\xi_i\mathbf E_i,
    \quad
    \xi_i\overset{\mathrm{i.i.d.}}{\sim}\mathsf N(0,1).
\]
Thus $\|\mathbf G\|_{\mathcal Z}^2=\sum_{i=1}^\infty \lambda_i\xi_i^2$, and the cumulant generating function of $\Vert\mathbf G\Vert_\mathcal{Z}^2$ is given by
\begin{align*}
    \log\bbE\left[e^{t\Vert\mathbf G\Vert_\mathcal{Z}^2}\right]&=\sum_{i=1}^\infty\log\bbE\left[e^{t\lambda_i\xi_i^2}\right]=-\frac{1}{2}\sum_{i=1}^\infty\log(1-2t\lambda_i),\\
    &=-\frac{1}{2}\sum_{i=1}^\infty\sum_{n=1}^\infty\frac{(2t\lambda_i)^n}{n}=\sum_{n=1}^\infty\frac{2^{n-1}}{n}\left(\sum_{i=1}^\infty\lambda_i^n\right) t^n,\quad t<\frac{1}{2}.
\end{align*}
Matching the coefficients in
\begin{equation*}
    \log\bbE\left[e^{t\Vert\mathbf G\Vert_\mathcal{Z}^2}\right]=\sum_{n=1}^\infty\frac{\kappa_n}{n!}t^n,
\end{equation*}
we obtain the cumulants of $\Vert\mathbf{G}\Vert_\mathcal{Z}^2$:
\begin{equation}
\kappa_n=2^{n-1}(n-1)!\sum_{i=1}^\infty\lambda_i^n=2^{n-1}(n-1)!\,\mathrm{Tr}(K^n),\quad n=1,2,\cdots.
\end{equation}
For fixed $n\in\mathbb N$, let $\Pi(n)$ denote the set of partitions of $\{1,\cdots,n\}$. We use the moment-cumulant formula for $\Vert\mathbf{G}\Vert_\mathcal{Z}^2$:
\begin{equation*}
    \bbE\left\Vert\mathbf G\right\Vert_\mathcal{Z}^{2n}=\sum_{\pi\in\Pi(n)}\prod_{V\in\pi}\kappa_{\vert V\vert}.
\end{equation*}
Then
\begin{equation*}
    \bbE\left\Vert\mathbf G\right\Vert_\mathcal{Z}^4=\kappa_2+\kappa_1^2=2\operatorname{Tr}(K^2)+(\operatorname{Tr} K)^2.
\end{equation*}
and
\begin{align*}
    \bbE\left\Vert\mathbf G\right\Vert_\mathcal{Z}^8&=\kappa_4+4\kappa_1\kappa_3+3\kappa_2^2+6\kappa_1^2\kappa_2+\kappa_1^4\\
    &=48\operatorname{Tr}(K^4)+32\operatorname{Tr}(K)\operatorname{Tr}(K^3)+12\operatorname{Tr}(K^2)^2+12(\operatorname{Tr} K)^2\operatorname{Tr}(K^2)+\operatorname{Tr}(K)^4.
\end{align*}
Finally, since $K$ is positive trace-class, we have
\begin{equation*}
    \operatorname{Tr}(K^n)=\sum_{i=1}^\infty\lambda_i^n\leq\left(
    \sum_{i=1}^\infty\lambda_i\right)^n=(\operatorname{Tr} K)^n.
\end{equation*}
Combining the last three displays gives the desired estimates.
\end{proof}

These moment estimates control the higher-order remainder terms generated by Gaussian perturbations. The following lemma combines them with a local Fr\'echet expansion to obtain a general perturbation estimate.

\begin{lemma}[Local sensitivity under Gaussian perturbations]
\label{lem:local_jacobian_decoder_robustness}
Let $\mathcal Z$ and $\mathcal U$ be separable Hilbert spaces, and let $K:\mathcal Z\to\mathcal Z$ be positive, self-adjoint, and trace-class operator.
Let $\mathbf G\sim \mathsf N(0,K)$. Fix $\mathbf z\in\mathcal Z$ and $a>0$. Suppose that $F:\mathcal Z\to\mathcal U$ is Fr\'echet differentiable at $\mathbf z$, with Fr\'echet derivative
\[
    J_{F}(\mathbf z)\in\mathfrak{B}(\mathcal Z,\mathcal U),
\]
and assume that, on the relevant neighborhood of $\mathbf z$, the second-order remainder satisfies
\[
    \left\|F(\mathbf z+\mathbf h)-F(\mathbf z)-J_{F}(\mathbf z)\mathbf h\right\|_{\mathcal U}
    \le\frac{M}{2}\|\mathbf h\|_{\mathcal Z}^2 .
\]
Then
\[
    \mathbb E\left[\left\|F(\mathbf z+a\mathbf G)-F(\mathbf z)\right\|_{\mathcal U}^2
    \right]\le a^2\Vert J\Vert_\mathrm{op}^2\operatorname{Tr} K+\frac{3}{4}M^2a^4(\operatorname{Tr} K)^2.
\]
\end{lemma}

\begin{proof}
For brevity, write $J=J_{F}(\mathbf z)$. By the assumed second-order expansion, for every $\mathbf h$ in the relevant neighborhood,
\[
    F(\mathbf z+\mathbf h)-F(\mathbf z)
    =J\mathbf h+R(\mathbf h),
\]
where
\[
    \|R(\mathbf h)\|_{\mathcal U}\le\frac{M}{2}\|\mathbf h\|_{\mathcal Z}^2.
\]
Taking $\mathbf h=a\mathbf G$, we have
\[
    F(\mathbf z+a\mathbf G)-F(\mathbf z) = aJ\mathbf G+R(a\mathbf G).
\]
Therefore,
\[
\begin{aligned}
    \mathbb E\left[\|F(\mathbf z+a\mathbf G)-F(\mathbf z)\|_{\mathcal U}^2\right]=\mathbb E
    \left[\|aJ\mathbf G+R(a\mathbf G)\|_{\mathcal U}^2
    \right]\le 2a^2\mathbb E\|J\mathbf G\|_{\mathcal U}^2
    +2\mathbb E\|R(a\mathbf G)\|_{\mathcal U}^2.
\end{aligned}
\]
We first compute the linear term. Since $\mathbf G\sim \mathsf N(0,K)$ and $J$ is bounded linear, $J\mathbf G$ is a Gaussian random element in $\mathcal U$ with covariance operator $J K J^*$.
Hence
\[
    \mathbb E\|J\mathbf G\|_{\mathcal U}^2
    =\Vert J\Vert_\mathrm{op}^2\mathbb E\|\mathbf G\|_{\mathcal U}^2\leq\Vert J\Vert_\mathrm{op}^2\operatorname{Tr}(K).
\]
We next bound the remainder term. From the second-order remainder assumption,
\[
    \|R(a\mathbf G)\|_{\mathcal U}\le\frac{M}{2}a^2\|\mathbf G\|_{\mathcal Z}^2.
\]
Therefore
\[
\mathbb E\|R(a\mathbf G)\|_{\mathcal U}^2\le\frac{M^2}{4}a^4\mathbb E\|\mathbf G\|_{\mathcal Z}^4.
\]
Combining the remainder estimate and the moment estimate in Lemma \ref{lem:moment_gaussian} yields
\[
    \mathbb E\left[\left\|F(\mathbf z+a\mathbf G)-F(\mathbf z)\right\|_{\mathcal U}^2
    \right] \le
    a^2\Vert J\Vert_\mathrm{op}^2\operatorname{Tr} (K)+\frac{3}{4}M^2a^4(\operatorname{Tr} K)^2.
\]
which is the desired result.
\end{proof}

We now apply this perturbation estimate to the latent transition and decoder to prove Proposition~\ref{prop:gaussian_sensitivity_bounds}.

\begin{proof}[Proof of Proposition \ref{prop:gaussian_sensitivity_bounds}]
We first prove the Lipschitz robustness estimates. By the Lipschitz continuity of $T_\theta$ on the relevant regions,\[
    \|T_\theta(\mathbf Z_n^\star+\Xi_n)
      -T_\theta(\mathbf Z_n^\star)\|_{\mathcal Z}
    \leq
    \Lambda_T\|\Xi_n\|_{\mathcal Z}.
\]
Conditionally on $\mathbf U_n^\star$, one samples  $\Xi_n\sim\mathsf N\bigl(0,K_\phi(\mathbf U_n^\star)\bigr)$.
Then we have
\[
    \mathbb E
    \left[
        \|\Xi_n\|_{\mathcal Z}^2
        \,\middle|\,
        \mathbf U_n^\star
    \right]
    =
    \operatorname{Tr}
    K_\phi(\mathbf U_n^\star).
\]
Taking expectation over the reference trajectory therefore gives
\[
    \eta_n^2
    \leq\Lambda_T^2
    \mathbb E\left[\operatorname{Tr}
    K_\phi(\mathbf U_n^\star)\right]
    \leq\Lambda_T^2\varsigma.
\]
Similarly,
\[
\|D_\psi(\widetilde{\mathbf Z}_n^+) -D_\psi(T_\theta(\widetilde{\mathbf Z}_n))\|_{\mathcal U}\leq\Lambda_D\alpha_\theta(\widetilde{\mathbf Z}_n)\|\mathbf G_n\|_{\mathcal Z}.
\]
Using $\alpha_\theta\leq\bar\alpha$ and
$\mathbb E\|\mathbf G_n\|_{\mathcal Z}^2=\operatorname{Tr}(K)$ yields
\[
    \rho_n
    \leq
    \Lambda_D\bar\alpha\sqrt{\operatorname{Tr}(K)}.
\]
This proves \eqref{eq:lipschitz_sensitivity_bounds}.

Now we derive the local estimate for $\rho_n$. By Lemma \ref{lem:local_jacobian_decoder_robustness},
\begin{align*}
    &\bbE_{\mathbf{G}_n\sim\mathsf N(0,K)}\left\Vert D_\psi\bigl(T_\theta(\wt{\mathbf Z}_n)+\alpha_\theta(\wt{\mathbf Z}_n)\mathbf G_n\bigr)-D_\psi(T_\theta(\wt{\mathbf Z}_n))\right\Vert_\mathcal{U}^2\\
    &\quad\leq\alpha_\theta(\wt{\mathbf Z}_n)^2\bigl\Vert J_{D_\psi}(T_\theta(\wt{\mathbf Z}_n))\bigr\Vert_\mathrm{op}^2\operatorname{Tr} K +\frac{3}{4}M_D^2\alpha_\theta(\wt{\mathbf Z}_n)^4(\operatorname{Tr} K)^2.
\end{align*}
Taking expectation over the reference trajectory and Gaussian perturbation $\Xi_n$, and using the fact $\alpha_\theta(\wt{\mathbf Z}_n)\leq\ol\alpha$, we have
\begin{equation*}
    \rho_n\leq\ol\alpha\left(\bbE\left[\Vert J_{D_\psi}(T_\theta(\wt{\mathbf Z}_n))\Vert_\mathrm{op}^2\right]\right)^{1/2}\sqrt{\operatorname{Tr} K}+\frac{\sqrt{3}M_D\ol\alpha^2}{2}\operatorname{Tr} K,
\end{equation*}
which is the desired bound. Similarly, applying Lemma \ref{lem:local_jacobian_decoder_robustness} conditionally on $\mathbf U_n^\star$, with $a = 1$, and $K=K_\phi(\mathbf U_n^\star)$, and using $\operatorname{Tr} K_\phi(\mathbf U_n^\star)\leq\varsigma$ pointwise before taking any expectation,
\[
    \mathbb E\left[\left\Vert T_\theta(\wt{\mathbf Z}_n)-T_\theta(\mathbf Z_n^\star)\right\Vert_{\mathcal Z}^2\,\middle|\,\mathbf U_n^\star\right]
    \leq
    \left\Vert J_{T_\theta}(\mathbf Z_n^\star)\right\Vert_\mathrm{op}^2\varsigma+\frac{3}{4}M_T^2\varsigma^2.
\]
Taking expectation over the reference trajectory and applying $\sqrt{x+y}\leq\sqrt{x}+\sqrt{y}$ for any $x,y\geq0$ gives
\begin{equation*}
    \eta_n\leq\left(\bbE\left[\Vert J_{T_\theta}(\mathbf Z^\star_n)\Vert_\mathrm{op}^2\right]\varsigma\right)^{1/2}+\frac{\sqrt{3}M_T}{2}\varsigma,
\end{equation*}
which is \eqref{eq:tau_jacobian_bound}.
\end{proof}

We next turn from the local effect of Gaussian perturbations to the control provided by variational transition alignment over autoregressive rollout.

\subsection{Variational Alignment and Deterministic Rollout}
\label{app:kl-aware-rollout}

We now prove the KL-aware rollout results of \S\ref{subsubsec:variational_kl_alignment}. The key step is to convert the dynamic KL divergence into control of latent mean discrepancies. We begin with a Gaussian transportation inequality and its specialization to the Hilbert-space setting used for the latent dynamics.

We first recall a transportation-cost inequality for Gaussian measures. The quadratic transportation inequality for Gaussian measures originates from the classical work of \citet{talagrand1996transportation} and was later connected
to logarithmic Sobolev inequalities and related functional inequalities \citep{otto2000generalization}. For the infinite-dimensional setting considered here, we use the Banach-space formulation of \citet[Corollary~1.3]{riedel2017transportation}.

\begin{theorem}[Riedel's Gaussian transportation inequality]
\label{thm:riedel-gaussian-t2}
Let $(B,\mathcal H,\gamma)$ be a Gaussian Banach space, where $B$ is a separable Banach space,
$\gamma$ is a centered Gaussian measure on $B$, and $\mathcal H$ is its Cameron--Martin space.
Then, for every probability measure $\nu\in\mathcal P(B)$,
\[
    \inf_{\pi\in\Pi(\nu,\gamma)}
    \int_{B\times B}
    \|x-y\|_B^2\,\d\pi(x,y)
    \le
    2\sigma_\gamma^2
    D_{\mathrm{KL}}(\nu\,\Vert\,\gamma),
\]
where
\[
    \sigma_\gamma^2
    :=
    \sup_{\ell\in B^\ast,\ \|\ell\|_{B^\ast}\le 1}
    \int_B \ell(x)^2\,\d\gamma(x)
    <\infty .
\]
Equivalently,
\[
    W_2^2(\nu,\gamma)
    \le
    2\sigma_\gamma^2
    D_{\mathrm{KL}}(\nu\,\Vert\,\gamma),
\]
where $W_2$ is computed with respect to the Banach norm $\|\cdot\|_B$.
\end{theorem}

For Gaussian measures on the latent Hilbert space, the weak variance in this inequality reduces to the operator norm of the covariance.

\begin{lemma}[Hilbert-space specialization]
\label{lem:hilbert-gaussian-t2}
Let $\mathcal Z$ be a separable Hilbert space and let
\[
    \mu=\mathsf N(m,C),
\]
where $m\in\mathcal Z$ and $C:\mathcal Z\to\mathcal Z$ is positive, self-adjoint, and trace-class.
Then, for every $\nu\in\mathcal P_2(\mathcal Z)$,
\[
    W_2^2(\nu,\mu)
    \le
    2\|C\|_{\mathrm{op}}
    D_{\mathrm{KL}}(\nu\,\Vert\,\mu).
\]
Here $W_2$ is computed with respect to the Hilbert norm $\|\cdot\|_{\mathcal Z}$.
\end{lemma}

Combining this transportation bound with the fact that Wasserstein distance controls differences of means gives the KL mean-control result used in the main text.

\begin{proof}
It suffices to consider the centered case $m=0$, since translation preserves both Wasserstein distance and relative entropy. Let
\[
    \gamma=\mathsf N(0,C).
\]
By Theorem~\ref{thm:riedel-gaussian-t2}, it remains to identify the weak variance $\sigma_\gamma^2$.

Since $\mathcal Z$ is a Hilbert space, every continuous linear functional $\ell\in\mathcal Z^\ast$ has the form
\[
    \ell_h(x)=\langle h,x\rangle_{\mathcal Z}
\]
for a unique $h\in\mathcal Z$, with
\[
    \|\ell_h\|_{\mathcal Z^\ast}=\|h\|_{\mathcal Z}.
\]
If $X\sim\mathsf N(0,C)$, then
\[
    \int_{\mathcal Z}\ell_h(x)^2\,\d\gamma(x)
    =\mathbb E\langle h,X\rangle_{\mathcal Z}^2
    =\langle Ch,h\rangle_{\mathcal Z}.
\]
Therefore
\[
\begin{aligned}
    \sigma_\gamma^2
    &=
    \sup_{\ell\in\mathcal Z^\ast,\ \|\ell\|_{\mathcal Z^\ast}\le 1}
    \int_{\mathcal Z}\ell(x)^2\,\d\gamma(x)\sup_{\|h\|_{\mathcal Z}\le 1}
    \langle Ch,h\rangle_{\mathcal Z}.
\end{aligned}
\]
Since $C$ is positive and self-adjoint,
\[
    \sup_{\|h\|_{\mathcal Z}\le 1}
    \langle Ch,h\rangle_{\mathcal Z}=\|C\|_{\mathrm{op}}.
\]
Thus
\[
    \sigma_\gamma^2=\|C\|_{\mathrm{op}}.
\]
Substituting this into Theorem~\ref{thm:riedel-gaussian-t2} gives
\[
    W_2^2(\nu,\mathsf N(0,C))
    \le
    2\|C\|_{\mathrm{op}}
    D_{\mathrm{KL}}(\nu\,\Vert\,\mathsf N(0,C)).
\]
The translated case $\mu=\mathsf N(m,C)$ follows by applying the same result to the translated measures.
\end{proof}

\begin{proof}[Proof of Lemma \ref{lem:kl_latent_mean_control}]
Let
\[
    P_\mathbf{z}:=P_\theta(\cdot\,|\,\mathbf{z})
    =
    \mathsf N(T_\theta(\mathbf z),\alpha_\theta(\mathbf z)^2K).
\]
By Lemma~\ref{lem:hilbert-gaussian-t2},
\[
    W_2^2(Q,P_{\mathbf z})\le
    2\|\alpha_\theta(\mathbf{z})^2K\|_{\mathrm{op}}
    D_{\mathrm{KL}}(Q\,\Vert\,P_{\mathbf z}).
\]
Since
\[
    \|\alpha_\theta(\mathbf{z})^2K\|_{\mathrm{op}}
    =\alpha_\theta(\mathbf{z})^2\|K\|_{\mathrm{op}},
\]
we have
\[
    W_2(Q,P_\mathbf{z})
    \le
    \alpha_\theta(\mathbf{z})
    \sqrt{
    2\|K\|_{\mathrm{op}}
    D_{\mathrm{KL}}(Q\,\Vert\,P_\mathbf{z})
    }.
\]
It remains to observe that Wasserstein distance controls the distance between means. Indeed, for
any coupling $(\mathbf Y, \mathbf Y')$ of $Q$ and $P_\mathbf z$,
\[
\begin{aligned}
    \|m_Q-T_\theta(\mathbf{z})\|_{\mathcal Z}
    =
    \left\|
    \mathbb E \mathbf Y-\mathbb E\mathbf Y'
    \right\|_{\mathcal Z}\le
    \mathbb E\|\mathbf Y-\mathbf Y'\|_{\mathcal Z}\le
    \left(
    \mathbb E\|\mathbf Y-\mathbf Y'\|_{\mathcal Z}^2
    \right)^{1/2}.
\end{aligned}
\]
Taking the infimum over all couplings gives
\[
    \|m_Q-T_\theta(\mathbf z)\|_{\mathcal Z}
    \le
    W_2(Q,P_{\mathbf z}).
\]
Combining the two estimates proves the claim.
\end{proof}

We now apply Lemma~\ref{lem:kl_latent_mean_control} along the encoded reference trajectory and propagate the resulting one-step latent control through the deterministic mean rollout. Let
\[
    \{\mathbf U_n^\star\}_{n=0}^N\subset\mathcal U
\]
be a reference physical trajectory and recall our notations
\[
    \mathbf Z_n^\star:=m_\phi(\mathbf U_n^\star),
    \qquad
    \bar{\mathbf U}_n:=D_\psi(\mathbf Z_n^\star),
    \qquad
    \delta_n:=\|\bar{\mathbf U}_n-\mathbf U_n^\star\|_{\mathcal U}.
\]
For each $n=0,\ldots,N-1$, let
\[
    \Xi_n\sim\mathsf N(0,K_\phi(\mathbf U_n^\star)).
\]
The dynamic KL quantity is
\[
    \kappa_n^2
    :=
    \mathbb E_{\Xi_n}
    \left[
    D_{\mathrm{KL}}
    \left(
    Q_\phi(\cdot\mid \mathbf U_{n+1}^\star)
    \,\Vert\,
    P_\theta(\cdot\mid \mathbf Z_n^\star+\Xi_n)
    \right)
    \right],
\]
and the latent encoder-noise sensitivity of the transition map is
\[
    \eta_n
    :=\left(\mathbb E_{\Xi_n}
    \left[\left\|T_\theta(\mathbf Z_n^\star+\Xi_n)
    -T_\theta(\mathbf Z_n^\star)
    \right\|_{\mathcal Z}^2
    \right]\right)^{1/2}.
\]

\begin{proof}[Proof of Proposition \ref{prop:population_kl_rollout}]
\textit{Step I.} We first bound the latent one-step defect
\[
    \|T_\theta(\mathbf Z_n^\star)-\mathbf Z_{n+1}^\star\|_{\mathcal Z}.
\]
For each realization of $\Xi_n$, apply Lemma~\ref{lem:kl_latent_mean_control} with
\[
    Q=Q_\phi(\cdot\mid \mathbf U_{n+1}^\star),
    \qquad
    \mathbf{Z}=\mathbf Z_n^\star+\Xi_n.
\]
Since the mean of $Q_\phi(\cdot\mid \mathbf U_{n+1}^\star)$ is
\[
    m_\phi(\mathbf U_{n+1}^\star)=\mathbf Z_{n+1}^\star,
\]
we obtain
\begin{equation}
\bigl\|\mathbf Z_{n+1}^\star
    -T_\theta(\widetilde{\mathbf Z}_n)
\bigr\|_{\mathcal Z}\leq
\alpha_\theta(\widetilde{\mathbf Z}_n)
\sqrt{
        2\|K\|_{\mathrm{op}}
        D_{\mathrm{KL}}
        \left(
            Q_\phi(\cdot\mid\mathbf U_{n+1}^\star)
            \,\Vert\,
            P_\theta(\cdot\mid\widetilde{\mathbf Z}_n)
        \right)
    }.
    \label{eq:proof_mean_control}
\end{equation}
Using $\alpha_\theta\leq\bar\alpha$, squaring and taking expectation over both the reference trajectory and $\Xi_n$ gives
\begin{equation}
\left(\mathbb E\left[\bigl\|
\mathbf Z_{n+1}^\star-T_\theta(\widetilde{\mathbf Z}_n)
\bigr\|_{\mathcal Z}^2\right]\right)^{1/2}
\leq\ol{\alpha}\sqrt{2\Vert K\Vert_\mathrm{op}}\,\kappa_n.
\label{eq:proof_kl_mean_population}
\end{equation}
Now write
\[
    T_\theta(\mathbf Z_n^\star)-\mathbf Z_{n+1}^\star
    =
    \bigl[
        T_\theta(\mathbf Z_n^\star)
        -
        T_\theta(\widetilde{\mathbf Z}_n)
    \bigr]
    +
    \bigl[
        T_\theta(\widetilde{\mathbf Z}_n)
        -
        \mathbf Z_{n+1}^\star
    \bigr].
\]
Taking the $L^2$ norm over the joint distribution and applying Minkowski's inequality together with \eqref{eq:proof_kl_mean_population} yields
\[
    \left(
        \mathbb E
        \left[
            \left\|
                T_\theta(\mathbf Z_n^\star)
                -
                \mathbf Z_{n+1}^\star
            \right\|_{\mathcal Z}^2
        \right]
    \right)^{1/2}
    \leq
    \eta_n+\ol{\alpha}\sqrt{2\Vert K\Vert_\mathrm{op}}\kappa_n,
\]
which proves \eqref{eq:population_latent_one_step}.

\item\textit{Step II.} We next derive the latent rollout bound. Define
\[
    e_n
    :=
    \left(
        \mathbb E
        \left[
            \|\widehat{\mathbf Z}_n-\mathbf Z_n^\star\|_{\mathcal Z}^2
        \right]
    \right)^{1/2}.
\]
Note that $e_0=0$ since the rollout is initialized at the encoder mean. For every realization of the reference trajectory,
\begin{align*}
    \|\widehat{\mathbf Z}_{n+1}-\mathbf Z_{n+1}^\star\|_{\mathcal Z}
    &=
    \|T_\theta(\widehat{\mathbf Z}_n)-\mathbf Z_{n+1}^\star\|_{\mathcal Z}
    \\
    &\leq
    \|T_\theta(\widehat{\mathbf Z}_n)
      -T_\theta(\mathbf Z_n^\star)\|_{\mathcal Z}
    +
    \|T_\theta(\mathbf Z_n^\star)
      -\mathbf Z_{n+1}^\star\|_{\mathcal Z}.
\end{align*}
Taking the population $L^2$ norm and using Minkowski's inequality, Lipschitz continuity of $T_\theta$, and \eqref{eq:population_latent_one_step}, we obtain
\[
    e_{n+1}
    \leq
    \Lambda_T e_n+\eta_n+\ol{\alpha}\sqrt{2\Vert K\Vert_\mathrm{op}}\,\kappa_n.
\]
Iterating this recursion from $e_0=0$ gives
\begin{equation}
    e_n
    \leq
    \sum_{j=0}^{n-1}
    \Lambda_T^{\,n-1-j}
    \left(\eta_j+\ol{\alpha}\sqrt{2\Vert K\Vert_\mathrm{op}}\,\kappa_j\right).\label{eq:population_latent_rollout}
\end{equation}
For the physical rollout, use $\widehat{\mathbf U}_n=D_\psi(\widehat{\mathbf Z}_n)$ and insert the decoded reference state $D_\psi(\mathbf Z_n^\star)$:
\begin{align*}
    \|\widehat{\mathbf U}_n-\mathbf U_n^\star\|_{\mathcal U}
    &\leq
    \|D_\psi(\widehat{\mathbf Z}_n)
      -D_\psi(\mathbf Z_n^\star)\|_{\mathcal U}
    +
    \|D_\psi(\mathbf Z_n^\star)-\mathbf U_n^\star\|_{\mathcal U}.
\end{align*}
Taking the population $L^2$ norm and using Lipschitz continuity of $D_\psi$ yields
\[
\left(\mathbb E
\|\widehat{\mathbf U}_n-\mathbf U_n^\star\|_{\mathcal U}^2
\right)^{1/2}
\leq\Lambda_D e_n+\delta_n.
\]
Substituting \eqref{eq:population_latent_rollout} proves \eqref{eq:population_kl_physical_rollout}.

\item\textit{Step III.} It remains to incorporate the predictive loss. Define the clean decoded one-step error
\[
    q_n:=
    \left(\mathbb E\left[\left\|D_\psi(T_\theta(\mathbf Z_n^\star))-\mathbf U_{n+1}^\star\right\|_{\mathcal U}^2\right]\right)^{1/2}.
\]
For every realization of the reference trajectory and the Gaussian perturbations,
\begin{align*}
    &
    \left\|
        D_\psi(T_\theta(\mathbf Z_n^\star)) - \mathbf U_{n+1}^\star
    \right\|_{\mathcal U}
    \\
    &\quad\leq
    \left\|
        D_\psi(T_\theta(\mathbf Z_n^\star)) -D_\psi(T_\theta(\widetilde{\mathbf Z}_n))
    \right\|_{\mathcal U}
    +
    \left\|D_\psi(T_\theta(\widetilde{\mathbf Z}_n)) - D_\psi(\widetilde{\mathbf Z}_n^+)
    \right\|_{\mathcal U}
    +
    \left\|
        D_\psi(\widetilde{\mathbf Z}_n^+)
        - \mathbf U_{n+1}^\star
    \right\|_{\mathcal U}.
\end{align*}
Taking the joint $L^2$ norm and applying Minkowski's inequality gives
\begin{equation}
q_n\leq\Lambda_D\eta_n+\rho_n+\epsilon_n.\label{eq:population_predictive_one_step}
\end{equation}
We now obtain a second bound for the physical rollout error at time $n$. Since
\[
    \widehat{\mathbf U}_n=D_\psi(T_\theta(\widehat{\mathbf Z}_{n-1})),
\]
by Minkowski's inequality and Lipschitz continuity of $D_\psi$ and $T_\theta$, we have
\begin{align*}
    \left(
        \mathbb E
        \|\widehat{\mathbf U}_n-\mathbf U_n^\star\|_{\mathcal U}^2
    \right)^{1/2}&\leq \left(
        \mathbb E
        \|D_\psi(T_\theta(\wh{\mathbf Z}_n))-D_\psi(T_\theta(\mathbf Z_n^\star))\|_{\mathcal U}^2
    \right)^{1/2}+\left(
        \mathbb E
        \left[
            \left\|
                D_\psi(T_\theta(\mathbf Z_n^\star))
                -
                \mathbf U_{n+1}^\star
            \right\|_{\mathcal U}^2
        \right]
    \right)^{1/2}\\
    &\leq
    \Lambda_D\Lambda_T e_{n-1}+q_{n-1}.
\end{align*}
Using \eqref{eq:population_latent_rollout} at time $n-1$ and
\eqref{eq:population_predictive_one_step},
\begin{align}
    &
    \left(
        \mathbb E
        \|\widehat{\mathbf U}_n-\mathbf U_n^\star\|_{\mathcal U}^2
    \right)^{1/2}
    \leq
    \Lambda_D\sum_{j=0}^{n-2}\Lambda_T^{\,n-1-j} \left(\eta_j+\bar\alpha\sqrt{2\|K\|_{\mathrm{op}}}\,\kappa_j\right)+\Lambda_D\eta_{n-1}+\epsilon_{n-1}+\rho_{n-1}.
    \label{eq:proof_predictive_route}
\end{align}
Both \eqref{eq:proof_predictive_route} and \eqref{eq:population_kl_physical_rollout} are valid upper bounds for the same population rollout error. Taking the smaller of their final terms yields \eqref{eq:population_refined_rollout}.
\end{proof}

For comparison, we conclude with the generic autoregressive amplification bound \eqref{eq:generic_rollout} used to separate the effect of recursive deployment from mechanisms specific to variational latent dynamics.

\begin{proof}[Proof of Proposition \ref{prop:generic_rollout}]
Define the population rollout error
\[
e_n
:=
\left(
\mathbb E
\left[
\|\widehat{\mathbf U}_n-\mathbf U_n^\star\|_{\mathcal U}^2
\right]
\right)^{1/2}.
\]
Since $\widehat{\mathbf U}_0=\mathbf U_0^\star$, we have $e_0=0$. For every
$n\geq0$,
\begin{align*}
\|\widehat{\mathbf U}_{n+1}-\mathbf U_{n+1}^\star\|_{\mathcal U}
&=
\|F_\vartheta(\widehat{\mathbf U}_n)-\mathbf U_{n+1}^\star\|_{\mathcal U}
\\
&\leq
\|F_\vartheta(\widehat{\mathbf U}_n)
-F_\vartheta(\mathbf U_n^\star)\|_{\mathcal U}
+\|F_\vartheta(\mathbf U_n^\star)-\mathbf U_{n+1}^\star\|_{\mathcal U}.
\end{align*}
Taking the population $L^2$ norm and applying Minkowski's inequality together
with the Lipschitz continuity of $F_\vartheta$ gives
\[
e_{n+1}\leq\Lambda_F e_n+\epsilon_n^{\mathrm{dir}}.
\]
Iterating this recursion from $e_0=0$ yields
\[
e_n\leq\sum_{j=0}^{n-1}
\Lambda_F^{\,n-1-j}\epsilon_j^{\mathrm{dir}},
\]
which proves \eqref{eq:generic_rollout}.
\end{proof}
The preceding results concern the deterministic mean-map rollout used at inference. The next subsection gives a complementary high-probability analysis of stochastic transition rollouts around this mean trajectory.

\subsection{High-Probability Stochastic Envelope}
\label{app:stochastic_envelope}
The main analysis in \S\ref{sec:err_var_dyn} concerns the deterministic mean-map rollout used at inference. Here we give a complementary high-probability characterization of stochastic rollouts obtained by sampling from the learned transition kernel. We first establish the required Gaussian concentration bounds and then propagate them through the latent dynamics. The following result is similar to the conclusion of \cite{hsu2012tail}.

\begin{lemma}[Two-sided Gaussian norm concentration in a Hilbert space]
\label{lem:hilbert-gaussian-tail}
Let $\mathcal Z$ be a separable Hilbert space and let
\[
    \mathbf Z\sim \mathsf N(0,C),
\]
where $C:\mathcal Z\to\mathcal Z$ is positive, self-adjoint, and
trace-class. Then, for every $t\ge0$,
\begin{equation}
    \mathbb P\!\left(
        \|\mathbf Z\|_{\mathcal Z}^2
        \ge
        \operatorname{Tr}(C)
        +2\sqrt{\operatorname{Tr}(C^2)t}
        +2\|C\|_{\mathrm{op}}t
    \right)
    \le e^{-t},
    \label{eq:gaussian_norm_upper_tail}
\end{equation}
and
\begin{equation}
    \mathbb P\!\left(
        \|\mathbf Z\|_{\mathcal Z}^2
        \le
        \operatorname{Tr}(C)
        -2\sqrt{\operatorname{Tr}(C^2)t}
    \right)
    \le e^{-t}.
    \label{eq:gaussian_norm_lower_tail}
\end{equation}
Consequently, for $q\in(0,1/2)$, with probability at least $1-2q$,
\[
    r_C^{-}(q)
    \le
    \|\mathbf Z\|_{\mathcal Z}
    \le
    r_C^{+}(q),
\]
where
\begin{align}
    r_C^{-}(q)
    &:=
    \left[
        \operatorname{Tr}(C)
        -2\sqrt{\operatorname{Tr}(C^2)\log(1/q)}
    \right]_+^{1/2},
    \label{eq:gaussian_lower_radius}
    \\
    r_C^{+}(q)
    &:=
    \left[
        \operatorname{Tr}(C)
        +2\sqrt{\operatorname{Tr}(C^2)\log(1/q)}
        +2\|C\|_{\mathrm{op}}\log(1/q)
    \right]^{1/2}.
    \label{eq:gaussian_upper_radius}
\end{align}
Moreover, let $P:\mathcal Z\to\mathcal Z$ be any orthogonal projection satisfying $\operatorname{rank}(I-P)\le1$. Define
\begin{equation}
    r_C^{\perp}(q)
    :=
    \left[
        \operatorname{Tr}(C)
        -\|C\|_{\mathrm{op}}
        -2\sqrt{\operatorname{Tr}(C^2)\log(1/q)}
    \right]_+^{1/2}.
    \label{eq:gaussian_projected_lower_radius}
\end{equation}
Then
\begin{equation}
    \mathbb P\!\left(
        \|P\mathbf Z\|_{\mathcal Z}
        < r_C^{\perp}(q)
    \right)
    \le q.
    \label{eq:gaussian_projected_lower_tail}
\end{equation}
\end{lemma}

\begin{proof}
Since $C$ is positive, self-adjoint, and trace-class, there exists an orthonormal basis $\{E_i\}_{i\ge1}$ of $\mathcal Z$ and eigenvalues $\lambda_i\ge0$ such that
\[
    CE_i=\lambda_iE_i,
    \qquad
    \sum_{i=1}^{\infty}\lambda_i
    =
    \operatorname{Tr}(C)<\infty.
\]
By the Karhunen-Lo\`eve expansion,
\[
    \mathbf Z
    =
    \sum_{i=1}^{\infty}
    \sqrt{\lambda_i}\,\xi_i E_i,
    \qquad
    \xi_i\overset{\mathrm{i.i.d.}}{\sim}\mathsf N(0,1),
\]
and hence
\[
    \|\mathbf Z\|_{\mathcal Z}^2
    =
    \sum_{i=1}^{\infty}\lambda_i\xi_i^2
    \qquad\text{a.s.}
\]
Set
\[
    S:=\|\mathbf Z\|_{\mathcal Z}^2,
    \qquad
    \mu:=\operatorname{Tr}(C),
    \qquad
    v:=\operatorname{Tr}(C^2),
    \qquad
    L:=\|C\|_{\mathrm{op}}.
\]
If $v=0$, then $C=0$ and the claims are immediate, so assume $v>0$. 

\item (i) For the upper tail, for every $0\le s<1/(2L)$,
\[
\begin{aligned}
    \log\mathbb E e^{s(S-\mu)}
    &=
    \sum_{i=1}^{\infty}
    \left[
        -s\lambda_i
        -\frac12\log(1-2s\lambda_i)
    \right]                                                     \\
    &\le
    \sum_{i=1}^{\infty}
    \frac{s^2\lambda_i^2}{1-2s\lambda_i}
    \le
    \frac{s^2v}{1-2sL},
\end{aligned}
\]
where we used
\[
    -\frac12\log(1-2x)-x
    \le
    \frac{x^2}{1-2x},
    \qquad 0\le x<\frac12.
\]
Chernoff's inequality therefore gives
\[
    \mathbb P(S-\mu\ge a)
    \le
    \exp\!\left(
        -sa+\frac{s^2v}{1-2sL}
    \right).
\]
Taking
\[
    a=2\sqrt{vt}+2Lt,
    \qquad
    s=
    \frac{\sqrt t}{\sqrt v+2L\sqrt t},
\]
yields
\[
    \mathbb P
    \left(
        S-\mu\ge2\sqrt{vt}+2Lt
    \right)
    \le e^{-t},
\]
which proves \eqref{eq:gaussian_norm_upper_tail}.

\item (ii) For the lower tail, for every $s\ge0$,
\[
\begin{aligned}
    \log\mathbb E e^{-s(S-\mu)}
    &=
    \sum_{i=1}^{\infty}
    \left[
        s\lambda_i
        -\frac12\log(1+2s\lambda_i)
    \right]\le
    s^2\sum_{i=1}^{\infty}\lambda_i^2
    =s^2v,
\end{aligned}
\]
where we used
\[
    \log(1+2x)\ge2x-2x^2,
    \qquad x\ge0.
\]
Thus, for $a>0$,
\[
    \mathbb P(S-\mu\le-a)
    \le
    \exp(-sa+s^2v).
\]
Choosing
\[
    a=2\sqrt{vt},
    \qquad
    s=\sqrt{\frac{t}{v}},
\]
gives
\[
    \mathbb P
    \left(
        S-\mu\le-2\sqrt{vt}
    \right)
    \le e^{-t},
\]
which proves \eqref{eq:gaussian_norm_lower_tail}. Taking $t=\log(1/q)$ and applying a union bound gives the two-sided interval.

\item (iii) It remains to prove \eqref{eq:gaussian_projected_lower_tail}. Since $P$ is bounded and linear,
\[
    P\mathbf Z\sim\mathsf N(0,PCP).
\]
Because $\operatorname{rank}(I-P)\le1$,
\[
    \operatorname{Tr}(PCP)
    \ge
    \operatorname{Tr}(C)-\|C\|_{\mathrm{op}},
\]
while
\[
    \operatorname{Tr}\!\left((PCP)^2\right)
    =
    \|PCP\|_{\mathrm{HS}}^2
    \le
    \|C\|_{\mathrm{HS}}^2
    =
    \operatorname{Tr}(C^2).
\]
Applying \eqref{eq:gaussian_norm_lower_tail} to $P\mathbf Z$ with $t=\log(1/q)$ therefore yields
\[
    \mathbb P\!\left(
        \|P\mathbf Z\|_{\mathcal Z}
        <
        r_C^\perp(q)
    \right) \le q,
\]
which concludes the proof.
\end{proof}

The preceding lemma controls both the magnitude of each Gaussian transition perturbation and the component that remains after projecting out one potentially canceling direction. We now use these estimates to obtain a two-sided latent envelope around the deterministic mean-map rollout.

\begin{proposition}[Two-sided stochastic envelope around the mean-map rollout]
\label{prop:two_sided_latent_envelope}
Let $\mathcal U$ be the physical state space and let $\mathcal Z$ be a separable Hilbert space. Let $K:\mathcal Z\to\mathcal Z$ be positive, self-adjoint, and trace-class. Given an initial physical state $\mathbf U_0\in\mathcal U$, define the mean-map latent rollout by
\begin{equation}
    \widehat{\mathbf Z}_0=m_\phi(\mathbf U_0),
    \qquad
    \widehat{\mathbf Z}_{n+1}=T_\theta(\widehat{\mathbf Z}_n),
    \qquad n=0,\ldots,N-1,
    \label{eq:mean_map_rollout}
\end{equation}
and the corresponding physical rollout by
\[
    \widehat{\mathbf U}_0=\mathbf U_0,
    \qquad
    \widehat{\mathbf U}_n
    =D_\psi(\widehat{\mathbf Z}_n),
    \qquad n=1,\ldots,N.
\]
Starting from the same physical initial condition, consider the stochastic latent rollout
\begin{equation}
    \widetilde{\mathbf Z}_0
    =
    m_\phi(\mathbf U_0),
    \qquad
    \widetilde{\mathbf Z}_{n+1}
    =
    T_\theta(\widetilde{\mathbf Z}_n)
    +
    \alpha_\theta(\widetilde{\mathbf Z}_n)\mathbf G_n,
    \qquad
    \mathbf G_n
    \overset{\mathrm{i.i.d.}}{\sim}
    \mathsf N(0,K),
    \label{eq:stochastic_latent_rollout}
\end{equation}
with physical states
\[
    \widetilde{\mathbf U}_0=\mathbf U_0,
    \qquad
    \widetilde{\mathbf U}_n
    =
    D_\psi(\widetilde{\mathbf Z}_n),
    \qquad n=1,\ldots,N.
\]
Assume that $T_\theta$ is $\Lambda_T$-Lipschitz and $D_\psi$ is $\Lambda_D$-Lipschitz on the relevant latent regions, and that
\[
    0<\underline{\alpha}\le\alpha_\theta(\mathbf Z)\le\bar\alpha
\]
throughout the stochastic rollout tube. Then the following hold.

\begin{enumerate}[label=(\alph*)]
\item \emph{One-step conditional spread.}
For every $n=0,\ldots,N-1$ and $q\in(0,1/2)$, conditionally on $\widetilde{\mathbf Z}_n$, with probability at least $1-2q$,
\begin{equation}
    \underline{\alpha}
    r_K^-(q)
    \le
    \left\|
        \widetilde{\mathbf Z}_{n+1}
        -
        T_\theta(\widetilde{\mathbf Z}_n)
    \right\|_{\mathcal Z}
    \le
    \bar\alpha r_K^+(q).
    \label{eq:conditional_transition_band}
\end{equation}

\item \emph{Two-sided latent envelope and physical upper envelope.}
For any $\gamma\in(0,1)$, with probability at least $1-\gamma$, simultaneously for all $n=1,\ldots,N$, the latent ensemble deviation satisfies
\begin{equation}
    \underline{\alpha}
    r_K^\perp\!\left(\frac{\gamma}{2N}\right)
    \le
    \|\widetilde{\mathbf Z}_n-\widehat{\mathbf Z}_n\|_{\mathcal Z}
    \le
    \bar\alpha
    r_K^+\!\left(\frac{\gamma}{2N}\right)
    \sum_{j=0}^{n-1}
    \Lambda_T^{\,n-1-j}.
    \label{eq:two_sided_latent_envelope}
\end{equation}
Consequently,
\begin{equation}
    \left\|
        \widetilde{\mathbf U}_n
        -
        \widehat{\mathbf U}_n
    \right\|_{\mathcal U}
    \le
    \Lambda_D\bar\alpha
    r_K^+\!\left(\frac{\gamma}{2N}\right)
    \sum_{j=0}^{n-1}
    \Lambda_T^{\,n-1-j}.
    \label{eq:physical_stochastic_envelope}
\end{equation}
The latent lower envelope is nontrivial whenever
\begin{equation}
    \operatorname{Tr}(K)-\|K\|_{\mathrm{op}}
    >
    2\sqrt{
        \operatorname{Tr}(K^2)
        \log\!\left(\frac{2N}{\gamma}\right)
    }.
    \label{eq:nontrivial_lower_envelope}
\end{equation}
\end{enumerate}
\end{proposition}

\begin{proof}
For part (a), condition on $\widetilde{\mathbf Z}_n$. Since
$\mathbf G_n$ is independent of the preceding transition noises,
\[
    \widetilde{\mathbf Z}_{n+1}
    -
    T_\theta(\widetilde{\mathbf Z}_n)
    =
    \alpha_\theta(\widetilde{\mathbf Z}_n)\mathbf G_n.
\]
The amplitude $\alpha_\theta(\widetilde{\mathbf Z}_n)$ is fixed under this conditioning. Applying Lemma~\ref{lem:hilbert-gaussian-tail} to
$\mathbf G_n\sim\mathsf N(0,K)$ gives
\begin{equation*}
    \alpha_\theta(\widetilde{\mathbf Z}_n)
    r_K^-(q)
    \le
    \left\|
        \widetilde{\mathbf Z}_{n+1}
        -
        T_\theta(\widetilde{\mathbf Z}_n)
    \right\|_{\mathcal Z}
    \le
    \alpha_\theta(\widetilde{\mathbf Z}_n)
    r_K^+(q).
\end{equation*}
Since $\underline{\alpha}\le\alpha_\theta(\mathbf Z)\le\bar\alpha$, we obtain \eqref{eq:conditional_transition_band}.

For part (b), define the latent ensemble deviation
\[
    \mathbf E_n
    :=\widetilde{\mathbf Z}_n-\widehat{\mathbf Z}_n.
\]
Let
\[
    \mathcal F_n
    :=
    \sigma(\mathbf G_0,\ldots,\mathbf G_{n-1}),
    \qquad
    \mathcal F_0:=\{\varnothing,\Omega\}.
\]
Then $\widetilde{\mathbf Z}_n$ and $\mathbf E_n$ are
$\mathcal F_n$-measurable, while $\mathbf G_n$ is independent of $\mathcal F_n$. Since the two rollouts are initialized from the same encoder mean,
\[
    \mathbf E_0=0.
\]
Subtracting \eqref{eq:mean_map_rollout} from
\eqref{eq:stochastic_latent_rollout} gives
\begin{equation}
    \mathbf E_{n+1}
    =
    \underbrace{
        T_\theta(\widetilde{\mathbf Z}_n)
        -
        T_\theta(\widehat{\mathbf Z}_n)
    }_{\mathbf D_n}
    +
    \alpha_\theta(\widetilde{\mathbf Z}_n)\mathbf G_n,
    \label{eq:ensemble_deviation_recursion}
\end{equation}
where $\mathbf D_n$ is $\mathcal F_n$-measurable.

For the upper envelope, Lipschitz continuity of $T_\theta$ and
$\alpha_\theta\le\bar\alpha$ give
\[
    \|\mathbf E_{n+1}\|_{\mathcal Z}
    \le
    \Lambda_T\|\mathbf E_n\|_{\mathcal Z}
    +
    \bar\alpha\|\mathbf G_n\|_{\mathcal Z}.
\]
Iterating from $\mathbf E_0=0$ yields
\begin{equation}
    \|\mathbf E_n\|_{\mathcal Z}
    \le
    \bar\alpha
    \sum_{j=0}^{n-1}
    \Lambda_T^{\,n-1-j}
    \|\mathbf G_j\|_{\mathcal Z}.
    \label{eq:upper_envelope_recursion}
\end{equation}
By Lemma~\ref{lem:hilbert-gaussian-tail} and a union bound, with probability at least $1-\gamma/2$,
\[
    \|\mathbf G_j\|_{\mathcal Z}
    \le
    r_K^+\!\left(\frac{\gamma}{2N}\right),
    \qquad
    j=0,\ldots,N-1.
\]
Substitution into \eqref{eq:upper_envelope_recursion} proves the upper
bound in \eqref{eq:two_sided_latent_envelope} simultaneously for all
$n\le N$.

For the lower envelope, fix $n\in\{0,\ldots,N-1\}$ and let $P_n$ be the orthogonal projection onto $\{\mathbf D_n\}^\perp$, with $P_n=I$ when $\mathbf D_n=0$. Then $P_n$ is $\mathcal F_n$-measurable,
$\operatorname{rank}(I-P_n)\le1$, and $P_n\mathbf D_n=0$. Applying
$P_n$ to \eqref{eq:ensemble_deviation_recursion} gives
\[
    P_n\mathbf E_{n+1}=\alpha_\theta(\widetilde{\mathbf Z}_n)P_n\mathbf G_n.
\]
Hence
\[
    \|\mathbf E_{n+1}\|_{\mathcal Z}\ge
    \|P_n\mathbf E_{n+1}\|_{\mathcal Z}=
    \alpha_\theta(\widetilde{\mathbf Z}_n)
    \|P_n\mathbf G_n\|_{\mathcal Z}\ge
    \underline{\alpha}\,
    \|P_n\mathbf G_n\|_{\mathcal Z}.
\]
Conditionally on $\mathcal F_n$, the projection $P_n$ is fixed and $\mathbf G_n\sim\mathsf N(0,K)$ remains independent of
$\mathcal F_n$. Therefore,
\[
    \mathbb P\!\left(
        \|P_n\mathbf G_n\|_{\mathcal Z}
        <
        r_K^\perp\!\left(\frac{\gamma}{2N}\right)
        \,\middle|\,
        \mathcal F_n
    \right)
    \le
    \frac{\gamma}{2N}
\]
by \eqref{eq:gaussian_projected_lower_tail}. Taking expectations and applying a union bound over $n=0,\ldots,N-1$ shows that, with probability
at least $1-\gamma/2$,
\[
    \|\mathbf E_n\|_{\mathcal Z}
    \ge
    \underline{\alpha}
    r_K^\perp\!\left(\frac{\gamma}{2N}\right),
    \qquad
    n=1,\ldots,N.
\]
Intersecting the upper- and lower-envelope events gives probability at least $1-\gamma$ and proves
\eqref{eq:two_sided_latent_envelope}.

Finally, for every $n\ge1$, Lipschitz continuity of the decoder gives
\[
\begin{aligned}
    \left\|
        \widetilde{\mathbf U}_n
        -
        \widehat{\mathbf U}_n
    \right\|_{\mathcal U}
    &=
    \left\|
        D_\psi(\widetilde{\mathbf Z}_n)
        -
        D_\psi(\widehat{\mathbf Z}_n)
    \right\|_{\mathcal U} \\
    &\le
    \Lambda_D
    \|\mathbf E_n\|_{\mathcal Z},
\end{aligned}
\]
and \eqref{eq:physical_stochastic_envelope} follows from the latent upper
envelope. The condition
\eqref{eq:nontrivial_lower_envelope} is precisely the condition under which
$r_K^\perp(\gamma/(2N))>0$.
\end{proof}

\paragraph{Relation to the deterministic rollout analysis.}
The stochastic-envelope result complements the deterministic rollout analysis in \S\ref{sec:err_var_dyn}. Proposition~\ref{prop:population_kl_rollout} characterizes the error of the mean-map rollout relative to the reference PDE trajectory, with the latent transition mismatch controlled by the variational KL term and the encoder-noise sensitivity. In contrast, Proposition~\ref{prop:two_sided_latent_envelope} conditions on the initial physical state and characterizes the spread induced by sampling from the learned Gaussian transition around this mean-map trajectory. Thus, the two results respectively describe the location and the stochastic spread of the learned latent dynamics.

Notably, both bounds are governed by the same transition-stability factor $\Lambda_T$. In the deterministic analysis, powers of $\Lambda_T$ amplify local transition discrepancies relative to the reference trajectory; in the stochastic analysis, they amplify the Gaussian perturbations around the mean-map rollout. The upper envelope therefore shows that stochastic spread remains controlled whenever the mean transition is sufficiently stable, while the lower latent envelope shows that, when the transition covariance has sufficient variance outside any single direction, sampled rollouts do not collapse onto the mean-map trajectory. These results characterize the spread generated by the learned transition distribution, but do not imply that this spread is calibrated to the true prediction error.

\section{Benchmark Details}\label{app:benchmark_details}

This appendix provides the detailed specifications of the three fluid-dynamics benchmarks used in \S\ref{sec:experiments}, including their governing equations, data generation, numerical solvers, and evaluation metrics. Across all benchmarks, the main extrapolation occurs along the temporal direction, with test rollouts extending substantially beyond the training horizon.

\subsection{1D compressible Euler equation} 
We consider the 1D compressible Euler equations, a canonical inviscid model for compressible fluid dynamics. The system describes the evolution of density, momentum, and total energy through conservation laws. In conservative form, the equation is
\begin{equation}
    \partial_t\mathbf{U} + \partial_x \mathbf{F}(\mathbf{U}) = 0,
    \qquad x\in \mathbb{T}_L,
    \label{eq:euler1d_conservative}
\end{equation}
where the conserved variables and flux are given by
\begin{equation}
    \mathbf{U} =
    \begin{pmatrix}
        \rho \\
        \rho u \\
        E
    \end{pmatrix},
    \qquad
    \mathbf{F}(\mathbf{U}) =
    \begin{pmatrix}
        \rho u \\
        \rho u^2 + p \\
        u(E+p)
    \end{pmatrix},
    \label{eq:euler1d_flux}
\end{equation}
where $\rho(t,x)$ denotes the fluid density, $u(t,x)$ denotes the velocity, $p(t,x)$ denotes the pressure, and $E(t,x)$ denotes the total energy density. We close the system with the ideal gas equation of state
\begin{equation*}
    p = (\gamma-1)
    \left(
        E - \frac{1}{2}\rho u^2
    \right),
    \label{eq:euler1d_eos}
\end{equation*}
where  $\gamma = 1.4$ is the constant adiabatic index. Equivalently, the system can be written as
\begin{equation*}
\begin{cases}
    \partial_t \rho + \partial_x(\rho u) &= 0,\\
    \partial_t(\rho u) + \partial_x(\rho u^2+p) &= 0,\\
    \partial_t E + \partial_x\bigl(u(E+p)\bigr) &= 0.
\end{cases}
\label{eq:euler1d_expanded}
\end{equation*}
The Euler system contains no explicit viscous or thermal diffusion. Therefore, wave-like structures are not smoothed by physical dissipation, and the dynamics can develop sharp gradients, contact discontinuities, rarefaction waves, and shocks.

In our experiments, we use smooth periodic initial conditions on $\mathbb{T}_L=\bbR/L\mathbb{Z}$ and represent the state by the primitive variables
\begin{equation}
    \mathbf{V}(t,x) = \bigl(\rho(t,x), u(t,x), p(t,x)\bigr),\quad t\in [0,T],\ x\in\mathbb{T}_L.
\end{equation}
The learning task is to approximate the time-$h$ solution map
\begin{equation}
    \mathcal{S}_{h}:
    \bigl(\rho(t),u(t),p(t)\bigr)
    \mapsto
    \bigl(\rho(t+h),u(t+h),p(t+h)\bigr),
    \label{eq:euler1d_solution_map}
\end{equation}
which is then rolled out autoregressively. This benchmark provides a compact test of deterministic compressible wave dynamics. Compared with dissipative 1D equations such as Burgers' equation, the absence of viscosity makes long-horizon prediction more sensitive to phase errors and accumulated wave-interaction errors. At the same time, the one-dimensional setting avoids the full cost of multidimensional compressible flow, making it useful as a controlled inviscid-fluid benchmark.

\paragraph{Evaluation metrics.} For the Euler benchmark, we report rollout-aggregated relative $L^1$ and $L^2$ errors separately for density, velocity, and pressure, together with their channel average. Let $u_i^n$ and $\widehat u_i^n$ denote the reference and predicted fields for trajectory $i$ at rollout step $n$. The relative errors are computed as
\[
    \operatorname{Rel}\text{-}L^1
    =\frac{1}{N_{\mathrm{test}}}
    \sum_{i=1}^{N_{\mathrm{test}}}
    \frac{
        \sum_{n=1}^{T}\|\widehat u_i^n-u_i^n\|_{L^1}
    }{
        \sum_{n=1}^{T}\|u_i^n\|_{L^1}
    },\quad 
    \operatorname{Rel}\text{-}L^2
    =\frac{1}{N_{\mathrm{test}}}
    \sum_{i=1}^{N_{\mathrm{test}}}
    \frac{
        \left(
            \sum_{n=1}^{T}\|\widehat u_i^n-u_i^n\|_{L^2}^2
        \right)^{1/2}
    }{
        \left(
            \sum_{n=1}^{T}\|u_i^n\|_{L^2}^2
        \right)^{1/2}
    }.
\]
The $L^2$ metric emphasizes larger-amplitude pointwise errors, while $L^1$ is less dominated by a small number of large local discrepancies and is therefore useful when sharp fronts or discontinuous structures are present. We do not report $H^1$ error for Euler because the inviscid dynamics may develop shocks and contact discontinuities, for which derivative-based errors become highly sensitive to grid resolution, numerical shock thickness, and small spatial shifts of the discontinuity.

\subsubsection{Data Generation for the 1D Euler Equations}
\label{app:euler_data_generation}

\paragraph{Initial conditions.}
For each trajectory, the density, pressure, and velocity fields are initialized from independent periodic
Gaussian random fields on $\mathbb{T}_L$. At first, independent periodic random fields $g_\rho,g_p,g_u$ are generated spectrally with power
\[
    S_k=\exp\left(
        -2\pi^2\ell_{\mathrm{GRF}}^2\frac{k^2}{L^2}
    \right),\qquad\ell_{\mathrm{GRF}}=1,
\]
with the zero mode removed. Each realization is subsequently centered and normalized to unit root-mean-square amplitude: 
\[
    g\leftarrow
    \frac{g-\frac{1}{n_x}\sum_{j=1}^{n_x}g(x_j)
    }{\left[\frac{1}{n_x}\sum_{j=1}^{n_x}
    \left(g(x_j)-\frac{1}{n_x}\sum_{m=1}^{n_x}g(x_m)\right)^2
    \right]^{1/2}+10^{-8}
    }.
\]
Then the primitive conditions are constructed as
\[
\rho(0,x)=\rho_0\left(1+a_\rho g_\rho(x)\right),\quad p(0,x)=p_0\left(1+a_p g_p(x)\right),\quad u(0,x)=M c_0 g_u(x),
\]
where
\[
    \rho_0=p_0=1,\qquad
    a_\rho=0.15,\qquad
    a_p=0.10,\qquad
    c_0=\sqrt{\frac{\gamma p_0}{\rho_0}},
\]
and the trajectory-level Mach amplitude is sampled as
$M\sim\operatorname{Uniform}(0.05,0.35)$. The physical correlation length $\ell_{\mathrm{GRF}}=1$ is fixed across all domain lengths, so increasing $L$ enlarges the domain while preserving the local scale of the initial perturbations.

\paragraph{Spatial and temporal discretization.}
We use
\[
\begin{array}{c|cccc}
    L & 5 & 10 & 15 & 20\\
    \hline
    n_x & 1024 & 2048 & 3072 & 4096
\end{array}
\]
so that the grid spacing is fixed at $\Delta x=5/1024$ across all settings. No spatial downsampling is applied. Trajectories are recorded every $\Delta t_{\mathrm{data}}=0.1$. Training and validation trajectories cover $t\in[0,1.5]$, while the long-horizon test trajectories extend to $T_{\mathrm{test}}=10$.

\paragraph{Numerical solver.}
We solve the Euler equations using a first-order finite-volume scheme with the local Lax-Friedrichs (Rusanov) flux \citep{rusanov1962calculation} and periodic boundary conditions. The internal time step is chosen adaptively using a CFL number of $0.45$, with a maximum solver step of $10^{-3}$. Density and pressure are positivity-clipped when necessary to avoid invalid numerical states. Numerical integration is performed in double precision and the trajectory data are stored in single precision.

\paragraph{Dataset splits.}
For each domain length, we generate 1200 training trajectories and 300 validation trajectories over $t\in[0,1.5]$, together with an independently sampled test set of 200 trajectories over $t\in [0,10]$. Training, validation, and test initial conditions are drawn from the same GRF-based distribution. Thus, the primary distribution shift considered in this benchmark is the substantially
longer temporal horizon at test time rather than a change in the initial-condition distribution.

\subsection{2D compressible fluid dynamics} 
We consider the 2D compressible Navier-Stokes equations following the Compressible Fluid Dynamics (CFD) setting used in PDEBench \citep{takamoto2022pdebench}. The system describes the evolution of density, velocity, pressure, and total energy for a viscous compressible fluid. 

Specifically, let $\rho(t,x)$ denote the mass density, $\mathbf{v}(t,x)=(v_1(t,x),v_2(t,x))$ the velocity field, $p(t,x)$ the gas pressure, and $\epsilon = p/(\Gamma-1)$ the internal energy density, where $\Gamma=5/3$ is the heat capacity ratio. The governing equations are
\begin{equation}
\begin{cases}
    \partial_t \rho + \nabla\cdot(\rho \mathbf{v}) = 0,\\
    \displaystyle\rho\left(\partial_t \mathbf{v} + \mathbf{v}\cdot\nabla \mathbf{v}\right)
    = -\nabla p + \eta \Delta\mathbf{v}
    + \left(\zeta+\frac{\eta}{3}\right)\nabla(\nabla\cdot \mathbf{v}),\vspace{0.2cm}\\
    \displaystyle\partial_t\left(\epsilon+\frac{1}{2}\rho |\mathbf{v}|^2\right)
    +
    \nabla\cdot\left[\left(\epsilon+p+\frac{1}{2}\rho |\mathbf{v}|^2\right)\mathbf{v}
        - \mathbf{v}\cdot\sigma^\prime
    \right]
    =0.
\end{cases}
\label{eq:cfd2d}
\end{equation}
Here $\eta$ and $\zeta$ denote the shear and bulk viscosity coefficients, respectively, and $\sigma'$ is the viscous stress tensor. Because viscosity governs the dissipation of flow structures, larger values of $\eta$ and $\zeta$ cause the random initial perturbations to decay more rapidly, with the density, pressure, and velocity fields eventually approaching nearly spatially uniform states over the long simulation horizon. To retain nontrivial dynamics beyond the early stage of the trajectory, we therefore focus on a low-viscosity, low-Mach-number regime,
\[
    \mu=\zeta=10^{-8}, \qquad M=0.1.
\]
In this regime, spatial structures and coupled density-pressure-velocity interactions remain dynamically active well beyond the training interval rather than rapidly relaxing to a near-stationary state. The fixed Mach number $M=0.1$ maintains a stable low-Mach regime while avoiding strong shocks that could make the comparison overly sensitive to numerical-solver artifacts.

We consider the equations on a two-dimensional periodic domain and represent the learned state using the primitive variables
\[
    \mathbf{U}(t,x)
    =
    \bigl(\rho(t,x),v_1(t,x),v_2(t,x),p(t,x)\bigr).
\]
Under this setting, the learning task is to approximate the time-$h$ solution map
\[
    \mathcal S_h:
    \bigl(\rho(t),v_1(t),v_2(t),p(t)\bigr)
    \mapsto
    \bigl(\rho(t+h),v_1(t+h),v_2(t+h),p(t+h)\bigr),
\]
which is then applied autoregressively over the prediction horizon. The benchmark therefore tests long-horizon prediction of coupled multi-field dynamics, where errors in density, velocity, and pressure can interact and propagate through both nonlinear transport and viscous effects.

\paragraph{Evaluation metrics.} For the two-dimensional compressible-flow benchmark, we report
rollout-aggregated relative $L^2$ and $H^1$ errors for density, pressure, and velocity, together with their channel average. For
$\mathcal X\in\{L^2,H^1\}$, the metric is
\[
    \operatorname{Rel}\text{-}\mathcal X
    =
    \frac{1}{N_{\mathrm{test}}}
    \sum_{i=1}^{N_{\mathrm{test}}}
    \frac{
        \left(
            \sum_{n=1}^{T}
            \|\widehat u_i^n-u_i^n\|_{\mathcal X}^2
        \right)^{1/2}
    }{
        \left(
            \sum_{n=1}^{T}
            \|u_i^n\|_{\mathcal X}^2
        \right)^{1/2}
    },
\]
with
\[
    \|u\|_{H^1}^2=\|u\|_{L^2}^2+\|\nabla u\|_{L^2}^2.
\]
For the velocity field $\mathbf v=(v_x,v_y)$, the component norms are combined as
\[
    \|\mathbf v\|_{\mathcal X}^2=\|v_x\|_{\mathcal X}^2+\|v_y\|_{\mathcal X}^2.
\]
The relative $L^2$ error measures overall field accuracy, whereas the $H^1$ error additionally penalizes errors in spatial derivatives and is therefore more sensitive to the preservation of interfaces, gradients, and fine spatial structure. Unlike the inviscid Euler benchmark, the viscous compressible-flow trajectories remain sufficiently regular for the derivative-sensitive $H^1$ metric to be meaningful.

\subsubsection{Data Generation for 2D Compressible Fluid Dynamics}
\label{app:cfd2d_data_generation}
\paragraph{Initial conditions.}
For each trajectory, density, pressure, and the two velocity components are initialized from independent periodic Gaussian random fields on $\mathbb{T}^2$. A correlation length
\[
    \ell\sim\operatorname{Uniform}(0.05,0.15)
\]
is sampled once per trajectory and shared across the four physical channels. Conditional on $\ell$, independent random fields $f_\rho,f_{v_x},f_{v_y},f_p$ are generated spectrally:
\[
    \wh{f}(\mathbf k)=S_\ell(\mathbf k)^{1/2}\xi_{\mathbf k},\quad S_\ell(\mathbf{k})=\exp\left(-2\pi^2\ell^2|\mathbf{k}|^2\right),\quad\mathbf k = (k_x,k_y)\in\mathbb{Z}^2\backslash\{(0,0)\}.
\]
with the zero Fourier mode removed, $\wh{f}(0,0)=0$. This construction corresponds to a centered stationary Gaussian random field with radial a basis function kernel. Each realization is subsequently normalized to zero spatial mean and unit empirical variance.

After normalization, the primitive initial conditions are generated from
\[
\begin{aligned}
    &\rho(0,\mathbf{x}) = \rho_0 + 0.1\,f_\rho(\mathbf{x}),\quad p(0,\mathbf{x}) = p_0 \left(1+0.05\,f_p(\mathbf{x})\right),\\
    &v_x(0,\mathbf{x}) = Mc_0f_{v_x}(\mathbf{x}),\quad v_y(0,\mathbf{x}) = Mc_0f_{v_y}(\mathbf{x}),
\end{aligned}
\]
where the background density $\rho_0=1$, the background pressure $p_0=1/\Gamma$, the sound velocity $c_0=\sqrt{\Gamma p_0/\rho_0}=1$, and the Mach number $M=0.1$. To avoid invalid numerical states, density and pressure are positivity-clipped by $10^{-6}$ when necessary. 

\paragraph{Numerical solver.}
The equations are solved on a uniform $256\times256$ periodic grid. Inviscid fluxes are computed using an HLLC approximate Riemann solver \citep{toro1994restoration} with second-order MUSCL reconstruction \citep{van1979towards} and a minmod slope limiter. Viscous and heat-conduction terms are discretized using second-order centered finite differences, and time integration is performed using a two-stage strong-stability-preserving Runge-Kutta method \citep{shu1988efficient}. The internal time step is selected adaptively using an acoustic CFL number of $0.45$ together with the corresponding diffusive stability restriction. Density and pressure  positivity floors are applied after each Runge-Kutta stage. Numerical evolution is performed in double precision. The resulting trajectories are spectrally downsampled to $64\times64$, which is the resolution used for all learning experiments.

\paragraph{Temporal sampling and dataset splits.}
Snapshots are recorded every $\Delta t_{\mathrm{data}}=0.1$. Training and validation trajectories cover $t\in[0,1]$ and contain 10 one-step transitions. We generate 1200 training trajectories and 300 validation trajectories. The independently sampled test set contains 200 trajectories over $t\in[0,10]$, corresponding to 100 autoregressive prediction steps. Training, validation, and test initial conditions follow the same distribution, so the primary extrapolation in this benchmark is along the temporal direction.

\subsection{2D Incompressible Navier-Stokes equation}

We consider the 2D incompressible Navier-Stokes equations on the periodic domain $\mathbb{T}^2$. In velocity form,
\begin{equation}
    \partial_t\mathbf{u}
    +\mathbf{u}\cdot\nabla\mathbf{u}
    =
    -\nabla p+\nu\Delta\mathbf{u}+\mathbf{f}_u,
    \qquad
    \nabla\cdot\mathbf{u}=0,
    \label{eq:ns2d_velocity}
\end{equation}
where $\mathbf{u}(t,x)=(u_1(t,x),u_2(t,x))\in\mathbb{R}^2$ is the velocity field, $p(t,x)$ is the pressure, $\nu>0$ is the kinematic viscosity, and $\mathbf{f}_u$ is an external body force. Following \citet{li2020fourier, li2022transformer,chen2024positional}, we work with the scalar vorticity
\begin{equation*}
    \omega
    =
    \nabla\times\mathbf{u}
    =
    \partial_xu_2-\partial_yu_1.
\end{equation*}
Taking the curl of \eqref{eq:ns2d_velocity} eliminates the pressure and gives
\begin{equation}
    \partial_t\omega+\mathbf{u}\cdot\nabla\omega
    =
    \nu\Delta\omega+f,
    \qquad
    f=\nabla\times\mathbf{f}_u,
    \label{eq:ns2d_vorticity}
\end{equation}
where the velocity is recovered nonlocally from vorticity through the stream function:
\begin{equation}
    \mathbf{u}=\nabla^\perp\psi,
    \qquad
    -\Delta\psi=\omega,
    \qquad
    \nabla^\perp\psi=(\partial_y\psi,-\partial_x\psi).
    \label{eq:ns2d_biot_savart}
\end{equation}

In our experiments, the forcing field is time-independent, so the dynamics take the form
\begin{equation}
    \partial_t\omega+\mathbf{u}\cdot\nabla\omega
    =
    \nu\Delta\omega+f,
    \qquad
    \mathbf{u}=\nabla^\perp(-\Delta)^{-1}\omega.
    \label{eq:ns2d_forced}
\end{equation}
The learning task is to approximate the conditioned time-$\Delta t$ solution map
\begin{equation}
    \mathcal{S}_{\Delta t}:
    (\omega(t),f)\mapsto\omega(t+\Delta t),
    \label{eq:ns2d_solution_map}
\end{equation}
and to apply the learned map autoregressively over the test horizon.

This benchmark is challenging because the predicted vorticity determines the velocity field that transports future vorticity. Thus, small off-manifold errors in $\omega$ can induce nonlocal velocity errors, which then feed back into the advection term $\mathbf{u}\cdot\nabla\omega$. Under weak viscosity and persistent forcing, these errors may accumulate over long horizons and generate localized high-frequency artifacts. We therefore evaluate not only pointwise rollout accuracy, but also spectral and derivative-sensitive diagnostics such as enstrophy and palinstrophy.

\paragraph{Evaluation metrics.}
We report rollout-aggregated relative $L^2$ and $H^1$ errors of the vorticity field using the same
definitions as in the two-dimensional compressible-flow benchmark. The $L^2$ error measures
field-level accuracy, while the $H^1$ error is additionally sensitive to spatial gradients.

We further evaluate enstrophy and palinstrophy,
\[
    \mathsf{Ens}(\omega)
    =
    \frac12\|\omega\|_{L^2}^2,
    \qquad
    \mathsf{Pal}(\omega)
    =
    \frac12\|\nabla\omega\|_{L^2}^2,
\]
which respectively measure the overall vorticity magnitude and the strength of vorticity gradients.
For either scalar diagnostic
$\mathsf q\in\{\mathsf{Ens},\mathsf{Pal}\}$, we compute
\[
    \operatorname{RelErr}_{\mathsf q}
    =
    \frac{1}{N_{\mathrm{test}}}
    \sum_{i=1}^{N_{\mathrm{test}}}
    \frac{
        \left(
            \sum_{n=1}^{T}
            |\mathsf q(\widehat\omega_i^n)-\mathsf q(\omega_i^n)|^2
        \right)^{1/2}
    }{
        \left(
            \sum_{n=1}^{T}
            |\mathsf q(\omega_i^n)|^2
        \right)^{1/2}
    }.
\]
Finally, we compare the isotropic energy spectra of the predicted and reference flows. Let $E_i^n(k)$ and $\widehat E_i^n(k)$ denote the radially aggregated spectral energies at wavenumber shell $k$. The spectral error is
\[
    \operatorname{SpecErr}
    =
    \frac{1}{N_{\mathrm{test}}}
    \sum_{i=1}^{N_{\mathrm{test}}}
    \frac{
        \left(
            \sum_{n=1}^{T}\sum_k
            |\widehat E_i^n(k)-E_i^n(k)|^2
        \right)^{1/2}
    }{
        \left(
            \sum_{n=1}^{T}\sum_k
            |E_i^n(k)|^2
        \right)^{1/2}
    }.
\]
Together, these metrics assess field accuracy, spatial regularity, physically relevant flow statistics, and the distribution of energy across spatial scales.

\subsubsection{Data Generation for 2D Incompressible Navier-Stokes Benchmark}
\label{app:ns2d_data_generation}

\paragraph{Initial-vorticity distribution.}
For all viscosities and forcing families, the initial vorticity is sampled on the mean-zero subspace as
\[
    \omega_0
    \sim
    \mathsf{N}\left(
        0,\,
        7^3(-\Delta+49I)^{-5/2}
    \right).
\]
This distribution has the same spectral decay as the Gaussian prior used in \citet{li2020fourier}, with a larger overall amplitude to reduce the mismatch between the weak initial vorticity fields and the higher-amplitude states produced later under persistent forcing. The zero Fourier mode is removed to enforce zero spatial mean.

\paragraph{Trajectory-dependent forcing fields.}
Unlike the original FNO benchmark in \citet{li2020fourier}, which uses a fixed forcing function, we independently sample one time-independent forcing field for each trajectory. We consider two forcing families.
\begin{itemize}[topsep=0pt, itemsep=0pt]
\item For the \textsc{GRF} family, the forcing is sampled from a periodic squared-exponential Gaussian random field with spectral power
\[
    S_\ell(\mathbf{k})
    \propto
    \exp\left(-2\pi^2\ell^2|\mathbf{k}|^2\right),
\]
where the correlation length $\ell$ is sampled independently for each trajectory. The zero mode is
removed and each realization is rescaled to a randomly sampled target $L^\infty$ amplitude.

\item For the \textsc{Wave} family, the forcing is a sparse mixture of periodic Fourier modes,
\[
    f(\mathbf{x})
    =
    \sum_{j=1}^{J}
    a_j
    \cos\left(
        2\pi\mathbf{k}_j\cdot\mathbf{x}
        +\phi_j
    \right),
\]
with $J\in\{1,2,3,4\}$. The wavevectors have nonzero integer components with maximum magnitude three, the phases are sampled uniformly from $[0,2\pi)$, and the coefficients decay proportionally to $|\mathbf{k}_j|^{-2}$. Each realization is projected to zero mean and rescaled to a randomly sampled target $L^\infty$ amplitude.
\end{itemize}

The forcing distributions are summarized in Table~\ref{tab:ns2d_forcing_settings}. For the high-viscosity regime $\nu=10^{-3}$, stronger viscous dissipation causes the trajectories to approach a slowly varying regime more rapidly. We therefore use stronger and more spatially correlated forcing to retain nontrivial dynamics over the long simulation horizon.

\begin{table}[t]
    \centering
    \caption{
    Forcing distributions used for the two-dimensional incompressible Navier-Stokes datasets. Amplitude ranges refer to the target $L^\infty$ norm of each sampled forcing field.
    }
    \label{tab:ns2d_forcing_settings}
    \begin{tabular}{ccccc}
        \toprule
        Viscosity & Forcing & Length scale & Number of modes & Amplitude range \\
        \midrule
        $10^{-3}$ & \textsc{GRF} & $[0.05,0.25]$ & -- & $[0.10,0.30]$ \\
        $10^{-3}$ & \textsc{Wave} & -- & $1$--$4$ & $[0.10,0.30]$ \\
        $10^{-4},10^{-5}$ & \textsc{GRF} & $[0.05,0.15]$ & -- & $[0.08,0.20]$ \\
        $10^{-4},10^{-5}$ & \textsc{Wave} & -- & $1$--$4$ & $[0.08,0.16]$ \\
        \bottomrule
    \end{tabular}
\end{table}

\paragraph{Numerical solver.}
Trajectories are generated on a $256\times256$ periodic grid
using a pseudospectral method \citep{orszag1971numerical,canuto1986spectral}. The velocity and vorticity gradients are computed spectrally, while the nonlinear advection term is evaluated in physical space and transformed back to Fourier space. A two-thirds
spectral mask is applied to de-alias the nonlinear term.

Diffusion is advanced using a Crank-Nicolson discretization
\citep{crank1947practical}, while advection and forcing are treated explicitly. For an internal time step $\delta t$, the Fourier-space update is
\[
    \left(
        1+\frac{\delta t\,\nu}{2}|\mathbf{k}|^2
    \right)
    \widehat\omega^{\,m+1}_{\mathbf{k}}
    =
    \left(
        1-\frac{\delta t\,\nu}{2}|\mathbf{k}|^2
    \right)
    \widehat\omega^{\,m}_{\mathbf{k}}
    -
    \delta t\,
    \widehat{\mathbf{v}^m\cdot\nabla\omega^m}_{\mathbf{k}}
    +
    \delta t\,\widehat f_{\mathbf{k}}.
\]
We use $\delta t=10^{-4}$, following the standard FNO Navier-Stokes benchmark \citep{li2020fourier}. The resulting advective CFL number remains below $0.03$ across all settings. Snapshots are recorded every one physical time unit, and the zero vorticity mode is removed after each step. The simulated vorticity and forcing fields are then spectrally truncated from $256\times256$ to $64\times64$ for learning.

\paragraph{Training and evaluation datasets.}
For each viscosity regime, a single model is trained jointly on both forcing families. The training set
contains 1,600 trajectories, equally divided between \textsc{GRF} and \textsc{Wave} forcing, and
the validation set contains 400 trajectories with the same split. At test time, the model is evaluated separately on the two forcing families.

The training horizons are $T_{\mathrm{train}}=10,12$ and $8$ for $\nu=10^{-3},10^{-4}$ and $10^{-5}$, respectively, while the corresponding test horizons are $T_{\mathrm{test}}=50,30$ and $20$. Because snapshots are stored at unit intervals, these correspond directly to 10, 12, and 8 training transitions and 50, 30, and 20 autoregressive test steps. Initial-vorticity and forcing fields are sampled independently across the training, validation, and test sets from the same viscosity-dependent distributions. Thus, the primary extrapolation is again along the temporal direction.

\section{Implementation Details}\label{app:implementation_details}

This appendix gives the architectural and optimization details for the models compared in \S\ref{sec:experiments}. All models are implemented in PyTorch and trained as one-step predictors on adjacent snapshots. At test time, all methods are rolled out autoregressively from the exact initial condition, using deterministic mean predictions. 

\subsection{Common Data Handling and Model Selection}
For each benchmark, the stored trajectories are converted into one-step training pairs
\[
    (\mathbf u_n,\mathbf f,\mathbf u_{n+1}),
\]
where $\mathbf f$ is the trajectory-dependent forcing field when present and a zero-valued compatibility channel for the unforced compressible benchmarks. The same one-step pairs, trajectory validation loader, and train/validation split are used for all compared methods on a fixed benchmark.

All models are optimized with AdamW, learning rate $10^{-4}$, weight decay $10^{-5}$, and gradient clipping with maximum norm $1$. We also use a StepLR scheduler with decay factor $0.5$ every 100 epochs. All models are trained for 500 epochs and batch size 8. Validation is performed after each epoch. The checkpoint used for reporting is selected by the validation autoregressive rollout relative $L^2$ error when a trajectory validation loader is available; otherwise it is selected by the one-step validation loss.

For multichannel compressible states, the one-step prediction and reconstruction losses use channel-weighted mean-squared error, the channel weights are set to the inverse empirical variance of each physical channel on the training trajectories and normalized to have unit mean. This prevents high-variance channels from dominating the optimization objective. For the scalar vorticity benchmark no channel reweighting is needed.

\paragraph{Rollout safeguard.}
The reported autoregressive evaluations use deterministic model outputs. During validation and model selection, we optionally clip a single-step increment in $L^\infty$ norm to avoid numerical overflow from already-diverged checkpoints. The clip thresholds used in the reported runs are reported in Table~\ref{tab:implementation_hyperparameters}.

\subsection{FNO Backbone}
The direct FNO baseline \citep{li2020fourier} maps the current physical state to the next physical state,
\[
    \widehat{\mathbf u}_{n+1}=\mathbf u_n+\mathcal F_\theta(\mathbf u_n,\mathbf f,\mathbf x),
\]
where $\mathbf x$ denotes the periodic grid-coordinate channels. The residual form is used in all reported experiments. The model first lifts the concatenated input channels by a $1\times1$ convolution, applies a stack of Fourier layers with pointwise local convolutions and GELU activations, and projects back to the physical channels by two $1\times1$ convolutions. All FNO runs use width 64 and six Fourier layers. The 1D Euler experiments retain 32 one-dimensional Fourier modes, while the 2D compressible-flow and incompressible Navier-Stokes experiments retain $16\times16$ modes.

\subsection{Deterministic Latent Baseline}
FNO-AE uses the same spatial latent representation as VAMO but removes all stochastic components. The encoder is a resolution-preserving residual convolutional network
\[
    E_\phi:\mathbf u_n\mapsto \mathbf z_n
    \in\mathbb R^{C_z\times N_{\mathrm{res}}},
\]
with four residual blocks and hidden width 64. We use $C_z=32$ latent channels for the 1D Euler benchmark and $C_z=16$ for both the 2D compressible fluid dynamics and incompressible Navier-Stokes benchmarks. For multichannel states, the encoder input contains both the physical fields and centered finite-difference gradient features; for scalar vorticity, the same construction is applied channelwise. The decoder has the mirrored residual-convolutional structure and output dimension equal to the number of physical state channels.

The deterministic latent transition is a residual FNO,
\[
    \mathbf z_{n+1}=\mathbf z_n+\mathcal G_\theta(\mathbf z_n,\mathbf f,\mathbf x),
\]
using the same retained modes and layer counts as the corresponding VAMO latent transition. The deterministic latent objective is
\[
    \mathcal L_{\mathrm{AE}}
    =
    \mathcal L_{\mathrm{step}}
    +\lambda_{\mathrm{rec}}\mathcal L_{\mathrm{rec}},
\]
where $\mathcal L_{\mathrm{step}}$ is the one-step prediction loss and $\mathcal L_{\mathrm{rec}}$ reconstructs the current state from the encoded latent. The deterministic latent runs use $\lambda_{\mathrm{rec}}=1$.

\subsection{VAMO Architecture and Structured Noise}
VAMO uses the same encoder, decoder, and residual latent FNO mean transition as FNO-AE, but the encoder outputs both a mean field and a log-variance field:
\[
    (\mu_\phi(\mathbf u),\ell_\phi(\mathbf u))
    =\mathrm{Enc}_\phi(\mathbf u).
\]
During training, the posterior log-variance is clamped to $[-8,2]$ before sampling. A posterior latent sample is generated by
\[
    \mathbf z_n=\mu_\phi(\mathbf u_n)+\exp\left(\frac12\ell_\phi(\mathbf u_n)\right)
    \odot \varepsilon_{\mathrm{enc}},
\]
where $\varepsilon_{\mathrm{enc}}$ is a spectrally filtered standard Gaussian field. This filtering is implemented by FFT as
\[
    \widehat{\varepsilon}(\mathbf k)
    =
    \left(1+\ell^2 |2\pi\mathbf k|^2\right)^{-s/2}
    \widehat{\xi}(\mathbf k),
\]
with orthonormal FFT normalization. This is the discrete implementation of the Laplacian-resolvent covariance described in \S\ref{subsec:structured_latent_perturbations}. The spectral decay exponent is $s=2$ in all VAMO runs.

The latent prior mean is the residual FNO transition
\[
    \mu_p=T_\theta(\mathbf z_n),
\]
and the transition noise amplitude is a scalar predicted separately for each sample. The amplitude head consists of two convolutional layers with GELU activations, global spatial averaging, and a final $1\times1$ convolution. In the multichannel implementation used for the reported Euler, compressible-flow, and incompressible Navier-Stokes experiments, the final scalar is passed through a softplus and then shifted and clipped:
\[
    \alpha_\theta(\mathbf z,\mathbf f)
    =\min\{\alpha_{\max},\, \alpha_{\min}
    +\operatorname{softplus}(a_\theta(\mathbf z,\mathbf f))\}.
\]
The softplus component is initialized to $0.10$. Thus the initial effective amplitude before clipping is approximately $0.1001$ for Euler and CFD2D, where $\alpha_{\min}=10^{-4}$, and $0.11$ for NS2D, where $\alpha_{\min}=10^{-2}$.

The stochastic transition sample is
\[
    \mathbf z_{n+1}
    =
    \mu_p
    +
    \alpha_\theta(\mathbf z_n,\mathbf f)\varepsilon_{\mathrm{tr}},
\]
where $\varepsilon_{\mathrm{tr}}$ is generated by the same spectral filter with the transition correlation length. During training, the CFD2D and NS2D VAMO samplers add a spectral variance floor $10^{-2}$ before taking the square root of the filter variance. Inference uses the deterministic mean only: the initial state is encoded as $\mu_\phi(\mathbf u_0)$, the latent is advanced by $T_\theta$, and each latent state is decoded by $D_\psi$.

The VAMO training loss is
\[
    \mathcal L_{\mathrm{VAMO}}
    =\mathcal L_{\mathrm{step}}
    +\beta_{\mathrm{KL}}\mathcal L_{\mathrm{KL}}
    +\lambda_{\mathrm{rec}}\mathcal L_{\mathrm{rec}} .
\]
The prediction and reconstruction terms are the same channel-weighted MSE terms used for the deterministic latent baseline. The KL term compares the encoded posterior at the next state with the predicted latent prior. In the numerical implementation this KL is evaluated using the encoder's pointwise diagonal posterior variance and the transition's scalar per-sample variance $\alpha_\theta^2$; the spectral covariance is used for  reparameterized sampling of the stochastic latent states.

\begin{table}[t]
    \centering
    \caption{
    Main implementation hyperparameters from the experiment configurations. 
    }
    \label{tab:implementation_hyperparameters}
    \setlength{\tabcolsep}{4pt}
    \resizebox{\textwidth}{!}{
    \begin{tabular}{lccccc}
        \toprule
        Setting
        & \makecell{1D Euler\\ $L\in\{5,10,15\}$}
        & \makecell{1D Euler\\ $L=20$}
        & \makecell{2D CFD\\ $\nu=\eta=10^{-8}$}
        & \makecell{2D Navier-Stokes\\ $\nu=10^{-3}$}
        & \makecell{2D Navier-Stokes\\ $\nu\in\{10^{-4},10^{-5}$\}} \\
        \midrule
        \# Training trajectories
        & 1200 & 1200 & 1200 & 1600 & 1600 \\
        \# Validation trajectories
        & 300 & 300 & 300 & 400 & 400 \\
        Resolution $N_{\rm res}$ & \{1024, 2048, 3072\} & 4096 & $64\times 64$ & $64\times 64$ & $64\times 64$\\
        Batch size
        & 8 & 8 & 8 & 8 & 8 \\
        \# Epochs
        & 500 & 500 & 500 & 500 & 500 \\
        Training transitions $N_{\rm train}$
        & 15 & 15 & 10 & 10 & \{12, 8\} \\
        FNO width
        & 64 & 64 & 64 & 64 & 64 \\
        \# FNO layers
        & 6 & 6 & 6 & 6 & 6 \\
        Fourier modes
        & 32 & 32 & $16\times16$ & $16\times16$ & $16\times16$ \\
        \# Latent channels $C_z$
        & 32 & 32 & 16 & 16 & 16 \\
        \# Encoder/decoder blocks
        & 4 / 4 & 4 / 4 & 4 / 4 & 4 / 4 & 4 / 4 \\
        Encoder filter $(\ell_\mathrm{enc},s)$
        & $(0.01,2)$ & $(0.01,2)$ & $(0.01,2)$ & $(0.01,2)$ & $(0.01,2)$\\
        Transition filter $(\ell_\mathrm{tr},s)$
        & $(0.05,2)$ & $(0.05,2)$ & $(0.1,2)$ & $(0.05,2)$ & $(0.05,2)$ \\
        Amplitude range $[\alpha_{\min},\alpha_{\max}]$
        & $[10^{-4},5]$ & $[10^{-4},5]$ & $[10^{-4},5]$ & $[10^{-2},5]$ & $[10^{-2},5]$ \\
        Loss weights $(\beta_{\mathrm{KL}},\lambda_{\mathrm{rec}})$
        & $(10^{-3},1)$ & $(10^{-2},1)$ & $(10^{-4},1)$ & $(10^{-3},1)$ & $(10^{-2},1)$ \\
        FNO+Noise lift std.
        & 0.02 & 0.02 & 0.02 & 0.1 & 0.1 \\
        FNO+Noise lift length
        & 0.05 & 0.05 & 0.1 & 0.05 & 0.05 \\
        Rollout increment clip
        & 5 & 5 & 10 & 10 & 10 \\
        \bottomrule
    \end{tabular}
    }
\end{table}

\subsection{FNO+Noise Baseline and Model Complexity}
FNO+Noise uses the same architecture, optimizer, data loaders, and model selection rule as the direct FNO baseline. The only change is that during training we add spectrally filtered Gaussian noise to the lifted FNO feature field before the Fourier layers:
\[
    h \leftarrow h+\sigma_{\mathrm{lift}}\varepsilon_{\mathrm{lift}},
\]
where $\varepsilon_{\mathrm{lift}}$ is sampled with the same FFT filter family as above and spectral decay $s=2$. The experiments use $\sigma_{\mathrm{lift}}=0.02$ for Euler1D and CFD2D, and $\sigma_{\mathrm{lift}}=0.1$ for NS2D, which are comparable to the VAMO counterparts. The corresponding lift-noise correlation lengths are reported in  Table~\ref{tab:implementation_hyperparameters}. This noise is disabled during validation and test rollout. The baseline therefore tests generic feature-space noise injection without a latent posterior, learned transition uncertainty, or KL alignment term.

\paragraph{Model size.}
For completeness, Table~\ref{tab:model_params} reports the number of trainable parameters for the models used in each benchmark. The FNO+Noise baseline has exactly the same trainable architecture as the direct FNO baseline; its additional stochastic perturbation introduces no learnable parameters. Similarly, FNO-AE and VAMO share the same encoder-transition-decoder backbone. The additional parameters in VAMO arise only from the encoder variance output and the lightweight transition amplitude head.

\begin{table}[t]
    \centering
    \caption{
    Number of trainable parameters for the evaluated models. FNO+Noise has the same architecture and parameter count as FNO. Counts are independent of the spatial resolution because all models are fully convolutional
    or neural-operator based.
    }
    \label{tab:model_params}
    \small
    \begin{tabular}{lrrr}
        \toprule
        Benchmark & FNO & FNO-AE & FNO-VAMO \\
        \midrule
        1D Euler
        & 1.607M & 1.945M & 1.979M \\
        2D compressible fluid dynamics
        & 25.200M & 25.801M & 25.817M \\
        2D incompressible Navier-Stokes
        & 25.200M & 25.796M & 25.811M \\
        \bottomrule
    \end{tabular}
\end{table}

\section{Additional Experimental Results}
\subsection{Comparison with Autoregressive Curriculum Training}
\label{app:curriculum_training}

Autoregressive curriculum training has been used to reduce the discrepancy between teacher-forced training and autoregressive deployment by progressively exposing a model to longer rollouts generated from its own predictions \citep{li2022transformer,hagnberger2025calm}. To assess whether a deterministic baseline trained with explicit multistep rollouts can match VAMO's long-horizon performance, we construct a curriculum-trained FNO baseline for the two more challenging Navier-Stokes regimes, $\nu=10^{-4}$ and $\nu=10^{-5}$. The comparison between standard and curriculum-trained FNO isolates the effect of the training strategy, as the underlying architecture and evaluation protocol remain unchanged.

\paragraph{Curriculum setting.} Specifically, let $K$ denote the current curriculum length. Starting from the ground-truth initial state, the model is applied autoregressively for the first $K$ steps, with each prediction fed back as the input to the next step. For the remaining steps in the training trajectory, the model uses ground-truth inputs as in standard teacher forcing. Thus, increasing $K$ gradually shifts the objective from one-step prediction toward full autoregressive rollout training. The curriculum begins with $K=1$, which is maintained for the first 100 epochs. The rollout length is then increased linearly until it reaches the full training horizon, namely $K=12$ for $\nu=10^{-4}$ and $K=8$ for $\nu=10^{-5}$, at the end of 500 training epochs.

The curriculum baseline uses the same optimizer settings and FNO architecture as the standard
FNO baseline, with width 64, 6 Fourier layers, and $16\times16$ retained Fourier modes. At test time, all methods are evaluated identically: starting from the initial condition, each model is rolled out autoregressively and without intermediate ground-truth correction until $T_{\mathrm{test}}$. Thus, the comparison isolates the effect of the curriculum training strategy
while keeping the model architecture and evaluation protocol unchanged.

Table~\ref{tab:curriculum_training} reports relative $L^2$ and $H^1$ errors over both the training horizon $[0,T_{\mathrm{train}}]$ and the full rollout horizon $[0,T_{\mathrm{test}}]$. We restrict this comparison to these two metrics because its purpose is to assess predictive accuracy and spatial regularity under different training strategies; the additional physical diagnostics are studied separately in the main experiments.

\begin{table}[t]
\centering
\caption{
Comparison with autoregressive curriculum training on the two more challenging Navier-Stokes regimes. Errors are evaluated over the training horizon $[0,T_{\mathrm{train}}]$ and the full rollout horizon $[0,T_{\mathrm{test}}]$. All entries report mean $\pm$ standard deviation over 200 test trajectories.
}
\resizebox{\textwidth}{!}{%
\begin{tabular}{llllcccc}
\toprule
\multicolumn{4}{c}{Evaluation setting} &
\multicolumn{2}{c}{Training horizon} &
\multicolumn{2}{c}{Full horizon} \\
\cmidrule(lr){1-4}
\cmidrule(lr){5-6}
\cmidrule(lr){7-8}
Viscosity & Forcing & Method & Training
& Rel.\,$L^2$
& Rel.\,$H^1$
& Rel.\,$L^2$
& Rel.\,$H^1$ \\
\midrule

\multirow{6}{*}{$\nu=10^{-4}$}
& \multirow{3}{*}{GRF}
& FNO & Standard
& $0.0123_{\,\pm0.0341}$ & $0.1013_{\,\pm0.3245}$ & $0.4068_{\,\pm0.5094}$ & $3.1094_{\,\pm4.0204}$ \\
&
& FNO & Curriculum
& $0.0093_{\,\pm0.0053}$ & $0.0513_{\,\pm0.0357}$ & $0.1693_{\,\pm0.1981}$ & $0.9812_{\,\pm1.2503}$ \\
&
& VAMO & Standard
& $\mathbf{0.0057}_{\,\pm0.0017}$ & $\mathbf{0.0218}_{\,\pm0.0070}$ & $\mathbf{0.0403}_{\,\pm0.0562}$ & $\mathbf{0.1139}_{\,\pm0.1547}$ \\
\cmidrule(lr){2-8}

& \multirow{3}{*}{\textsc{Wave}}
& FNO & Standard
& $0.0118_{\,\pm0.0363}$ & $0.0828_{\,\pm0.2990}$ & $0.3570_{\,\pm0.6102}$ & $1.6131_{\,\pm2.7778}$ \\
&
& FNO & Curriculum
& $0.0109_{\,\pm0.0162}$ & $0.0470_{\,\pm0.0786}$ & $0.2096_{\,\pm0.2991}$ & $0.5641_{\,\pm0.8087}$ \\
&
& VAMO & Standard
& $\mathbf{0.0058}_{\,\pm0.0031}$ & $\mathbf{0.0182}_{\,\pm0.0110}$ & $\mathbf{0.1198}_{\,\pm0.1785}$ & $\mathbf{0.2212}_{\,\pm0.2680}$ \\
\midrule

\multirow{6}{*}{$\nu=10^{-5}$}
& \multirow{3}{*}{GRF}
& FNO & Standard
& $\mathbf{0.0132}_{\,\pm0.0050}$ & $\mathbf{0.1186}_{\,\pm0.0502}$ & $0.2987_{\,\pm0.4330}$ & $2.2621_{\,\pm3.6585}$ \\
&
& FNO & Curriculum
& $0.0197_{\,\pm0.0044}$ & $0.1668_{\,\pm0.0421}$ & $0.1503_{\,\pm0.2719}$ & $0.7867_{\,\pm0.7627}$ \\
&
& VAMO & Standard
& $0.0139_{\,\pm0.0018}$ & $0.1238_{\,\pm0.0230}$ & $\mathbf{0.0841}_{\,\pm0.0560}$ & $\mathbf{0.3858}_{\,\pm0.1561}$ \\
\cmidrule(lr){2-8}

& \multirow{3}{*}{\textsc{Wave}}
& FNO & Standard
& $\mathbf{0.0122}_{\,\pm0.0059}$ & $\mathbf{0.0998}_{\,\pm0.0472}$ & $0.2993_{\,\pm0.4083}$ & $1.9685_{\,\pm3.1219}$ \\
&
& FNO & Curriculum
& $0.0184_{\,\pm0.0052}$ & $0.1453_{\,\pm0.0524}$ & $0.1899_{\,\pm0.1456}$ & $0.6596_{\,\pm0.4390}$ \\
&
& VAMO & Standard
& $0.0126_{\,\pm0.0024}$ & $0.1041_{\,\pm0.0294}$ & $\mathbf{0.1469}_{\,\pm0.1127}$ & $\mathbf{0.4509}_{\,\pm0.1981}$ \\
\bottomrule
\end{tabular}}
\label{tab:curriculum_training}
\end{table}

As shown in Table~\ref{tab:curriculum_training}, curriculum training substantially improves the FNO baseline over the full rollout horizon, reducing its relative $L^2$ errors by $36.6\%$--$58.4\%$ and its relative $H^1$ errors by $65.0\%$--$68.4\%$ across the four evaluation settings. Nevertheless, VAMO achieves the lowest full-horizon errors in all cases, further reducing the $L^2$ and $H^1$ errors of FNO-Curriculum by $22.6\%$--$76.2\%$ and $31.6\%$--$88.4\%$, respectively. VAMO also exhibits consistently smaller standard deviations over the full rollout horizon, indicating substantially lower variability across test trajectories and more reliable long-horizon predictions. Notably, for $\nu=10^{-5}$, the improved full-horizon performance of FNO-Curriculum is accompanied by larger training-horizon errors than standard FNO under both forcing families. This suggests an accuracy-stability tradeoff rather than a uniform improvement across the temporal regimes represented during training and deployment.

Overall, autoregressive curriculum training considerably improves the deterministic FNO baseline by exposing it to model-generated states, but it does not close the gap to VAMO in any of the four full-horizon comparisons. Notably, VAMO attains stronger long-horizon performance using only a local variational objective, whereas FNO-Curriculum explicitly optimizes progressively longer autoregressive rollouts up to the full training horizon. Nevertheless, the two mechanisms are potentially complementary. A promising extension is to derive a rollout-aware multistep variational objective for VAMO and incorporate a rollout-length curriculum into its prediction terms. Such an extension could further reduce the mismatch between local training and long-horizon autoregressive deployment while retaining structured variational regularization of the latent dynamics.

\subsection{Additional Quantitative Results}
\label{app:additional_quantitative}

We provide additional quantitative results complementing the aggregate comparisons and temporal error analysis in \S\ref{sec:experiments}. In particular, we report per-trajectory error distributions for the 1D Euler benchmark and the $\nu=10^{-3}$ incompressible
Navier-Stokes setting, together with the remaining temporal error trends and physical diagnostics
omitted from the main text for brevity.

\paragraph{1D Euler equations.}
Figures~\ref{fig:euler1d_error_distribution1}--\ref{fig:euler1d_error_distribution2} show the distributions of per-trajectory full-rollout relative $L^2$ and $L^1$ errors for all four domain lengths. Across the evaluated settings, VAMO generally exhibits lower typical errors and reduced upper tails, with the clearest separation appearing in the more challenging large-domain regime. For $L=20$, in particular, the deterministic latent baseline develops a pronounced heavy tail associated with unstable rollouts, whereas the VAMO error distribution remains substantially more concentrated. These results complement the aggregate errors in Table~\ref{tab:euler1d_results} by showing that the long-horizon improvement is reflected across the test distribution rather than being driven by a small number of trajectories.

Figure~\ref{fig:euler_trend_2} further reports the temporal evolution of the Euler rollout errors for the additional domain-length settings. The same pattern observed in \S\ref{sec:error_growth} persists: the methods can remain comparable during the early rollout, while their behavior separates as prediction proceeds beyond the training horizon. VAMO exhibits slower late-time error growth, consistent with the full-rollout results in Table~3 and the representative settings shown in the main text.

\begin{figure}
    \centering
    \begin{subfigure}{\linewidth}
        \centering
        \caption{1D Euler per-trajectory rollout errors, domain length $L=5$.\vspace{0.1cm}}
        \includegraphics[width=\linewidth]{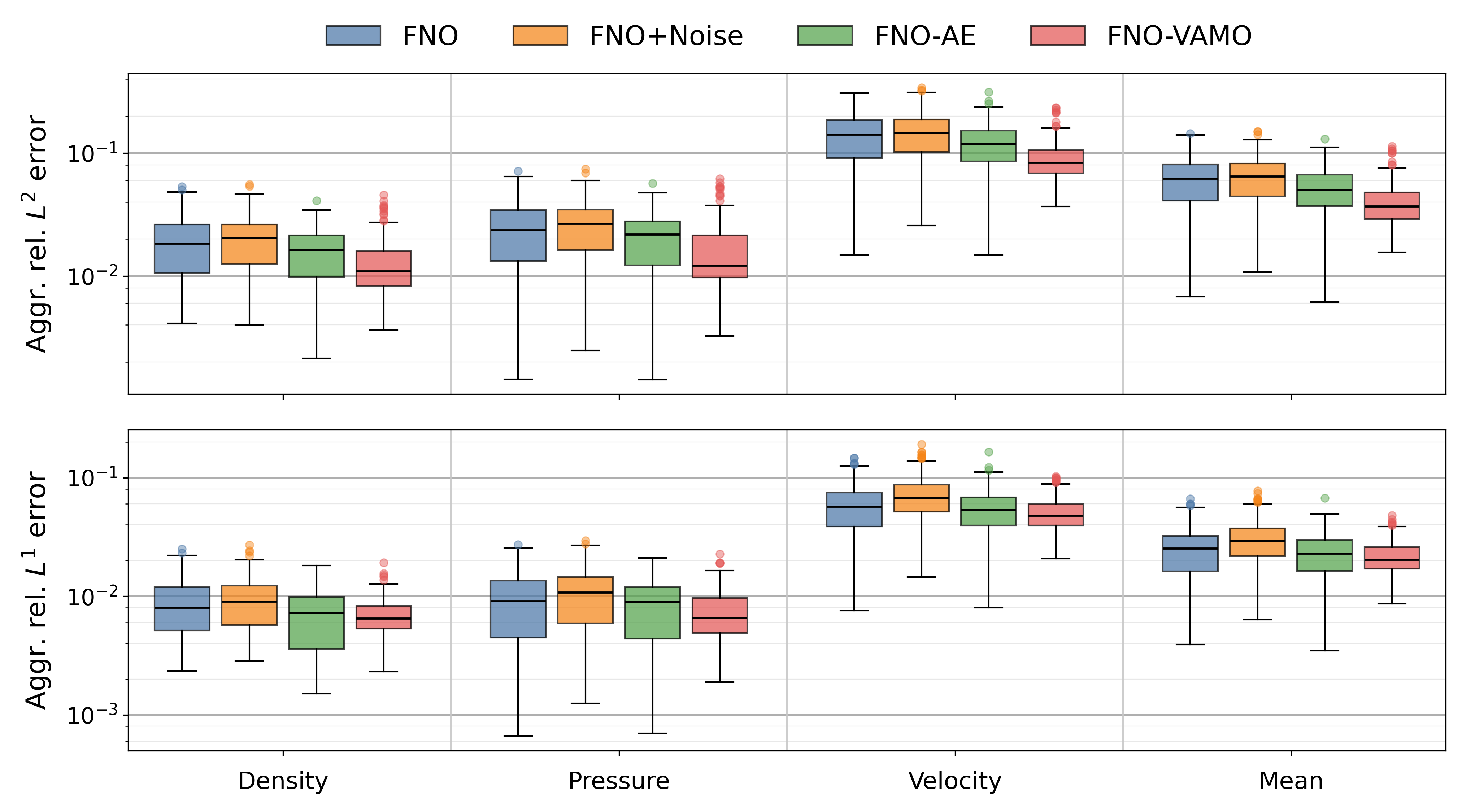}\vspace{-0.2cm}
        \label{fig:euler1d_l5_error_distribution}
    \end{subfigure}
    \begin{subfigure}{\linewidth}
        \centering
        \caption{1D Euler per-trajectory rollout errors, domain length $L=10$.\vspace{0.1cm}}
        \includegraphics[width=\linewidth]{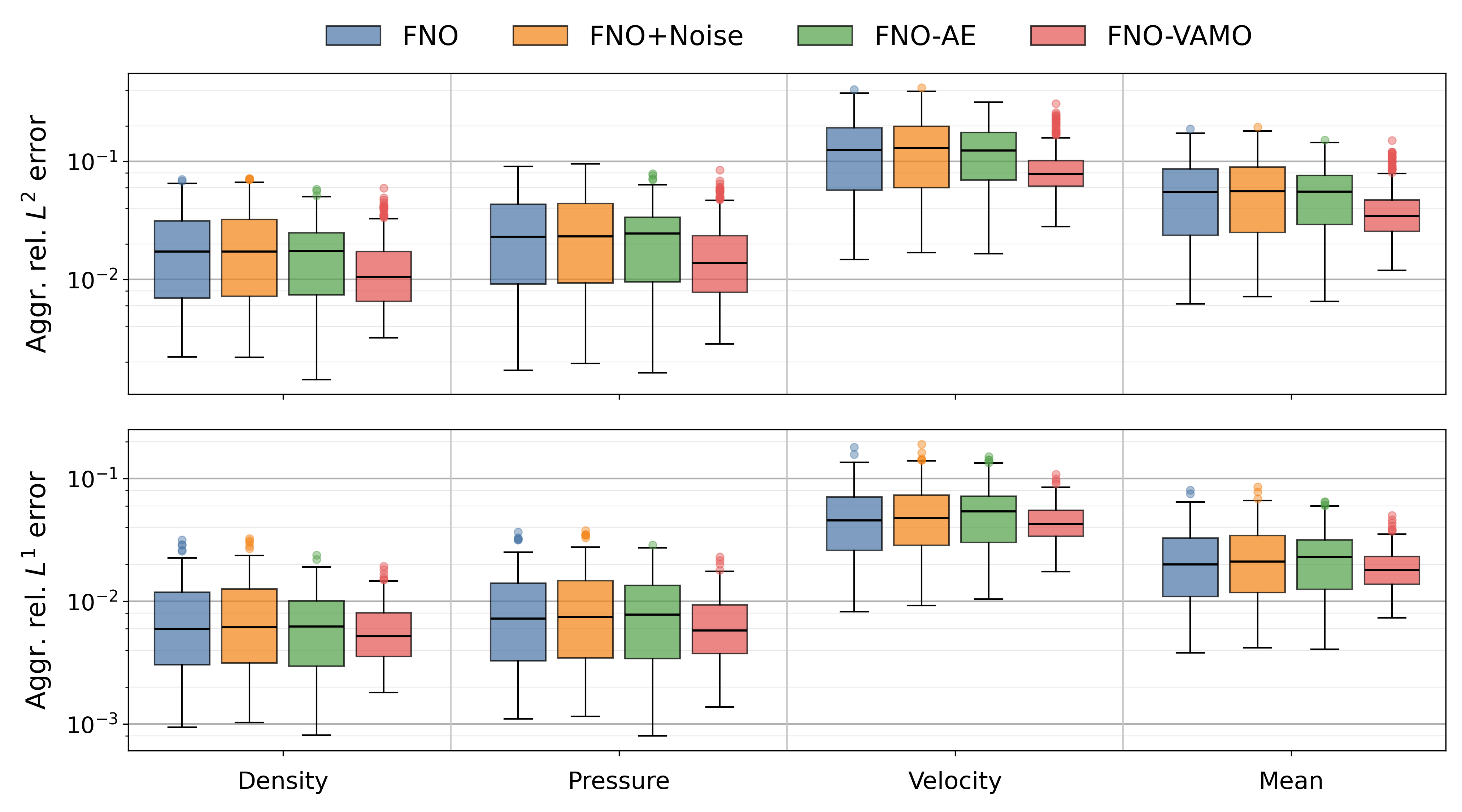}\vspace{-0.2cm}
        \label{fig:euler1d_l10_error_distribution}
    \end{subfigure}
    \caption{
    Distribution of per-trajectory aggregate relative $L^2$ and $L^1$ errors for the 1D Euler benchmark. The logarithmic scale highlights both the typical error and the heavy upper tails associated with unstable rollouts. 
    }
    \label{fig:euler1d_error_distribution1}
\end{figure}

\begin{figure}
    \centering
    \begin{subfigure}{\linewidth}
        \centering
        \caption{1D Euler per-trajectory rollout errors, domain length $L=15$.\vspace{0.1cm}}
        \includegraphics[width=\linewidth]{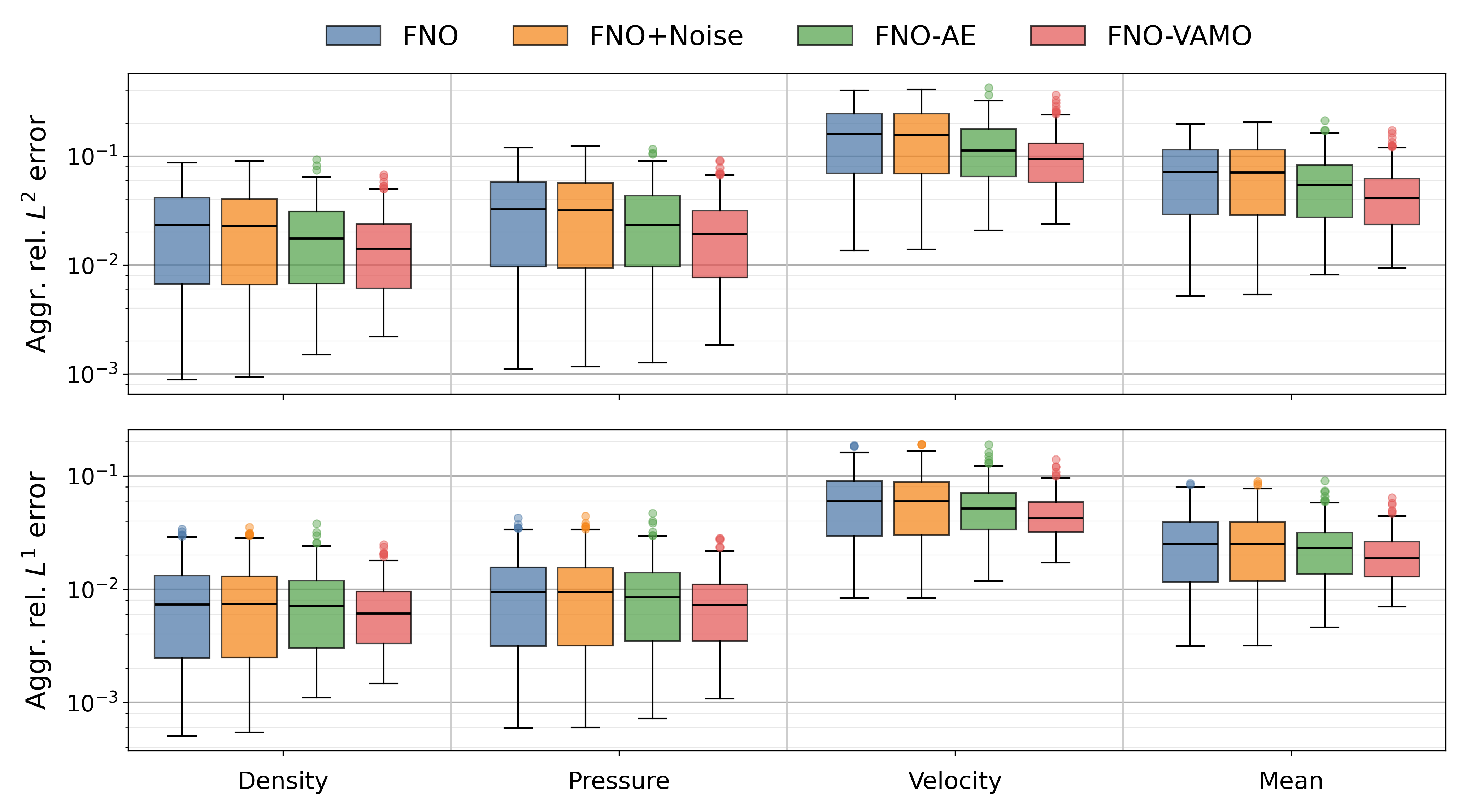}\vspace{-0.2cm}
        \label{fig:euler1d_l15_error_distribution}
    \end{subfigure}
    \begin{subfigure}{\linewidth}
        \centering
        \caption{1D Euler per-trajectory rollout errors, domain length $L=20$.\vspace{0.1cm}}
        \includegraphics[width=\linewidth]{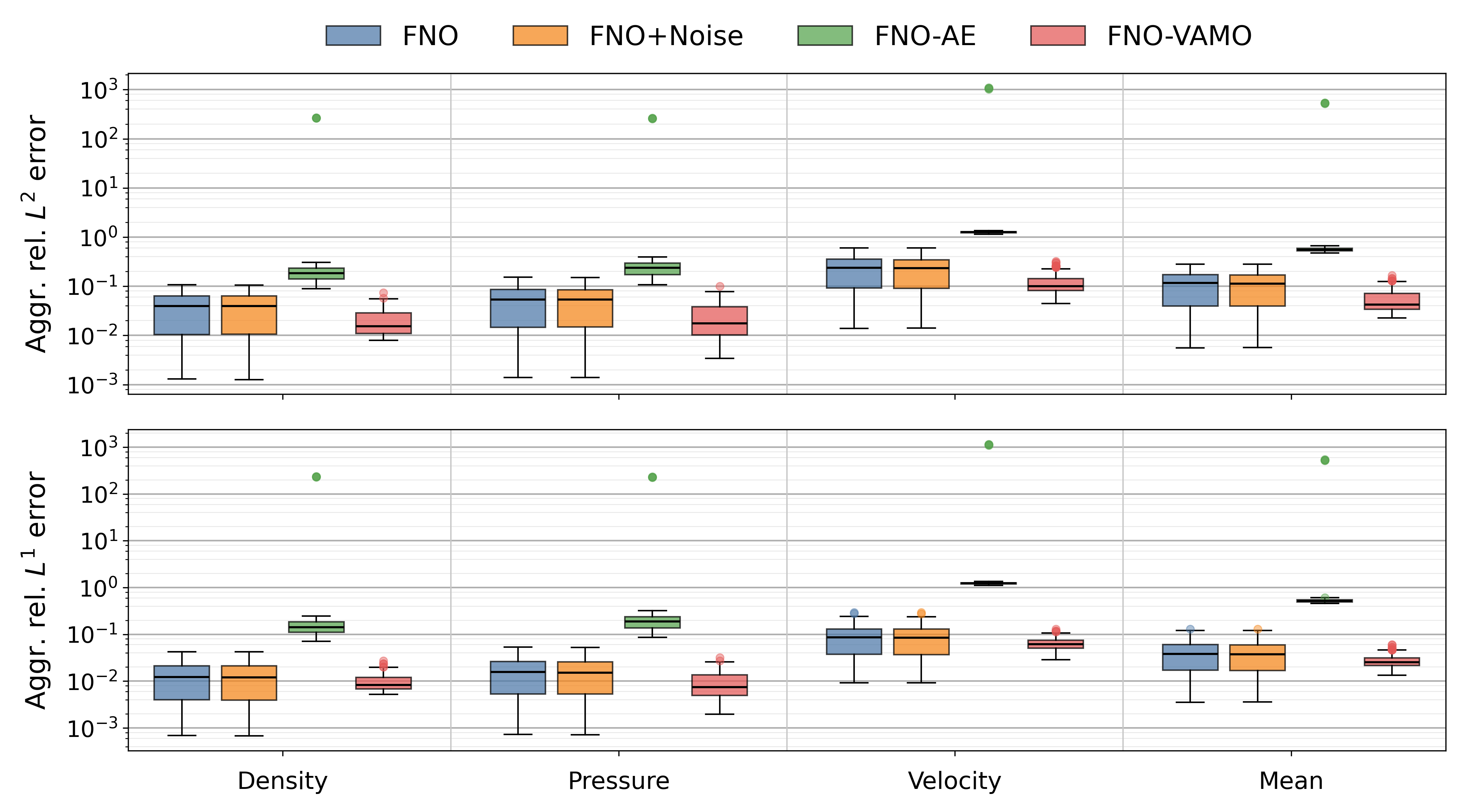}\vspace{-0.2cm}
        \label{fig:euler1d_l20_error_distribution}
    \end{subfigure}
    \caption{
    Distribution of per-trajectory aggregate relative $L^2$ and $L^1$ errors for the 1D Euler benchmark. The logarithmic scale highlights both the typical error and the heavy upper tails associated with unstable rollouts. 
    }
    \label{fig:euler1d_error_distribution2}
\end{figure}

\begin{figure}
    \centering
    \begin{subfigure}[h]{\linewidth}
        \centering
        \includegraphics[width=1.0\linewidth]{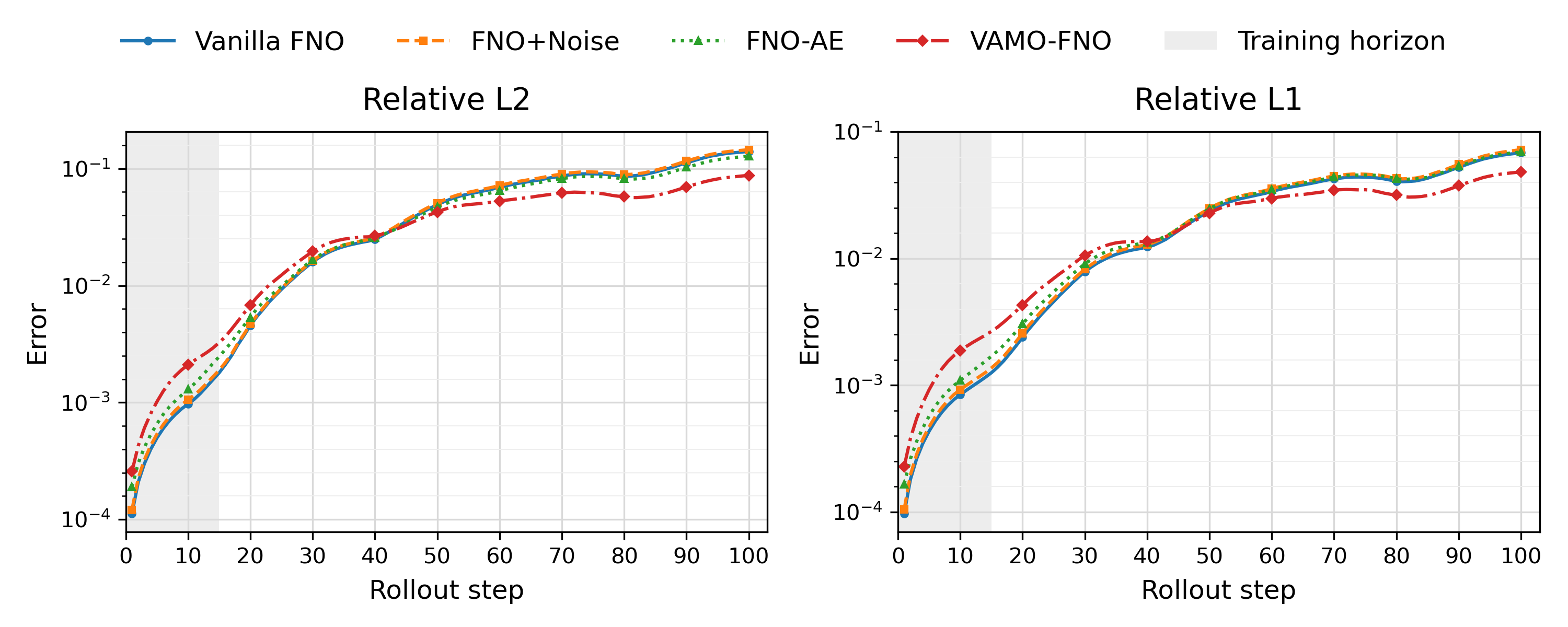}
        \caption{Domain length $L=10$.}
        \label{fig:eulerl10}
    \end{subfigure}

    \vspace{0.5em}
    
    \begin{subfigure}[h]{\linewidth}
        \centering
        \includegraphics[width=1.0\linewidth]{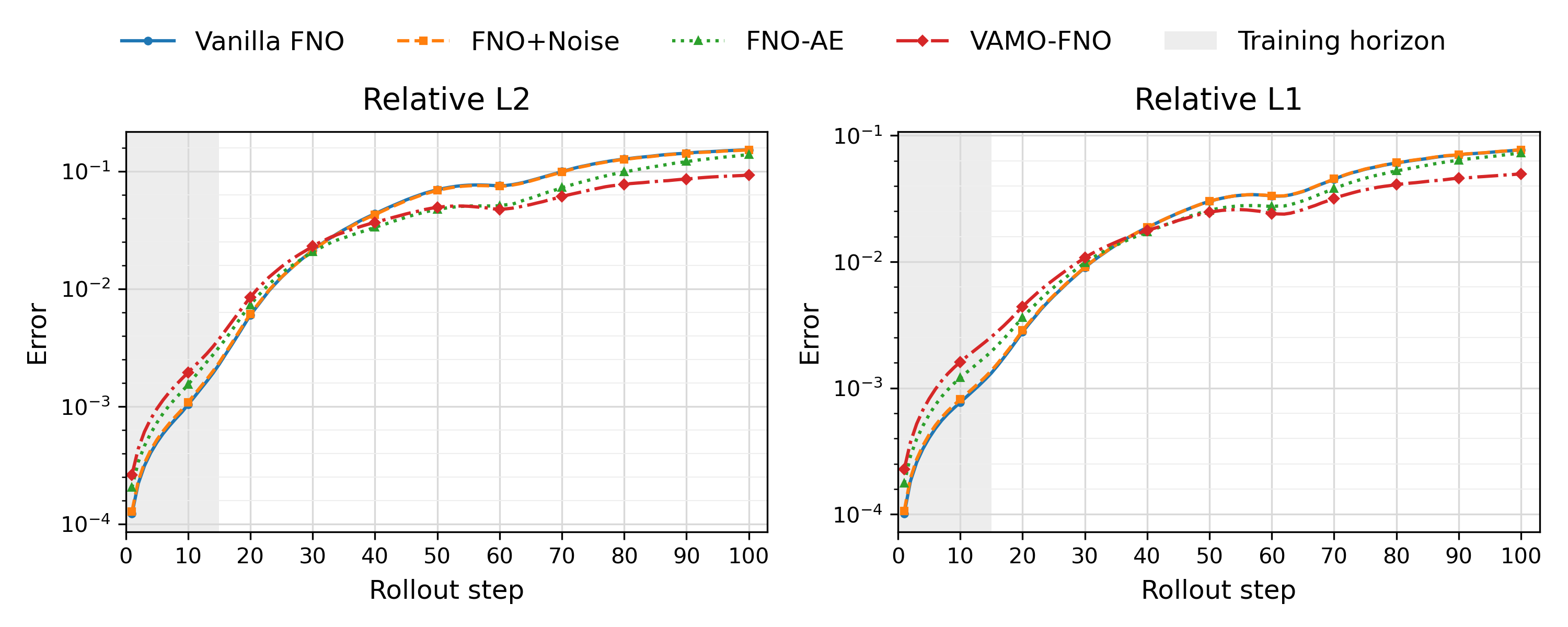}
        \caption{Domain length $L=20$.}
        \label{fig:eulerl15}
    \end{subfigure}
    \caption{Error trend for 1D Euler Benchmarks with $L\in\{10,15\}$.}
    \label{fig:euler_trend_2}
\end{figure}

\paragraph{2D incompressible Navier-Stokes equations.}
We additionally report results for the more dissipative $\nu=10^{-3}$ regime, which is omitted from the detailed temporal analysis in the main text. Figure~\ref{fig:ns2d_nu3_error_distribution} shows the per-trajectory full-rollout relative $L^2$ and $H^1$ error distributions under both GRF and Wave forcing. VAMO maintains low typical errors and comparatively tight distributions for both forcing families. The deterministic latent baseline is substantially less stable and exhibits pronounced heavy-tailed errors, while the direct and noise-injection baselines remain more stable but generally less accurate than VAMO.

The corresponding temporal errors are shown in Figure~\ref{fig:ns2d_nu3_trend}. Under both forcing families, VAMO limits the growth of field-level, derivative-sensitive, and physical-statistic errors throughout the long rollout. The advantage is particularly clear relative to the deterministic latent model, whose errors grow rapidly after the training horizon. Compared with the lower-viscosity regimes studied in the main text, the direct FNO and FNO+Noise baselines are more competitive at $\nu=10^{-3}$, consistent with the stronger physical dissipation in this setting.

Finally, Figure~\ref{fig:ns2d_physical_diagnostics_nu3} compares the enstrophy and palinstrophy trajectories for representative $\nu=10^{-3}$ test samples. The deterministic latent baseline frequently develops severe late-time inflation in both statistics, while the direct baselines exhibit smaller but still visible deviations on
some trajectories. VAMO more closely tracks the scale and temporal evolution of the reference statistics across both forcing families. Together with the results for $\nu=10^{-4}$ and $\nu=10^{-5}$ in \S\ref{sec:physical_diagnostics}, these observations show that the improved physical-statistic stability of VAMO persists across the full range of viscosities considered in our experiments.

\begin{figure}
    \centering
    \includegraphics[width=\linewidth]{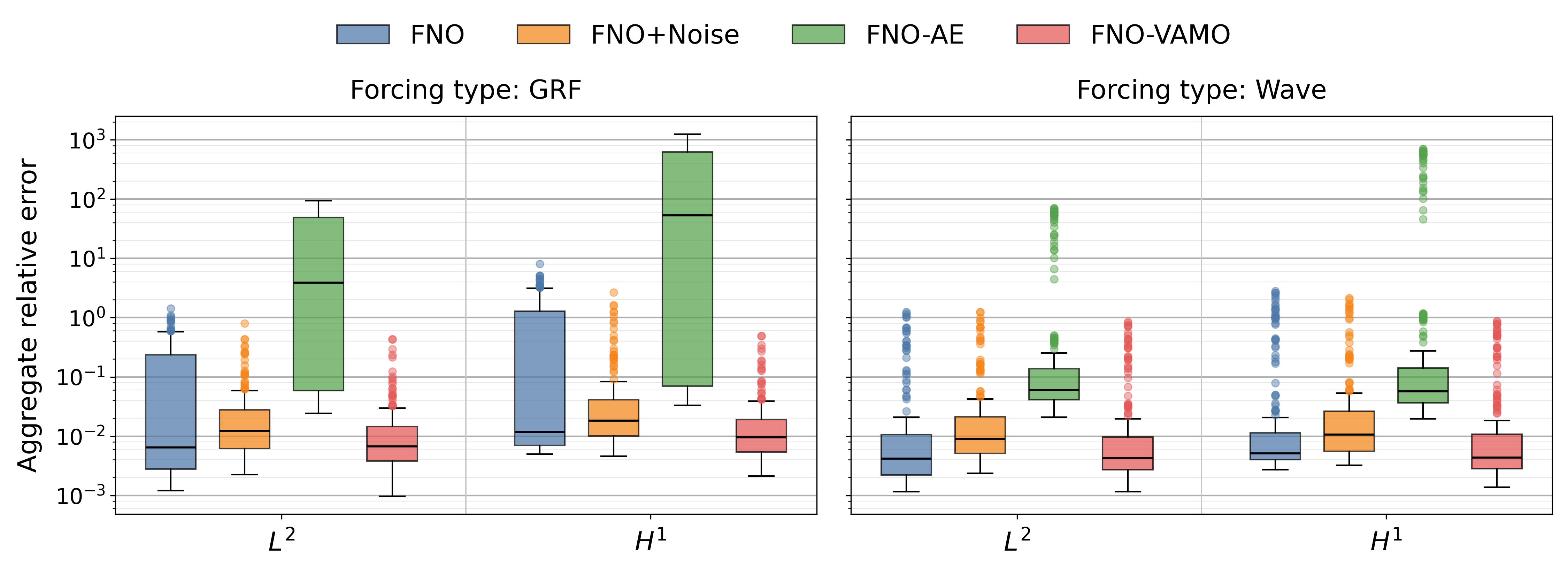}
    \caption{
    Distribution of per-trajectory full-rollout relative $L^2$ and $H^1$ errors for the 2D incompressible Navier-Stokes benchmark, $\nu=10^{-3}$. Results are shown separately for GRF and \textsc{Wave} forcing. The logarithmic scale highlights both the typical error and the heavy upper tails associated with unstable rollouts.}
    \label{fig:ns2d_nu3_error_distribution}
\end{figure}

\begin{figure}
    \centering
    \begin{subfigure}[t]{\linewidth}
        \centering
        \includegraphics[width=1.0\linewidth]{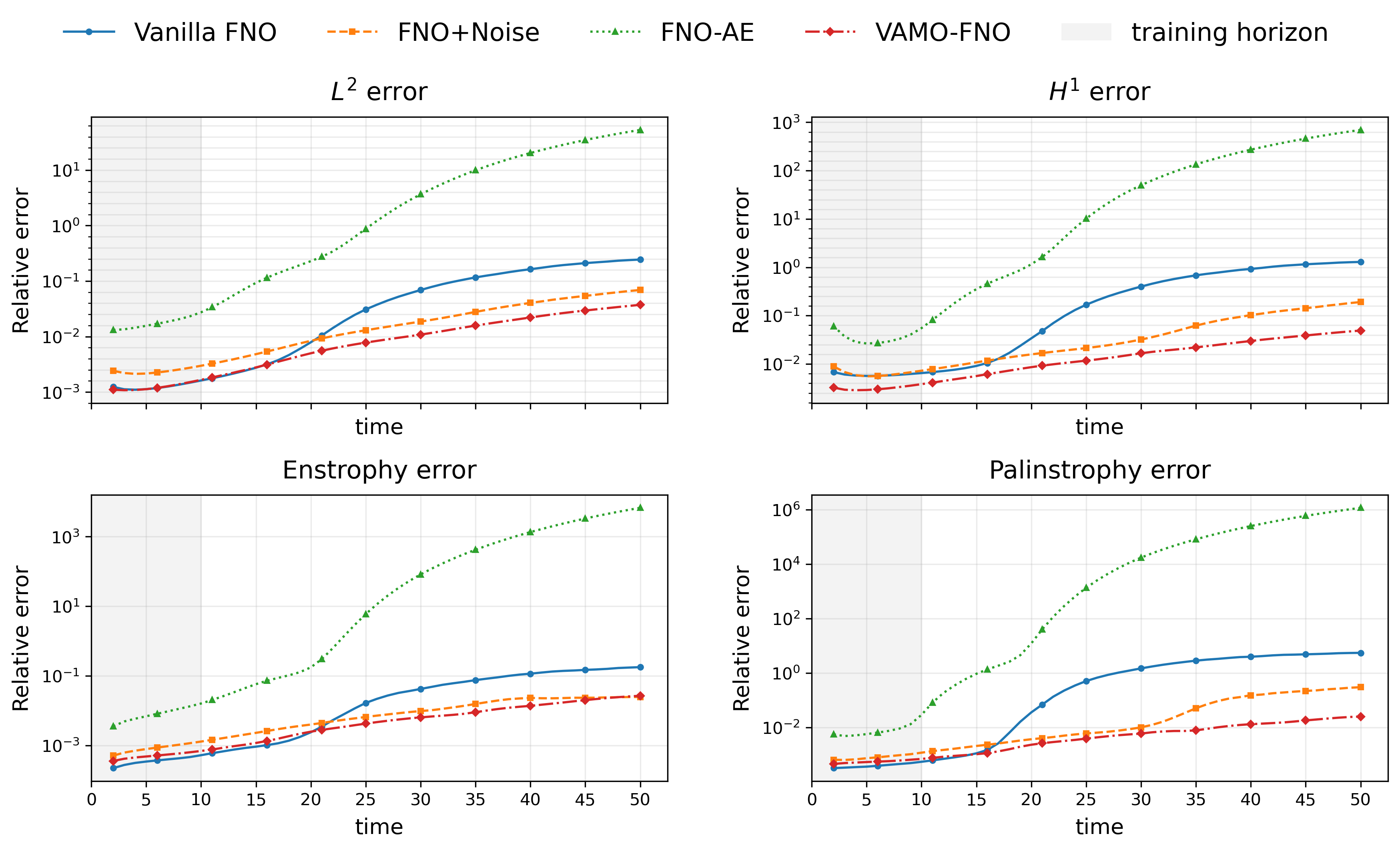}
        \caption{\textsc{GRF} forcing type.}
        \label{fig:nu3grf}
    \end{subfigure}

    \vspace{0.5em}
    
    \begin{subfigure}[t]{\linewidth}
        \centering
        \includegraphics[width=1.0\linewidth]{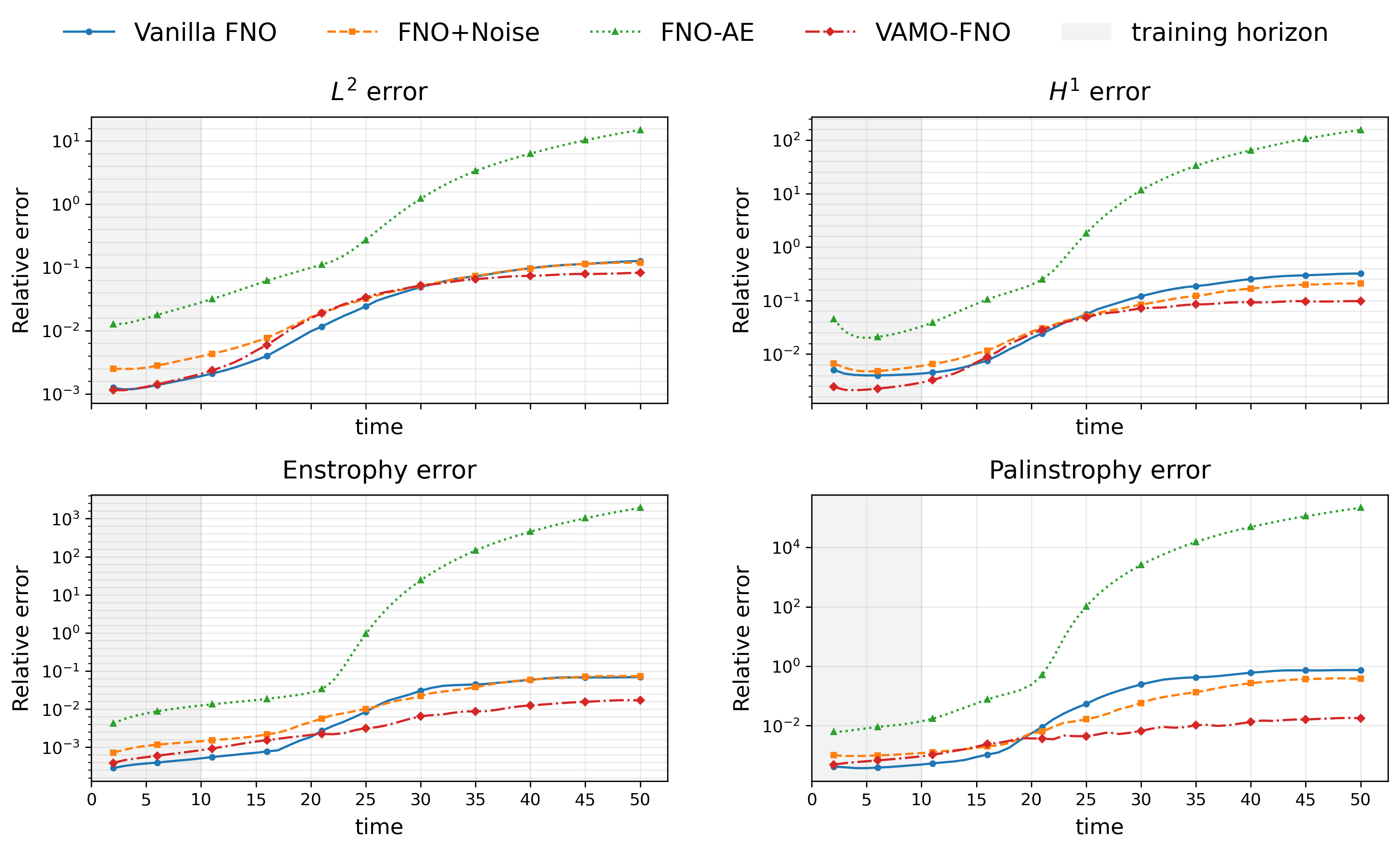}
        \caption{\textsc{Wave} forcing type.}
        \label{fig:nu3wave}
    \end{subfigure}
    \caption{Error trend for 2D incompressible NS Benchmark with $\nu=10^{-3}$.}
    \label{fig:ns2d_nu3_trend}
\end{figure}

\begin{figure}[p]
    \centering

    \begin{subfigure}{\linewidth}
        \centering
        \includegraphics[width=\linewidth]
        {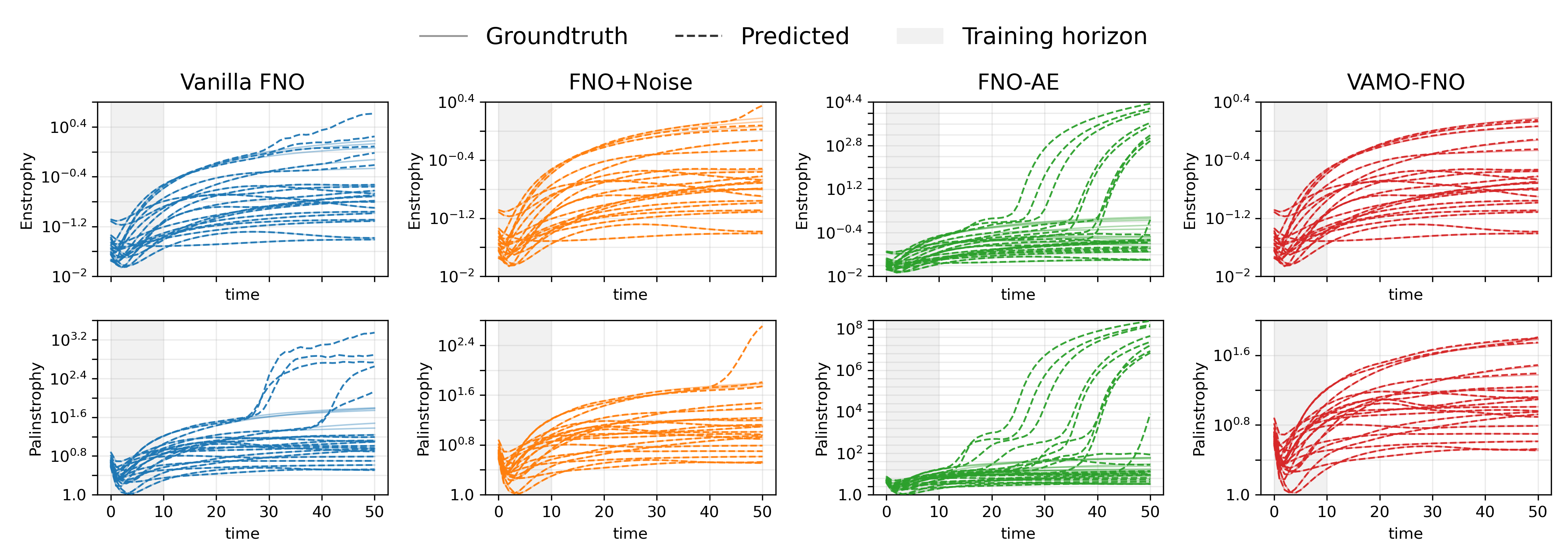}
        \caption{Viscosity $\nu=10^{-3}$, \textsc{GRF} forcing.}
        \label{fig:nsdiag_nu3_grf}
    \end{subfigure}

    \vspace{0.5em}

    \begin{subfigure}{\linewidth}
        \centering
        \includegraphics[width=\linewidth]
        {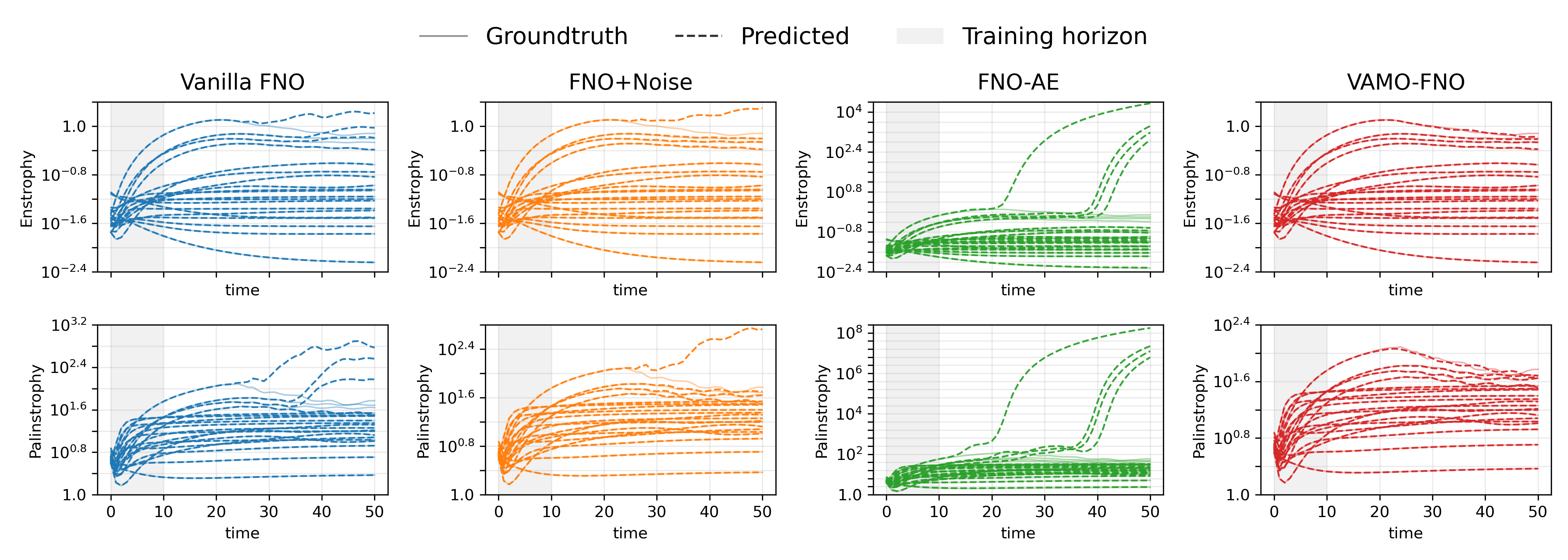}
        \caption{Viscosity $\nu=10^{-3}$, \textsc{Wave} forcing.}
        \label{fig:nsdiag_nu3_wave}
    \end{subfigure}

    \caption{
    Physical diagnostics for the 2D incompressible
    Navier-Stokes benchmark. Each panel compares enstrophy and palinstrophy along 16 ground-truth and predicted trajectories. Solid translucent curves denote ground truth, dashed curves denote predictions, and the shaded region marks the training horizon.
    }
    \label{fig:ns2d_physical_diagnostics_nu3}
\end{figure}

\subsection{Additional Qualitative Results}
\label{app:additional_qualitative}

We provide additional representative autoregressive rollouts from the test sets to complement the
quantitative results in \S\ref{sec:experiments}. The visualizations illustrate how the accumulated rollout errors
manifest in the predicted physical fields and highlight the qualitative differences among the
evaluated models.

\paragraph{1D Euler equations.}
Figures~\ref{fig:euler1d_L5_qualitative}--\ref{fig:euler1d_L20_qualitative} show representative trajectories for the four domain lengths $L\in\{5,10,15,20\}$. At early rollout times, all methods generally reproduce the large-scale wave structure of the reference solution. As the rollout proceeds, differences become increasingly visible around sharp fronts and interacting waves. The deterministic latent baseline is particularly susceptible to overshoots and oscillatory artifacts, with the degradation becoming more pronounced for the larger domains. The direct FNO and FNO+Noise baselines remain more stable but accumulate visible phase and amplitude errors at later times. VAMO remains more closely aligned with the reference density, velocity, and pressure profiles over the full prediction horizon, consistent with the error trends reported in \S\ref{sec:error_growth}.

\paragraph{2D compressible fluid dynamics.}
Figure~\ref{fig:cfd2d_qualitative} presents a representative rollout of the 2D compressible fluid dynamics benchmark. The differences among the methods are modest near the beginning of the trajectory but become clearer beyond the training horizon. The baseline predictions progressively develop larger spatial errors and fine-scale artifacts, particularly in the velocity fields, whereas VAMO better preserves the coherent spatial organization of the density, velocity, and pressure fields throughout the rollout. These qualitative differences agree with the substantially lower full-horizon $L^2$ and $H^1$ errors reported in Table~4.

\paragraph{2D incompressible Navier--Stokes equations.}
Figures~\ref{fig:ns2d_nu3_qualitative}--\ref{fig:ns2d_nu5_qualitative} provide representative Navier-Stokes rollouts for $\nu\in\{10^{-3},10^{-4},10^{-5}\}$ under both GRF and Wave forcing. In each panel, all predictions are shown using the corresponding ground-truth color scale, allowing deviations in vorticity magnitude and spatial structure to be compared directly. The baseline methods generally track the reference dynamics during the early rollout but can develop increasingly pronounced small-scale and oscillatory artifacts at later times. This behavior is most severe for the deterministic latent baseline, which becomes unstable in several examples. The direct FNO and FNO+Noise models remain more controlled but still exhibit substantial late-time deviations in the more challenging trajectories.

In contrast, VAMO more consistently preserves the dominant vorticity structures and avoids the severe high-frequency artifacts observed in the unstable baseline rollouts. The qualitative behavior is consistent across viscosity levels and forcing families and complements the aggregate field errors, temporal error trends, and physical-statistic diagnostics reported in the main text. Together, these examples provide a spatial view of the central empirical finding: the primary advantage of VAMO emerges during prolonged autoregressive evolution, where the variational formulation helps limit the accumulation of unphysical structure.

\begin{figure}[p]
    \centering
    \includegraphics[width=\linewidth]
    {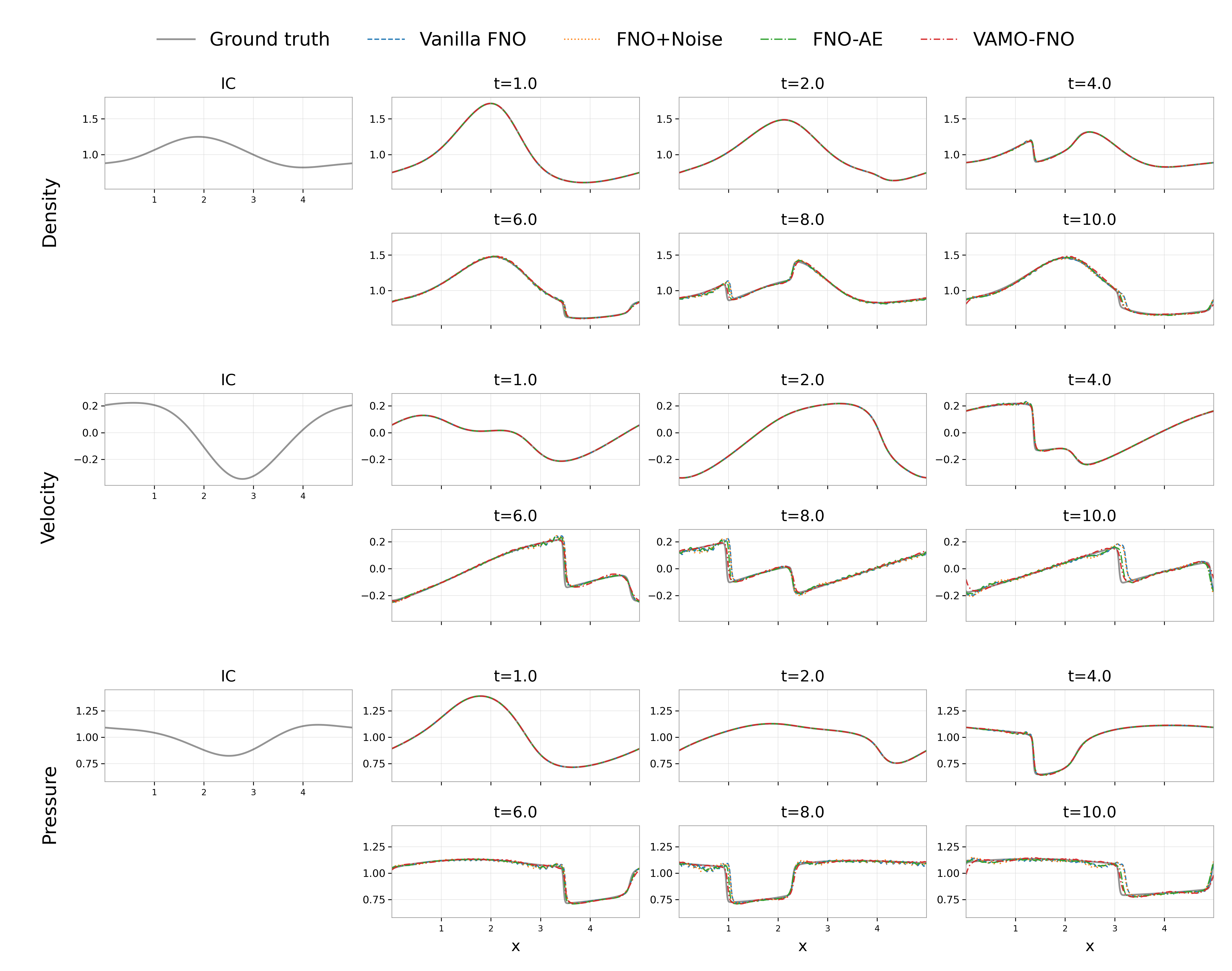}
    \caption{
    Representative autoregressive rollout for the 1D Euler equations with $L=5$. Rows show density $\rho$, velocity $u$, and pressure $p$, while columns correspond to selected rollout times. The solid gray curve denotes the ground truth, and the colored dashed curves denote predictions from the evaluated models. All methods remain accurate at early times, whereas the deterministic latent baseline develops increasingly pronounced overshoots and oscillatory errors around sharp wave structures during the later rollout. VAMO remains closely aligned with the reference trajectory over the full prediction horizon.
    }
    \label{fig:euler1d_L5_qualitative}
\end{figure}

\begin{figure}[p]
    \centering
    \includegraphics[width=\linewidth]
    {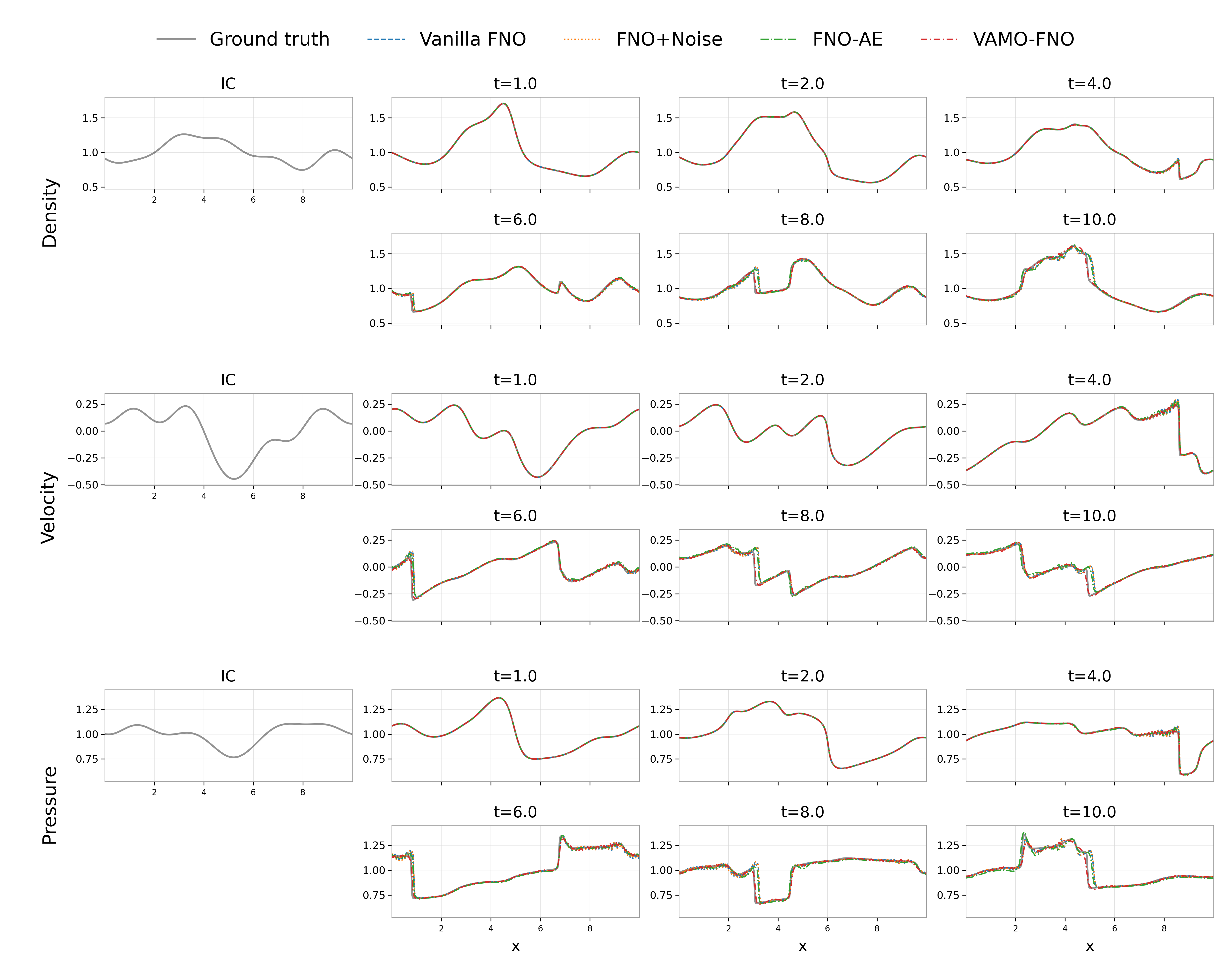}
    \caption{
    Representative autoregressive rollout for the 1D Euler equations with $L=10$. Rows show density $\rho$, velocity $u$, and pressure $p$, while columns correspond to selected rollout times. The solid gray curve denotes the ground truth, and the colored dashed curves denote predictions from the evaluated models. All methods remain accurate at early times, whereas the deterministic latent baseline develops increasingly pronounced overshoots and oscillatory errors around sharp wave structures during the later rollout. VAMO remains closely aligned with the reference trajectory over the full prediction horizon.
    }
    \label{fig:euler1d_L10_qualitative}
\end{figure}

\begin{figure}[p]
    \centering
    \includegraphics[width=\linewidth]
    {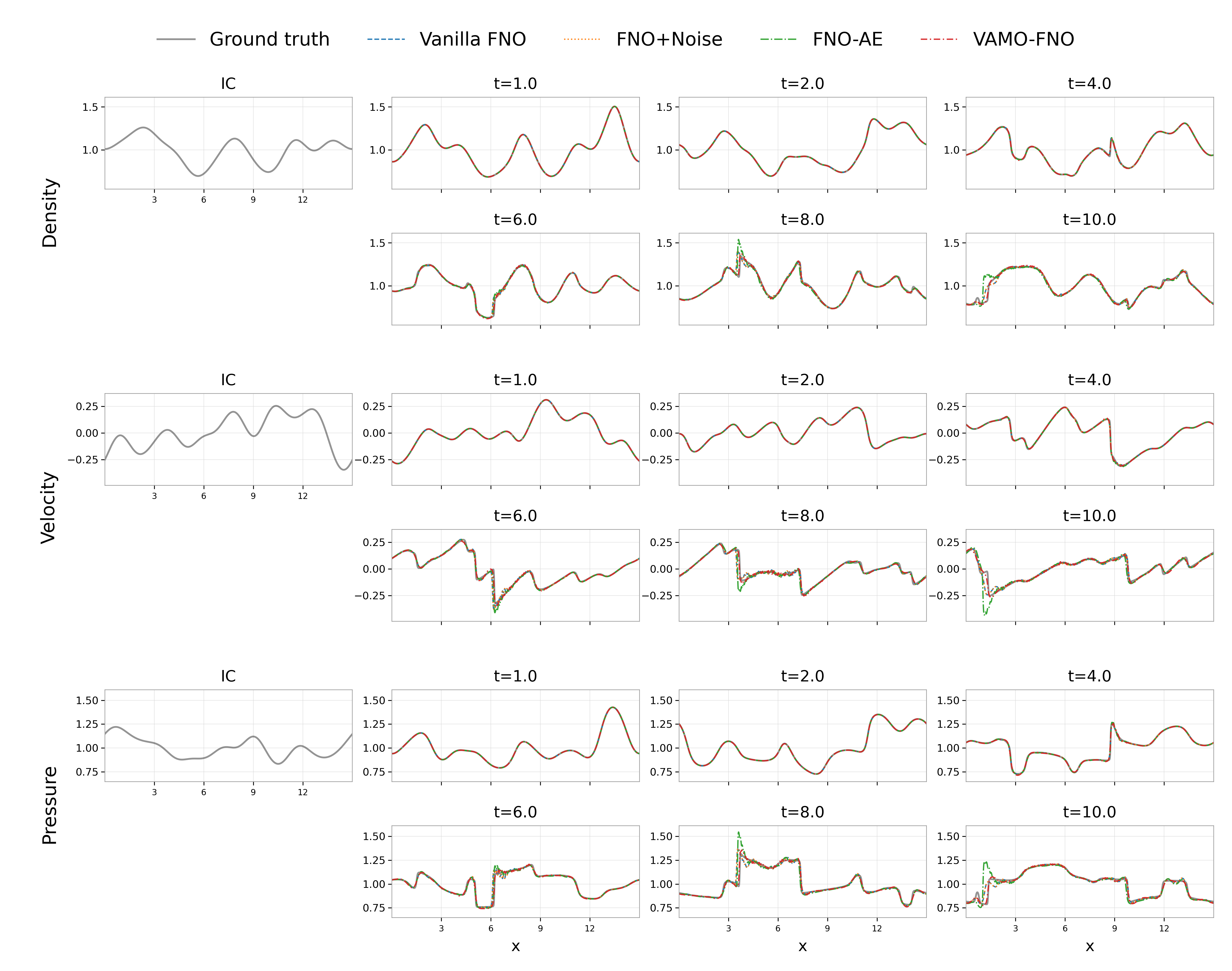}
    \caption{
    Representative autoregressive rollout for the 1D Euler equations with $L=15$. Rows show density $\rho$, velocity $u$, and pressure $p$, while columns correspond to selected rollout times. The solid gray curve denotes the ground truth, and the colored dashed curves denote predictions from the evaluated models. All methods remain accurate at early times, whereas the deterministic latent baseline develops increasingly pronounced overshoots and oscillatory errors around sharp wave structures during the later rollout. VAMO remains closely aligned with the reference trajectory over the full prediction horizon.
    }
    \label{fig:euler1d_L15_qualitative}
\end{figure}

\begin{figure}[p]
    \centering
    \includegraphics[width=\linewidth]
    {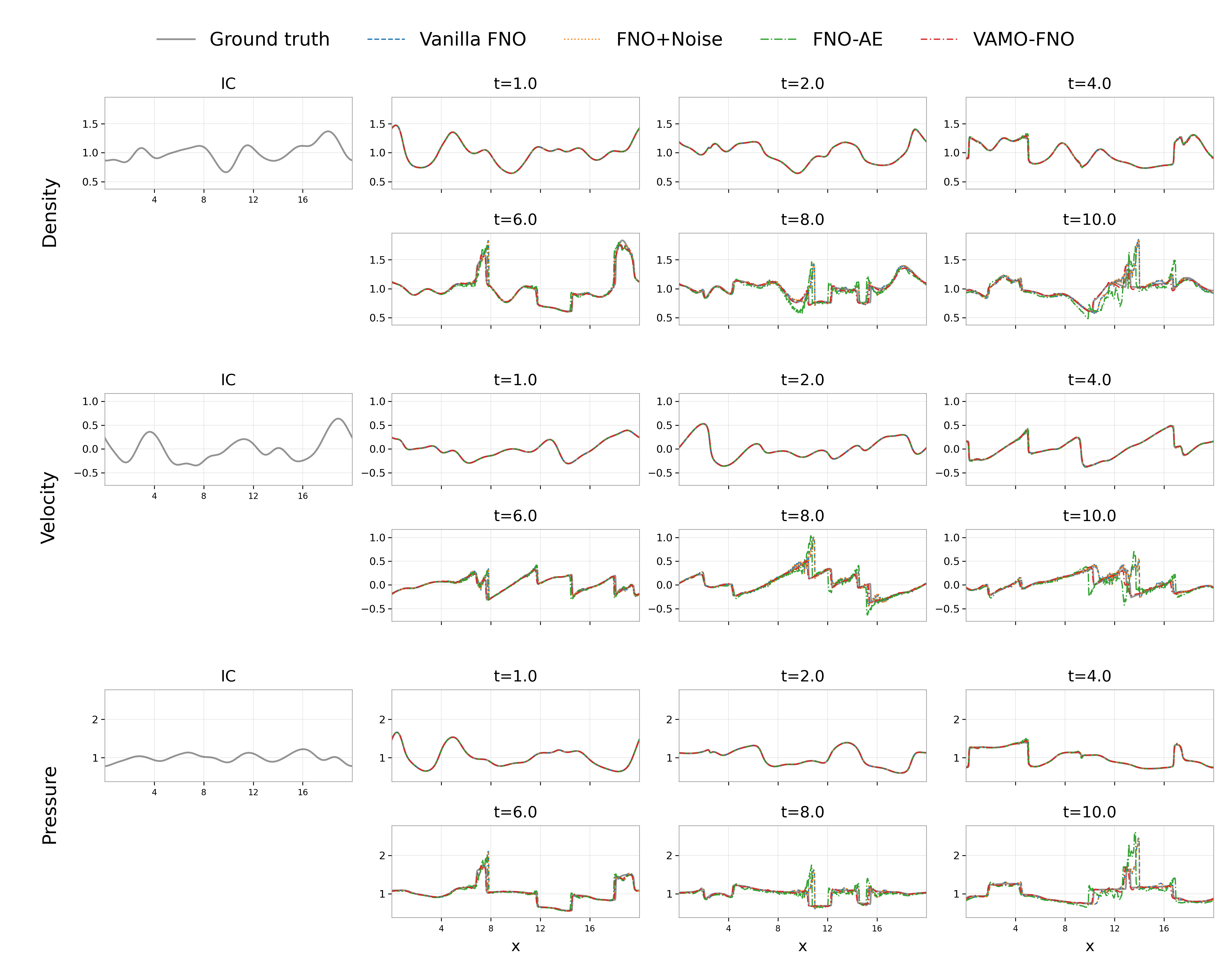}
    \caption{
    Representative autoregressive rollout for the 1D Euler equations with $L=20$. Rows show density $\rho$, velocity $u$, and pressure $p$, while columns correspond to selected rollout times. The solid gray curve denotes the ground truth, and the colored dashed curves denote predictions from the evaluated models. All methods remain accurate at early times, whereas the deterministic latent baseline develops increasingly pronounced overshoots and oscillatory errors around sharp wave structures during the later rollout. VAMO remains closely aligned with the reference trajectory over the full prediction horizon.
    }
    \label{fig:euler1d_L20_qualitative}
\end{figure}

\newpage 
\begin{figure}[ht]
    \centering
    \includegraphics[width=0.96\linewidth]
    {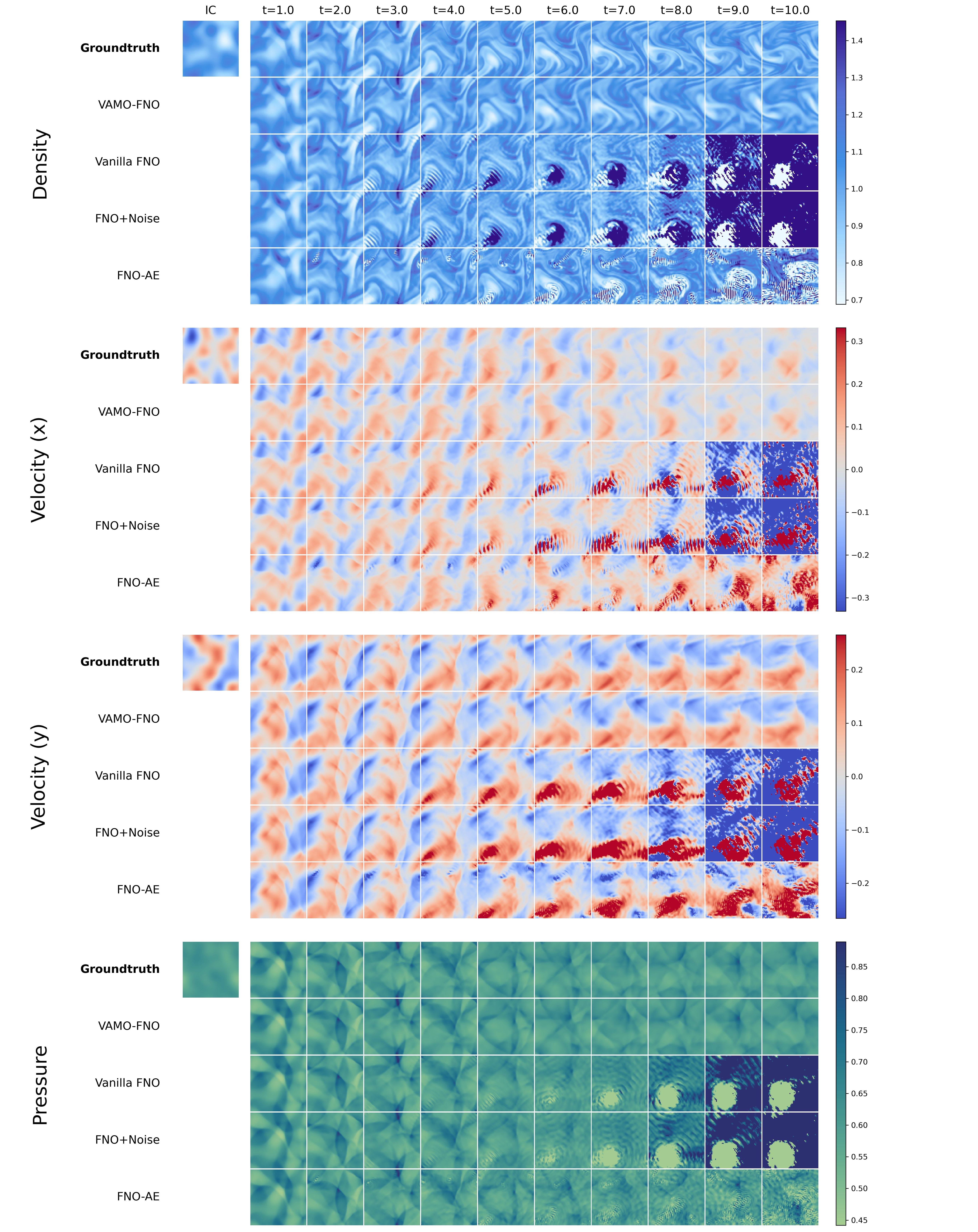}
    \caption{
    Representative autoregressive rollout for the 2D compressible fluid-dynamics benchmark. Rows compare the ground truth with predictions from FNO, FNO+Noise, FNO-AE, and FNO-VAMO for density, the two velocity components, and pressure.
    }
    \label{fig:cfd2d_qualitative}
\end{figure}

\begin{figure}[p]
    \centering

    \begin{subfigure}[t]{\linewidth}
        \centering
        \includegraphics[width=\linewidth]
        {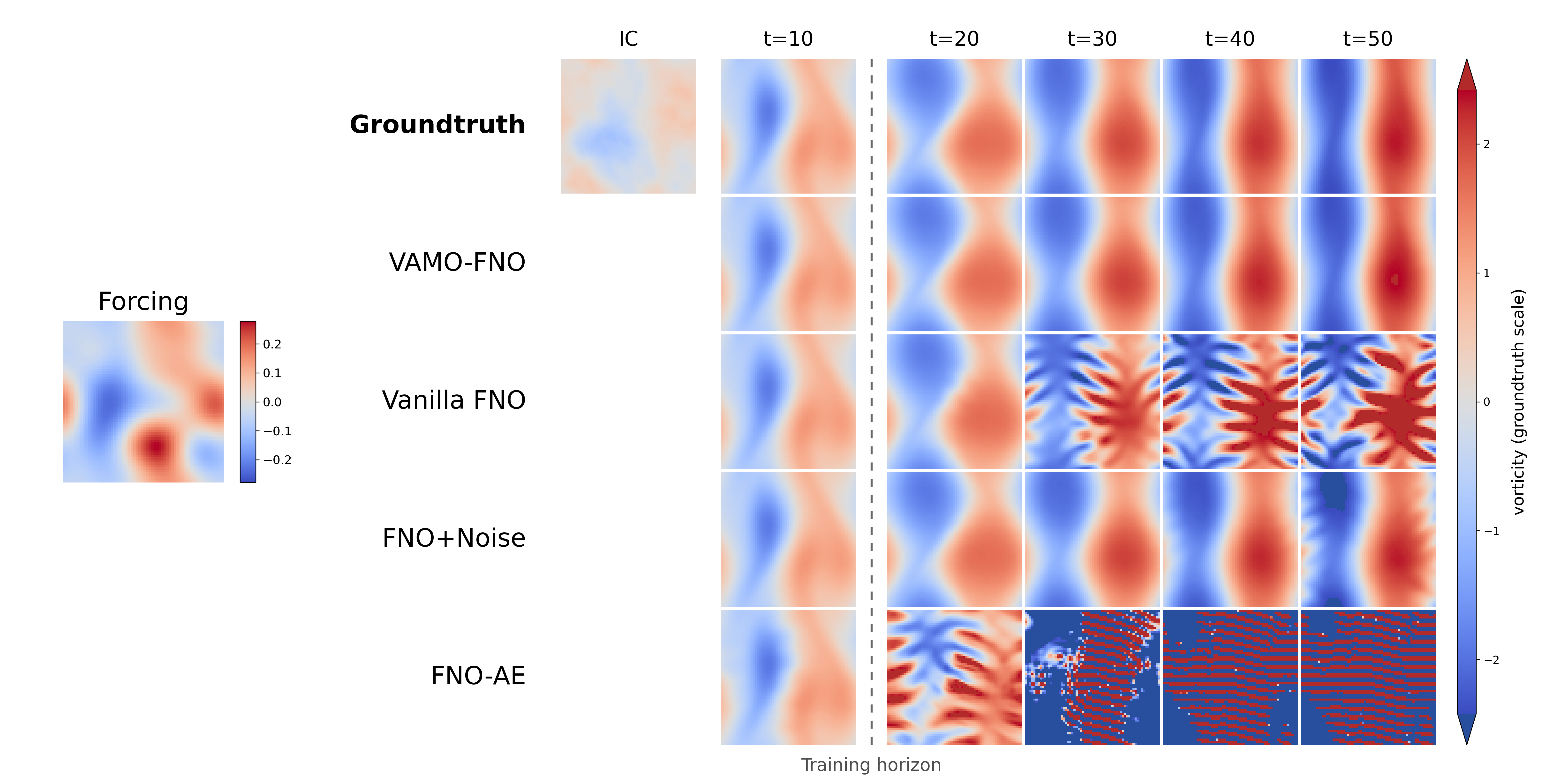}
        \caption{$\nu=10^{-3}$, \textsc{GRF} forcing, Sample 177.}
        \label{fig:ns2d_nu3_grf_trajectory}
    \end{subfigure}

    \vspace{0.8em}

    \begin{subfigure}[t]{\linewidth}
        \centering
        \includegraphics[width=\linewidth]
        {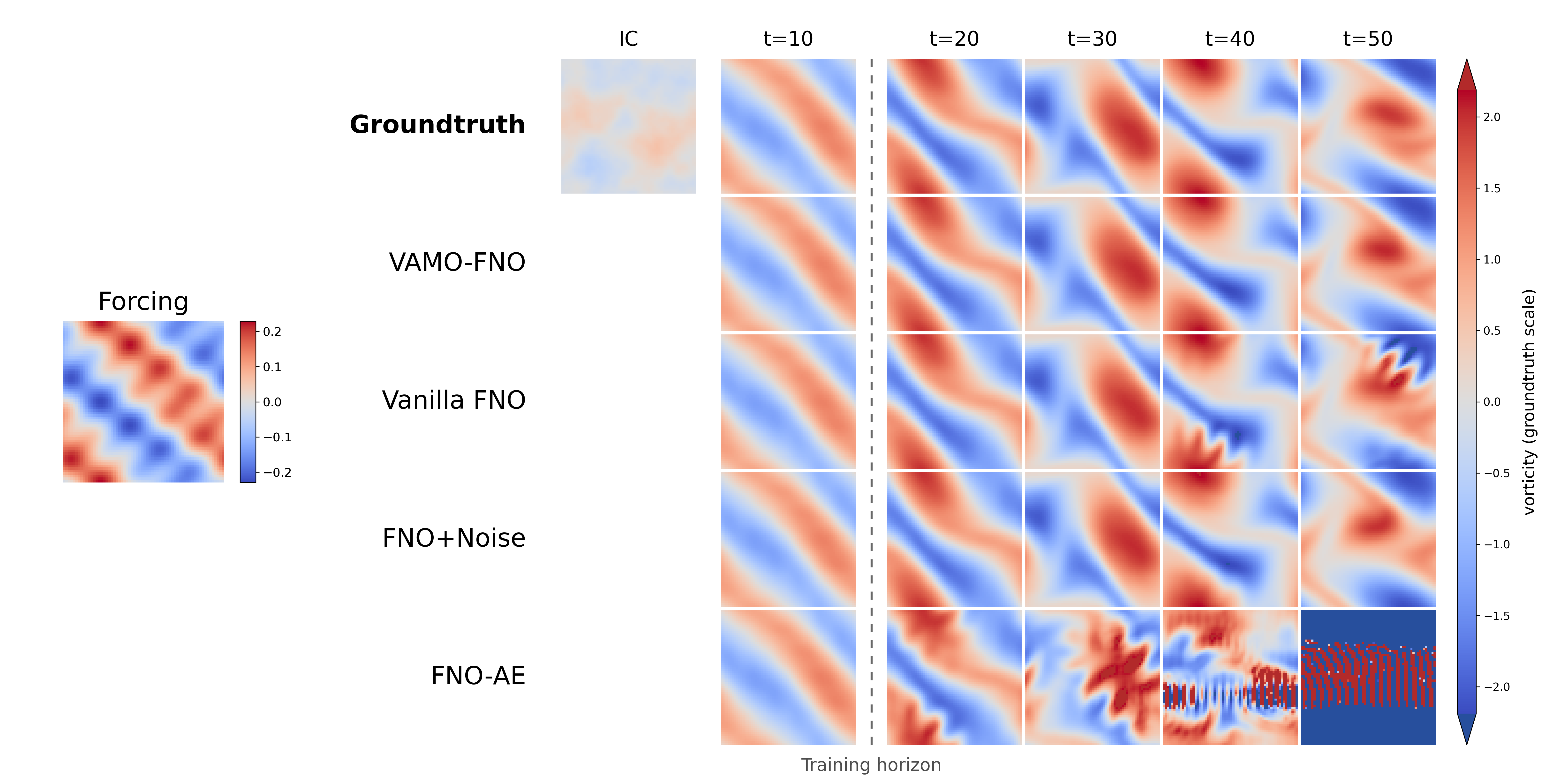}
        \caption{$\nu=10^{-3}$, \textsc{Wave} forcing, Sample 70.}
        \label{fig:ns2d_nu3_wave_trajectory}
    \end{subfigure}

    \caption{
    Representative autoregressive rollouts for the 2D incompressible Navier-Stokes equations with $\nu=10^{-3}$. Each panel shows the initial condition, the trajectory-specific forcing, the ground-truth vorticity, and predictions from the four evaluated methods at selected times. All predicted snapshots use the corresponding ground-truth color scale.
    }
    \label{fig:ns2d_nu3_qualitative}
\end{figure}

\begin{figure}[p]
    \centering

    \begin{subfigure}[t]{\linewidth}
        \centering
        \includegraphics[width=\linewidth]
        {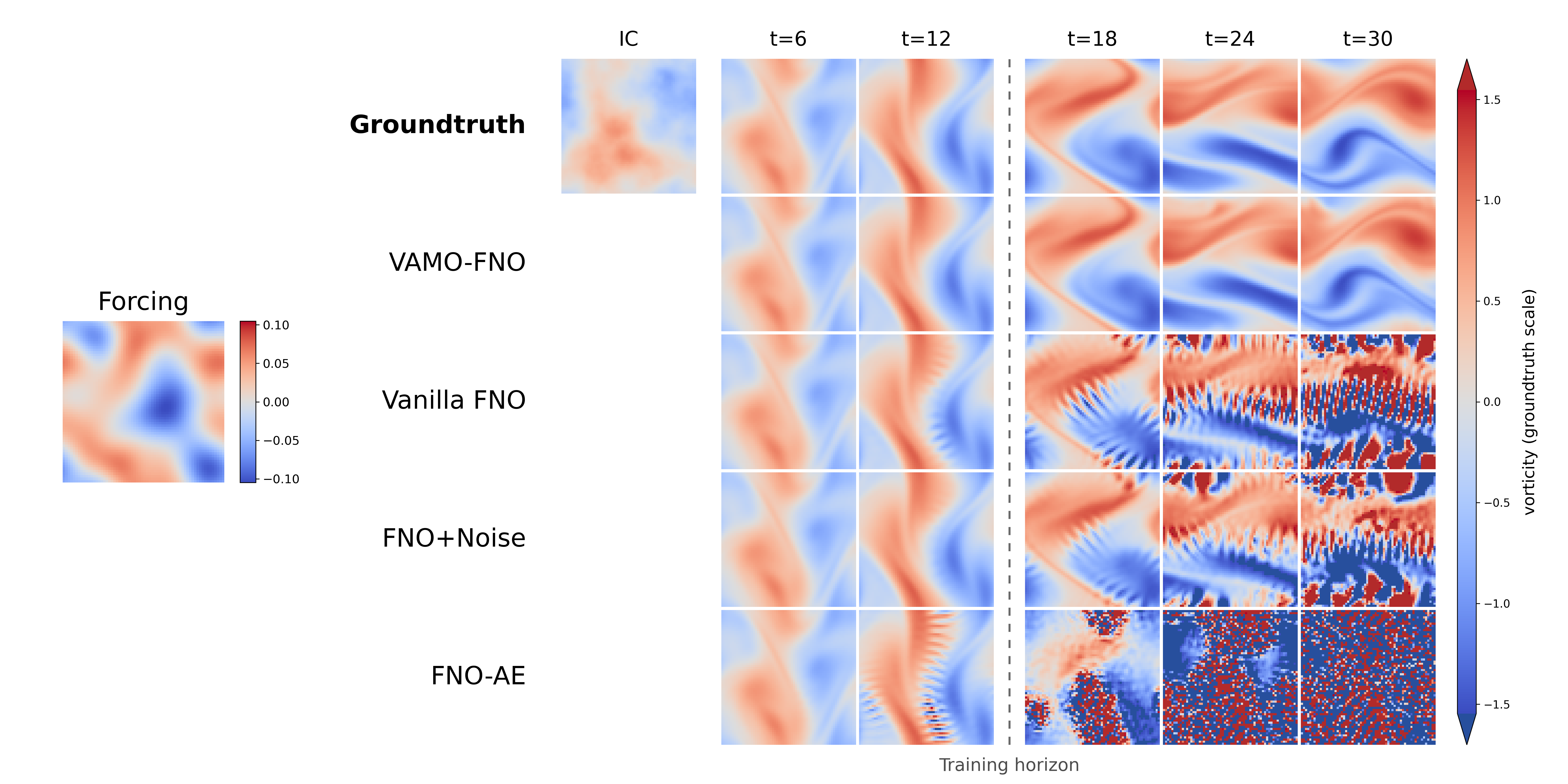}
        \caption{$\nu=10^{-4}$, \textsc{GRF} forcing, Sample 0.}
        \label{fig:ns2d_nu4_grf_trajectory}
    \end{subfigure}

    \vspace{0.8em}

    \begin{subfigure}[t]{\linewidth}
        \centering
        \includegraphics[width=\linewidth]
        {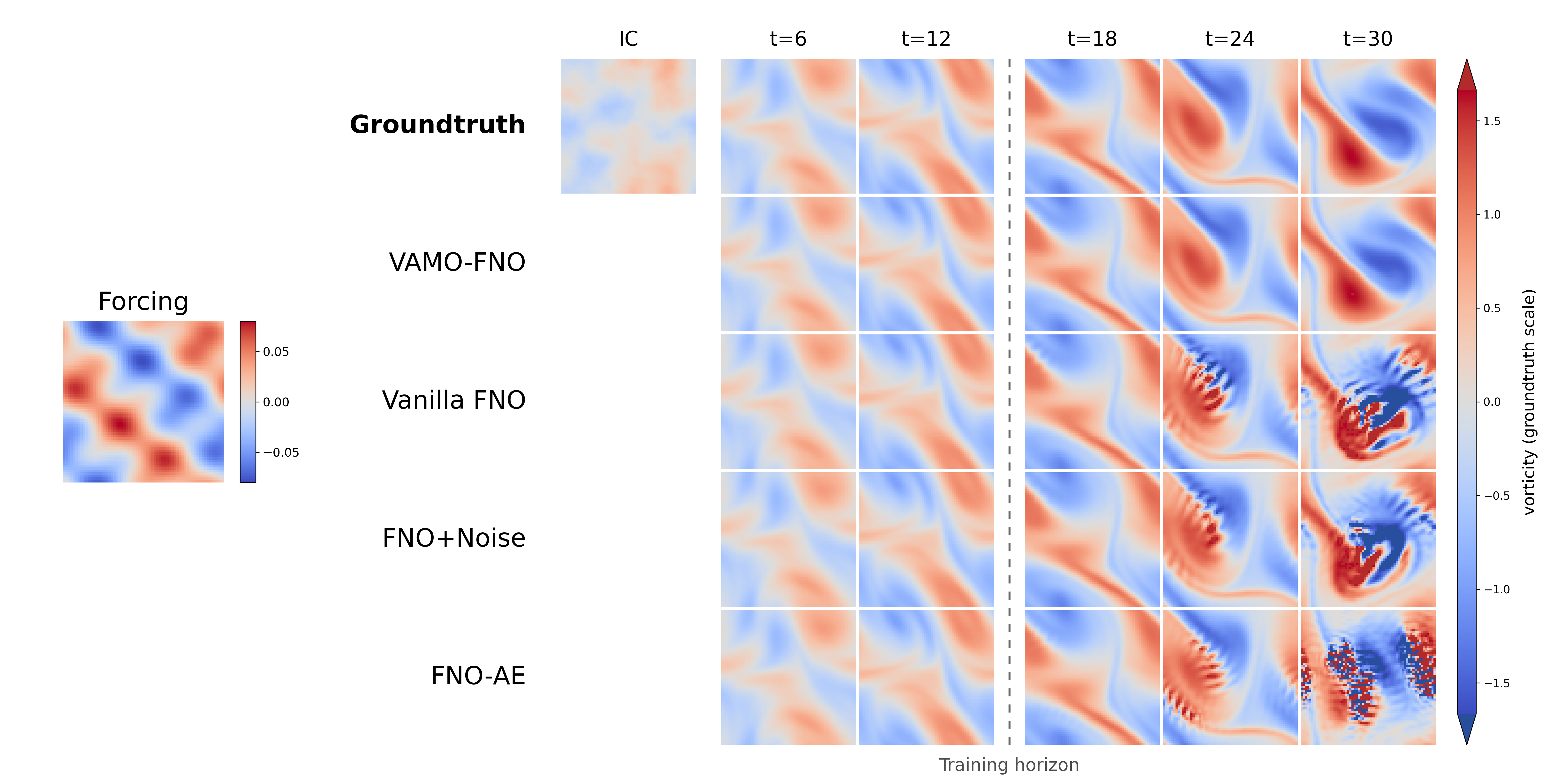}
        \caption{$\nu=10^{-4}$, \textsc{Wave} forcing, Sample 136.}
        \label{fig:ns2d_nu4_wave_trajectory}
    \end{subfigure}

    \caption{
    Representative autoregressive rollouts for the 2D incompressible Navier-Stokes equations with $\nu=10^{-4}$. Each panel shows the initial condition, the trajectory-specific forcing, the ground-truth vorticity, and predictions from the four evaluated methods at selected times. All predicted snapshots use the corresponding ground-truth color scale.
    }
    \label{fig:ns2d_nu4_qualitative}
\end{figure}

\begin{figure}[p]
    \centering

    \begin{subfigure}[t]{\linewidth}
        \centering
        \includegraphics[width=\linewidth]
        {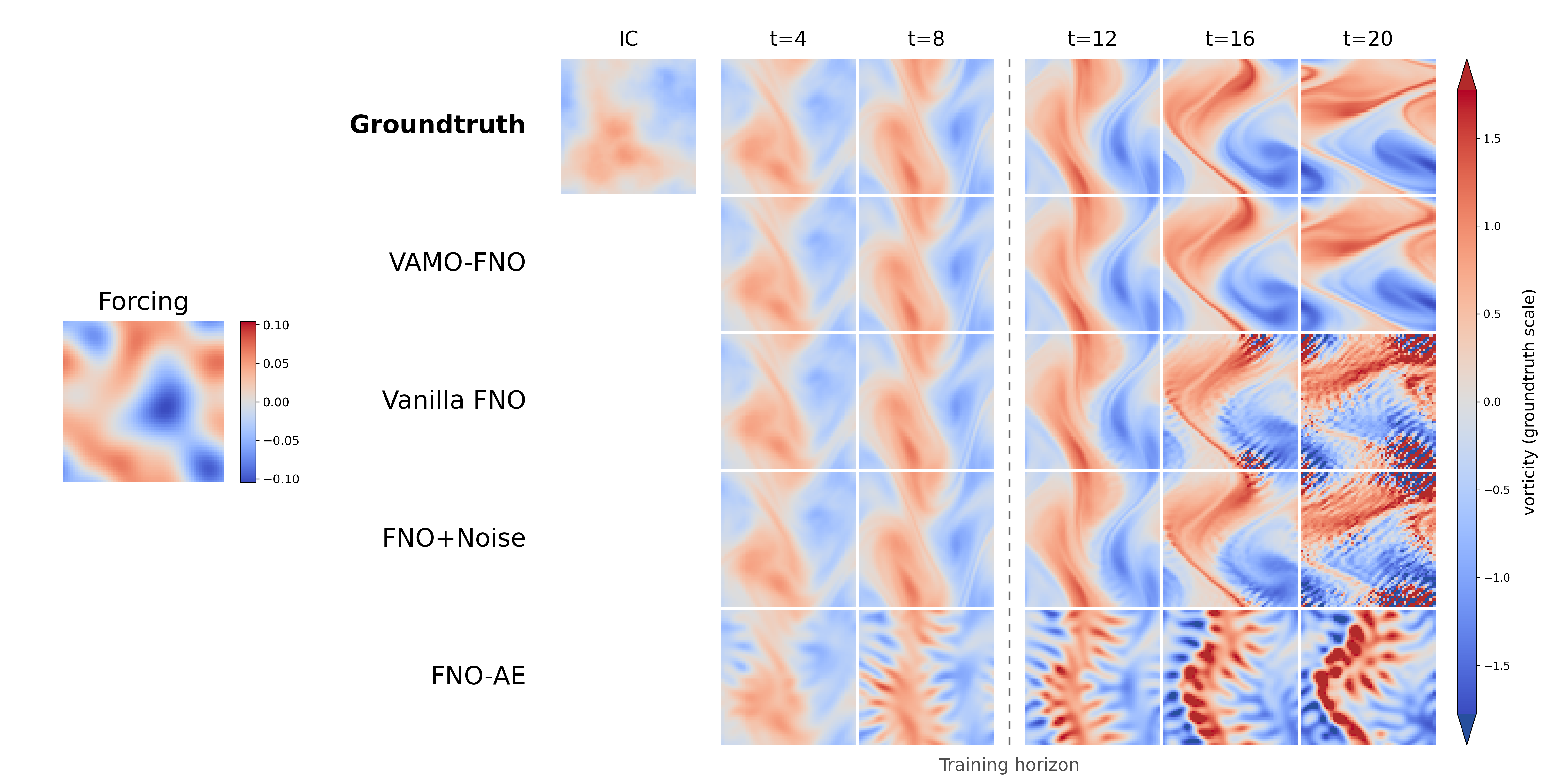}
        \caption{$\nu=10^{-5}$, \textsc{GRF} forcing, Sample 0.}
        \label{fig:ns2d_nu5_grf_trajectory}
    \end{subfigure}

    \vspace{0.8em}
    
    \begin{subfigure}[t]{\linewidth}
        \centering
        \includegraphics[width=\linewidth]
        {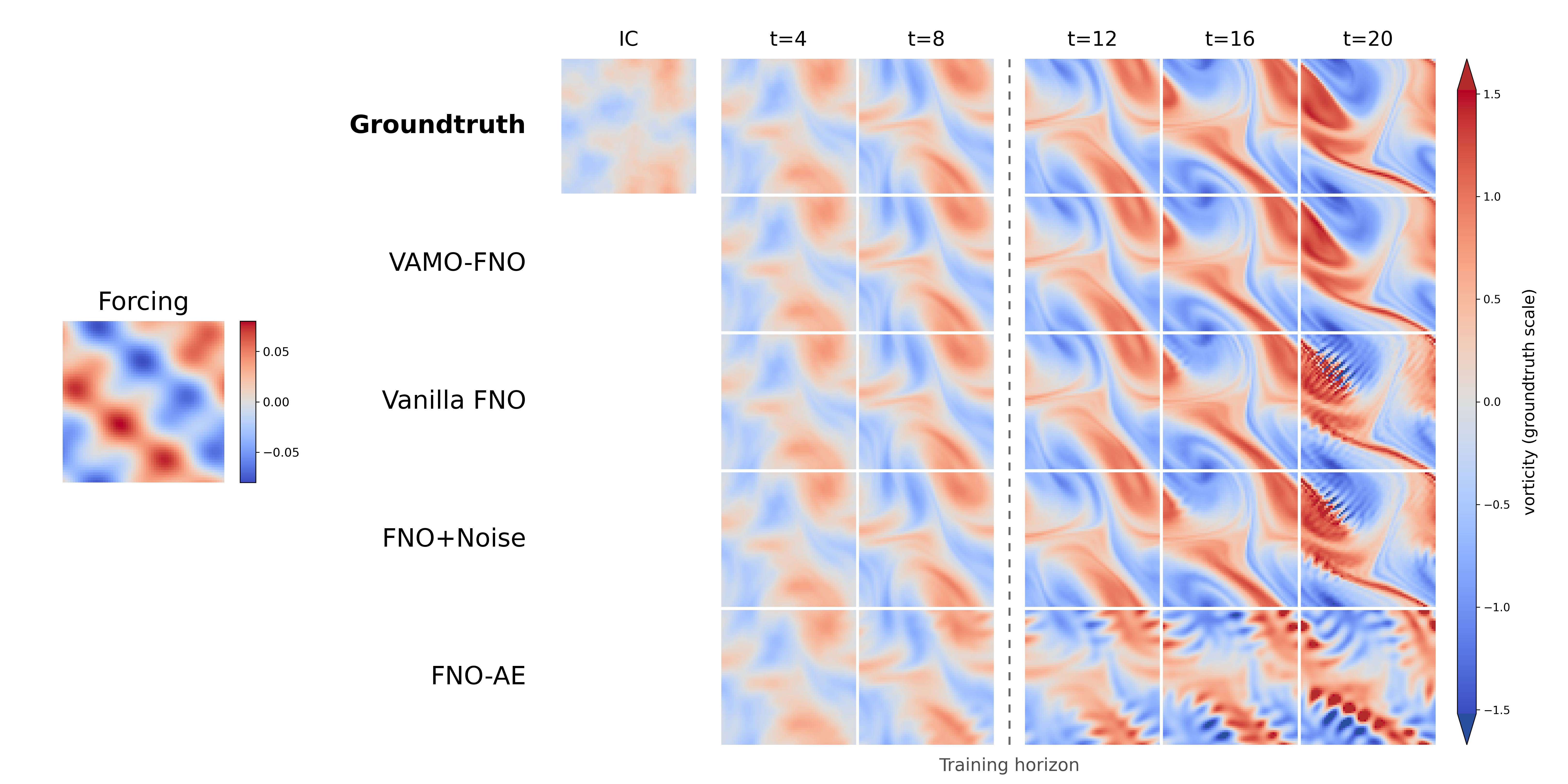}
        \caption{$\nu=10^{-5}$, \textsc{Wave} forcing, Sample 136.}
        \label{fig:ns2d_nu5_wave_trajectory}
    \end{subfigure}
    \caption{
    Representative autoregressive rollouts for the 2D incompressible Navier--Stokes equations with $\nu=10^{-5}$. Each panel shows the initial condition, the trajectory-specific forcing, the ground-truth vorticity, and predictions from the four evaluated methods at selected times. All predicted snapshots use the corresponding ground-truth color scale. 
    }
    \label{fig:ns2d_nu5_qualitative}
\end{figure}

\end{document}